\documentclass[verbose=true,letterpaper,pass]{article}

\usepackage{PRIMEarxiv}

\usepackage{graphicx}

\usepackage{microtype}      
\usepackage{amsfonts}       
\usepackage{nicefrac}       
\usepackage{pgfplotstable}
\tikzstyle{every node}=[font=\small]
\usepackage{pgfplots}
\usepackage{subcaption}
\usepackage{graphicx}
\usepackage[numbers]{natbib}
\usepackage{adjustbox}
\usepackage{booktabs} 
\usepackage{multirow}

\usepackage{amsmath,amsfonts,bm}

\def\1{\bm{1}}

\DeclareMathAlphabet{\mathsfit}{\encodingdefault}{\sfdefault}{m}{sl}
\SetMathAlphabet{\mathsfit}{bold}{\encodingdefault}{\sfdefault}{bx}{n}

\newcommand{\R}{\mathbb{R}}

\DeclareMathOperator*{\argmin}{arg\,min}

\usepackage{setspace}
\usepackage{hyperref}
\usepackage[nameinlink]{cleveref}
\usepackage{xspace}

\newcommand{\xit}{x_i^t}

\newcommand{\thetaijtkll}{\theta_{i,j}^{t,k,l+1}}
\newcommand{\thetaijtkl}{\theta_{i,j}^{t,k,l}}
\newcommand{\thetaijtkstar}{\theta_{i,j}^{t,k,*}}

\newcommand{\thetaijtkL}{\theta_{i,j}^{t,k,L}}
\newcommand{\thetabaritk}{\bar{\theta}_{i}^{t,k}}
\newcommand{\thetaijtklstar}{\theta_{i,j}^{t,k,l^*}}
\newcommand{\fij}{f_{i,j}}
\newcommand{\ftildeij}{\tilde{f}_{i,j}}

\newcommand{\witk}{w_{i}^{t, k}}
\newcommand{\witK}{w_{i}^{t, K}}
\newcommand{\witkk}{w_{i}^{t, k+1}}
\newcommand{\witstar}{w_{i}^{t,*}}
\newcommand{\witmin}{w_{i}^{t, k^*}}

\newcommand{\xt}{x^t}
\newcommand{\xtt}{x^{t+1}}
\newcommand{\xT}{x^{T}}

\newcommand{\gt}{g^{t}}
\newcommand{\git}{g_i^{t}}

\newcommand{\qitk}{q_i^{t, k}}
\newcommand{\pij}{p_{i, j}}
\newcommand{\avgNi}{\frac{1}{N_i}\sum_{j=1}^{N_i}}
\newcommand{\avgK}{\frac{1}{K}\sum_{k=0}^{K-1}}
\newcommand{\avgT}{\sum_{t=0}^{T-1}}
\newcommand{\avgM}{\frac{1}{M}\sum_{i=1}^{M}}
\newcommand{\avgL}{\frac{1}{L} \sum_{l=0}^{L-1}}

\newcommand{\ait}{\witK}
\newcommand{\Ftildei}{\tilde{F}_i}
\newcommand{\muFtilde}{\mu_{\tilde{F}}}
\newcommand{\LFtilde}{L_{\tilde{F}}}
\newcommand{\Lftilde}{L_{\tilde{f}}}
\newcommand{\muflambda}{\mu_f+\lambda}
\newcommand{\Lflambda}{L_f+\lambda}

\newcommand{\deltaij}{\delta_{i,j}}
\newcommand{\deltai}{\delta_{i}}

\newcommand{\norm}[1]{\|#1\|}
\newcommand{\nsq}[1]{\|#1\|^2}
\newcommand{\prox}{\mathrm{prox}}

\newcommand{\ip}[2]{\langle#1,#2\rangle}
\usepackage{bbold}
\usepackage{comment}
\usepackage{color, colortbl}
\definecolor{LightCyan}{rgb}{0.88,1,1}
\usepackage{enumitem}
\usepackage{amsthm}
\usepackage{microtype}
\newtheorem{theorem}{Theorem}

\newtheorem{proposition}{Proposition}
\newtheorem{definition}{Definition}

\newtheorem{remark}{Remark}
\newcommand{\ourmodel}[1]{PerMFL}
\usepackage{algorithm}
\usepackage{algorithmic}
\renewcommand{\algorithmiccomment}[1]{\bgroup\hfill//~#1\egroup}

\renewcommand{\paragraph}[1]{
   \textbf{#1}
}
\usepackage{xcolor}
\pgfplotsset{compat=1.18}
\usepackage{url}            
\usepackage{booktabs}       
\usepackage{lipsum}
\usepackage{fancyhdr}       
\usepackage{graphicx}       
\graphicspath{{media/}}     
\begin{document}
\title{Personalized Multi-tier Federated Learning}
\author{Sourasekhar Banerjee, Ali Dadras, Alp Yurtsever, Monowar Bhuyan \\
Ume\aa University \\
Ume\aa \\
Sweden\\
\texttt{\{firstname.lastname\}@umu.se}
}

\maketitle          

\begin{abstract}
The key challenge of personalized federated learning (PerFL) is to capture the statistical heterogeneity properties of data with inexpensive communications and gain customized performance for participating devices. To address these, we introduced personalized federated learning in multi-tier architecture (\ourmodel{}\footnote{ \url{https://github.com/sourasb05/PerMFL_1.git}}) to obtain optimized and personalized local models when there are known team structures across devices. We provide theoretical guarantees of \ourmodel{}, which offers linear convergence rates for smooth strongly convex problems and sub-linear convergence rates for smooth non-convex problems. We conduct numerical experiments demonstrating the robust empirical performance of \ourmodel{}, outperforming the state-of-the-art in multiple personalized federated learning tasks.

\end{abstract}

\section{Introduction}

Federated learning (FL) is a distributed on-device learning framework that employs the heterogeneous data privately available at the edge for learning. In classical machine learning, edge devices are supposed to send data to the centralized server for training. However, FL relaxes this restriction by enabling the training of each model on the end devices and aggregating them on the global server. Classical FL learns a single global model by utilizing the private data of the devices locally and exchanging only the model information in a communication-efficient and privacy-preserving manner \cite{mcmahan2017communication}. 

The early FL literature focuses mainly on a simplistic network architecture where all devices communicate directly with a single server \cite{mcmahan2017communication}. An example of such a network architecture is typical on a local area network (LAN) with all devices connected to a single server. However, real-world Internet applications often occur on more complex, multi-tiered network architecture, e.g., wide area networks (WAN) that span multiple geographic locations and connect heterogeneous LANs. Hence, the conventional FL architecture is not suitable for such a setting. To address this, we study a multi-tier FL model \Cref{fig: proposed} where an intermediate layer of mediators, called \emph{team servers ($TS_i$)}, acts as a bridge between the end devices ($N_i$) and the global server ($GS$) and facilitates intermediate aggregations. From a systems standpoint, this multi-tier FL model resembles the cloud-edge continuum architecture \cite{wu2021hierarchical}. In this model, the distant cloud serves as a central server responsible for generating a global model, and the edge servers located in various geographical regions act as team servers that establish connections with both the distant cloud and end devices. In contrast, the end devices perform local model computations. Multi-tier FL models, which are also referred to as hierarchical FL, have been studied by researchers in various applications and have shown significant advantages such as cost efficiency and scalability \cite{ekmefjord2022scalable}, reduced communication overheads \cite{abad2020hierarchical, liu2022hierarchical}, enhanced privacy \cite{wainakh2020enhancing}, improved system adaptability and performance \cite{lin2018don}, and faster convergence speeds (both theoretically and empirically) and reduced training time \cite{liu2020client}.

Data heterogeneity poses a significant challenge in FL, as data across different devices may exhibit varying characteristics and originate from diverse data distributions. The conventional FL systems assume that a single global model fits all devices, hindering the convergence speed and more crucially, the model accuracy when data is disseminated in a non-independent and identically distributed (non-IID) manner. In this scenario, using a `global model for all' disallows adaptation to unique needs and preferences embedded in each user's data characteristics, possibly leading to subpar performance and user dissatisfaction \cite{tan2022towards}. To tackle this challenge, personalized FL (PerFL) methods learn local models suitable for each user's unique needs for downstream tasks while still benefiting from collaborative training to achieve customized performance. Popular PerFL approaches include regularization techniques that penalize the distance between local and global models \cite{li2021ditto}, smoothing techniques such as Moreau envelopes \cite{t2020personalized}, and other heuristics such as taking extra local steps after the global model has converged \cite{finn2017model, fallah2020personalized}. 
 
Multi-tier FL also suffers from data heterogeneity, but the existing literature on personalized multi-tier FL is limited. To address this challenge, we introduce a personalized multi-tier federated learning method (\ourmodel{}). Our method is based on a new problem formulation that explicitly incorporates individual models for each team and device, in addition to the global server's model. We enforce the proximity between team models and the global model and between device models and their associated team models using squared Euclidean distance regularization. As a result, our method leverages the multi-tiered architecture and simultaneously learns three models: (1) a global model, (2) a personalized model for each team, and (3) a personalized model for each device. Geometrically, the global model serves as a central estimate that is agreed upon by all teams and end devices. On the contrary, personalized models are designed to deviate from the global model in specific directions that align with their local data distributions. To facilitate efficient communication, our algorithm restricts direct communication between devices and the global server, allowing devices to communicate only with the team servers, which in turn communicate with the global server.  Our primary goal with \ourmodel{} is to achieve personalized on-device model performance while still maintaining comparably high accuracy for the global model, which can compete with conventional FL methods, all while ensuring efficient communication and collaboration among the devices and team servers.

\vspace{2em}
\noindent Our key \textbf{contributions} are as follows: 
\begin{enumerate}

\item We formulate an optimization problem for multi-tier personalized FL by introducing personal decision variables for teams and devices through squared Euclidean distance regularization, and we propose an algorithm (\ourmodel{}) to solve this problem. Our algorithm flexibly accommodates local objectives during joint training of global and personalized models.
\item To provide a theoretical basis for our approach, we analyze the convergence guarantees of the proposed algorithm under the assumptions of smooth strongly convex and smooth non-convex loss functions. We show that the method converges with linear and sublinear rates, respectively, and we derive explicit theoretical bounds on hyperparameter settings to provide guidance for implementation. 
\item We conduct extensive numerical experiments to evaluate the empirical performance of \ourmodel{}. We compare our method to state-of-the-art (SOTA) approaches for both conventional and multi-tier FL settings using benchmark datasets (MNIST, FMNIST, EMNIST-10, FEMNIST,CIFAR100) and non-image tabular synthetic datasets with non-IID data dissemination. Moreover, we have examined the effect of hyperparameters on the convergence of \ourmodel{}, ablation studies on different team formations, and ablation studies on team and client participation on the convergence of \ourmodel{}.
\end{enumerate}

\section{Related Work} \label{sec : rel_work}
 
This section reviews existing studies and summarizes the differences between existing works and proposed efforts. We emphasize three different categories of FL models - multi-tier and personalized FL. 

\textbf{Multi-tier FL.} A multi-tier (\textit{aka} hierarchical) FL model leverages the combined capabilities of cloud and edge devices within the FL framework. In \cite{liu2020client}, the authors demonstrated that a multi-tier FL strategy, both theoretically and empirically, exhibits a faster convergence rate compared to traditional FL algorithms. In \cite{wang2022demystifying}, the concept of ``upward'' and ``downward'' divergences were introduced, exploring their implications in the context of multi-tier FL. Additionally, \cite{das2022cross} presented Cross-Silo FL, which encompasses multi-tier networks and utilizes vertical and horizontal data partitioning strategies. In  \cite{ekmefjord2022scalable}, they introduced FEDn, a multi-tier FL framework designed explicitly for horizontally scalable distributed deployments.  \cite{wainakh2020enhancing} identified several potential advantages of multi-tier FL in addressing privacy concerns. Firstly, multi-tier FL helps reduce the concentration of power and control in a central server, promoting a more distributed and decentralized approach. Secondly, multi-tier FL allows for the flexible placement of defense and verification mechanisms within the hierarchical structure, enabling the practical application of these methods. Lastly, multi-tier FL leverages the trust between users to mitigate the number of potential threats. 

\textbf{Personalized FL.} FedAvg \cite{mcmahan2017communication} is a classical baseline method in FL, renowned for its simplicity and low communication cost. However, it faces challenges when dealing with heterogeneous (non-iid) data, leading to unstable performance due to the concept drift problem. Concept drift arises when a single global model does not perform well for all clients. To address these issues, FL has increasingly leaned toward personalized models \cite{karimireddy2020scaffold, mansour2020three, tan2022towards}. In \cite{finn2017model}, the authors proposed a personalized version of model-agnostic meta-learning (MAML) for FL. They identified the problem of how fast the initial shared model adapts to their local dataset with fewer gradient descent steps using individual client data. In \cite{mansour2020three}, a systematic learning-theoretic study of personalization led to the proposal of three model-agnostic approaches: user clustering, data interpolation, and model interpolation. Another approach, presented in \cite{t2020personalized}, formulated a personalized FL problem that utilized Moreau envelopes to regularize devices' loss functions. In \cite{lyu2022personalized}, the authors introduced an Asynchronous Loopless Local Gradient Descent (Async-L2GD) method for users from multiple known clusters, simultaneously training three models: a global model, a model-specific cluster, and a personalized model for each device. The architecture is similar to \ourmodel{}, except that L2GD is an asynchronous approach. \cite{nguyen2022self} introduced DemLearn, an FL algorithm that employs hierarchical agglomeration clustering. Unlike our model, where teams remain static throughout the FL process, DemLearn dynamically assembles teams after each global round. Recent research on personalized FL also includes works such as \cite{zhang2022personalized, gasanov2022flix, pillutla2022federated, chen2022pfl, tan2022towards, tziotis2022straggler, li2021ditto}.

\paragraph{Motivation.}
A hierarchical structure in FL addresses the scalability and failure tolerance limitations observed in the centralized architecture \cite{wainakh2020enhancing}. In addition, it is also beneficial in addressing the management difficulties, system adaptivity, and performance \cite{lin2018don} that arise due to fully decentralized architecture \cite{wainakh2020enhancing}. The key challenge of personalization is the heterogeneity of data, locally customized models, and identifying and collaborating among clients those having similar information \cite{lyu2022personalized}. By adding personalization at team and device levels, we aim to capture customized and refined personalized properties of the model that align well with real-world applications \cite{wu2021hierarchical, zhou2023toward}. Communication with the global server is often the most expensive step in FL \cite{abdellatif2022communication}, and communications within a team are typically cheaper \cite{liu2020client}. By employing the multi-tier architecture and accommodating a large portion of the communication within the teams, \ourmodel{} economizes significantly on the communication iterations with the global server, preventing biases to local clients, and achieving faster convergence \cite{wainakh2020enhancing}. These advantages motivated us to propose multi-tier FL.

\section{\ourmodel{}} \label{sec : prop_app}

A multi-tier FL framework (see \Cref{fig: proposed}) is different from the conventional FL framework (see \Cref{fig: existing}) as it follows a hierarchy. All devices are divided into M teams. Each team has a team server ($TS_i$), which is connected with $N_i$ devices of the respective team. Global server $(GS)$ only communicates with $TS_i$'s team. 

\begin{figure}[!ht]
    \centering
    \begin{subfigure}{0.5\linewidth}
        \includegraphics[width=\textwidth]{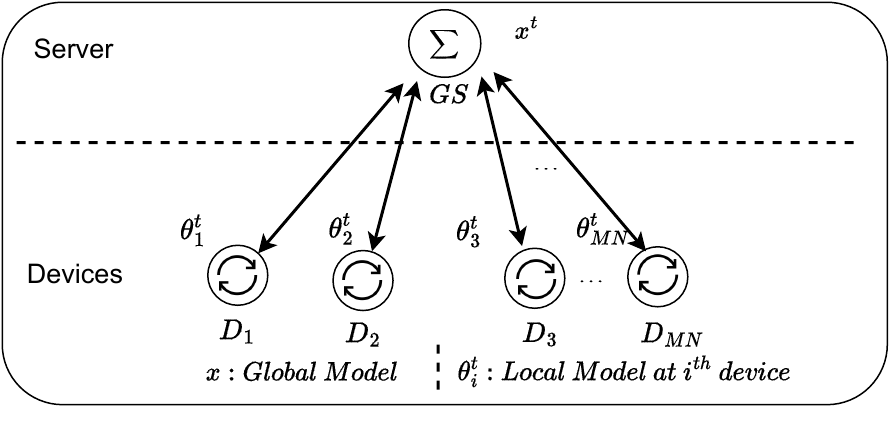}
        \vspace{-2.25em}
        \caption{Conventional FL framework}
        \label{fig: existing}
    \end{subfigure}
    \begin{subfigure}{0.48\linewidth}
        \includegraphics[width=\textwidth]{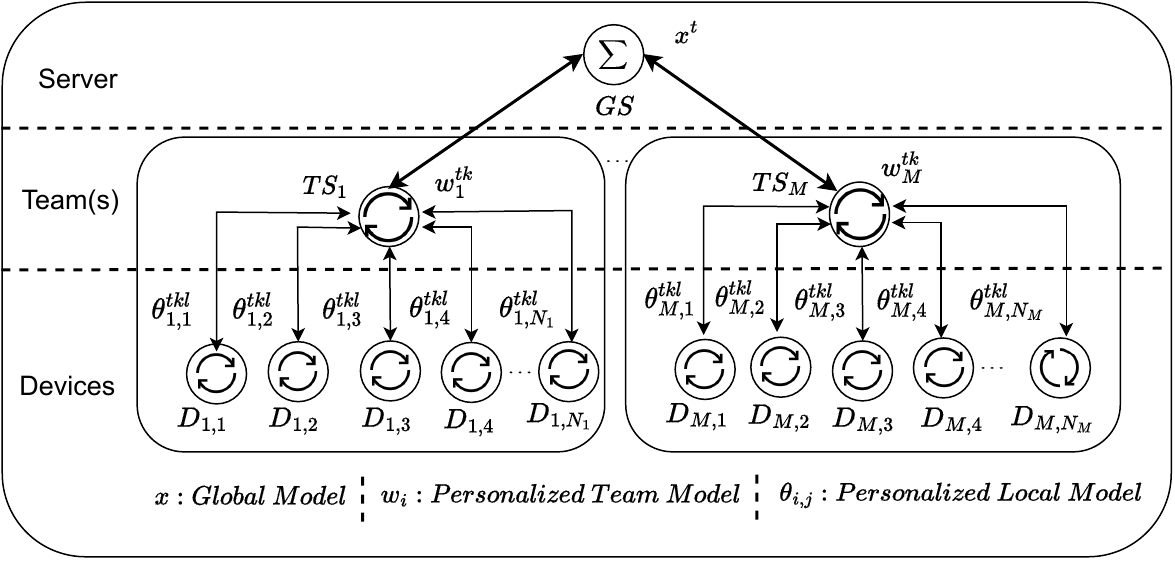}
        \caption{Multi-tier FL framework}
        \label{fig: proposed}
    \end{subfigure}
    \caption{Federated learning work-flow}
    \label{fig: framework}
\end{figure}
\vspace{-0.3in}
\subsection{Team formation } \label{subsec: team}
By definition, FL can be seen as a cross-silo or cross-device setting \cite{kairouz2021advances}. In cross-silo settings, a limited number of devices, typically between 2 and 100, participate in each round, and their participation remains fixed. On the other hand, in cross-device setups, the client pool is extensive, potentially in the millions, but only a small fraction of devices take part in each iteration \cite{kairouz2021advances}.   In a multi-tier structure, teams can be constituted in four ways: (1) Both teams and devices within teams have full participation, (2) Teams have full participation, whereas devices within teams are participating partially, (3) Teams have partial participation, but all devices within teams are participating, and (4) Teams and devices both have partial participation.

Various FL methods with different team formation strategies have been proposed in the literature \cite{long2023multi, yan2023clustered}. \ourmodel{} does not explicitly address the creation of teams. Instead, it exhibits adaptability to accommodate any team formation mechanism. It is important to note that \ourmodel{} is most efficient when communication with the team servers is cheaper compared to communication with the global server. This situation often occurs when devices are geographically distributed and connected to their nearest teams, similar to the Cloud-Edge model \cite{bittencourt2018internet}. We can also argue that team-level personalization is more effective when the data within each team exhibits distinctive characteristics that differentiate it from the data in other teams. Nonetheless, we provide a comprehensive ablation study that thoroughly examines different types of team selection, including poorly constructed and randomly constructed team formations. In cases where teams do not exhibit such distinctive characteristics, our model can still benefit from device-level personalization and reduced communication costs with the global server.

\subsection{\ourmodel{} formulation}

We consider a multi-tier FL setup, 
consisting of $M$ teams, each with $N_i$ devices ($i = 1,\ldots,M$). 
In this setup, empirical risk minimization can be expressed as:
\begin{equation}\label{eqn:ERM-multi}
    \min_{x \in \mathbb{R}^d} \quad  \frac{1}{M} \sum_{i=1}^M \frac{1}{N_i} \sum_{j=1}^{N_i} f_{i,j} (x).
\end{equation}
Here, $f_{i,j}(\cdot)$ represents the loss function of the $j^{th}$ device from the $i^{th}$ team. 
This formulation relies on a single decision variable $x$ that all clients are expected to converge upon. 
However, this model is unsuitable for scenarios involving non-homogeneous data distributions, as the existence of a global model capable of accommodating all devices becomes unreasonable. 

To address this challenge, we introduce decision variables for every team and device, denoted as $w_i$ and $\theta_{i,j}$, respectively. Our goal is to find a global model representing a rough average, capturing common characteristics across all teams and devices; personalized team models that are close to the global model but can deviate from aligning with shared features within each team; and personalized device models that resemble the team model but can deviate to accommodate the unique characteristics of each device. We adopt a quadratic penalty approach to enforce that the device models are close to the team models and the team models to the global model. The parameters $\gamma \geq 0$ and $\lambda \geq 0$ control the degree of personalization impact at the team and device levels respectively. 
\begin{equation}
\label{eqn:model-problem}
    \hspace{-0.5em} \min_{x \in \mathbb{R}^d} \, \min_{w_i \in \mathbb{R}^d} \, \min_{\theta_{i,j} \in \mathbb{R}^d} \frac{1}{M} \sum_{i=1}^M  
    \frac{1}{N_i} \sum_{j=1}^{N_i}  \Big( f_{i,j} (\theta_{i,j}) \!+\!  \frac{\lambda}{2} \| \theta_{i,j} - w_i \|^2 \!+\! \frac{\gamma}{2} \| w_i - x \|^2  \Big).
\end{equation}
We are now prepared to outline the algorithm design. Our approach involves the use of three iteration counters: $t$ for the global rounds, $k$ for team-level rounds, and $\ell$ for device-level rounds. 

\paragraph{Device-level updates.} Given a team model $\smash{w_i^{t,k}}$, the objective of the device $(i,j)$ is to solve the following subproblem:
\begin{equation}
\label{eqn:moreau-device-subproblem}
   \tilde{f}_{i,j}^\lambda (w_i^{t,k}) := \min_{\theta_{i,j} \in \mathbb{R}^d} ~ f_{i,j}(\theta_{i,j} ) + \frac{\lambda}{2} \| \theta_{i,j} - w_i^{t,k} \|^2
\end{equation}
The exact solution to this problem is $\smash{\mathrm{prox}_{f_{i,j}/\lambda}(w_i^{t,k})}$; however, in general, there is no closed-form solution available for this proximal operator. Consequently, each device employs the gradient method to approximate the solution to \eqref{eqn:moreau-device-subproblem}.
Starting with an initial value of $\smash{\theta_{i,j}^{t,k,0} = w_i^{t,k}}$, and utilizing a positive step-size $\alpha$, we perform the following update rule for $l=0,1,\ldots,L-1$:
\begin{equation} \label{eqn:device_p_GD}
\theta_{i,j}^{t,k,l+1} \! = \theta_{i,j}^{t,k,l} \!-\! \alpha_{i,j} \nabla f_{i,j}(\theta_{i,j}^{t,k,l}) \!-\! \alpha_{i,j}  \lambda (\theta_{i,j}^{t,k,l} - w_i^{t,k}).
\end{equation}
\paragraph{Team-level updates.} Similarly, the goal in team-level updates is to solve a regularized subproblem. We define the team-level loss function as 
\begin{equation}
\begin{gathered}
F_i(w_i) := \frac{1}{N_i} \sum_{j=1}^{N_i} \tilde{f}_{i,j}^\lambda (w_i). 
\end{gathered}
\end{equation}
With the addition of regularization towards the global server's model $x^t$, Team $(i)$ aims to solve the following subproblem:
\begin{equation}
\label{eqn:moreau-team-subproblem}
clust    \tilde{F}_i^\gamma (x^t) := \min_{w_i \in \mathbb{R}^d} ~ F_i(w_i) + \frac{\gamma}{2} \| w_i - x^t \|^2.
\end{equation}
Once again, the solution is given by the proximal operator, $\mathrm{prox}_{F_i/\gamma}(x^t)$, which can be difficult to compute. Instead, we can find an approximate solution by using the gradient method. Starting from $\smash{w_i^{t,0} = x^t}$, and using a positive step-size $\eta > 0$, the gradient method update becomes
\begin{equation}\label{eqn:prox_Fi_GD-exact}
\begin{aligned}
    w_i^{t,k+1} 
    & = w_i^{t,k} - \eta_{i}  \nabla F_i(w_i^{t,k}) - \eta_{i}  \gamma (w_i^{t,k} - x^t) \\
    & = w_i^{t,k} -  \frac{\eta_{i} }{N_i} \sum_{j=1}^{N_i} \nabla \tilde{f}_{i,j}^\lambda (w_i^{t,k}) - \eta_{i}  \gamma (w_i^{t,k} - x^t). 
\end{aligned}
\end{equation}
Here, the second line follows from the definition of $F_i$. 
Since we do not have the exact gradient $\nabla \tilde{f}_{i,j} (w_i^{t,k})$, we approximate it by:
\begin{equation}
\label{eqn:approx-device-grad}
    \nabla \tilde{f}_{i,j}^\lambda (w_i^{t,k})
    = \lambda \big( w_i^{t,k} - \mathrm{prox}_{f_{i,j}/\lambda}(w_i^{t,k}) \big) 
    \approx \lambda ( w_i^{t,k} - \theta_{i,j}^{t,k,L} ).
\end{equation}
By combining \eqref{eqn:prox_Fi_GD-exact} and \eqref{eqn:approx-device-grad}, we construct the following update rule for $k=0,1,\ldots,K-1$:
\begin{equation}\label{eqn:prox_Fi_GD}
\begin{aligned}
    w_i^{t,k+1} \!=\! \big(1 \!-\! \eta_{i} (\lambda\!+\!\gamma)\big) w_i^{t,k} + \eta_{i}  \gamma x^t + \frac{\lambda\eta_{i}}{N_i} \sum_{j=1}^{N_i} \theta_{i,j}^{t,k,L}.
\end{aligned}
\end{equation}
\paragraph{Server-level updates.} Finally, the server applies the gradient method for 
\begin{equation}\label{eq:definition of phi}
    \min_{x\in \mathbb{R}^d} ~~ \phi(x) := \frac{1}{M} \sum_{i=1}^M \tilde{F}_i^\gamma(x). 
\end{equation}
Starting from an initial state $\smash{x^0 \in \R^d}$ and using a positive step-size $\beta$, the gradient update becomes:
\begin{equation}\label{eqn:global_update-pre}
    x^{t+1} 
    = x^t - \frac{\beta}{M} \sum_{i=1}^M \nabla \tilde{F}_i^\gamma (x^t). 
\end{equation}
Again, we use an approximation for $\smash{\nabla \tilde{F}_i (x^t)}$ based on the team-level models:
\begin{equation}\label{eqn:grad-moreau}
    \nabla \tilde{F}_i^\gamma(x^t) = \gamma \big( x^t - \mathrm{prox}_{F_i/\gamma}(x^t) \big) \approx \gamma ( x^{t} - w_{i}^{t,K} ).
\end{equation}
Finally, we construct the server-level update rule by combining \eqref{eqn:global_update-pre} and \eqref{eqn:grad-moreau}. 
For $t=0,1,\ldots,T-1$, the server performs the following update rule:
\begin{equation}\label{eqn:global_update}
    x^{t+1} = (1-\beta \gamma) x^t + \frac{\beta\gamma}{M}\sum_{i=1}^M  w_{i}^{t,K}.
\end{equation}
\paragraph{Synthesis.} 
By combining device, team, and server-level updates, we propose \ourmodel{} (\Cref{alg:pfedmt_algo}) for personalized multi-tier FL. 

\begin{algorithm}[ht!]
\scriptsize
{
\caption{\ourmodel{} : Personalized Multi-tier FL }\label{alg:pfedmt_algo}
\textbf{Input :} $x^0$ 

\textbf{Output :} $\forall_{i=1}^M \forall_{j=1}^{N_i} \theta^{T,K,L}_{i,j}$, $x^T$ 

\textbf{Initialize :} $\forall_{i=1}^{M} w_i^{0,0} = x^0$ , $\forall_{i=1}^M \forall_{j=1}^{N_i} \theta_{i,j} = w_i^{0,0}$, $T$,$K$,$L$,$\alpha_{i,j}$, $\beta$, $\gamma$, $\lambda$, $\eta_i$ \\[0.5em]

\begin{algorithmic}[1]
{
\FOR[Global iterations]{$t = 0,1,\ldots,T-1$} \label{step: ge} 
    \STATE global server sends $x^t$ to the teams. \label{step: gm}  \COMMENT{Global model}
    \FOR[Team iterations]{$k = 0,1,\ldots,K-1$}  \label{step: te}   
        \STATE Teams send $\witk$ to the devices. \label{step: tm} \COMMENT{Team-level personalized model}
        \FOR[Local iterations]{$l = 0,1, \ldots,L-1$} \label{step: le}  
            \STATE $\theta_{i,j}^{t,k,l+1} = \theta_{i,j}^{t,k,l} - \alpha_{i,j} \nabla f_{i,j}(\theta_{i,j}^{t,k,l}) - \alpha_{i,j} \lambda (\theta_{i,j}^{t,k,l} - w_i^{t,k})$
            \label{step: theta}
            \COMMENT{Personalized local models}
        \ENDFOR \label{step: end_le}
        \STATE $\bar{\theta}_i^{t,k} =  \frac{1}{N_i}\sum_{j=1}^{ N_i} \theta_{i,j}^{t,k,L}$ \label{step: theta_bar} \COMMENT{Aggregation within a team}
        \STATE $w_i^{t,k+1} =(1- \eta_i \lambda - \eta_i\gamma) w_i^{t,k} + \eta_i \gamma x^t + \lambda\eta_i \bar{\theta}_i^{t,k} $ \label{step: w}\COMMENT{Personalized team update}
    \ENDFOR \label{step: end_te}
    \STATE  $\bar{w}^t = \frac{1}{M}\sum_{i=1}^{M} w_{i}^{t,K}$ \label{step: g_agg}\COMMENT{Global aggregation} 
    \STATE $x^{t+1} = (1-\beta\gamma)x^t + \beta\gamma\bar{w}^t$ \label{step: gu} \COMMENT{Global update} 
\ENDFOR \label{step: end_ge}
}
\end{algorithmic}
}
\end{algorithm}

\textit{Initialization:} Global server initializes the global model ($x^0$). Every team connected to the global server, copies the global model to their respective team model ($\forall_{i=1}^{M}w^{0,0}_i$ = $x^0$). Each device within a team copies the initial team model ($w_i^{0,0}$) as their local model ($\theta_{i,j}$). Global server initializes the total number of global iterations, team iterations, and local iterations as $T, K,$ and $L$, respectively.

\textit{Iterations:} At each global iteration $t$, the global server broadcasts the global model $x^t$ to every team. Similarly, at each team iteration $k$, each team broadcasts $w_i^{t,k}$ to all the devices within the team. For each local iteration ($l$), each device solves \eqref{eqn:moreau-device-subproblem} separately, but in parallel to obtain the personalized model $\smash{\theta_{i,j}^{t,k,l}}$. The team server of each team collects the device updates from the respective devices after $L$ local iterations and performs aggregation ($\smash{\bar \theta_i^{t,k}}$) on the device updates. Each team broadcasts the updated team-level model to the devices registered with that team and continues steps 3 to 10 for the next $K-1$ team iterations. After all teams finished $K$ team iterations, the global server collects team updates ($w_i^{t,K}$) from each team and performs averaging ($\bar w^t$) on team updates over M teams. The global server produces a global update ($x^t$) by solving \eqref{eqn:global_update}.
The global server broadcasts the updated global model to teams and continues steps 1 to 13 for the upcoming $T-1$ global iterations.

\begin{remark}
The quadratic penalty approach leads to the concept of Moreau envelopes, a mathematical tool frequently employed in optimization theory for smoothing functions. For a function $g: \mathbb{R}^d \to \mathbb{R}$, we define the Moreau envelope $\tilde{g}^\sigma: \mathbb{R}^d \to \mathbb{R}$ with parameter $\sigma \geq 0$ as 
\begin{equation}
\label{eqn:moreau-envelope}
    \tilde{g}^\sigma(x) := \min_{u \in \mathbb{R}^d}  ~ \Big\{ g(u) + \frac{\sigma}{2} \| u- x\|^2 \Big\}.
\end{equation}

Clearly, we can interpret \eqref{eqn:moreau-device-subproblem} and \eqref{eqn:moreau-team-subproblem} 
as the Moreau envelopes of $f_{i,j}$ and $F_i$, respectively. It is important to note that Moreau envelopes have been used before in \cite{t2020personalized} for personalization in the conventional FL setting. 
\end{remark}

\subsection{Convergence guarantees} \label{sec : convergence_analysis}

This section presents the convergence guarantees of \ourmodel{}. We consider two different settings, with strongly convex and non-convex loss functions. In both cases, we assume that the loss functions are smooth in the sense that they have Lipschitz continuous gradients. 
The next Theorem formalizes the guarantees of when $\fij$ is strongly convex.

\begin{theorem}[Strongly convex]  \label{PfedMT:theorem: analysis of the strongly convex case}
Consider the minimization problem $\min_x \phi(x)$ when $\phi(x)$ is defined in \eqref{eq:definition of phi} with $L_f$-smooth and $\mu_f$-strongly convex loss functions $\fij(x)$. For large enough numbers of inner iterations of orders $L=\Omega \big(K \big)$ and $K=\Omega \big(T \big)$, see the supplementary copy for the bounds, estimation $\{\xt\}_{t=0}^T$ generated by \ourmodel{} with step-size  $\beta$ satisfies:
\begin{align}
\|\xT-x^*\|^2 \leq 2(1-\beta )^T \|x^0- x^*\|^2.
\end{align}
 where learning rates should satisfy $\beta \leq \frac{\muFtilde}{4\gamma}$, $\eta_i\leq \frac{1}{2(\lambda+\gamma)}$, $\alpha_{i,j}  \leq \frac{1}{\Lflambda}$, $\muFtilde:=\frac{\lambda \gamma \mu_f}{\lambda \mu_f+ \gamma \mu_f+\lambda \gamma }$, and $\gamma > 2 \lambda > 4 L_f $.
\end{theorem}

\noindent\textit{Proof sketch}.~
We analyze the algorithm in three levels: 
{\em(i)}~The devices find approximate solutions to problem (\ref{eqn:moreau-device-subproblem}) by using the gradient method. We can control the accuracy of this stage by choosing $L$ (the number of iterations for the gradient method) large enough. {\em(ii)}~The teams solve problem (\ref{eqn:moreau-team-subproblem}) approximately, again by using a gradient method. We use an inexact gradient at this stage since the exact gradient requires exact solutions from the devices to which we do not have access. At this stage, $K$ is the number of iterations, and $L$ modulates the accuracy of our gradients. By choosing both $K$ and $L$ large enough, we can control the solution accuracy achieved at the teams' level. {\em(iii)} Finally, the Server solves the problem (\ref{eq:definition of phi}), the original FL problem, by using the information provided by the Teams. 
It is worth noting that for a fixed $T$, we can decrease the error (down to a threshold) by increasing $K$ and $L$. More precisely, we achieve linear convergence rates when we choose $K$ and $L$ in the order of $\Omega{(T)}$. The complete proof is left to the supplementary material. 

The next Theorem shows that \ourmodel{} finds a first-order stationary point with sublinear rates when $\fij$ are smooth but non-convex.  
 
\begin{theorem}[Non-convex] \label{theorem} \label{PfedMT:theorem: analysis of the non-convex case}
Consider the minimization problem $\min_x \phi(x)$ when $\phi(x)$ is defined in \eqref{eq:definition of phi} with non-convex $L_f$-smooth loss functions $\fij(x)$. For large enough numbers of inner iterations of orders $ L=\Omega \big( K\big) $ and $ K=\Omega \big( T\big) $, see the appendix for the bounds, then, estimation $\{\xt\}_{t=0}^T$ generated by \ourmodel{}  with step-size $\beta$ satisfies: 
\begin{align}  \label{PfedMT:theorem: analysis of the non-convex case eq1}
\mathbb{E} \big[ \nsq{\nabla\phi(x^{\tilde{t}})}  \big] 
&\leq  \frac{\phi(x^0)-\phi(x^*)}{\beta T}
\end{align}
where $\beta \leq \frac{1}{4\gamma}$, $\eta_i \leq \frac{1}{\lambda+\gamma}$, $\alpha_{i,j} \leq \frac{1}{\lambda}$, $\gamma > 2 \lambda > 4 L_f $, and $\Tilde{t}$ is uniformly sampled from $\{0, \ldots, T-1\}$
\end{theorem}

\noindent\textit{Proof sketch}.~
The analysis follows a similar structure to the previous setting. The main difference is that the errors in subproblems are guaranteed as a bound on the gradient norms, which we translate to error bounds on objective residual by tuning $\lambda$ and $\gamma$. The proof can be found in the supplementary material. 

\begin{remark}
At first glance, it may come as a surprise that our guarantees do not necessitate a bounded drift condition, which is common in FL methods involving multiple local steps. It is important to note a fundamental distinction between the conventional FL template \eqref{eqn:ERM-multi} and our personalized multi-tier FL template \eqref{eqn:model-problem}. The former lacks consideration for data heterogeneity as it relies solely on a single global variable. This can result in `drift-away' issues when multiple local steps are taken, especially in the presence of data heterogeneity. In contrast, our formulation explicitly incorporates team and device-level variables, and our local steps are tailored to solve device and team-level subproblems \eqref{eqn:moreau-device-subproblem} and \eqref{eqn:moreau-team-subproblem}. The regularization employed in these subproblems prevents `over-drifting' by explicitly penalizing the divergence between device, team, and global models in a suitable manner.
\end{remark}

\section{Experiments} \label{sec : exp_res}
We studied classification problems to validate \ourmodel{} using both benchmarks (MNIST \citep{lecun1998gradient}, FMNIST \citep{xiao2017online} EMNIST \citep{cohen2017emnist} with 10 classes  (EMNIST-10), EMNIST with 62 classes (FEMNIST), CIFAR100 \citep{krizhevsky2009learning}) and non-image synthetic datasets. For the MNIST, FMNIST, EMNIST-10, and synthetic datasets, the data was distributed among multiple devices in a non-iid manner, ensuring each device had data from at most two classes. Subsequently, the devices were randomly grouped into four teams, each consisting of 10 devices, before performing \ourmodel{}. We considered the full participation of teams and devices in each global round for the performance and convergence evaluation. For the FEMNIST and CIFAR100 datasets, the data was distributed to 3,500 and 350 devices, respectively, with each device holding data from the 3 classes. The devices are arranged into 5 teams. All datasets are split into training and validation sets with a 3:1 ratio. 

 We considered a multi-class logistic regression (MCLR) model with a softmax activation function for strongly convex scenarios. For synthetic datasets, we constructed deep neural networks with two hidden layers, while for image datasets, we built two-layered convolutional neural networks for non-convex scenarios. More details of the experimental setup, datasets, and learning models are in the supplementary copy. 

In this paper, we conducted (1) the performance comparison between \ourmodel{} with FedAvg \citep{mcmahan2017communication}, pFedMe\citep{t2020personalized}, Per-FedAvg \citep{fallah2020personalized}, pFedBayes \citep{zhang2022personalized}, Ditto \citep{li2021ditto}, and performance and convergence comparison with two hierarchical FL algorithms, such as hierarchical-SGD (h-SGD)\citep{liu2022hierarchical}, Asynchronous L2GD (AL2GD)\citep{lyu2022personalized}, and a hierarchical-clustered FL algorithm DemLearn \citep{nguyen2022self}.
(2) We investigated the impact of $\beta$, $\gamma$, and $\lambda$ on the convergence of \ourmodel{}. (3) An ablation study to explore team formation. (4) An ablation study to analyze the influence of team and device participation. Moreover, in the supplementary copy, we gave an ablation study that explores the effects of team iterations on the convergence of \ourmodel{}.  Throughout the experiments, we denoted the personalized model and global model as (PM) and (GM), respectively We have made the implementation of \ourmodel{} available at \url{https://github.com/sourasb05/PerMFL_1.git}

\subsection{Results and Analysis}
\subsubsection{Performance:}
From \cref{tab: perf}, we observed that \ourmodel{}(PM) outperformed the state-of-the-art for non-convex cases in all datasets. \ourmodel{}(GM) outperformed other global models, including FedAvg(GM), pFedMe(GM), and Ditto(GM), and nearly equivalent with h-SGD(GM) and DemLearn(GM) for the MNIST dataset. For the synthetic dataset, \ourmodel{}(GM) outperformed the state-of-the-art. For FMNIST and EMNIST-10 datasets, the performance of \ourmodel{}(GM) is better than the conventional FL models and DemLearn(GM). For strongly convex cases, \ourmodel{}(PM) outperformed the state of the art in MNIST and Synthetic datasets. For the FMNIST dataset, \ourmodel{}(PM)'s performance is better than the conventional FL state-of-the-art. Moreover, the performance of \ourmodel{}(PM) is nearly equivalent to the DemLearn(PM) for FMNIST and EMNIST-10. \ourmodel{}(PM) also achieved better performance than h-SGD and AL2GD on FEMNIST and CIFAR100 given in the supplementary copy. \ourmodel{}(GM) outperformed the state-of-the-art in Synthetic and FMNIST datasets. \ourmodel{}(GM) performs better than conventional methods in all datasets and is nearly equivalent to h-SGD(GM) on MNIST and EMNIST-10. From these observations, we can infer that \ourmodel{}(PM) performs better than the 7 state-of-the-art methods on 6 out of 8 experiments. The reason could be, the personalization in both team and devices helps to get better performance.

\begin{table*}[!ht]
\centering
\scriptsize
\caption{Performance (Validation accuracy(mean/std)(\%)) comparison of \ourmodel{} with state-of-the-art.  } \label{tab: perf}
\begin{adjustbox}{max width=\textwidth}
\setstretch{1.1}
\begin{tabular}{| c | c | c | c | c | c |}
    \hline
    Architecture & \multicolumn{5}{| c |}{MCLR (Strongly convex)}  \\
    \cline{2-6} 
    \parbox[t]{0.5mm}{\multirow{7}{*}{\rotatebox[origin=c]{90}{Conventional}}}& Algorithm & MNIST  & Synthetic & FMNIST & EMNIST-10 \\ 
   \hline
   & FedAvg(GM)\citep{mcmahan2017communication} & 84.87 ($\pm$ 0.054) & 84.87($\pm$ 0.054) & 79.80 ($\pm 0.002$) & 91.60($\pm 0.001$) \\
   & Per-FedAvg (PM)\citep{fallah2020personalized}  & 94.81($\pm$ 0.00) & 83.91($\pm 0.15$)  & 94.75 ($\pm 0.00$) & 97.57($\pm 0.0 $)  \\

   & pFedMe(GM)\citep{t2020personalized} & 75.50($\pm 0.00$) &  81.93($\pm 0.21$) &  83.45($\pm 0.21$) & 88.78($\pm 0.01$) \\
   & pFedMe(PM)\citep{t2020personalized}  & 88.89($\pm 0.001$)  & 87.61($\pm 0.32$) & 91.32 ($\pm 0.08$) & 91.23($\pm 0.01 $) \\
   & pFedBayes(PM)\citep{zhang2022personalized} & 94.13($\pm 0.27$) & 87.05($\pm 0.5$) & 92.14($\pm 0.001$) & 94.13($\pm 0.001$) \\

   & Ditto (GM) \citep{li2021ditto} & 84.81($\pm 0.001 $) & 82.35($\pm 0.001 $) & 74.02($\pm 0.001 $) & 91.03 ($\pm 0.0003 $) \\
    \hline
   \parbox[t]{0.5mm}{\multirow{4}{*}{\rotatebox[origin=c]{90}{Multi-tier}}} & h-SGD (GM) \citep{wang2022demystifying}  & \textbf{87.41}($\pm 6.35 $) & 84.29($\pm 5.18$) & 81.653($\pm 1.8$) & \textbf{92.33}($\pm 0.001$) \\
   & AL2GD(PM) \citep{lyu2022personalized} & 93.70 ($\pm$0.13)  & 84.75($\pm 0.03$) & \textbf{98.52($\pm$0.004)} & \textbf{98.72($\pm 0.001 $)} \\

     & DemLearn (GM)\citep{nguyen2022self} & 87.32($\pm 0.002 $) & 67.93($\pm 0.04$) & 62.60($\pm 0.002$) & 69.09($\pm 0.12$) \\
   & DemLearn (PM)\citep{nguyen2022self} & 91.26 ($\pm 0.01 $)  & 81.21($\pm 0.01$) & 97.50($\pm 0.0$) & 97.24($\pm 0.005$) \\

   & \cellcolor{LightCyan} \textbf{\ourmodel{}(GM)} [ ours ]  & \cellcolor{LightCyan} 86.92($\pm 0.013$) & \cellcolor{LightCyan} \textbf{84.92}($\pm 0.06$) & \cellcolor{LightCyan} \textbf{83.71}($\pm 0.001$) & \cellcolor{LightCyan} 91.68($\pm 0.0$)\\

   & \cellcolor{LightCyan} \textbf{\ourmodel{}(PM)} [ ours ]  & \cellcolor{LightCyan} \textbf{96.87($\pm 0.0$)}  & \cellcolor{LightCyan} \textbf{87.94($\pm 0.001$)} & \cellcolor{LightCyan} 96.77 ($\pm 0.0$) & \cellcolor{LightCyan} 96.49($\pm 0.0$) \\  \hline

   \multicolumn{6}{| c |}{DNN or CNN (Non-convex)}  \\ \hline
   
   \parbox[t]{0.5mm}{\multirow{5}{*}{\rotatebox[origin=c]{90}{Conventional}}} & FedAvg(GM) & 93.17 ($\pm 0.02$) & 84.53($\pm0.067$) &  \textbf{84.14}($\pm 0.00$) & 92.73($\pm 0.003 $) \\
   & Per-FedAvg(PM)  & 91.845($\pm 0.00$) & 75.93 ($\pm 0.18$) & 88.69($\pm 0.269$) & 97.37($\pm 0.01$) \\
   
   & pFedMe(GM) & 80.12($\pm 0.01$) & 81.23($\pm 0.19$) & 68.64 ($\pm 0.009 $) & 91.81 ($\pm 0.0002 $) \\
   & pFedMe(PM) & 97.40($\pm 0.001$) & 87.86($\pm 0.06$) & 96.30 ($\pm 0.001 $) & 97.18($\pm0.0003$) \\

   & Ditto(GM) & 87.30($\pm 0.03$) & 81.12($\pm 0.006 $) & 57.80($\pm 0.001 $) & 90.58($\pm 0.004 $) \\
    \hline
   \parbox[t]{0.5mm}{\multirow{6}{*}{\rotatebox[origin=c]{90}{Multi-tier}}} & h-SGD(GM) & 86.59 ($\pm 7.14 $) & 87.42 ($\pm 5.67$) & 79.84 ($\pm 0.035 $) & 96.03 ($\pm 0.001$) \\

   & AL2GD(PM) & 91.04 ($\pm 0.035$) & 84.92 ($\pm 0.02$) & 71.32($\pm 0.13$) & 92.94 ($\pm 0.14 $) \\

   & DemLearn(GM) & \textbf{90.75($\pm 0.001 $)} & 68.91($\pm 0.05$) & 64.84 ($\pm 0.002 $) & \textbf{96.63 ($\pm 0.005$)} \\
   
   & DemLearn(PM) & 97.20($\pm 0.001$) & 82.74($\pm 0.008$) & 98.64($\pm 0.0 $) & 98.74($\pm 0.0 $) \\

   & \cellcolor{LightCyan} \textbf{\ourmodel{}(GM)} [ ours ]  & \cellcolor{LightCyan} 89.39 ($\pm 0.001$) & \cellcolor{LightCyan} \textbf{87.53 ($\pm 0.0$) }& \cellcolor{LightCyan} 79.15($\pm 0.0$) & \cellcolor{LightCyan} 93.12($\pm 0.0$) \\
   & \cellcolor{LightCyan} \textbf{\ourmodel{}(PM)} [ ours ]  & \cellcolor{LightCyan} \textbf{98.15 ($\pm 0.0$)} & \cellcolor{LightCyan} \textbf{87.89($\pm 0.0$)} & \cellcolor{LightCyan} \textbf{98.67}($\pm 0.0$) & \cellcolor{LightCyan} \textbf{98.79}($\pm 0.0$) \\ \hline

\end{tabular}
\end{adjustbox}
\end{table*}

\subsubsection{Convergence:}
 From \cref{Fig:convergence_sota_fmnist}, we observed that the convergence of \ourmodel{}(PM) is equivalent to DemLearn and is faster than h-SGD and AL2GD. Similar findings were also observed in EMNIST-10, Synthetic, and MNIST datasets given in the supplementary copy. It is because, inside each team multiple iterations are happening, that helps the personalized model to converge quickly.


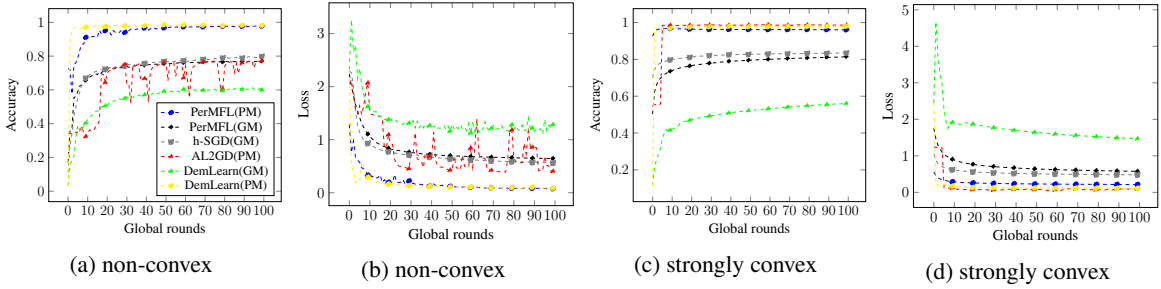
\begin{figure*}[h!]
        \centering
        \begin{subfigure}[!t]{0.23\linewidth}
            \resizebox{\linewidth}{!}{
                \begin{tikzpicture}
                \tikzstyle{every node}=[font=\fontsize{12}{12}\selectfont]

\begin{axis}[
  xlabel= Global rounds,
  ylabel= Accuracy,
  legend style ={nodes={scale=0.9, transform shape}},
  legend pos=south east,
  xtick={0,10,20,30,40,50,60,70,80,90,100}
]
\addplot[blue,dashed,thick,mark=otimes*, mark repeat=10, mark phase=10] table [y=PerMFL(PM) ,x=GR]{final-perf-fmnist-cnn-acc.dat};
\addlegendentry{\ourmodel{}(PM)}

\addplot[black,dashed,thick,mark=diamond*, mark repeat=10, mark phase=10]  table [y=PerMFL(GM) ,x=GR]{final-perf-fmnist-cnn-acc.dat};
\addlegendentry{\ourmodel{}(GM)}

\addplot[gray,dashed,thick,mark=square*, mark repeat=10, mark phase=10]  table [y=h-sgd,x=GR]{final-perf-fmnist-cnn-acc.dat};
\addlegendentry{h-SGD(GM)}

\addplot[red,dashed,thick,mark=triangle*, mark repeat=10, mark phase=10]  table [y=AL2GD,x=GR]{final-perf-fmnist-cnn-acc.dat};
\addlegendentry{AL2GD(PM)}

\addplot[green,thick,dashed,mark=triangle*, mark repeat=10, mark phase=10]  table [y=DemLearn(C-GEN),x=GR]{final-perf-fmnist-cnn-acc.dat};
\addlegendentry{DemLearn(GM)}

\addplot[yellow,thick,dashed,mark=*, mark repeat=10, mark phase=10]  table [y=DemLearn(C-SPE),x=GR]{final-perf-fmnist-cnn-acc.dat};
\addlegendentry{DemLearn(PM)}

\end{axis}
\end{tikzpicture}
            }
        \caption{non-convex}
        \label{Fig:fmnist_non_conv_acc_sota}
    \end{subfigure} 
    \begin{subfigure}[!t]{0.23\linewidth}
        \resizebox{\linewidth}{!}{
        \begin{tikzpicture}
        \tikzstyle{every node}=[font=\fontsize{12}{12}\selectfont]
\begin{axis}[
  xlabel= Global rounds,
  ylabel= Loss,
  legend style ={nodes={scale=1.45, transform shape}},
  legend pos= north east,
  xtick={0,10,20,30,40,50,60,70,80,90,100}
]
\addplot[blue,dashed,thick,mark=otimes*, mark repeat=10, mark phase=10] table [y=PerMFL(PM) ,x=GR]{final-perf-fmnist-cnn-loss.dat};

\addplot[black,dashed,thick,mark=diamond*, mark repeat=10, mark phase=10]  table [y=PerMFL(GM) ,x=GR]{final-perf-fmnist-cnn-loss.dat};

\addplot[gray,dashed,thick,mark=square*, mark repeat=10, mark phase=10]  table [y=h-sgd,x=GR]{final-perf-fmnist-cnn-loss.dat};

\addplot[red,dashed,thick,mark=triangle*, mark repeat=10, mark phase=10]  table [y=AL2GD,x=GR]{final-perf-fmnist-cnn-loss.dat};

\addplot[green,thick,dashed,mark=triangle*, mark repeat=10, mark phase=10]  table [y=DemLearn(C-GEN),x=GR]{final-perf-fmnist-cnn-loss.dat};

\addplot[yellow,thick,dashed,mark=*, mark repeat=10, mark phase=10]  table [y=DemLearn(C-SPE),x=GR]{final-perf-fmnist-cnn-loss.dat};

\end{axis}
\end{tikzpicture}
}

    \caption{non-convex}
    \label{Fig:fmnist_non_conv_loss_sota}
\end{subfigure}
\begin{subfigure}[!t]{0.23\linewidth}
            \resizebox{\linewidth}{!}{
                \begin{tikzpicture}
                \tikzstyle{every node}=[font=\fontsize{12}{12}\selectfont]

\begin{axis}[
  xlabel= Global rounds,
  ylabel= Accuracy,
  legend style ={nodes={scale=1.45, transform shape}},
  legend pos= south east,
   xtick={0,10,20,30,40,50,60,70,80,90,100},
  legend style ={nodes={scale=1.3, transform shape}},
  legend pos= south east
]
\addplot[blue,thick,dashed,mark=otimes*, mark repeat=10, mark phase=10] table [y=PerMFL(PM) ,x=GR]{final-perf-fmnist-mclr-acc.dat};

\addplot[black,dashed,thick,mark=diamond*, mark repeat=10, mark phase=10]  table [y=PerMFL(GM) ,x=GR]{final-perf-fmnist-mclr-acc.dat};

\addplot[gray,dashed,thick,mark=square*, mark repeat=10, mark phase=10]  table [y=h-sgd,x=GR]{final-perf-fmnist-mclr-acc.dat};

\addplot[red,dashed,thick,mark=triangle*, mark repeat=10, mark phase=10]  table [y=AL2GD,x=GR]{final-perf-fmnist-mclr-acc.dat};

\addplot[green,thick,dashed,mark=triangle*, mark repeat=10, mark phase=10]  table [y=DemLearn(C-GEN),x=GR]{final-perf-fmnist-mclr-acc.dat};

\addplot[yellow,thick,dashed,mark=*, mark repeat=10, mark phase=10]  table [y=DemLearn(C-SPE),x=GR]{final-perf-fmnist-mclr-acc.dat};

\end{axis}

\end{tikzpicture}

            }
        \caption{strongly convex}
        \label{Fig:fmnist_conv_acc_sota}
    \end{subfigure} 
    \begin{subfigure}[!t]{0.23\linewidth}
        \resizebox{\linewidth}{!}{
        \begin{tikzpicture}
        \tikzstyle{every node}=[font=\fontsize{12}{12}\selectfont]
        \begin{axis}[
  xlabel= Global rounds,
  ylabel= Loss,
  legend style ={nodes={scale=1.45, transform shape}},
  legend pos= north east,
    xtick={0,10,20,30,40,50,60,70,80,90,100}
]
\addplot[blue,dashed,thick,mark=otimes*, mark repeat=10, mark phase=10] table [y=PerMFL(PM) ,x=GR]{final-perf-fmnist-mclr-loss.dat};

\addplot[black,dashed,thick,mark=diamond*, mark repeat=10, mark phase=10]  table [y=PerMFL(GM) ,x=GR]{final-perf-fmnist-mclr-loss.dat};

\addplot[gray,dashed,thick,mark=square*, mark repeat=10, mark phase=10]  table [y=h-sgd,x=GR]{final-perf-fmnist-mclr-loss.dat};

\addplot[red,dashed,thick,mark=triangle*, mark repeat=10, mark phase=10]  table [y=AL2GD,x=GR]{final-perf-fmnist-mclr-loss.dat};

\addplot[green,thick,dashed,mark=triangle*, mark repeat=10, mark phase=10]  table [y=DemLearn(C-GEN),x=GR]{final-perf-fmnist-mclr-loss.dat};

\addplot[yellow,thick,dashed,mark=*, mark repeat=10, mark phase=10]  table [y=DemLearn(C-SPE),x=GR]{final-perf-fmnist-mclr-loss.dat};

\end{axis}
\end{tikzpicture}
}

    \caption{strongly convex}
    \label{Fig:fmnist_conv_loss_sota}
\end{subfigure}

\caption{ Convergence of \ourmodel{} with multi-tier SOTA in strongly convex and non-convex settings on FMNIST }
\label{Fig:convergence_sota_fmnist}
\end{figure*}


\subsubsection{Effect of hyghperparameters $\beta$, $\gamma$, and $\lambda$:}
From \cref{fig: mnist_conv_beta_gamma_lamda_cnn_mclr_main}, we observed if we increase the value of $\beta$, $\gamma$, and $\lambda$ separately then \ourmodel{}(PM) converge faster. A similar observation is found for FMNIST, and the synthetic dataset is given in the supplementary copy. In all experiments, the hyperparameters followed the bounds given in \cref{PfedMT:theorem: analysis of the strongly convex case} for strongly convex and \cref{PfedMT:theorem: analysis of the non-convex case} for non-convex and smooth problems.


\begin{figure*}[!ht]
        \centering
        \begin{subfigure}[!t]{0.3\linewidth}
            \resizebox{\linewidth}{!}{
                \begin{tikzpicture}
                \tikzstyle{every node}=[font=\fontsize{12}{12}\selectfont]

                \begin{axis}[
  xlabel= Global rounds,
  ylabel= Personalized Loss,
  legend style={at={(0.5,1.03)},anchor=south}, 
  legend columns=2
   xtick={0,10,20,30,40,50,60,70,80,90,100}
]
\addplot[red,dashed,thick,mark=diamond*]  table [y=pbeta_0.1,x=GR]{hyperparameters_mnist_beta_loss_cnn.dat};
\addlegendentry{$\beta=0.1$ (CNN)}

\addplot[green,dashed,thick,mark=triangle*]  table [y=pbeta_0.2,x=GR]{hyperparameters_mnist_beta_loss_cnn.dat};
\addlegendentry{$\beta=0.2$ (CNN)}

\addplot[blue,dashed,thick,mark=square*]  table [y=pbeta_0.3,x=GR]{hyperparameters_mnist_beta_loss_cnn.dat};
\addlegendentry{$\beta=0.3$ (CNN)}

\addplot[cyan,dashed,thick,mark=diamond*]  table [y=pbeta_0.1,x=GR]{hyperparameters_mnist_beta_loss_mclr.dat};
\addlegendentry{$\beta=0.1$ (MCLR)}

\addplot[gray,dashed,thick,mark=triangle*]  table [y=pbeta_0.2,x=GR]{hyperparameters_mnist_beta_loss_mclr.dat};
\addlegendentry{$\beta=0.2$ (MCLR) }

\addplot[black,dashed,thick,mark=square*]  table [y=pbeta_0.3,x=GR]{hyperparameters_mnist_beta_loss_mclr.dat};
\addlegendentry{$\beta=0.3$ (MCLR)}

\end{axis}

\end{tikzpicture}
}

    \caption{Effect of $\beta$ }
    \label{Fig: mnist_conv_beta_pm_cnn_loss_main}
\end{subfigure}
    \begin{subfigure}[!t]{0.3\linewidth}
        \resizebox{\linewidth}{!}{
        \begin{tikzpicture}
        \tikzstyle{every node}=[font=\fontsize{12}{12}\selectfont]

\begin{axis}[
  xlabel= Global rounds,
  ylabel= Personalized Loss,
  legend style={at={(0.5,1.03)},anchor=south}, 
  legend columns=2
  xtick={0,10,20,30,40,50,60,70,80,90,100}
]
\addplot[red,dashed,thick,mark=diamond*]  table [y=pgamma_2.0,x=GR]{hyperparameters_mnist_gamma_loss_cnn.dat};
\addlegendentry{$\gamma=2.0$ (CNN)}

\addplot[green,dashed,thick,mark=triangle*]  table [y=pgamma_4.0,x=GR]{hyperparameters_mnist_gamma_loss_cnn.dat};
\addlegendentry{$\gamma=4.0$ (CNN)}

\addplot[blue,dashed,thick,mark=square*]  table [y=pgamma_8.0,x=GR]{hyperparameters_mnist_gamma_loss_cnn.dat};
\addlegendentry{$\gamma=8.0$ (CNN)}

\addplot[cyan,dashed,thick,mark=diamond*]  table [y=pgamma_2.0,x=GR]{hyperparameters_mnist_gamma_loss_mclr.dat};
\addlegendentry{$\gamma=2.0$ (MCLR)}

\addplot[gray,dashed,thick,mark=triangle*]  table [y=pgamma_4.0,x=GR]{hyperparameters_mnist_gamma_loss_mclr.dat};
\addlegendentry{$\gamma=4.0$ (MCLR)}

\addplot[black,dashed,thick,mark=square*]  table [y=pgamma_8.0,x=GR]{hyperparameters_mnist_gamma_loss_mclr.dat};
\addlegendentry{$\gamma=8.0$ (MCLR)}

\end{axis}

\end{tikzpicture}
}

    \caption{Effect of $\gamma$}
    \label{Fig: mnist_conv_gamma_pm_cnn_loss_main}
\end{subfigure}
    \begin{subfigure}[!t]{0.3\linewidth}
        \resizebox{\linewidth}{!}{
        \begin{tikzpicture}
        \tikzstyle{every node}=[font=\fontsize{12}{12}\selectfont]

\begin{axis}[
  xlabel= Global rounds,
  ylabel= Personalized Loss,
  legend style={at={(0.5,1.03)},anchor=south}, 
  legend columns=2
  xtick={0,10,20,30,40,50,60,70,80,90,100}
]
\addplot[red,dashed,thick,mark=diamond*]  table [y=plamdba_0.1,x=GR]{hyperparameters_mnist_lamda_loss_cnn.dat};
\addlegendentry{$\lambda=0.1$ (CNN)}

\addplot[green,dashed,thick,mark=triangle*]  table [y=plamdba_0.3,x=GR]{hyperparameters_mnist_lamda_loss_cnn.dat};
\addlegendentry{$\lambda=0.3$ (CNN)}

\addplot[blue,dashed,thick,mark=square*]  table [y=plamdba_0.5,x=GR]{hyperparameters_mnist_lamda_loss_cnn.dat};
\addlegendentry{$\lambda=0.5$ (CNN)}

\addplot[red,dashed,thick,mark=diamond*]  table [y=plamdba_0.1,x=GR]{hyperparameters_mnist_lamda_loss_mclr.dat};
\addlegendentry{$\lambda=0.1$ (MCLR)}

\addplot[green,dashed,thick,mark=triangle*]  table [y=plamdba_0.3,x=GR]{hyperparameters_mnist_lamda_loss_mclr.dat};
\addlegendentry{$\lambda=0.3$ (MCLR)}

\addplot[black,dashed,thick,mark=square*]  table [y=plamdba_0.5,x=GR]{hyperparameters_mnist_lamda_loss_mclr.dat};
\addlegendentry{$\lambda=0.5$ (MCLR)}
\end{axis}

\end{tikzpicture}
}

    \caption{Effect of $\lambda$}
    \label{Fig: mnist_conv_lamda_pm_cnn_loss_main}
\end{subfigure}
\caption{Effect of $\beta$, $\gamma$, and $\lambda$ on convergence of \ourmodel{}(PM) in non-convex(CNN) and strongly convex(MCLR) settings using MNIST dataset}
\label{fig: mnist_conv_beta_gamma_lamda_cnn_mclr_main}
\end{figure*}
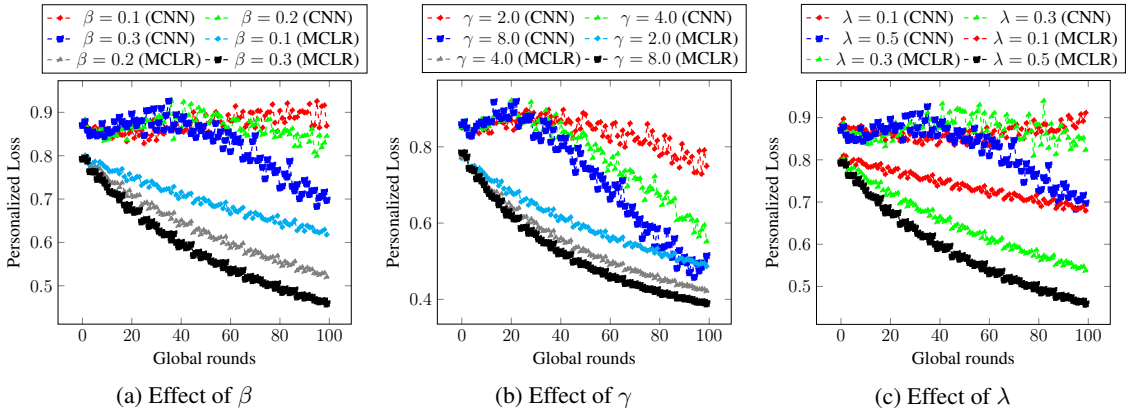

\subsubsection{Ablation studies on team formation:} \label{team_formation}
In \cref{tab: team_formation}, we evaluated \ourmodel{} for both worst-case (team 1 with labels $\{0,1,2,3,4\}$ and team 2 with $\{5,6,7,8,9\}$) and average-case (teams with overlapping labels, team 1 with labels $\{0,1,2,3,4,5,6\}$ and team 2 with $\{ 5,6,7,8,9,0,1\}$)  over 400 global iterations which include 10 and 20, team and local iterations respectively, with the hyperparameter settings $\lambda$ = 0.5, $\gamma$ = 1.5, $\beta$ = 0.6, and, $\alpha$ = 0.01. PerMFL(GM) showed a 4\% improvement in average-case over the worst-case with FMNIST data. Likewise, PerMFL(PM) had slightly better results in average-case than worst-case for non-convex setups (CNN) with MNIST and FMNIST datasets, indicating \ourmodel{}(PM)'s performance is mostly unaffected by team formation.

\begin{table*}[!ht]
\centering
\caption{Performance of \ourmodel{} (Validation accuracy (\%)) on worst-case and average-case team formation}
\label{tab: team_formation}
\begin{adjustbox}{width=0.8\textwidth}
\scriptsize
\begin{tabular}{| c | c | c | c | c | c | c | c |}
    \hline
    \centering
     {\begin{tabular}[c]{@{}c@{}}Team \\  Formation \end{tabular}} &
     Algorithm & \multicolumn{2}{ c |} {MNIST} &  \multicolumn{2}{ c |} {FMNIST} & \multicolumn{2}{ c |} {EMNIST-10} \\ 
   \cline{3-8}
   & & MCLR(\%) & CNN(\%) & MCLR(\%) & CNN(\%) & MCLR(\%) & CNN(\%) \\
   \hline
    {\multirow{2}{*}{{\begin{tabular}[c]{@{}c@{}}Worst \\ case \end{tabular}}}} 
    & \ourmodel{}(PM) & 96.86 & 95.80 & \textbf{97.14} & 95.62 & \textbf{96.57} & 98.13\\
   & \ourmodel{}(GM) & 80.48 & 82.21  & \textbf{76.18} & 70.28 & 88.05 & 87.05 \\ \hline
   
{\multirow{2}{*}{{\begin{tabular}[c]{@{}c@{}}Average \\ case \end{tabular}}}}  
& \ourmodel{}(PM) & \textbf{97.01} & \textbf{97.02} & 96.72 & \textbf{97.38} &  96.39 & \textbf{98.15}\\

&  \ourmodel{}(GM)  &  \textbf{80.86} & \textbf{83.59} &  74.45  & \textbf{74.66} &  \textbf{90.36} & \textbf{87.43}\\ \hline

   \end{tabular}
\end{adjustbox}
\end{table*}

 \subsubsection{Ablation study on teams and clients participation:}
 \ourmodel{} achieves quick convergence with complete participation from both teams and devices (\cref{fig: mnist_ftfc_mclr_pl_main}) or when teams fully participate but devices do so partially (\cref{fig: Mnist_ftpc_pl_5_main}), in contrast to slower convergence under partial participation from both teams and devices (\cref{fig: mnist_ptpc_team_2_mclr_pl_main}). Moreover, expanding the number of teams does not impact the convergence speed of \ourmodel{}(PM) when there is full participation from all teams and devices throughout all global rounds (\cref{fig: mnist_ftfc_mclr_pl_main}). Increased device involvement leads to faster convergence (\cref{fig: Mnist_ftpc_pl_5_main}), whereas lower team engagement (\cref{fig: mnist_ptfc_mclr_pl_main}) decelerates it. Nonetheless, when team participation reaches 50\% in each global round, the convergence rate is comparable to that observed in scenarios with complete participation (\cref{fig: mnist_ftfc_mclr_pl_main}). When all teams are fully participating, increasing the number of team iterations leads to quicker convergence. However, in scenarios where both team and device participation is minimal (2\%) as shown in \cref{fig: mnist_ptpc_team_2_mclr_pl_main}, \ourmodel{}(PM) requires more global and team iterations to achieve convergence. Extended experimental results are reported in the supplementary copy.


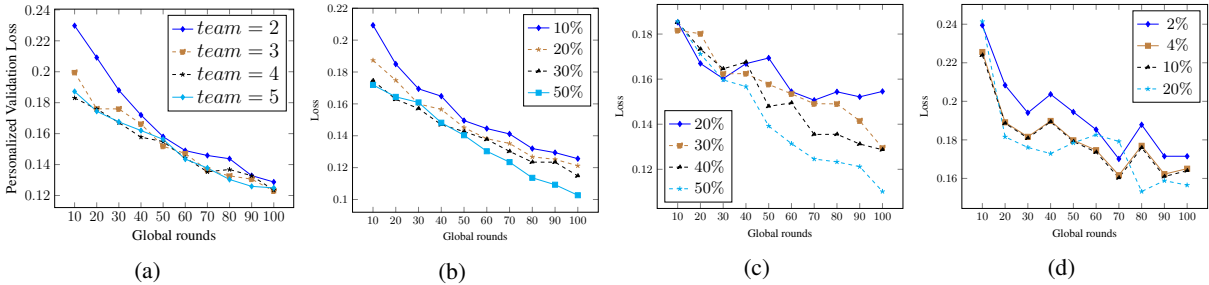
\begin{figure*}[!ht]
        \centering
        \begin{subfigure}[!t]{0.24\linewidth}
            \resizebox{\linewidth}{!}{
                \begin{tikzpicture}
                \tikzstyle{every node}=[font=\fontsize{12}{12}\selectfont]
\begin{axis}[
  xlabel= Global rounds,
  ylabel= Personalized Validation Loss,
  legend style ={nodes={scale=1.4, transform shape}},
  legend pos=north east,
  xtick={0,10,20,30,40,50,60,70,80,90,100}
]
\addplot[blue,thick,mark=diamond*] table [y=pteam_2,x=GR]{ablation_study_team_Mnist_Full_team_full_client_mclr_loss.dat};
\addlegendentry{$team = 2$}

\addplot[brown,dashed,thick,mark=square*]  table [y=pteam_3,x=GR]{ablation_study_team_Mnist_Full_team_full_client_mclr_loss.dat};
\addlegendentry{$ team = 3$}

\addplot[black,dashed,thick,mark=star]  table [y=pteam_4,x=GR]{ablation_study_team_Mnist_Full_team_full_client_mclr_loss.dat};
\addlegendentry{$ team = 4$}

\addplot[cyan,thick,mark=diamond*] table [y=pteam_5,x=GR]{ablation_study_team_Mnist_Full_team_full_client_mclr_loss.dat};
\addlegendentry{$ team = 5$}

\end{axis}
\end{tikzpicture}

            }
        \caption{}
        \label{fig: mnist_ftfc_mclr_pl_main}
    \end{subfigure} 
        \begin{subfigure}[!t]{0.24\linewidth}
            \resizebox{\linewidth}{!}{
                \begin{tikzpicture}
\begin{axis}[
  xlabel= Global rounds,
  ylabel= Loss,
  legend pos=north east,
  legend style ={nodes={scale=1.4, transform shape}},
  xtick={0,10,20,30,40,50,60,70,80,90,100}
]
\addplot[blue,thick,mark=diamond*] table [y=p_10,x=GR]{ablation_study_team_Mnist_Full_team_part_client_team_5_loss.dat};
\addlegendentry{$ 10\%$}

\addplot[brown,dashed,thick,mark=star]  table [y=p_20,x=GR]{ablation_study_team_Mnist_Full_team_part_client_team_5_loss.dat};
\addlegendentry{$ 20\%$}

\addplot[black,dashed,thick,mark=triangle*]  table [y=p_30,x=GR]{ablation_study_team_Mnist_Full_team_part_client_team_5_loss.dat};
\addlegendentry{$ 30\%$}

\addplot[cyan,thick,mark=square*] table [y=p_50,x=GR]{ablation_study_team_Mnist_Full_team_part_client_team_5_loss.dat};
\addlegendentry{$ 50\%$}

\end{axis}
\end{tikzpicture}

            }
    \caption{}
    \label{fig: Mnist_ftpc_pl_5_main}
\end{subfigure} 
\begin{subfigure}[!t]{0.24\linewidth}
    \resizebox{\linewidth}{!}{
    \begin{tikzpicture}
\begin{axis}[
  xlabel= Global rounds,
  ylabel= Loss,
  legend pos=south west,
  legend style ={nodes={scale=1.4, transform shape}},
xtick={0,10,20,30,40,50,60,70,80,90,100}
]
\addplot[blue,thick,mark=diamond*] table [y=pteams_2,x=GR]{ablation_study_team_Mnist_Part_team_full_client_mclr_loss.dat};
\addlegendentry{$ 20\%$}

\addplot[brown,dashed,mark=square*] table [y=pteams_3,x=GR]{ablation_study_team_Mnist_Part_team_full_client_mclr_loss.dat};
\addlegendentry{$ 30\%$}

\addplot[black,dashed,thick,mark=triangle*]  table [y=pteams_4,x=GR]{ablation_study_team_Mnist_Part_team_full_client_mclr_loss.dat};
\addlegendentry{$ 40\%$}

\addplot[cyan,dashed,thick,mark=star]  table [y=pteams_5,x=GR]{ablation_study_team_Mnist_Part_team_full_client_mclr_loss.dat};
\addlegendentry{$ 50\%$}

\end{axis}
\end{tikzpicture}

            }
\caption{}
\label{fig: mnist_ptfc_mclr_pl_main}
\end{subfigure}
\begin{subfigure}[!t]{0.24\linewidth}
    \resizebox{\linewidth}{!}{
    \begin{tikzpicture}
\begin{axis}[
  xlabel= Global rounds,
  ylabel= Loss,
  legend pos=north east,
  legend style ={nodes={scale=1.4, transform shape}},
  xtick={0,10,20,30,40,50,60,70,80,90,100}
]
\addplot[blue,thick,mark=diamond*] table [y=pusers_2,x=GR]{ablation_study_team_Mnist_Part_team_part_client_team_2_mclr_loss.dat};
\addlegendentry{$ 2\%$}

\addplot[brown,thick,mark=square*] table [y=pusers_4,x=GR]{ablation_study_team_Mnist_Part_team_part_client_team_2_mclr_loss.dat};
\addlegendentry{$ 4\%$}

\addplot[black,dashed,thick,mark=triangle*]  table [y=pusers_10,x=GR]{ablation_study_team_Mnist_Part_team_part_client_team_2_mclr_loss.dat};
\addlegendentry{$ 10\%$}

\addplot[cyan,dashed,thick,mark=star]  table [y=pusers_20,x=GR]{ablation_study_team_Mnist_Part_team_part_client_team_2_mclr_loss.dat};
\addlegendentry{$ 20\%$}

\end{axis}
\end{tikzpicture}

}
\caption{}
\label{fig: mnist_ptpc_team_2_mclr_pl_main}
\end{subfigure}

\caption{Ablation study on team and client's participation on MNIST datasets in convex settings (MCLR): \cref{fig: mnist_ftfc_mclr_pl_main} Full teams and devices participation, \cref{fig: Mnist_ftpc_pl_5_main} Full participation of 5 teams but partial participation of devices, \cref{fig: mnist_ptfc_mclr_pl_main} partial participation of teams but full participation of devices, and \cref{fig: mnist_ptpc_team_2_mclr_pl_main} Partial participation of teams (2\%) with partial participation of clients}
\label{Fig: mnist_ftfc_mclr_main}
\end{figure*}

\vspace{-0.2in}
\section{Conclusions}
We introduced, \ourmodel{}, a personalized multi-tier federated learning approach involving global servers, teams, and devices. \ourmodel{} utilized squared euclidean distance regularization on both devices and teams. By employing \ourmodel{}, we are able to generate personalized models for each device and simultaneously obtain a global model. The performance of \ourmodel{} demonstrates linear baseline rates for strongly convex scenarios and sublinear baseline rates for non-convex and smooth scenarios. Empirically we observed that \ourmodel{}(PM) converges quickly and outperformed the state-of-the-art. \ourmodel{} performs best when both teams and devices participate fully. Very low team and device participation degrade the performance of \ourmodel{}(GM). It requires more global and team iterations to converge. Also, in worst-case team formation, the performance of the global model decreases, while the personalized model is able to maintain its performance. Moreover right combination of $\lambda$, $\beta$, and $\gamma$ enhances the convergence of \ourmodel{}.
\section{Acknowledgement}
This work was supported by the Wallenberg AI, Autonomous Systems and Software Program (WASP) funded by the Knut and Alice Wallenberg Foundation. The computations were enabled by the Berzelius resource provided by the Knut and Alice Wallenberg Foundation at the National Supercomputer Centre.

\begingroup
\vspace*{0.5cm}  
\let\clearpage\relax 
\bibliographystyle{unsrt}
\bibliography{bibliography}
\endgroup
\newpage

\newgeometry{top=2cm, bottom=2.5cm, left=2cm, right=2cm}
\appendix
\section{Preliminaries and Supporting Lemmas}

\begin{definition}[Strong convexity]
\label{appendix:Background:def:strongly convex}
A differentiable function $g : \R^d \to \R$ is $\mu_g$-strongly convex if there exists a positive constant $\mu_g$ such that
\begin{align}
\label{appendix:Background:def:strongly convex:eq1}
g(x) \leq g(y) + \ip{\nabla g(x)}{x-y} -\frac{\mu_g}{2} 
\nsq{x-y}, \quad \forall x,y \in \R^d. 
\end{align}
If $g$ is $\mu_g$-strongly convex, then 
\begin{equation}
 \label{appendix: properties strongly-convex eq1}
\mu_g \|x-y\| \leq \|\nabla g(x)-\nabla g(y)\|, \quad \forall x,y \in \R^d.
\end{equation}
\end{definition}
\begin{definition}[Smoothness]
\label{appendix:Background:def:smoothness}
A differentiable function $g : \R^d \to \R$ is $L_g$-smooth if there exists a non-negative constant $L_g$ such that
\begin{equation}
    \label{appendix:Background:def:smoothness:eq1}
    \norm{\nabla g(y)-\nabla g(x)} \leq L_g \norm{x-y}, \quad \forall x,y \in \R^d.
\end{equation} 
If $g$ is $L_g$-smooth, then 
\begin{equation}
    \label{appendix:Background:def:smoothness:eq2}
    |g(y)-g(x) -\ip{\nabla g(x)}{y-x}| \leq \frac{L_g}{2} \nsq{y-x}, \quad \forall x,y \in \R^d.
\end{equation}
\end{definition}
\begin{definition}[Expected Smoothness]
\label{appendix:Background:def:expected smoothness}
The second moment of the stochastic gradient $\tilde{\nabla} g(x) $ sattisfies
\begin{align}
\label{appendix:Background:def:expected smoothness:eq1}
\mathbb{E} \nsq{\tilde{\nabla} g(x)} \leq 2 A \big( g(x)-g(x^*)\big) +B \cdot  \nsq{\nabla g(x)} +C,
\end{align}
for some $A, B, C \geq 0$  and $\forall x  \in \mathbb{R}^d$.
\end{definition}
\begin{definition}[Moreau Envelope]
\label{appendix:Background:def:Moreau Envelope}
For a function $g$, its Moreau envelope $\tilde{g}$ is defined as
\begin{equation}
    \tilde{g}(x) = \min_{u  \in \mathbb{R}^d} \left( g(u) + \frac{\lambda}{2} \| u- x\|^2 \right).
\end{equation}
\end{definition}
\begin{proposition}\label{appendix: Moreau Envelope properties convex}
Let $g$ be a convex function and $\tilde{g}$ be its Moreau Envelope with parameter $\lambda$. Then
\begin{enumerate}[itemsep=0.5em]
\item $\tilde{g}$ is convex.
\item $\tilde{g}$ is continuously differentiable (even if f is not) and
\begin{align}\label{appendix:eq: grad of Moreau Envelope}
\nabla \tilde{g}(x) &= \lambda \big( x - \mathrm{prox}_{g/\lambda}(x) \big), 
\quad \text{where} \quad \mathrm{prox}_{g/\lambda}(x):=\argmin_{u \in \mathbb{R}^d} ~\{g(u)+\frac{\lambda}{2} \|u-x\|^2 \}.
\end{align}
Moreover $\tilde{g}$ is $\lambda$-smooth.
\item If $g$ is $\mu$-strongly convex, then $\tilde{g}$ is $\mu_{\tilde{g}}$-strongly convex with  $\mu_{\tilde{g}}=\frac{\mu \lambda}{\mu+\lambda}$.
\end{enumerate}
\end{proposition} 
\begin{proposition}\label{appendix: Moreau Envelope properties non-convex}
If $g$ is non-convex and $L_g$-smooth, then 
$\nabla \tilde{g}(x)$ is $\lambda$-smooth with the condition that $\lambda >2 L_g $.
\end{proposition} 
\begin{proposition}\label{appendix: properties L-smooth}
If $g:\R^d \to \R$ is convex and $L_g$-smooth, then 
\begin{equation*}
\|\nabla g(x)-\nabla g(y)\|^2 \leq 2L_g 
\Big( g(x)-g(y)-\langle \nabla g(y) , x-y\rangle
\Big), \quad \forall x,y \in \R^d.
\end{equation*}
\end{proposition}
\begin{proposition}\label{appendix: properties strongly-convex}
If $g:\R^d \to \R$ is  $\mu_g$-strongly convex, then 
\begin{equation}
\label{appendix: properties strongly-convex eq2}
\frac{\mu_g}{2} \nsq{x-x^*} \leq g(x)-g(x^*), \quad \forall x \in \R^d.
\end{equation}
\end{proposition}
\begin{remark}\label{appendix: general properties}
For any set of vectors $\{y_i\}_{i=1}^n$, we have
    \begin{align}
        \frac{1}{n}\sum_{i=1}^n~ \| y_i\|^2
        = \frac{1}{n}\sum_{i=1}^n~ \| y_i-\Bar{y}\|^2+ \|\Bar{y}\|^2 
    \end{align}
that leads to 
\begin{align}
        \|\sum_{i=1}^n~ y_i\|^2 \leq n\cdot \sum_{i=1}^n~ \| y_i\|^2.
\end{align}
where $\Bar{y}:= \frac{1}{n}\sum_{i=1}^n y_i$.
\end{remark}
\begin{remark}[Young's inequality]
\label{appendix: Young's inequality}
For any $\epsilon >0$ and vectors $y,z \in \mathbb{R}^d$, we have 
\begin{equation*}
2|\langle y, z \rangle| \leq \frac{2}{\epsilon} \|y\|^2    +\frac{\epsilon}{2} \|z\|^2
\end{equation*}  
\end{remark}
\begin{remark}[Parameters in \ourmodel{}]
\label{appendix: Parameters in pFedMT}
Here, we consider two setups.
\begin{itemize}
    \item If $\fij$ is $\mu_{f}$-strongly convex, then
\begin{itemize}
    \item $\ftildeij$ is $\mu_{\tilde{f}}$-strongly convex where     $\mu_{\tilde{f}}=\frac{\lambda\mu_{f}}{\lambda+\mu_{f}}$
    \item  $\ftildeij$ is $\lambda$-smooth.
    \item $F_i$ is $\mu_F$-strongly convex where $\mu_{F}=\mu_{\tilde{f}}=\frac{\lambda\mu_{f}}{\lambda+\mu_{f}}$.
    \item $\tilde{F}_i$ is $\mu_{\tilde{F}}$-strongly convex where 
    \begin{equation}
    \mu_{\tilde{F}}=\frac{\mu_F \gamma}{\mu_F+\gamma}=\frac{\frac{\lambda\mu_{f}}{\lambda+\mu_{f}} \gamma}{\frac{\lambda\mu_{f}}{\lambda+\mu_{f}}+\gamma}=\frac{\lambda \gamma \mu_f}{\lambda \mu_f+ \gamma \mu_f+\lambda \gamma }.   
    \end{equation}
    \item $\phi$ is $\mu_\phi$-strongly convex where $\mu_\phi= \mu_{\tilde{F}}=\frac{\lambda \gamma \mu_f}{\lambda \mu_f+ \gamma \mu_f+\lambda \gamma }$
\end{itemize}

\item If $\fij$ is $L_f$-smooth then
\begin{itemize}
    \item $\ftildeij$ is $\Lftilde$-smooth with the condition that $\Lftilde=\lambda >2 L_f$.
    \item $F_i$ is $L_F$-smooth with the condition that $L_F=\Lftilde=\lambda >2 L_f$.
    \item $\tilde{F}_i$ is $\LFtilde$-smooth with the condition that $\LFtilde=\gamma > 2 \lambda > 4 L_f $.
    \item $\phi$ is $L_\phi$-smooth with the condition that $L_\phi=\LFtilde=\gamma > 2 \lambda > 4 \Lftilde $.
\end{itemize}
\end{itemize}
\end{remark}
\begin{definition}[Inexact Gradient Update]
\label{appendix: Defenition of gt and git}
Consider an equivalent algorithm to \Cref{alg:pfedmt_algo} as \Cref{appendix: analysis: pFedMT algorithm}.
Note that, We can re-write the update rule in \Cref{alg:pfedmt_algo} as
\begin{equation}\label{appendix:eq: pFedMT analysis with approx prox smooth}
    \xtt =\avgM \xit
    =\avgM (\xt -\beta \git)
    =\xt -\beta \avgM \git
    =\xt-\beta \gt
\end{equation}
where $\gt:=\avgM \git$.

\begin{algorithm}
\setstretch{1.2}
\caption{}
\begin{algorithmic}
\label{appendix: analysis: pFedMT algorithm}
\STATE \textbf{set} $x^1, \lambda, \eta, \beta$
\FOR {round $t = 1,2,\ldots,T$}
        \vspace{0.25em}
        \STATE --- \textbf{Team}-level training ----------------------------------- 
        \FOR {team $i = 1,2,\ldots,M$}
        \STATE $\ait=\textbf{approx}  \Big\{\argmin_{u \in \mathbb{R}^d} ~\{F_i(u)+\frac{\gamma}{2} \|u-\xt\|^2 \} \Big\} $
        \STATE $\git=\gamma(\xt- \ait)$ 
        \STATE $\xit=\xt-\beta \git$
        \STATE Client communicates $\xit$ to the server. 
        \ENDFOR
        \vspace{0.25em}
        \STATE --- \textbf{Server}-level aggregation ------------------------------------
        \STATE $\xtt= \avgM \xit $
        \vspace{0.15em}
        \STATE Server communicates $\xtt$ to the clients.
        \vspace{0.35em}
\ENDFOR
\end{algorithmic}
\end{algorithm}
 
\end{definition}

\section{ Convergence of \ourmodel{} (strongly convex case) \label{appendix:sec Team-level convergence} }
Let us consider the team-level update as in \Cref{appendix: analysis: pFedMT algorithm}. We start with analyzing the convergence of this algorithm when  having access to the approximate solution to  $\mathrm{prox}_{F_i/\gamma}(.)$ in \Cref{appendix:Th: Team-level bound: inexact grad strongly convex}. Secondly we upper bound the approximation error in \Cref{appendix:Th: K bound: inexact grad: strongly convex}. We analyze the client level convergence in \Cref{appendix:L:Th: client level strongly convex:error bound}. 
 Finally, we combine these results in \Cref{Discussion on error bounds (strongly convex case)} and complete the proof of \Cref{PfedMT:theorem: analysis of the strongly convex case}.


\begin{theorem}\label{appendix:Th: Team-level bound: inexact grad strongly convex}
Let $\fij$ be $\mu_{f}$-strongly convex and assume that $\ait$ is an approximate solution to $\mathrm{prox}_{F_i/\gamma}(x^t)$ such that  
$\|\ait-\mathrm{prox}_{F_i/\gamma}(x^t)\|^2 \leq \nu_i$.
Suppose that the sequence $\{x^t\}$ is generated by  \ourmodel{} with $ \beta \leq  \frac{\muFtilde}{4\LFtilde}$. Then, we have  
\begin{equation}
\label{appendix:Convergence of PfedMT:Th:strongly convex:eq1}
 \|\xT-x^*\|^2 \leq (1-\beta)^T \|x^0- x^*\|^2
+ \frac{12\gamma^2 \nu}{\muFtilde^2},
\end{equation}
where $\muFtilde=\frac{\lambda \gamma \mu_f}{\lambda \mu_f+ \gamma \mu_f+\lambda \gamma }$, $\LFtilde=\gamma$, and $\nu:=\avgM \nu_i$.
\end{theorem}

\begin{proof}
Using \eqref{appendix:eq: pFedMT analysis with approx prox smooth}, we have
 \begin{equation}\label{appendix:eq: pFedMT analysis approx prox s-convex one step update}
 \|\xtt-x^*\|^2 = \|   \xt- \beta  \gt - x^*\|^2 
 = \|\xt- x^*\|^2 -2\beta \langle  \gt, \xt- x^* \rangle +\beta^2 \|\gt\|^2.
 \end{equation}
Now, we bound $J_1:= -\langle  \gt, \xt- x^* \rangle$ and $J_2:=  \|\gt\|^2$ seperately. 

    \begin{itemize}
    \item {Upper bound for $J_1$:}
        \begin{equation}
    \label{appendix:eq: pFedMT analysis J_1}
    J_1 = -\langle  \gt-\nabla \phi(\xt)+\nabla \phi(\xt), \xt- x^* \rangle
    = \underbrace{ -\langle  \gt-  \nabla \phi(\xt), \xt- x^* \rangle}_{=:J_3} 
    \underbrace{-\langle \nabla \phi(\xt), \xt- x^* \rangle }_{=:J_4}.
    \end{equation}
    We upper bound $J_3$ and $J_4$.
    \begin{itemize}
        \item {Upper bound for $J_3$}: Using Young’s inequality, see \Cref{appendix: Young's inequality}, we have
        \begin{align}
           2 J_3 \leq \frac{2}{\epsilon} \| \gt-  \nabla \phi(\xt)\|^2
           +
           \frac{\epsilon}{2} \|\xt- x^*\|^2, \qquad \forall \epsilon>0. 
        \end{align}
        \item {Upper bound for $J_4$}: Using the  fact that  $\Ftildei$ is $\muFtilde$-strongly convex, we have 
        \begin{equation}
        J_4 =-\langle \nabla \phi(\xt), \xt- x^* \rangle
        \leq \phi(x^*)-\phi(\xt) - \frac{\muFtilde}{2} 
        \|\xt-x^*\|^2.
        \end{equation}
    \end{itemize}
    
Substituting upper bounds on $J_3$ and $J_4$ into \eqref{appendix:eq: pFedMT analysis J_1}, we have 
\begin{align}\label{appendix:eq: pFedMT analysis J_1 part2}
    J_1=J_3+J_4 
    &\leq 
    \frac{1}{\epsilon} \| \gt-  \nabla \phi(\xt)\|^2
           +
           \frac{\epsilon}{4} \|\xt- x^*\|^2 +\phi(x^*)-\phi(\xt) - \frac{\muFtilde}{2} 
        \|\xt-x^*\|^2
    \nonumber \\
    &=\frac{1}{\epsilon} \| \gt-  \nabla \phi(\xt)\|^2
    -\Big(\frac{\muFtilde}{2}- \frac{\epsilon}{4}\Big) \|\xt- x^*\|^2+
    \phi(x^*)-\phi(\xt) 
\end{align} 

\item {Upper bound for $J_2$}:
\begin{align}
   J_2= \|\gt\|^2
    &=\|  \gt  -\nabla \phi(\xt) +\nabla \phi(\xt)\|^2
    \nonumber \\
    &\leq 2~\| \gt  -\nabla \phi(\xt) \big)\|^2 + 2 ~ \|\nabla \phi(\xt)
    \|^2 \qquad (\text{see \Cref{appendix: general properties}})
   \nonumber \\
    &\leq 2~\| \gt  -\nabla \phi(\xt) \big)\|^2 + 4\LFtilde (\phi(\xt)-\phi(x^*)),
\end{align}
where in the last line, we used the fact that the average of $\LFtilde$-smooth functions $\Ftildei(.)$ is a $\LFtilde$-smooth function $\phi(x)=\avgM \Ftildei(x)$, therefore, we can apply \Cref{appendix: properties L-smooth} with $x=\xt, x'=x^*$, and $\nabla \phi(x^*)=0$.
\end{itemize}

Substituting $J_1$ and $J_2$ into \eqref{appendix:eq: pFedMT analysis approx prox s-convex one step update}, we can write
\begin{align} \label{appendix:eq: pFedMT analysis approx prox s-convex one step update part2}
 \|\xtt-x^*\|^2 &\leq
\|\xt- x^*\|^2 +2\beta  J_1 +\beta^2 J_2
\nonumber\\ 
&\leq 
\|\xt- x^*\|^2 +2\beta \bigg( \frac{1}{\epsilon} \Big(  \| \gt-  \nabla \phi\|^2 \Big) 
    -\Big(\frac{\muFtilde}{2}- \frac{\epsilon}{4}\Big) \|\xt- x^*\|^2+
    \phi(x^*)-\phi(\xt)  \bigg) 
    \nonumber\\
    &+\beta^2 \bigg(  2  \|  \gt  -\nabla \phi(\xt) \|^2 + 4\LFtilde (\phi(\xt)-\phi(x^*)) \bigg)
    \nonumber\\
    &=\Big(1-\beta(\muFtilde- \frac{\epsilon}{2})\Big)  \|\xt- x^*\|^2
    -2\beta\big(1-2\beta \LFtilde \big) \Big( \phi(\xt)-\phi(x^*) \Big)
    \nonumber\\
    &
    +2\beta(\beta+\frac{1}{\epsilon})
    \Big( \|  \gt  -\nabla \phi(\xt) \|^2 \Big)
    \nonumber\\
    &\leq (1-\frac{\beta \muFtilde}{2}) \|\xt- x^*\|^2
    +2\beta(\beta+\frac{1}{\muFtilde})
    \Big(  \|  \gt  -\nabla \phi(\xt) \|^2 \Big)
    \nonumber\\
    &\leq (1-\beta) \|\xt- x^*\|^2
    +\frac{8}{\muFtilde^2}\beta(\beta+\frac{1}{2})
    \underbrace{
    \Big(  \|  \gt  -\nabla \phi(\xt) \|^2 \Big)}_{J_5}
\end{align}
where in the last line we re-defined $\beta$ with the new $\beta \leq \frac{\muFtilde}{2}  \min \big\{ \frac{1}{2\LFtilde}, \frac{2}{\muFtilde} \big \}=\min \big\{ \frac{\muFtilde}{4\LFtilde} ,1 \big \} $ and chose $\epsilon=\muFtilde$. Also note that  $\muFtilde \leq \LFtilde$ results in $\beta \leq \frac{\muFtilde}{4\LFtilde}$. 
Let us now bound $J_5$ using the definitions of $\git$ and the second part of \Cref{appendix: Moreau Envelope properties convex}.
\begin{align}
J_5= \|  \gt  -\nabla \phi(\xt) \|^2 &= \avgM \|  \git  -\nabla \Ftildei(\xt) \|^2
\nonumber \\    
&= \avgM \| \gamma(\xt- \ait)-\gamma(\xt- \mathrm{prox}_{F_i/\gamma}(x^t) )    \|^2
\nonumber \\    
&=\gamma^2 \avgM \|\ait-\mathrm{prox}_{F_i/\gamma}(x^t)  \|^2
\nonumber \\    
& \leq  \gamma^2 \avgM \nu_i
\quad (\text{assumption})
\nonumber \\
&=:\gamma^2 \nu
\end{align}
where $\nu:=\avgM \nu_i$.
Substituting $J_5$ into \eqref{appendix:eq: pFedMT analysis approx prox s-convex one step update part2}, we get 
\begin{equation}\label{appendix:Convergence of PfedMT:Th:strongly convex:eq1 in proof}
 \|\xtt-x^*\|^2 \leq (1-\beta) \|\xt- x^*\|^2
    +\frac{8}{\muFtilde^2}\beta(\beta+\frac{1}{2})
\gamma^2 \nu.
\end{equation}
Substituting $t=T-1$ and unrolling \eqref{appendix:Convergence of PfedMT:Th:strongly convex:eq1 in proof}, we get
\begin{align}
\label{appendix:Convergence of PfedMT:Th:strongly convex:eq2 in proof}
 \|\xT-x^*\|^2 &\leq (1-\beta)^T \|x^0- x^*\|^2
    +\frac{8}{\muFtilde^2}\beta(\beta+\frac{1}{2})
\gamma^2 \nu \Big(\frac{1- \big(1- \frac{\beta \muFtilde}{2}\big)^T}{1-(1-\frac{\beta \muFtilde}{2})} \Big)
\nonumber \\    
&\leq (1-\beta)^T \|x^0- x^*\|^2
+\frac{8}{\muFtilde^2}\beta (\beta+\frac{1}{2}) \gamma^2 \nu 
 \Big(\frac{1}{\frac{\beta \muFtilde}{2}} \Big)
\nonumber \\    
&\leq (1-\beta)^T \|x^0- x^*\|^2
+(\beta+\frac{1}{2})\frac{16}{\muFtilde^3}  \gamma^2 \nu.
\end{align}
We can further simplify \eqref{appendix:Convergence of PfedMT:Th:strongly convex:eq2 in proof} using $\beta \leq \frac{\muFtilde}{4\LFtilde} \leq \frac{1}{2}$
\begin{align}
\label{appendix:Convergence of PfedMT:Th:strongly convex:eq3 in proof}
 \|\xT-x^*\|^2 \leq (1-\beta)^T \|x^0- x^*\|^2
+ \frac{16}{\muFtilde^3}  \gamma^2 \nu.
\end{align}
That completes the proof.
\end{proof} 

\subsection{Team-level convergence (strongly convex case)}
\begin{theorem}
\label{appendix:Th: K bound: inexact grad: strongly convex}
Consider Problem \ref{eqn:moreau-team-subproblem} with   Gradient descent steps \eqref{eqn:prox_Fi_GD}. Let us assume that we have approximate solution to $prox_{\fij/\gamma}(.)$ so that we have $\nsq{ \prox_{\fij/\lambda}(\witk) -\thetaijtkL} \leq \deltaij$, then 

\begin{equation}
\label{appendix:Team-level:Th:strongly convex:eq1}
\nsq{\witK-\prox_{F_i/\gamma}(\xt)}  \leq \big(1-\eta_i\frac{\mu_F+\gamma}{2}\big)^K \frac{2L_{\tilde{F}}}{\gamma^2} \big( \Ftildei(w_i^{t,0}) -\Ftildei(x^*)\big)+12\big(\frac{\lambda}{\mu_F+\gamma}\big)^2 \avgNi \deltaij
    :=\nu_i
\end{equation}
where $\eta_i\leq \frac{1}{2(L_F+\gamma)}$.
\end{theorem}

\begin{proof}
Let us re-write \eqref{eqn:prox_Fi_GD} as 
$\witkk= \witk -\eta_i \qitk$, where $\qitk$ is the inexact gradient of $p_i(y):=F_i(y)+\frac{\gamma}{2} \nsq{y- \xt}$ at $y=\witk$.
\begin{align}\label{appendix:eq:client level}
    \nsq{\witkk-\witstar} 
    & = \nsq{w_i^{t,k} - \eta_i \qitk-\witstar} 
    \nonumber \\
    & =\nsq{\witk-\witstar}
    \underbrace{    -2\eta_i\ip{\qitk}{w_i^{t,k}-\witstar}
    }_{J_1}+
    \eta_i^2
    \underbrace{\nsq{\qitk}
    }_{J_2},
\end{align}
where $\witstar$ is the minimizer of $p_i(.)$.  
Now, we upper bound $J_1$ and $J_2$.
\begin{itemize}
    \item {\bf Bound on $J_1$:}
    \begin{align}\label{appendix:eq:client level dummy1}
    J_1=2\eta_i\Big( 
    \underbrace{
    -\ip{\qitk-\nabla p_i(\witk)}{w_i^{t,k}-\witstar} 
    }_{J_3}
    \underbrace{-
    \ip{\nabla p_i(\witk)}{w_i^{t,k}-\witstar}
    }_{J_4}\Big)
    \end{align}
    \begin{itemize}
        \item {\bf Bound on $J_3$:} Note that $\qitk=(\lambda+\gamma)\witk-\gamma \xt- \lambda \thetabaritk$ and $\nabla p_i(\witk)=(\lambda+\gamma)\witk-\gamma \xt- \frac{\lambda}{N_i} \sum_{j=1}^{N_i}\mathrm{prox}_{f_{i,j}/\lambda}(w_i^{t,k})$. We have 
        \begin{align}
            J_3&=-\lambda\ip{ \avgNi \prox_{\fij/\lambda}(\witk) -\thetabaritk}{w_i^{t,k}-\witstar}
            \nonumber  \\
            &=-\lambda\avgNi \ip{\prox_{\fij/\lambda}(\witk)-\thetaijtkL}{w_i^{t,k}-\witstar}
        \quad (\thetabaritk:=\avgNi \thetaijtkL).
        \end{align}
Using Young’s inequality with $\epsilon=(\mu_F+\gamma)/\lambda$, see Remark \ref{appendix: Young's inequality}, we get
          \begin{align}
            2J_3&\leq \lambda\avgNi\Big(
            \frac{2\lambda}{\mu_F+\gamma} \nsq{\prox_{\fij/\lambda}(\witk)-\thetaijtkL}  
            + 
            \frac{\mu_F+\gamma}{2\lambda} \nsq{\witk-\witstar}
            \Big)
            \nonumber \\
            &\leq  \frac{2\lambda^2}{\mu_F+\gamma} \avgNi \deltaij+
            \frac{\mu_F+\gamma}{2} \nsq{\witk-\witstar}.
         \end{align}        
        \item {\bf Bound on $J_4$:} Note that $p_i(.)$ is $(\mu_F+\gamma)$-strongly convex, therefore we have
            \begin{equation}
             J_4\leq p_i(\witstar)-p_i(\witk)-
            \frac{\mu_F+\gamma}{2}\nsq{w_i^{t,k}-\witstar}.
        \end{equation}
    \end{itemize}
    Substituting upper bounds on $J_3$ and $J_4$ into \eqref{appendix:eq:client level dummy1}, we have
    \begin{align}\label{appendix:eq:client level dummy2}
    J_1\leq 2\eta_i \Big( \frac{\lambda^2}{\mu_F+\gamma} \avgNi\deltaij
    +p_i(\witstar)-p_i(\witk)-
            \frac{\mu_F+\gamma}{4}\nsq{w_i^{t,k}-\witstar}
    \Big).
    \end{align}    
\item {\bf Upper bound for $J_2$}
\begin{align}
   J_2= \|\qitk\|^2 &=\|  \qitk  -\nabla p_i(\witk) +\nabla p_i(\witk)
    \|^2
    \nonumber \\
    &\leq 2 \big( \nsq{\qitk  -\nabla p_i(\witk)} + \nsq{\nabla p_i(\witk)}
    \big)  \quad (\text{see Remark \ref{appendix: general properties}})
    \nonumber \\
    &= 
    2 \big( \lambda^2 \nsq{\avgNi \prox_{\fij/\lambda}(\witk) -\thetaijtkL} + \nsq{\nabla p_i(\witk)}
    \big)
        \nonumber \\
    &\leq 
    2 \lambda^2 \avgNi \nsq{ \prox_{\fij/\lambda}(\witk) -\thetaijtkL} +2 \nsq{\nabla p_i(\witk)}
    \quad (\text{see Remark \ref{appendix: general properties}})
    \nonumber \\
    &\leq  2\lambda^2 \avgNi \deltaij+4(L_F+\gamma)\big(p_i(\witk)-p_i(\witstar) \big)
    \quad (\text{$p_i(.)$ is $(L_F+\gamma)$-Smooth})
\end{align}
\end{itemize}
Substituting $J_1$ and $J_2$ into \eqref{appendix:eq:client level}, we can write
\begin{align}\label{appendix:eq:client level dummy3}
    \nsq{\witkk-\witstar} 
    &\leq \nsq{\witk-\witstar}+2\eta_i \Big( \frac{\lambda^2}{\mu_F+\gamma} \avgNi\deltaij
    +p_i(\witstar)-p_i(\witk)-
            \frac{\mu_F+\gamma}{4}\nsq{w_i^{t,k}-\witstar}
    \Big)
    \nonumber \\
    &+2\eta_i^2\lambda^2 \avgNi \deltaij+4\eta_i^2(\Lftilde+\gamma)\big(p_i(\witk)-p_i(\witstar) \big)
    \nonumber \\
    &=(1-\eta_i\frac{\mu_F+\gamma}{2}) \nsq{\witk-\witstar}
    -2\eta_i\big(1-2\eta_i (L_F+\gamma)\big)\Big( p_i(\witk)-p_i(\witstar)\Big)
    \nonumber \\
    &+2\eta_i\lambda^2\big(\eta_i+ \frac{1}{\mu_F+\gamma} \big) \avgNi \deltaij
    \nonumber \\
    &\leq (1-\eta_i\frac{\mu_F+\gamma}{2}) \nsq{\witk-\witstar}+2\eta_i\lambda^2\big(\eta_i+ \frac{1}{\mu_F+\gamma} \big) \avgNi \deltaij
    \quad(\eta_i\leq \frac{1}{2(L_F+\gamma)})
    \nonumber \\
    &\leq (1-\eta_i\frac{\mu_F+\gamma}{2}) \nsq{\witk-\witstar}+\big(\frac{6\eta_i\lambda^2}{\mu_F+\gamma} \big) \avgNi \deltaij
    \quad(\eta_i\leq \frac{2}{(\mu_F+\gamma)}).
\end{align} 
Substituting $k=K-1$ and unrolling \eqref{appendix:eq:client level dummy3}, we get
\begin{align}
\label{appendix:eq:client level dummy4}
    \nsq{\witK-\witstar} 
    &\leq \big(1-\eta_i\frac{\mu_F+\gamma}{2}\big)^K \nsq{w_{i}^{t,0}-\witstar}  + \big(\frac{6\eta_i\lambda^2}{\mu_F+\gamma} \big) \avgNi \deltaij
 \Big(\frac{ 1-\big(1-\eta_i\frac{\mu_F+\gamma}{2}\big)^K}{1-(1-\eta_i\frac{\mu_F+\gamma}{2})}\Big) 
    \nonumber \\
    &\leq\big(1-\eta_i\frac{\mu_F+\gamma}{2}\big)^K \nsq{w_{i}^{t,0}-\witstar}+12\big(\frac{\lambda}{\mu_F+\gamma}\big)^2 \avgNi \deltaij.
\end{align}
Substituting $\witstar=\prox_{F_i/\gamma}(\xt)$ and $w_{i}^{t,0}=\xt$ into the equation above, we can write 
\begin{align}
\label{appendix:eq:client level dummy5}
    \nsq{\witK-\prox_{F_i/\gamma}(\xt)} 
    &\leq\big(1-\eta_i\frac{\mu_F+\gamma}{2}\big)^K \nsq{\xt-\prox_{F_i/\gamma}(\xt)}+12\big(\frac{\lambda}{\mu_F+\gamma}\big)^2 \avgNi \deltaij
    \nonumber \\
    &=\big(1-\eta_i\frac{\mu_F+\gamma}{2}\big)^K \frac{1}{\gamma^2}\nsq{\nabla \Ftildei(\xt)}
    +12\big(\frac{\lambda}{\mu_F+\gamma}\big)^2 \avgNi \deltaij
    \nonumber \\
    &\leq
    \big(1-\eta_i\frac{\mu_F+\gamma}{2}\big)^K \frac{2L_{\tilde{F}}}{\gamma^2} \big( \Ftildei(\xt) -\Ftildei(x^*)\big)
    +12\big(\frac{\lambda}{\mu_F+\gamma}\big)^2 \avgNi \deltaij
    \nonumber \\
    &=\big(1-\eta_i\frac{\mu_F+\gamma}{2}\big)^K \frac{2L_{\tilde{F}}}{\gamma^2} \big( \Ftildei(w_i^{t,0}) -\Ftildei(x^*)\big)+12\big(\frac{\lambda}{\mu_F+\gamma}\big)^2 \avgNi \deltaij
    \nonumber \\
    &:=\nu_i
\end{align}
where in the second line, we used   
$\nabla \Ftildei(\xt)=\gamma(\xt-\prox_{F_i/\gamma}(\xt))$, and in the third line we used the  smoothness of $\Ftildei(.)$. Note also that the last two lines of \eqref{appendix:eq:client level dummy3}  require $\eta_i\leq \min \{ \frac{1}{2(L_F+\gamma)},\frac{2}{(\mu_F+\gamma)}\}$  that can be reduced to $\eta_i\leq \frac{1}{2(L_F+\gamma)}$ because we always have $\mu_F \leq L_F$.
\end{proof}

\subsection{Client-level convergence (strongly convex case)}

\begin{theorem}
\label{appendix:L:Th: client level strongly convex:error bound}
Let $\fij$ be $\mu_f$-strongly convex and $L_f$-smooth functions. Consider Problem \ref{eqn:moreau-device-subproblem} with Gradient descent steps \eqref{eqn:device_p_GD} 
, we have
\begin{equation}
\label{appendix:L:Th:strongly convex:error bound:  eq1}
\nsq{\thetaijtkL- \prox_{\fij/\lambda}(\witk)} \leq (1- \alpha \mu_f)^L \frac{1}{\lambda^2} \nsq{\nabla \ftildeij(\theta_{i,j}^{t,k,0})} =: \deltaij,  
\end{equation}
where $ \alpha \leq  \frac{1}{\Lflambda}$.

\end{theorem}

\begin{proof}
Let us define $\pij(y):= \fij(y)+\frac{\lambda}{2} \nsq{y- \witk}$. Note that $p_{i,j}(.)$ is $(\mu_f+\lambda)$-strongly convex and $(L_f+\lambda)$-smooth. We use \eqref{eqn:device_p_GD} , and write
\begin{align}
\nsq{\thetaijtkll-\thetaijtkstar}
&=\nsq{\thetaijtkl-\alpha \nabla \pij(\thetaijtkl) -\thetaijtkstar}
\nonumber \\
&=\nsq{\thetaijtkl-\thetaijtkstar} 
-2\alpha \ip{\nabla \pij(\thetaijtkl)}{\thetaijtkl-\thetaijtkstar}
+\alpha^2 \nsq{\nabla \pij(\thetaijtkl)}
\nonumber \\
&\leq \nsq{\thetaijtkl-\thetaijtkstar} +2\alpha \Big( \pij(\thetaijtkstar)-\pij(\thetaijtkl)-\frac{\muflambda}{2}\nsq{\thetaijtkl-\thetaijtkstar}  \Big)
+\alpha^2 \nsq{\nabla \pij(\thetaijtkl)}
\nonumber \\
&\leq (1- \alpha (\muflambda)) \nsq{\thetaijtkl-\thetaijtkstar} - 2\alpha \big(
\pij(\thetaijtkl)-
\pij(\thetaijtkstar)
 \big)+2\alpha^2 (\Lflambda) \big(
\pij(\thetaijtkl)-
\pij(\thetaijtkstar)
\nonumber \\
&\leq (1- \alpha (\muflambda)) \nsq{\thetaijtkl-\thetaijtkstar}
-2\alpha(1-\alpha (\Lflambda))\big( \pij(\thetaijtkl)-
\pij(\thetaijtkstar) \big) 
\nonumber \\
&\leq (1- \alpha (\muflambda)) \nsq{\thetaijtkl-\thetaijtkstar}
\quad ( \alpha \leq \frac{1}{\Lflambda})
\end{align}
In the third and fourth lines, we used strong convexity of $p_{i,j}$ and \Cref{appendix: properties L-smooth}. Substituting $l=L-1$,  
$\theta_{i,j}^{t,k,0}=\witk$, and $\thetaijtkstar=\prox_{\fij/\lambda}(\witk) $, we get
\begin{align}
\nsq{\thetaijtkL-\prox_{\fij/\lambda}(\witk) }
&\leq (1- \alpha (\muflambda))^L \nsq{\witk-\prox_{\fij/\lambda}(\witk)}
\nonumber \\
&=(1- \alpha (\muflambda))^L \frac{1}{\lambda^2} \nsq{\nabla \ftildeij(\witk)}
\nonumber \\
&=(1- \alpha (\muflambda))^L \frac{1}{\lambda^2} \nsq{\nabla \ftildeij(\theta_{i,j}^{t,k,0})}
\quad (\theta_{i,j}^{t,k,0}=\witk)
\nonumber \\
&=:\deltaij.
\end{align}
Note that we require $ \alpha \leq \min\{ \frac{1}{\Lflambda}, \frac{1}{\muflambda} \}$ that can be reduced to $ \alpha \leq \frac{1}{\Lflambda}$ since we always have $\mu_f \leq L_f$.
That completes the proof.
\end{proof}

\subsection{ Discussion on error bounds (strongly convex case)}
\label{Discussion on error bounds (strongly convex case)}
In this section we discuss the relation between $L, K$, and $T$ in the strongly convex setup. 

\subsubsection{ How to choose \texorpdfstring{$L$}{AL}}
Let us start with \eqref{appendix:Team-level:Th:strongly convex:eq1}
\begin{align}
    \nsq{\witK-\prox_{F_i/\gamma}(\xt)}  &\leq \big(1-\eta_i\frac{\mu_F+\gamma}{2}\big)^K \frac{2L_{\tilde{F}}}{\gamma^2} \big( \Ftildei(w_i^{t,0}) -\Ftildei(x^*)\big)+12\big(\frac{\lambda}{\mu_F+\gamma}\big)^2 \avgNi \deltaij
\nonumber \\
&\leq \big(1-\eta_i\frac{\mu_F+\gamma}{2}\big)^K \frac{\Gamma^{(1)}_i}{2}+\Gamma^{(2)} \delta_i 
\end{align}
where $\Gamma^{(1)}_i:=4\frac{L_{\tilde{F}}}{\gamma^2} \big( \Ftildei(w_i^{t,0}) -\Ftildei(x^*)\big)$ and $\Gamma^{(2)}:=12\big(\frac{\lambda}{\mu_F+\gamma}\big)^2$.
 We require $\delta_i=\big(1-\eta_i\frac{\mu_F+\gamma}{2}\big)^K  \frac{\Gamma^{(1)}_i}{2\Gamma^{(2)}}$, then
\begin{align}\label{appendix:Discussion on error bounds: dummy 2}
    \nsq{\witK-\prox_{F_i/\gamma}(\xt)} 
    \leq \big(1-\eta_i\frac{\mu_F+\gamma}{2}\big)^K \Gamma^{(1)}_i =:\nu_i.
\end{align}
Note also that from \eqref{appendix:L:Th:strongly convex:error bound:  eq1}
\begin{align}\label{appendix:Discussion on error bounds: dummy 1}
 (1- \alpha \mu_f)^L \frac{1}{\lambda^2} \nsq{\nabla \ftildeij(\theta_{i,j}^{t,k,0})} \leq \deltaij \ \xrightarrow[]{} \ 
(1- \alpha \mu_f)^L \frac{1}{\lambda^2} 
 \avgNi \nsq{\nabla \ftildeij(\theta_{i,j}^{t,k,0})} \leq \deltai.
\end{align}
Using the equation above and $\delta_i=\big(1-\eta_i\frac{\mu_F+\gamma}{2}\big)^K  \frac{\Gamma^{(1)}_i}{2\Gamma^{(2)}}$, we write
\begin{align} \label{appendix: eq: elation KL convex}
(1- \alpha \mu_f)^L \Gamma^{(3)}_i
 \leq (1-\eta_i\frac{\mu_F+\gamma}{2})^K \frac{\Gamma^{(1)}_i}{2\Gamma^{(2)}}
 ~ \rightarrow \ 
L \geq \frac{\ln{(1-\eta_i\frac{\mu_F+\gamma}{2})}}{\ln{(1- \alpha \mu_f)}} K
+\frac{1}{\ln{(1- \alpha \mu_f)}}
 \ln{\Big( \frac{\Gamma^{(1)}_i}{2\Gamma^{(3)}_i \Gamma^{(2)}}} \Big)
\end{align}
where $\Gamma^{(3)}_i:= \frac{1}{\lambda^2} 
 \avgNi \nsq{\nabla \ftildeij(\theta_{i,j}^{t,k,0})}$.
 \emph{ Note that \eqref{appendix: eq: elation KL convex} allows different teams  to have a different number of loops to  provide the needed accuracy.}
 
\subsubsection{ How to choose \texorpdfstring{$K$}{Ak}}
Consider \eqref{appendix:Convergence of PfedMT:Th:strongly convex:eq1}   

\begin{align}\label{appendix: eq: relation KTL assumption convex}
\|\xT-x^*\|^2 &\leq (1-\frac{\beta \muFtilde}{2})^T \|x^0- x^*\|^2+\frac{6\gamma^2 \nu}{\muFtilde^2}
\nonumber \\
&=2(1-\frac{\beta \muFtilde}{2})^T \|x^0- x^*\|^2
\end{align}
where in the last line we require  $\nu=\frac{\muFtilde^2}{6\gamma^2} (1-\frac{\beta \muFtilde}{2})^T \|x^0- x^*\|^2$. 
Also note that from \eqref{appendix:Discussion on error bounds: dummy 2} we can write 
\begin{align}
\big(1-\eta_i\frac{\mu_F+\gamma}{2}\big)^K \Gamma^{(1)}_i =\nu_i
     \ \xrightarrow[]{} \ 
     \big(1-\eta_i\frac{\mu_F+\gamma}{2}\big)^K   \avgM  \Gamma^{(1)}_i 
     \leq \big(1-\eta_{\min}\frac{\mu_F+\gamma}{2}\big)^K   \avgM \Gamma^{(1)}_i 
= \avgM \nu_i=\nu,
\end{align}
where $\eta_{\min}=\min_i\{\eta_i\}$. Finally, we can write 
\begin{align}\label{appendix: eq: elation KT convex}
\big(1-\eta_{\min}\frac{\mu_F+\gamma}{2}\big)^K \Gamma^{(4)} \leq   \Gamma^{(5)} (1-\frac{\beta \muFtilde}{2})^T
     \ \xrightarrow \ 
K \geq \frac{\ln{(1-\frac{\beta \muFtilde}{2})}}{\ln{(1-\eta_{\min}\frac{\mu_F+\gamma}{2})}} T+ \frac{1}{\ln{(1-\eta_{\min}\frac{\mu_F+\gamma}{2})}} \ln{\big ( \frac{\Gamma^{(5)}}{\Gamma^{(4)}}}  \big),  
\end{align}
where $\Gamma^{(4)}:= \avgM \Gamma^{(1)}_i$ and $\Gamma^{(5)}:= \frac{\muFtilde^2}{6\gamma^2} \|x^{0}- x^*\|^2$.

\subsubsection{Combining \texorpdfstring{\Cref{appendix:Th: Team-level bound: inexact grad strongly convex}, \Cref{appendix:Th: K bound: inexact grad: strongly convex}, and \Cref{appendix:L:Th: client level strongly convex:error bound}}{Combining Theorems A, B, and C}}

Assuming \eqref{appendix: eq: elation KL convex} and \eqref{appendix: eq: elation KT convex} are in force, we  can derive \eqref{appendix: eq: relation KTL assumption convex}
\begin{align}
\|\xT-x^*\|^2 \leq 2(1-\frac{\beta \muFtilde}{2})^T \|x^0- x^*\|^2.
\end{align}
This completes the proof of \Cref{PfedMT:theorem: analysis of the strongly convex case}.

\section{ Convergence of \ourmodel{} (non-convex case) \label{appendix:sec Team-level convergence non-convex} }
Let us consider the team-level update as in Algorithm 2. We analyze the convergence of this algorithm when objectives are non-convex and smooth. We assume that we have access to the approximate solution to the $\mathrm{prox}_{F_i/\gamma}(.)$.

\begin{theorem}\label{appendix:Th: Team-level bound: inexact grad: smooth}
Suppose that the sequence $\{x^t\}$ is generated by  \ourmodel{}. Further assume that $\witmin$ is an approximate solution to $\prox_{F_i/\gamma}(\xt)$  such that  $\nsq{\witmin-\prox_{F_i/\gamma}(\xt)} \leq \nu_i$. Then, we have  

\begin{align}  \label{appendix: Convergence of PfedMT:Th:non-convex:eq1}
\mathbb{E} \big[ \nsq{\nabla\phi(x^{\tilde{t}})}  \big] 
&\leq 2 \frac{\phi(x^0)-\phi(x^*)}{\beta T}
+ \gamma^2 \nu.
\end{align}
where $\beta \leq \frac{1}{\LFtilde}$,  $\nu:=\avgM \nu_i$ and 
$\Tilde{t}$ is uniformly sampled from $[0, T-1]$.

\end{theorem}

\begin{proof}
Using smoothness of $F$, we can write

\begin{align}\label{appendix:eq:proof:T nonconvex}
\phi(\xtt)-\phi(\xt) 
&\leq \ip{\nabla\phi(\xt)}{\xtt-\xt} +\frac{\LFtilde}{2} \nsq{\xtt-\xt}
\nonumber \\
&=-\beta\ip{\nabla\phi(\xt)}{ \gt} +\beta^2 \frac{\LFtilde}{2} \nsq{\gt}
\quad (\text{see \eqref{appendix:eq: pFedMT analysis with approx prox smooth}})
\nonumber \\
&=-\beta\ip{\nabla\phi(\xt)}{ \gt} + \frac{\beta}{2}\nsq{\gt} + \frac{\beta}{2} (\beta\LFtilde - 1) \nsq{\gt}
\nonumber \\
&\leq-\beta\ip{\nabla\phi(\xt)}{ \gt} + \frac{\beta}{2}\nsq{\gt}
\quad (\text{by choosing $\beta \leq 1/\LFtilde$})
\nonumber \\
&=-\frac{\beta}{2} \nsq{\nabla\phi(\xt)} + \frac{\beta}{2} \underbrace{\nsq{\gt-\nabla\phi(\xt)}}_{J_1},
\end{align}
where in the last line, we used 
$2 \ip{a}{b} = \nsq{a} + \nsq{b} - \nsq{a-b} $.

We bound $J_1$ as follows:
\begin{itemize}
    \item {\bf Upper bound on $J_1$:}
\begin{align}\label{appendix:eq:bound on J_1 non-convex}
    J_1 &=\nsq{\gt-\nabla\phi(\xt)}
    \nonumber \\
    &=\nsq{\avgM ( \git -\nabla \Ftildei(\xt))}
    \nonumber \\
    &\leq \avgM \nsq{\git -\nabla \Ftildei(\xt)}
    \quad (\text{see: Remark \ref{appendix: general properties}})
    \nonumber \\    
&= \avgM \| \gamma(\xt- \witmin)-\gamma(\xt- \mathrm{prox}_{F_i/\gamma}(x^t) )    \|^2
\nonumber \\    
&=\gamma^2 \avgM \|\witmin-\mathrm{prox}_{F_i/\gamma}(x^t)  \|^2
\nonumber \\    
& \leq  \gamma^2 \avgM \nu_i
\nonumber \\
&=\gamma^2 \nu
\end{align}

\end{itemize}

Substituting $J_1$ into \eqref{appendix:eq:proof:T nonconvex}, we have
\begin{align}\label{appendix:eq:proof:T nonconvex dummy1}
\phi(\xtt)-\phi(\xt) 
&\leq -\frac{\beta}{2} \nsq{\nabla\phi(\xt)} + \frac{\beta}{2}\gamma^2\nu.
\end{align}
\eqref{appendix:eq:proof:T nonconvex dummy1}, can be re-written as
\begin{align}\label{appendix:eq:proof:T nonconvex dummy2}
\nsq{\nabla\phi(\xt)} \leq \frac{2}{\beta} \big(\phi(\xt)-\phi(\xtt)\big)+ \gamma^2\nu
\end{align}
Averaging  \eqref{appendix:eq:proof:T nonconvex dummy2} over $t$, we get
\begin{align}\label{appendix:eq:proof:T nonconvex dummy3}
\frac{1}{T}\avgT \nsq{\nabla\phi(\xt)} 
&\leq \frac{1}{T} \avgT \frac{2}{\beta} \big(\phi(\xt)-\phi(\xtt) \big) + \gamma^2\nu
\nonumber \\
&= 2 \frac{\phi(x^0)-\phi(x^T)}{\beta T}
+ \gamma^2\nu
\nonumber \\
&\leq 2 \frac{\phi(x^0)-\phi(x^*)}{\beta T}
+ \gamma^2\nu
\quad (\text{since $\phi(x^T) \geq \phi(x^*)$})
\end{align}
Assume $\tilde{t}$ is sampled uniformly at random from $[0, T-1]$, then we can write
\begin{align}\label{appendix:eq:proof:T nonconvex dummy5} 
\mathbb{E} \big[ \nsq{\nabla\phi(x^{\tilde{t}})}  \big] 
&\leq 2 \frac{\phi(x^0)-\phi(x^*)}{\beta T}
+ \gamma^2 \nu.
\end{align}

\end{proof}

\subsection{ Team-level convergence (non-convex case)}

\begin{theorem}\label{appendix:Th: witkstar bound: inexact grad: smooth}

Consider Problem \ref{eqn:moreau-team-subproblem} with   gradient descent steps \eqref{eqn:prox_Fi_GD} 
, and assume that we have $\nsq{ \prox_{\fij/\lambda}(\witk) -\thetaijtkL} \leq \delta_{i,j}$.  We can write
\begin{align}
\label{appendix:eq:proof:K nonconvex}
 \nsq{\witmin-\prox_{F_i/\gamma}(\xt)} &\leq \frac{1}{(\gamma-\Lftilde)^2} \bigg(
\frac{2}{K\eta_i} ( F_i(\xt)- \prox_{F_i/\gamma}(\xt))+\lambda^2 \avgNi \delta_{i,j}
\bigg)=: \nu_i
\end{align}
where $\eta_i \leq \frac{1}{(\Lftilde+\gamma)}$, $\gamma > \Lftilde$, and  $\witmin=\argmin_k \{ \nsq{\nabla p_i(\witk)} \}$.

\end{theorem}

\begin{proof}
Let us define $p_i(y):= F_i(y)+\frac{\gamma}{2} \nsq{y- \xt}$.Since $\ftildeij(.)$ is $\Lftilde$-Smooth therfore  $F_i(.)=\avgNi \ftildeij(.)$ and $p_i(.)$ are  $\Lftilde$-Smooth and $(\Lftilde+\gamma)$-Smooth, respectively. Using smoothness, we have 
\begin{align}
p_i(\witkk)-p_i(\witk) 
&\leq \ip{\nabla p_i(\witk) }{\witkk-\witk} +\frac{\Lftilde+\gamma}{2} \nsq{\witkk-\witk}
\nonumber \\
&
\end{align}
Now, let us re-write \eqref{eqn:prox_Fi_GD} as $\witkk= \witk -\eta_i \qitk$, where $\qitk$ is the inexact gradient of $p_i(y):=F_i(y)+\frac{\gamma}{2} \nsq{y- \xt}$ at $y=\witk$.
\begin{align}
p_i(\witkk)-p_i(\witk) 
&\leq \ip{\nabla p_i(\witk) }{\witkk-\witk} +\eta_i^2 \frac{\Lftilde+\gamma}{2} \nsq{\qitk}
\nonumber \\
&=-\eta_i \ip{\nabla p_i(\witk) }{\qitk} +\eta_i^2 \frac{\Lftilde+\gamma}{2} \nsq{\qitk}
\nonumber \\
&=- \frac{\eta_i}{2} \nsq{\nabla p_i(\witk)} - \frac{\eta_i}{2}  \nsq{\qitk} + 
\frac{\eta_i}{2} \nsq{\qitk- \nabla p_i(\witk)}
+\eta_i^2 \frac{\Lftilde+\gamma}{2} \nsq{\qitk}
\nonumber \\
&=- \frac{\eta_i}{2} \nsq{\nabla p_i(\witk)} - \frac{\eta_i}{2} \Big(1- \eta_i(\Lftilde+\gamma) \Big)  \nsq{\qitk} 
+
\frac{\eta_i}{2} \nsq{\qitk- \nabla p_i(\witk)}
\nonumber \\
&=- \frac{\eta_i}{2} \nsq{\nabla p_i(\witk)} 
+
\frac{\eta_i}{2} \nsq{\qitk- \nabla p_i(\witk)}
\end{align}
where in the third line, we used 
$2 \ip{a}{b} = \nsq{a} + \nsq{b} - \nsq{a-b} $ and in the last line we used  $\eta_i \leq \frac{1}{\Lftilde+\gamma}$.

Re-arranging and averaging over $k$, we get
\begin{align}
\avgK \bigg(
 \nsq{\nabla p_i(\witk)}
- \nsq{\qitk-\nabla p_i(\witk)}
\bigg)
&\leq\frac{2}{K\eta_i} ( p_i(w_{i}^{t,0}) - p_i(\witK))
\nonumber \\
&\leq \frac{2}{K\eta_i} ( p_i(\xt)- p_i(\witK))
\nonumber \\
&\leq \frac{2}{K\eta_i} ( p_i(\xt)- p_i(\witstar))
\nonumber \\
&\leq \frac{2}{K\eta_i} ( F_i(\xt)- \prox_{F_i/\gamma}(\xt)).
\end{align}

Substituting the assumption 
$\nsq{ \prox_{\fij/\lambda}(\witk) -\thetaijtkL} \leq \delta_{i,j}$, we get
\begin{align}
\label{appendix:eq:proof:K nonconvex dummy1} 
\avgK
 \nsq{\nabla p_i(\witk)}
&\leq \frac{2}{K\eta_i} ( F_i(\xt)- \prox_{F_i/\gamma}(\xt))+ \avgK \nsq{\qitk-\nabla p_i(\witk)}
\nonumber \\
&\leq \frac{2}{K\eta_i} ( F_i(\xt)- \prox_{F_i/\gamma}(\xt))+\lambda^2 \avgK 
\nsq{\avgNi (\prox_{\fij/\lambda}(\witk) -\thetaijtkL)}
\nonumber \\
&\leq \frac{2}{K\eta_i} ( F_i(\xt)- \prox_{F_i/\gamma}(\xt))+\lambda^2
 \avgK \avgNi
\nsq{ (\prox_{\fij/\lambda}(\witk) -\thetaijtkL)}
\nonumber \\
&\leq \frac{2}{K\eta_i} ( F_i(\xt)- \prox_{F_i/\gamma}(\xt))+\lambda^2 \avgNi \delta_{i,j}
\end{align}

Let us define $\witmin=\argmin_k \{ \nsq{\nabla p_i(\witk)} \}$, then we can write
\begin{align}
\label{appendix:eq:proof:K nonconvex dummy2}
\avgK
\nsq{\nabla p_i(\witk)}
\geq  \nsq{\nabla p_i(\witmin)} 
\geq (\gamma-\Lftilde)^2 \nsq{\witmin-\witstar} \geq 
(\gamma-\Lftilde)^2 \nsq{\witmin-\prox_{F_i/\gamma}(\xt)}
\end{align}
where we used \eqref{appendix: properties strongly-convex eq1} because $p_i(.)$ is $(\gamma-\Lftilde)$-strongly convex. Finally, we can combine \eqref{appendix:eq:proof:K nonconvex dummy1} and \eqref{appendix:eq:proof:K nonconvex dummy2}

\begin{equation}
\label{appendix:eq:proof:K nonconvex dummy3}
 \nsq{\witmin-\prox_{F_i/\gamma}(\xt)}
 \leq \frac{1}{(\gamma-\Lftilde)^2} \bigg(
\frac{2}{K\eta_i} ( F_i(\xt)- \prox_{F_i/\gamma}(\xt))+\lambda^2 \avgNi \delta_{i,j}
\bigg) 
=: \nu_i
\end{equation}
\end{proof}

\subsection{Client-level convergence (non-convex case)}

\begin{theorem}
\label{appendix:L:Th:strongly non-convex:error bound}
Let $\fij$ be $L_f$-smooth functions. Consider Problem \ref{eqn:moreau-device-subproblem} with gradient descent steps \eqref{eqn:device_p_GD} 
, we have
\begin{equation}
\label{appendix:L:Th:strongly non-convex:error bound:  eq1}
    \nsq{\thetaijtklstar-\prox_{\fij/\lambda}(\witk)} \leq  \frac{1}{L \alpha (1-\frac{L_F+\lambda}{2}\alpha)(\gamma-\Lftilde)^2} \big(\fij(\theta_{i,j}^{t,k,0})-\ftildeij(\theta_{i,j}^{t,k,0}) \big)=:\deltaij
\end{equation}
where $ \alpha \leq \frac{2}{L_F+\lambda} $ and $\thetaijtklstar=\argmin_l \{ \nsq{\nabla \pij(\thetaijtkl)} \}$.
\end{theorem}

\begin{proof}
Let us define $\pij(y):= \fij(y)+\frac{\lambda}{2} \nsq{y- \witk}$. Note that $p_{i,j}(.)$ is $(L_f+\lambda)$-smooth. Using smoothness, we have
\begin{align}
\pij(\thetaijtkll)-\pij(\thetaijtkl) 
&\leq \ip{\nabla \pij(\thetaijtkl) }{\thetaijtkll-\thetaijtkl} +\frac{L_F+\lambda}{2} \nsq{\witkk-\witk}
\nonumber \\
&=\ip{\nabla \pij(\thetaijtkl) }{-\alpha \nabla \pij(\thetaijtkl)} +\frac{L_F+\lambda}{2} \alpha^2 \nsq{\nabla \pij(\thetaijtkl)}
\quad(\text{see \eqref{eqn:device_p_GD}})
\nonumber \\
&=-\alpha(1-\alpha \frac{L_F+\lambda}{2} ) \nsq{\nabla \pij(\thetaijtkl)}.
\end{align}
Averaging over $l$ and assuming $\alpha \leq \frac{2}{L_F+\lambda} $, we get
\begin{align}
\label{appendix:eq:proof:L nonconvex dummy1}
\avgL \nsq{\nabla \pij(\thetaijtkl)}  &\leq  \frac{1}{L \alpha (1-\frac{L_F+\lambda}{2}\alpha)} \big(\pij(\theta_{i,j}^{t,k,0})-\pij(\theta_{i,j}^{t,k,L}) \big)
\nonumber \\
&\leq  \frac{1}{L \alpha (1-\frac{L_F+\lambda}{2}\alpha)} \big(\pij(\witk)-\pij(\theta_{i,j}^{t,k,*}) \big)
\quad (\theta_{i,j}^{t,k,0}=\witk, \pij(\theta_{i,j}^{t,k,L}) \geq \pij(\theta_{i,j}^{t,k,*}))
\nonumber \\
&=\frac{1}{L \alpha (1-\frac{L_F+\lambda}{2}\alpha)} \big(\fij(\witk)-\ftildeij(\witk) \big)
\quad(\thetaijtkstar=\prox_{\fij/\lambda}(\witk))
\nonumber \\
&\leq \frac{1}{L \alpha (1-\frac{L_F+\lambda}{2}\alpha)} \big(\fij(\theta_{i,j}^{t,k,0})-\ftildeij(\theta_{i,j}^{t,k,0}) \big)
\end{align}

Let us define $\thetaijtklstar=\argmin_l \{ \nsq{\nabla \pij(\thetaijtkl)} \}$, then we can write
\begin{align}
\label{appendix:eq:proof:L nonconvex dummy2}
\avgL
\nsq{\nabla \pij(\thetaijtkl)}
\geq  \nsq{\nabla \pij(\thetaijtklstar)} 
\geq (\gamma-\Lftilde)^2 \nsq{\thetaijtklstar-\thetaijtkstar} \geq 
(\gamma-\Lftilde)^2 \nsq{\thetaijtklstar-\prox_{\fij/\lambda}(\witk)}
\end{align}
where we used \eqref{appendix: properties strongly-convex eq1} with $x=\thetaijtklstar, y=\thetaijtkstar, {\nabla} \pij(\thetaijtkstar)=0$ because $\pij(.)$ is $(\gamma-\Lftilde)$-strongly convex.  Finally, we can combine \eqref{appendix:eq:proof:L nonconvex dummy1} and \eqref{appendix:eq:proof:L nonconvex dummy2}
\begin{align}\label{appendix:eq:proof:L nonconvex dummy3}
    \nsq{\thetaijtklstar-\prox_{\fij/\lambda}(\witk)} \leq  \frac{1}{L \alpha (1-\frac{L_F+\lambda}{2}\alpha)(\gamma-\Lftilde)^2} \big(\fij(\theta_{i,j}^{t,k,0})-\ftildeij(\theta_{i,j}^{t,k,0}) \big)=:\deltaij
\end{align}

\end{proof}

\subsection{ Discussion on error bounds (non-convex case)}
\label{Discussion on error bounds (non-convex case)}

In this section, we discuss the relation between $L, K$, and $T$ in the non-convex setup. 

\subsection{ How to choose \texorpdfstring{$L$}{AL}}
From \eqref{appendix:eq:proof:K nonconvex}, we have

\begin{align}
\label{appendix:Discussion on error bounds:eq:proof:K nonconvex eq1}
 \nsq{\witmin-\prox_{F_i/\gamma}(\xt)} &\leq \frac{1}{(\gamma-\Lftilde)^2} \bigg(
\frac{2}{K\eta_i} ( F_i(\xt)- \prox_{F_i/\gamma}(\xt))+\lambda^2 \avgNi \delta_{i,j}
\bigg)
\nonumber \\
&=\frac{1}{(\gamma-\Lftilde)^2} \bigg(
\frac{2}{K\eta_i} ( F_i(\xt)- \prox_{F_i/\gamma}(\xt))+\lambda^2 \delta_{i}
\bigg)
\nonumber \\
&\leq \frac{B^1_i}{K}=\nu_i  
\end{align}
where $B^1_i=\frac{4}{\eta_i(\gamma-\Lftilde)^2} ( F_i(\xt)- \prox_{F_i/\gamma}(\xt))$  by requiring $\delta_i \leq \frac{2}{\lambda^2 \eta_i K}  ( F_i(\xt)- \prox_{F_i/\gamma}(\xt))$. Note that, from \eqref{appendix:eq:proof:L nonconvex dummy3}, we can write
\begin{align}
\frac{1}{L_i \alpha (1-\frac{L_F+\lambda}{2}\alpha)(\gamma-\Lftilde)^2} \avgNi  \big(\fij(\theta_{i,j}^{t,k,0})-\ftildeij(\theta_{i,j}^{t,k,0}) \big)=\avgNi \deltaij=\delta_i,
\end{align}
and therefore
\begin{align}
\frac{1}{L_i \alpha (1-\frac{L_F+\lambda}{2}\alpha)(\gamma-\Lftilde)^2} \avgNi  \big(\fij(\theta_{i,j}^{t,k,0})-\ftildeij(\theta_{i,j}^{t,k,0}) \big) =\delta_i\leq\frac{2}{\lambda^2 \eta_i K}  ( F_i(\xt)- \prox_{F_i/\gamma}(\xt)).
\end{align}
This leads to 
\begin{equation} \label{appendix: eq: elation KL non-convex}
 L_i \geq B^2_i K
\end{equation}
where $B^2_i=\frac{\lambda^2 \eta_i}{2\alpha (1-\frac{L_F+\lambda}{2}\alpha)(\gamma-\Lftilde)^2} \avgNi  \big(\frac{\fij(\theta_{i,j}^{t,k,0})-\ftildeij(\theta_{i,j}^{t,k,0})}{ F_i(\xt)- \prox_{F_i/\gamma}(\xt)} \big) $. \emph{Note that  \eqref{appendix: eq: elation KL non-convex} allows different teams  to have a different number of loops to  provide the needed accuracy.}

\subsection{ How to choose \texorpdfstring{$K$}{Ak}}
From \eqref{appendix: Convergence of PfedMT:Th:non-convex:eq1}, we have
\begin{align}\label{appendix:Discussion on error bounds:eq:proof:K nonconvex e99}
\mathbb{E} \big[ \nsq{\nabla\phi(x^{\tilde{t}})}  \big] 
&\leq 2 \frac{\phi(x^0)-\phi(x^*)}{\beta T}
+ \beta \LFtilde \gamma^2 \nu  
\nonumber \\
&\leq \frac{B^3}{T}
\end{align}
where $B^3=4 \frac{\phi(x^0)-\phi(x^*)}{\beta }$ by requiring $\nu \leq 2 \frac{\phi(x^0)-\phi(x^*)}{\beta^2  \gamma^2 \LFtilde T}$. Uson equation above  and \eqref{appendix:Discussion on error bounds:eq:proof:K nonconvex eq1}, we get
\begin{align}
\frac{1}{K} \avgM B^1_i =\avgM \nu_i= \nu \leq 2 \frac{\phi(x^0)-\phi(x^*)}{\beta^2  \gamma^2 \LFtilde T}.   
\end{align}
 That leads to
 \begin{equation}\label{appendix: eq: elation KT non-convex}
 K \geq  B^4 ~ T
 \end{equation}
where $B^4:= \frac{\beta^2  \gamma^2 \LFtilde T}{2(\phi(x^0)-\phi(x^*))}\avgM B^1_i$.

\subsection{Combining \texorpdfstring{\Cref{appendix:Th: Team-level bound: inexact grad: smooth}, \Cref{appendix:Th: witkstar bound: inexact grad: smooth}, and \Cref{appendix:L:Th:strongly non-convex:error bound}}{Combining Theorems X, Y, and Z}}

Assuming \eqref{appendix: eq: elation KL non-convex} and \eqref{appendix: eq: elation KT non-convex} are in force, we  can derive \eqref{appendix:Discussion on error bounds:eq:proof:K nonconvex e99}
\begin{equation}
\mathbb{E}_{\tilde{t}} \big[ \nsq{\nabla\phi(x^{\tilde{t}})}  \big]   \leq 4 \frac{\phi(x^0)-\phi(x^*)}{\beta } \frac{1}{T}=\mathcal{O} \Big(\frac{1}{T} \Big).
\end{equation}
This can be translated to 
\begin{equation}
\mathbb{E}_{\tilde{t}} \big[ \norm{\nabla\phi(x^{\tilde{t}})}  \big]   \leq\mathcal{O} \Big(\frac{1}{\sqrt{T}} \Big).
\end{equation}

\section{Experiments} \label{appendix: experiment}
We studied classification problems to validate \ourmodel{} using both image (MNIST \citep{lecun1998gradient}, FMNIST \citep{xiao2017online} EMNIST \citep{cohen2017emnist} with 10 classes and 62 classes, CIFAR100 \citep{krizhevsky2009learning}) and non-image synthetic datasets, but \ourmodel{} is not limited to these two categories of data. Data distributions are heterogeneous and non-iid. The
distribution of the EMNIST dataset with 62 classes exhibits non-iid characteristics similar to the FEMNIST dataset. Consequently, we will distinguish EMNIST with 10 classes as 'EMNIST' and EMNIST with 62 classes as 'FEMNIST'. In the case of MNIST, FMNIST, EMNIST, and synthetic datasets, each device contains data for 2 classes. However, for FEMNIST and CIFAR100 datasets, each device carries data for 3 classes. There are no overlapping samples among devices. This supplementary material provides the following experiments and analysis to validate the \ourmodel{}.
\begin{enumerate}
    \item A detailed empirical study to investigate the impact of hyperparameters $\beta$, $\lambda$, and $\gamma$ on the convergence of PerMFL (see \Cref{appx: hyperparameter}).
    \item An ablation study to analyze the influence of team and device participation (see \Cref{appx: ablation_study_team}).
    \item An ablation study to explore team formation, i.e., the performance of \ourmodel{} on worst-case and average-case team formation (see \Cref{team_formation}).
    \item An ablation study that explores the effects of team iterations on the convergence of \ourmodel{} (see \Cref{appx: effect_team_iters}).
    \item Convergence analysis on MNIST and synthetic datasets with the multi-tier SOTA such as (AL2GD\citep{lyu2022personalized}, h-SGD \citep{liu2022hierarchical}, \citep{nguyen2022self}) (see \Cref{appx: exp_conv_analysis}). 
    \item Performance analysis of \ourmodel{} on FEMNIST and CIFAR100 datasets (see \Cref{appx: perf_femnist_cifar100}).

\end{enumerate}
Each experiments ran 10 times, from there we produced mean and standard deviation. 
\subsection{Hardware specification}  \label{appx: hardware}
The experiments were conducted using an NVIDIA DGX-A100 GPU with 40 GB of RAM. The DGX-A100 is based on the NVIDIA A100 Tensor Core GPU architecture, and each A100 GPU in the system has 40 GB of high-bandwidth memory (HBM2). The standard configuration of the DGX-A100 includes eight A100 GPUs connected via NVLink, providing a total of 320 GB of GPU memory (40 GB per GPU × 8 GPUs). However, we could utilize only one GPU at a time for our experiments, which allowed us to utilize 40 GB of GPU memory per experiment.

\subsection{Dataset description}
We conducted our experiments using MNIST, FMNIST, EMNIST, FEMNIST, CIFAR100 and Synthetic datasets. FEMNIST, CIFAR100, and synthetic datasets are considered to create larger and more complicated scenarios from these datasets. The description of the datasets is given below.

\subsubsection{MNIST}

The primary purpose of the MNIST dataset is to serve as a widely used benchmark for evaluating and comparing the performance of various models in image classification tasks. It consists of a total of 70,000 examples, with 60,000 examples used for training and 10,000 examples used for testing. Each example is a grayscale image measuring 28x28 pixels, representing a handwritten digit ranging from 0 to 9. Every image in the dataset is associated with a label that denotes the correct digit it represents. The labels themselves are integers ranging from 0 to 9, corresponding to the handwritten digits in the images.

\subsubsection{FMNIST}
The FMNIST dataset serves the purpose of evaluating and benchmarking machine learning algorithms, particularly in the areas of image classification and pattern recognition. It differs from the original MNIST dataset by focusing on fashion-related images rather than handwritten digits, providing a more complex task. The dataset consists of 70,000 grayscale images with dimensions of 28x28 pixels. These images are split into 60,000 training examples and 10,000 testing examples. They depict various fashion items, including clothing, shoes, bags, and accessories. Each image in the FMNIST dataset is associated with a label representing the corresponding fashion item category. There are a total of 10 classes representing different types of clothing and fashion accessories such as T-shirt/top, Trouser, Pullover, Dress, Coat, Sandal, Shirt, Sneaker, Bag, and Ankle boot. Compared to the MNIST dataset, FMNIST presents a more significant challenge due to the diversity of clothing types and the increased complexity of the images. These datasets are often utilized to assess the robustness and generalization capabilities of machine learning models.

\subsubsection{EMNIST}

The dataset is a collection of handwritten characters, including lowercase and uppercase letters and digits. It consists of six different splits or variations, each representing a different task or scenario. The first split, called EMNIST ByClass, contains a total of 814,255 images representing 62 character classes. These classes include 26 uppercase letters, 26 lowercase letters, and 10 digits. The second split, EMNIST ByMerge, merges similar characters into a single class, resulting in 47 classes. This split is useful and challenging for scenarios where distinguishing between similar characters, such as uppercase and lowercase letters. The third split, EMNIST Balanced, aims to balance the number of samples per class. It provides a balanced dataset with 131,600 images representing 47 classes. The fourth split, EMNIST Letters, focuses exclusively on uppercase and lowercase letters. It consists of 145,600 images representing 26 classes of letters. The fifth split, EMNIST Digits, contains only the digits from 0 to 9. It consists of 280,000 images representing the 10-digit classes. Lastly, the sixth split follows the structure of the original MNIST dataset that contains 70,000 images of digits from 0 to 9. Each image in the dataset is associated with a label indicating the corresponding character class that provided information about the specific representative character or digit. Here for all the experiments, we considered split by digits.

\subsubsection{FEMNIST} It is a Federated EMNIST dataset. It is a ByClass split over the EMNIST dataset containing 814,255 images representing 62 character classes (0-9, A-Z, and a-z). The data are distributed among 3500 devices in an unbalanced manner, where each device has access to a maximum of 3 classes. FEMNIST is similar to \citep{caldas2018leaf}.

\subsubsection{CIFAR100} This dataset is a collection of 60,000 32x32 colour images in 100 classes, with 600 images per class. The dataset is distributed among 350 clients in an unbalanced manner where each client can access 3 classes. 

\subsubsection{Synthetic} We generate a synthetic dataset with $\bar \alpha$ = 0.5 and $\bar \beta$ = 0.5.  The synthetic dataset is a tabular dataset. It has 60 features and 10 classes.  The sample size of each client ranges from 250 to 25810, and each client has almost 2 classes.  Finally, we distribute the data to N devices according to the power law in \citep{li2020federated}.

\subsubsection{Data division}
MNIST, FMNIST, EMNIST, and synthetic datasets have ten different class labels representing distinct data distributions. The data was divided among multiple devices in a non-iid manner that ensured each device had information from two classes. To ensure each device has two classes, first, we gave specific data from each of these two classes to that device and then randomly distributed the remaining samples from these two classes. A similar approach is taken for the other devices. Following that, the devices were further divided into teams randomly. In the experiments, teams have an equal number of devices.

\subsection{Learning models}

Our study examined different scenarios and used specific models to handle them. We employed a multi-nomial logistic regression (MLR) model with $l_2$ regularization and a softmax activation function for strongly convex scenarios. To handle non-convex scenarios, we adopted different approaches depending on the dataset. For synthetic datasets, we constructed deep neural networks with two hidden layers. On the other hand, for the MNIST, FMNIST, and EMNIST datasets, which also involve non-convex scenarios, we built three two-layered convolutional neural networks (CNNs). 

Throughout our experiments, we used the abbreviations (PM) and (GM) to refer to the personalized model and global model, respectively.

\subsection{Effect of hyperparameters} \label{appx: hyperparameter}

A series of tests were conducted on the MNIST, FMNIST, and Synthetic datasets to examine the impact of various hyperparameters, including $\lambda$, $\gamma$, and $\beta$, on the convergence of \ourmodel{}. These tests were performed for both smooth strongly convex and smooth non-convex scenarios. The entire experiment is performed with the full participation of teams and devices for each global rounds. The number of teams is four, and each team has ten devices. 

\paragraph{Effect of $\beta$:} In this study, we made the following observation: when we increased the value of $\beta$ ( see \Cref{Fig:mnist_conv_beta_cnn} to \Cref{Fig: synthetic_conv_beta_mclr}) while other hyperparameters $\lambda$, and $\gamma$ remains constant, both personalized and global model of \ourmodel{} exhibited faster convergence. This behaviour was consistent in both convex and non-convex settings. Very low values of $\beta$ do not generalize the global model well and delay the convergence of personalized and global models. A high value of $\beta$ better generalizes the global model and the convergence faster. In all experiments from \Cref{Fig:mnist_conv_beta_cnn} to \Cref{Fig: synthetic_conv_beta_mclr}, we set the value of $\gamma = 3.0$ and $\lambda = 0.5$, $\eta = 0.03$, and $\alpha = 0.01$. 

\paragraph{Effect of $\gamma$:} In this study, we examined the influence of the hyperparameter $\gamma$ on the convergence of \ourmodel{}. From  \Cref{Fig:mnist_conv_gamma_cnn} to \Cref{Fig: synthetic_conv_gamma_mclr}, we observed that increasing the value of $\gamma$ led to faster convergence of \ourmodel{} PM and GM. Like $\beta$, a low value of $\gamma$ does not generalize the model well and slows the convergence speed of personalized and global models. Increasing the value of $\gamma$ results in a better generalization with a faster convergence. In all experiments from \Cref{Fig:mnist_conv_gamma_cnn} to \Cref{Fig: synthetic_conv_gamma_mclr}, we set the value of $\lambda = 1.5$ and $\beta=0.1$, $\eta = 0.03$, and $\alpha = 0.01$. 

\paragraph{Effect of $\lambda$:} Here, we examined the influence of the hyperparameter $\lambda$ on the convergence of \ourmodel{}. From the experimental results  (see \Cref{Fig:mnist_conv_lamda_cnn} to \Cref{Fig:synthetic_conv_lamda_cnn}), we found that a low value of $\lambda$ tended to impede the convergence process of both the personalized and global models. However, by increasing the value of $\lambda$, we observed a significant improvement in convergence. All the experiments from \Cref{Fig:mnist_conv_lamda_cnn} to \Cref{Fig:synthetic_conv_lamda_cnn}, we set the value of $\beta = 0.3$ and $\gamma = 3.0$,  $\eta = 0.03$, and $\alpha = 0.01$.

\paragraph{\textit{Discussions :}} The hyperparameters $\beta$, $\gamma$, and $\lambda$ exhibit interdependencies that should be taken into consideration. Modifying the value of a single hyperparameter can significantly influence the learning of both the personalized and global models. Consequently, it is crucial to recognize the intricate relationship among these hyperparameters, as changes in one hyperparameter can have consequential effects on the overall learning process. Therefore, thoroughly evaluating these hyperparameters is essential to achieve optimal model performance and convergence.


\begin{figure*}[!ht]
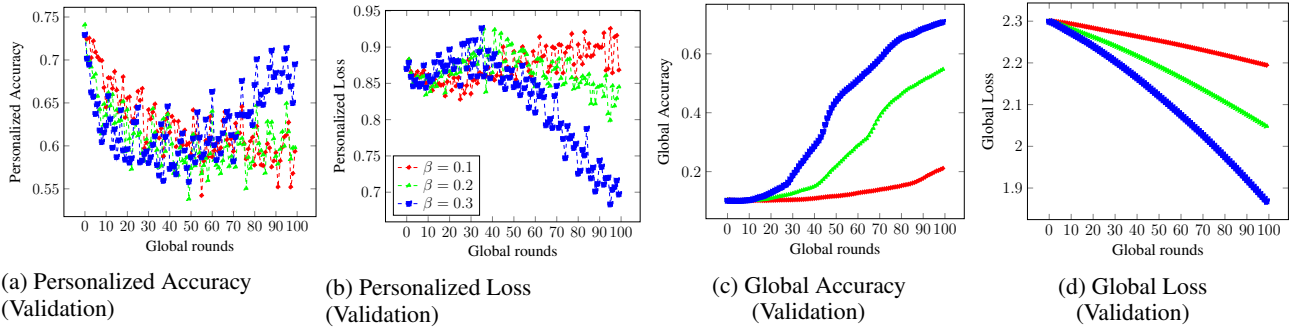

\centering
\begin{subfigure}[!t]{0.24\linewidth}
\resizebox{\linewidth}{!}{


}
\caption{Personalized Accuracy \\ (Validation)}
\label{Fig: mnist_conv_beta_pm_cnn_acc}
\end{subfigure} 
\begin{subfigure}[!t]{0.24\linewidth}
\resizebox{\linewidth}{!}{
%
}

\caption{Personalized Loss \\ (Validation)}
\label{Fig: mnist_conv_beta_pm_cnn_loss}
\end{subfigure}
\begin{subfigure}[!t]{0.24\linewidth}
\resizebox{\linewidth}{!}{
%

}
\caption{Global Accuracy \\ (Validation)}
\label{Fig: mnist_conv_beta_gm_cnn_acc}
\end{subfigure} 
\begin{subfigure}[!t]{0.24\linewidth}
\resizebox{\linewidth}{!}{
%
}

\caption{Global Loss \\ (Validation)}
\label{Fig: mnist_conv_beta_gm_cnn_loss}
\end{subfigure}

\caption{Effect of $\beta$ on convergence of \ourmodel{} in non-convex settings (CNN) using MNIST dataset}
\label{Fig:mnist_conv_beta_cnn}
\end{figure*}

\begin{figure*}[!ht]
\centering
\begin{subfigure}[!t]{0.24\linewidth}
\resizebox{\linewidth}{!}{
%

}
\caption{Personalized Accuracy \\ (Validation)}
\label{Fig: mnist_conv_beta_pm_mclr_acc}
\end{subfigure} 
\begin{subfigure}[!t]{0.24\linewidth}
\resizebox{\linewidth}{!}{
%
}

\caption{Personalized Loss \\ (Validation)}
\label{Fig: mnist_conv_beta_pm_mclr_loss}
\end{subfigure}
\begin{subfigure}[!t]{0.24\linewidth}
\resizebox{\linewidth}{!}{
%

}
\caption{Global Accuracy \\ (Validation)}
\label{Fig: mnist_conv_beta_gm_mclr_acc}
\end{subfigure} 
\begin{subfigure}[!t]{0.24\linewidth}
\resizebox{\linewidth}{!}{
%
}

\caption{Global Loss \\ (Validation)}
\label{Fig: mnist_conv_beta_gm_mclr_loss}
\end{subfigure}

\caption{Effect of $\beta$ on convergence of \ourmodel{} in strongly convex settings (MCLR) using MNIST dataset}
\label{Fig:mnist_conv_beta_mclr}
\end{figure*}


\begin{figure*}[!ht]
\centering
\begin{subfigure}[!t]{0.24\linewidth}
\resizebox{\linewidth}{!}{
%

}
\caption{Personalized Accuracy \\ (Validation)}
\label{Fig: fmnist_conv_beta_pm_cnn_acc}
\end{subfigure} 
\begin{subfigure}[!t]{0.24\linewidth}
\resizebox{\linewidth}{!}{
%
}

\caption{Personalized Loss \\ (Validation)}
\label{Fig: fmnist_conv_beta_pm_cnn_loss}
\end{subfigure}
\begin{subfigure}[!t]{0.24\linewidth}
\resizebox{\linewidth}{!}{
%

}
\caption{Global Accuracy \\ (Validation)}
\label{Fig: fmnist_conv_beta_gm_cnn_acc}
\end{subfigure} 
\begin{subfigure}[!t]{0.24\linewidth}
\resizebox{\linewidth}{!}{
%
}

\caption{Global Loss \\ (Validation)}
\label{Fig: fmnist_conv_beta_gm_cnn_loss}
\end{subfigure}

\caption{Effect of $\beta$ on convergence of \ourmodel{} in non-convex settings (CNN) using FMNIST dataset}
\label{Fig:fmnist_conv_beta_cnn}
\end{figure*}


\begin{figure*}[!ht]
\centering
\begin{subfigure}[!t]{0.24\linewidth}
\resizebox{\linewidth}{!}{
%

}
\caption{Personalized Accuracy \\ (Validation)}
\label{Fig: fmnist_conv_beta_pm_mclr_acc}
\end{subfigure} 
\begin{subfigure}[!t]{0.24\linewidth}
\resizebox{\linewidth}{!}{
%
}

\caption{Personalized Loss \\ (Validation)}
\label{Fig: fmnist_conv_beta_pm_mclr_loss}
\end{subfigure}
\begin{subfigure}[!t]{0.24\linewidth}
\resizebox{\linewidth}{!}{
%

}
\caption{Global Accuracy \\ (Validation)}
\label{Fig: fmnist_conv_beta_gm_mclr_acc}
\end{subfigure} 
\begin{subfigure}[!t]{0.24\linewidth}
\resizebox{\linewidth}{!}{
%
}

\caption{Global Loss \\ (Validation)}
\label{Fig: fmnist_conv_beta_gm_mclr_loss}
\end{subfigure}

\caption{Effect of $\beta$ on convergence of \ourmodel{} in strongly convex settings (MCLR) using FMNIST dataset}
\label{Fig:fmnist_conv_beta_mclr}
\end{figure*}


\begin{figure*}[!ht]
\centering
\begin{subfigure}[!t]{0.24\linewidth}
\resizebox{\linewidth}{!}{
%

}
\caption{Personalized Accuracy \\ (Validation)}
\label{Fig: synthetic_conv_beta_pm_cnn_acc}
\end{subfigure} 
\begin{subfigure}[!t]{0.24\linewidth}
\resizebox{\linewidth}{!}{
%
}

\caption{Personalized Loss \\ (Validation)}
\label{Fig: synthetic_conv_beta_pm_cnn_loss}
\end{subfigure}
\begin{subfigure}[!t]{0.24\linewidth}
\resizebox{\linewidth}{!}{
%

}
\caption{Global Accuracy \\ (Validation)}
\label{Fig: synthetic_conv_beta_gm_cnn_acc}
\end{subfigure} 
\begin{subfigure}[!t]{0.24\linewidth}
\resizebox{\linewidth}{!}{
%
}

\caption{Global Loss \\ (Validation)}
\label{Fig: synthetic_conv_beta_gm_cnn_loss}
\end{subfigure}

\caption{Effect of  $\beta$ on convergence of \ourmodel{} in non-convex settings (DNN) using Synthetic dataset}
\label{Fig: synthetic_conv_beta_cnn}
\end{figure*}


\begin{figure*}[!ht]
\centering
\begin{subfigure}[!t]{0.24\linewidth}
\resizebox{\linewidth}{!}{
%

}
\caption{Personalized Accuracy \\ (Validation)}
\label{Fig: synthetic_conv_beta_pm_mclr_acc}
\end{subfigure} 
\begin{subfigure}[!t]{0.24\linewidth}
\resizebox{\linewidth}{!}{
%
}

\caption{Personalized Loss \\ (Validation)}
\label{Fig: synthetic_conv_beta_pm_mclr_loss}
\end{subfigure}
\begin{subfigure}[!t]{0.24\linewidth}
\resizebox{\linewidth}{!}{
%

}
\caption{Global Accuracy \\ (Validation)}
\label{Fig: synthetic_conv_beta_gm_mclr_acc}
\end{subfigure} 
\begin{subfigure}[!t]{0.24\linewidth}
\resizebox{\linewidth}{!}{
%
}

\caption{Global Loss \\ (Validation)}
\label{Fig: synthetic_conv_beta_gm_mclr_loss}
\end{subfigure}

\caption{Effect of $\beta$ on convergence of \ourmodel{} in strongly convex settings (MCLR) using Synthetic dataset}
\label{Fig: synthetic_conv_beta_mclr}
\end{figure*}




\begin{figure*}[!ht]
\centering
\begin{subfigure}[!t]{0.24\linewidth}
\resizebox{\linewidth}{!}{
%

}
\caption{Personalized Accuracy \\ (Validation)}
\label{Fig: mnist_conv_gamma_pm_cnn_acc}
\end{subfigure} 
\begin{subfigure}[!t]{0.24\linewidth}
\resizebox{\linewidth}{!}{
%
}

\caption{Personalized Loss \\ (Validation)}
\label{Fig: mnist_conv_gamma_pm_cnn_loss}
\end{subfigure}
\begin{subfigure}[!t]{0.24\linewidth}
\resizebox{\linewidth}{!}{
%

}
\caption{Global Accuracy \\ (Validation)}
\label{Fig: mnist_conv_gamma_gm_cnn_acc}
\end{subfigure} 
\begin{subfigure}[!t]{0.24\linewidth}
\resizebox{\linewidth}{!}{
%
}

\caption{Global Loss \\ (Validation)}
\label{Fig: mnist_conv_gamma_gm_cnn_loss}
\end{subfigure}

\caption{Effect of $\gamma$ on convergence of \ourmodel{} in non-convex settings (CNN) using MNIST dataset}
\label{Fig:mnist_conv_gamma_cnn}
\end{figure*}

\begin{figure*}[!ht]
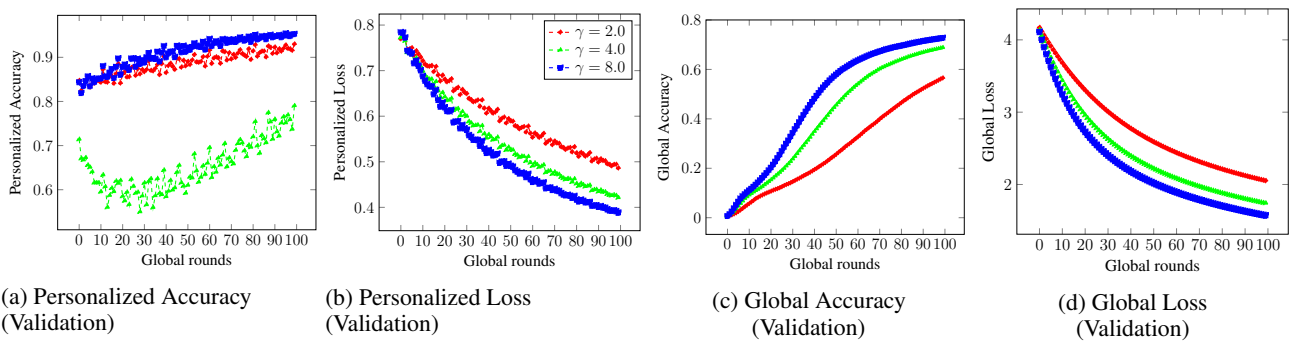

\centering
\begin{subfigure}[!t]{0.24\linewidth}
\resizebox{\linewidth}{!}{
%

}
\caption{Personalized Accuracy \\ (Validation)}
\label{Fig: mnist_conv_gamma_pm_mclr_acc}
\end{subfigure} 
\begin{subfigure}[!t]{0.24\linewidth}
\resizebox{\linewidth}{!}{
%
}

\caption{Personalized Loss \\ (Validation)}
\label{Fig: mnist_conv_gamma_pm_mclr_loss}
\end{subfigure}
\begin{subfigure}[!t]{0.24\linewidth}
\resizebox{\linewidth}{!}{
%

}
\caption{Global Accuracy \\ (Validation)}
\label{Fig: mnist_conv_gamma_gm_mclr_acc}
\end{subfigure} 
\begin{subfigure}[!t]{0.24\linewidth}
\resizebox{\linewidth}{!}{
%
}

\caption{Global Loss \\ (Validation)}
\label{Fig: mnist_conv_gamma_gm_mclr_loss}
\end{subfigure}

\caption{Effect of the hyperparameter $\gamma$ on convergence of \ourmodel{} in strongly convex settings (MCLR) using MNIST dataset}
\label{Fig:mnist_conv_gamma_mclr}
\end{figure*}



\begin{figure*}[!ht]
\centering
\begin{subfigure}[!t]{0.24\linewidth}
\resizebox{\linewidth}{!}{
%

}
\caption{Personalized Accuracy \\ (Validation)}
\label{Fig: fmnist_conv_gamma_pm_cnn_acc}
\end{subfigure} 
\begin{subfigure}[!t]{0.24\linewidth}
\resizebox{\linewidth}{!}{
%
}

\caption{Personalized Loss \\ (Validation)}
\label{Fig: fmnist_conv_gamma_pm_cnn_loss}
\end{subfigure}
\begin{subfigure}[!t]{0.24\linewidth}
\resizebox{\linewidth}{!}{
%

}
\caption{Global Accuracy \\ (Validation)}
\label{Fig: fmnist_conv_gamma_gm_cnn_acc}
\end{subfigure} 
\begin{subfigure}[!t]{0.24\linewidth}
\resizebox{\linewidth}{!}{
%
}

\caption{Global Loss \\ (Validation)}
\label{Fig: fmnist_conv_gamma_gm_cnn_loss}
\end{subfigure}

\caption{Effect of $\gamma$ on convergence of \ourmodel{} in non-convex settings (CNN) using FMNIST dataset}
\label{Fig: fmnist_conv_gamma_cnn}
\end{figure*}

\begin{figure*}[!ht]
\centering
\begin{subfigure}[!t]{0.24\linewidth}
\resizebox{\linewidth}{!}{
%

}
\caption{Personalized Accuracy \\ (Validation)}
\label{Fig: fmnist_conv_gamma_pm_mclr_acc}
\end{subfigure} 
\begin{subfigure}[!t]{0.24\linewidth}
\resizebox{\linewidth}{!}{
%
}

\caption{Personalized Loss \\ (Validation)}
\label{Fig: fmnist_conv_gamma_pm_mclr_loss}
\end{subfigure}
\begin{subfigure}[!t]{0.24\linewidth}
\resizebox{\linewidth}{!}{
%

}
\caption{Global Accuracy \\ (Validation)}
\label{Fig: fmnist_conv_gamma_gm_mclr_acc}
\end{subfigure} 
\begin{subfigure}[!t]{0.24\linewidth}
\resizebox{\linewidth}{!}{
%
}

\caption{Global Loss \\ (Validation)}
\label{Fig: fmnist_conv_gamma_gm_mclr_loss}
\end{subfigure}

\caption{Effect of $\gamma$ on convergence of \ourmodel{} in strongly convex settings (MCLR) using FMNIST dataset}
\label{Fig: fmnist_conv_gamma_mclr}
\end{figure*}


\begin{figure*}[!ht]
\centering
\begin{subfigure}[!t]{0.24\linewidth}
\resizebox{\linewidth}{!}{
%

}
\caption{Personalized Accuracy \\ (Validation)}
\label{Fig: synthetic_conv_gamma_pm_cnn_acc}
\end{subfigure} 
\begin{subfigure}[!t]{0.24\linewidth}
\resizebox{\linewidth}{!}{
%
}

\caption{Personalized Loss \\ (Validation)}
\label{Fig: synthetic_conv_gamma_pm_cnn_loss}
\end{subfigure}
\begin{subfigure}[!t]{0.24\linewidth}
\resizebox{\linewidth}{!}{
%

}
\caption{Global Accuracy \\ (Validation)}
\label{Fig: synthetic_conv_gamma_gm_cnn_acc}
\end{subfigure} 
\begin{subfigure}[!t]{0.24\linewidth}
\resizebox{\linewidth}{!}{
%
}

\caption{Global Loss \\ (Validation)}
\label{Fig: synthetic_conv_gamma_gm_cnn_loss}
\end{subfigure}

\caption{Effect of $\gamma$ on convergence of \ourmodel{} in non-convex settings (DNN) using Synthetic dataset}
\label{Fig: synthetic_conv_gamma_cnn}
\end{figure*}

\begin{figure*}[!ht]
\centering
\begin{subfigure}[!t]{0.24\linewidth}
\resizebox{\linewidth}{!}{
%

}
\caption{Personalized Accuracy \\ (Validation)}
\label{Fig: synthetic_conv_gamma_pm_mclr_acc}
\end{subfigure} 
\begin{subfigure}[!t]{0.24\linewidth}
\resizebox{\linewidth}{!}{
%
}

\caption{Personalized Loss \\ (Validation)}
\label{Fig: synthetic_conv_gamma_pm_mclr_loss}
\end{subfigure}
\begin{subfigure}[!t]{0.24\linewidth}
\resizebox{\linewidth}{!}{
%

}
\caption{Global Accuracy \\ (Validation)}
\label{Fig: synthetic_conv_gamma_gm_mclr_acc}
\end{subfigure} 
\begin{subfigure}[!t]{0.24\linewidth}
\resizebox{\linewidth}{!}{
%
}

\caption{Global Loss \\ (Validation)}
\label{Fig: synthetic_conv_gamma_gm_mclr_loss}
\end{subfigure}

\caption{Effect of $\gamma$ on convergence of \ourmodel{} in strongly convex settings (MCLR) using Synthetic dataset}
\label{Fig: synthetic_conv_gamma_mclr}
\end{figure*}



\begin{figure*}[!ht]
\centering
\begin{subfigure}[!t]{0.24\linewidth}
\resizebox{\linewidth}{!}{
%
}
\caption{Personalized Accuracy \\ (Validation)}
\label{Fig: mnist_conv_lamda_pm_cnn_acc}
\end{subfigure} 
\begin{subfigure}[!t]{0.24\linewidth}
\resizebox{\linewidth}{!}{
%
}

\caption{Personalized Loss \\ (Validation)}
\label{Fig: mnist_conv_lamda_pm_cnn_loss}
\end{subfigure}
\begin{subfigure}[!t]{0.24\linewidth}
\resizebox{\linewidth}{!}{
%

}
\caption{Global Accuracy \\ (Validation)}
\label{Fig: mnist_conv_lamda_gm_cnn_acc}
\end{subfigure} 
\begin{subfigure}[!t]{0.24\linewidth}
\resizebox{\linewidth}{!}{
%
}

\caption{Global Loss \\ (Validation)}
\label{Fig: mnist_conv_lamda_gm_cnn_loss}
\end{subfigure}

\caption{Effect of $\lambda$ on convergence of \ourmodel{} in non-convex settings (CNN) using MNIST dataset}
\label{Fig:mnist_conv_lamda_cnn}
\end{figure*}


\begin{figure*}[!ht]
\centering
\begin{subfigure}[!t]{0.24\linewidth}
\resizebox{\linewidth}{!}{
%
}
\caption{Personalized Accuracy \\ (Validation)}
\label{Fig: mnist_conv_lamda_pm_mclr_acc}
\end{subfigure} 
\begin{subfigure}[!t]{0.24\linewidth}
\resizebox{\linewidth}{!}{
%
}

\caption{Personalized Loss \\ (Validation)}
\label{Fig: mnist_conv_lamda_pm_mclr_loss}
\end{subfigure}
\begin{subfigure}[!t]{0.24\linewidth}
\resizebox{\linewidth}{!}{
%
}
\caption{Global Accuracy \\ (Validation)}
\label{Fig: mnist_conv_lamda_gm_mclr_acc}
\end{subfigure} 
\begin{subfigure}[!t]{0.24\linewidth}
\resizebox{\linewidth}{!}{
%
}

\caption{Global Loss \\ (Validation)}
\label{Fig: mnist_conv_lamda_gm_mclr_loss}
\end{subfigure}

\caption{Effect of $\lambda$ on convergence of \ourmodel{} in strongly convex settings (MCLR) using MNIST dataset}
\label{Fig:mnist_conv_lamda_mclr}
\end{figure*}



\begin{figure*}[!ht]
\centering
\begin{subfigure}[!t]{0.24\linewidth}
\resizebox{\linewidth}{!}{
%
}
\caption{Personalized Accuracy \\ (Validation)}
\label{Fig: fmnist_conv_lamda_pm_cnn_acc}
\end{subfigure} 
\begin{subfigure}[!t]{0.24\linewidth}
\resizebox{\linewidth}{!}{
%
}

\caption{Personalized Loss \\ (Validation)}
\label{Fig: fmnist_conv_lamda_pm_cnn_loss}
\end{subfigure}
\begin{subfigure}[!t]{0.24\linewidth}
\resizebox{\linewidth}{!}{
%

}
\caption{Global Accuracy \\ (Validation)}
\label{Fig: fmnist_conv_lamda_gm_cnn_acc}
\end{subfigure} 
\begin{subfigure}[!t]{0.24\linewidth}
\resizebox{\linewidth}{!}{
%
}

\caption{Global Loss \\ (Validation)}
\label{Fig: fmnist_conv_lamda_gm_cnn_loss}
\end{subfigure}

\caption{Effect of $\lambda$ on convergence of \ourmodel{} in non-convex settings (CNN) using FMNIST dataset}
\label{Fig:fmnist_conv_lamda_cnn}
\end{figure*}


\begin{figure*}[!ht]
\centering
\begin{subfigure}[!t]{0.24\linewidth}
\resizebox{\linewidth}{!}{
%
}
\caption{Personalized Accuracy \\ (Validation)}
\label{Fig: fmnist_conv_lamda_pm_mclr_acc}
\end{subfigure} 
\begin{subfigure}[!t]{0.24\linewidth}
\resizebox{\linewidth}{!}{
%
}

\caption{Personalized Loss \\ (Validation)}
\label{Fig: fmnist_conv_lamda_pm_mclr_loss}
\end{subfigure}
\begin{subfigure}[!t]{0.24\linewidth}
\resizebox{\linewidth}{!}{
%

}
\caption{Global Accuracy \\ (Validation)}
\label{Fig: fmnist_conv_lamda_gm_mclr_acc}
\end{subfigure} 
\begin{subfigure}[!t]{0.24\linewidth}
\resizebox{\linewidth}{!}{
%
}

\caption{Global Loss \\ (Validation)}
\label{Fig: fmnist_conv_lamda_gm_mclr_loss}
\end{subfigure}

\caption{Effect of $\lambda$ on convergence of \ourmodel{} in strongly convex settings (MCLR) using FMNIST dataset}
\label{Fig:fmnist_conv_lamda_mclr}
\end{figure*}



\begin{figure*}[!ht]
\centering
\begin{subfigure}[!t]{0.24\linewidth}
\resizebox{\linewidth}{!}{
%
}
\caption{Personalized Accuracy \\ (Validation)}
\label{Fig: synthetic_conv_lamda_pm_cnn_acc}
\end{subfigure} 
\begin{subfigure}[!t]{0.24\linewidth}
\resizebox{\linewidth}{!}{
%
}

\caption{Personalized Loss \\ (Validation)}
\label{Fig: synthetic_conv_lamda_pm_cnn_loss}
\end{subfigure}
\begin{subfigure}[!t]{0.24\linewidth}
\resizebox{\linewidth}{!}{
%

}
\caption{Global Accuracy \\ (Validation)}
\label{Fig: synthetic_conv_lamda_gm_cnn_acc}
\end{subfigure} 
\begin{subfigure}[!t]{0.24\linewidth}
\resizebox{\linewidth}{!}{
%
}

\caption{Global Loss \\ (Validation)}
\label{Fig: synthetic_conv_lamda_gm_cnn_loss}
\end{subfigure}

\caption{Effect of $\lambda$ on convergence of \ourmodel{} in non-convex settings (DNN) using Synthetic dataset}
\label{Fig:synthetic_conv_lamda_cnn}
\end{figure*}


\begin{figure*}[!ht]
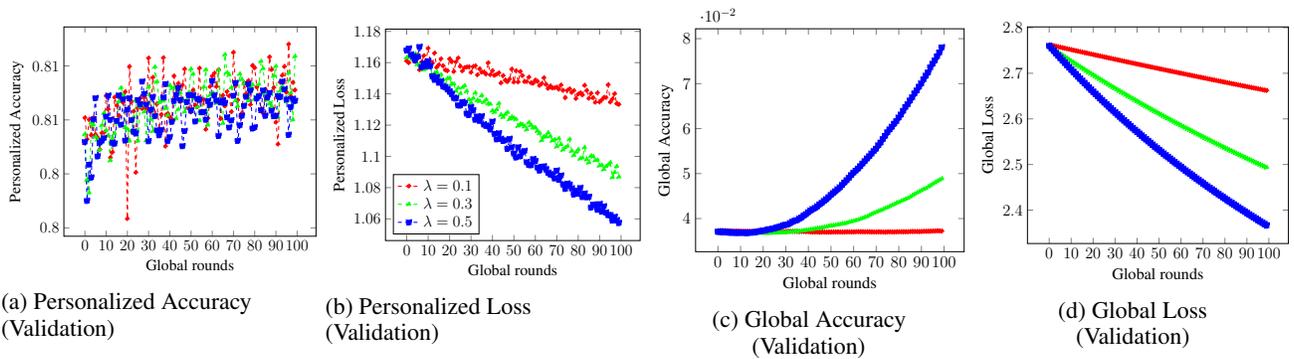

\centering
\begin{subfigure}[!t]{0.24\linewidth}
\resizebox{\linewidth}{!}{
%
}
\caption{Personalized Accuracy \\ (Validation)}
\label{Fig: synthetic_conv_lamda_pm_mclr_acc}
\end{subfigure} 
\begin{subfigure}[!t]{0.24\linewidth}
\resizebox{\linewidth}{!}{
%
}

\caption{Personalized Loss \\ (Validation)}
\label{Fig: synthetic_conv_lamda_pm_mclr_loss}
\end{subfigure}
\begin{subfigure}[!t]{0.24\linewidth}
\resizebox{\linewidth}{!}{
%

}
\caption{Global Accuracy \\ (Validation)}
\label{Fig: synthetic_conv_lamda_gm_mclr_acc}
\end{subfigure} 
\begin{subfigure}[!t]{0.24\linewidth}
\resizebox{\linewidth}{!}{
%
}

\caption{Global Loss \\ (Validation)}
\label{Fig: synthetic_conv_lamda_gm_mclr_loss}
\end{subfigure}

\caption{Effect of $\lambda$ on convergence of \ourmodel{} in strongly convex settings (MCLR) using Synthetic dataset}
\label{Fig:synthetic_conv_lamda_mclr}
\end{figure*}


\clearpage
\subsection{Ablation study on the effect of teams and devices participation on \ourmodel{}} \label{appx: ablation_study_team}

 In a multi-tier architecture, teams and devices can be constituted in four ways: (1) Both teams and devices within teams have full participation (see \Cref{Fig: mnist_ftfc_cnn} to \Cref{Fig: synthetic_ftfc_mclr}), (2) Teams have full participation, but devices within teams are partially participating (see \Cref{Fig: FMnist_ftpc_team_5} to \Cref{Fig: Emnist_ftpc_team_5}), (3) Teams have partial participation, but all devices within teams are participating (see \Cref{Fig: emnist_ptfc_mclr} to \Cref{Fig: fmnist_ptfc_cnn}), and (4) Teams and devices both have partial participation (see \Cref{Fig: emnist_ptpc_team_2_mclr} to \Cref{Fig: fmnist_ptpc_team_20_mclr}). In these experiments, we observed how well \ourmodel{} performs and converges in various team combinations. We are also studying how \ourmodel{} behaves when the number of participating teams and devices is limited. 

\subsubsection{Full participation of Teams and Devices}
We conducted experiments in both convex (see \Cref{Fig: mnist_ftfc_mclr}, \Cref{Fig: fmnist_ftfc_mclr}, \Cref{Fig: emnist_ftfc_mclr}, \Cref{Fig: synthetic_ftfc_mclr} ) and non-convex (See \Cref{Fig: mnist_ftfc_cnn}, \Cref{Fig: fmnist_ftfc_cnn}, \Cref{Fig: emnist_ftfc_cnn}, \Cref{Fig: synthetic_ftfc_cnn}) settings using MNIST, FMNIST, EMNIST, and Synthetic datasets. Our findings indicate that when it comes to team and device selection strategies, having full participation of teams and devices yields the best results compared to the other three types of participation strategies. Increasing the number of teams does not negatively impact the performance of the personalized model. However, in some cases, such as those depicted in \Cref{Fig: synthetic_ftfc_dnn_pl} and \Cref{Fig: synthetic_ftfc_dnn_gl}, increasing the number of teams can lead to a decrease in the convergence of the global model.  


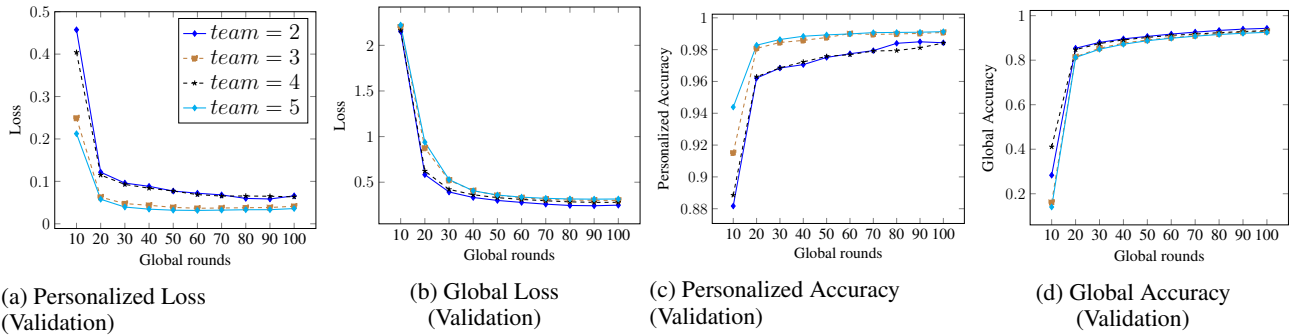
\begin{figure*}[!ht]
        \centering
        \begin{subfigure}[!t]{0.24\linewidth}
            \resizebox{\linewidth}{!}{
                \begin{tikzpicture}
                \tikzstyle{every node}=[font=\fontsize{12}{12}\selectfont]
\begin{axis}[
  xlabel= Global rounds,
  ylabel= Loss,
  legend style ={nodes={scale=1.4, transform shape}},
  legend pos=north east,
  xticklabels from table={ablation_study_team_Mnist_Full_team_full_client_cnn_loss.dat}{GR},xtick=data
]
\addplot[blue,thick,mark=diamond*] table [y=pteam_2,x=GR]{ablation_study_team_Mnist_Full_team_full_client_cnn_loss.dat};
\addlegendentry{$team = 2$}

\addplot[brown,dashed,thick,mark=square*]  table [y=pteam_3,x=GR]{ablation_study_team_Mnist_Full_team_full_client_cnn_loss.dat};
\addlegendentry{$ team = 3$}

\addplot[black,dashed,thick,mark=star]  table [y=pteam_4,x=GR]{ablation_study_team_Mnist_Full_team_full_client_cnn_loss.dat};
\addlegendentry{$ team = 4$}

\addplot[cyan,thick,mark=diamond*] table [y=pteam_5,x=GR]{ablation_study_team_Mnist_Full_team_full_client_cnn_loss.dat};
\addlegendentry{$ team = 5$}

\end{axis}
\end{tikzpicture}

            }
        \caption{Personalized Loss \\ (Validation)}
        \label{Fig: mnist_ftfc_cnn_pl}
    \end{subfigure} 
    \begin{subfigure}[!t]{0.24\linewidth}
        \resizebox{\linewidth}{!}{
        \begin{tikzpicture}
        \tikzstyle{every node}=[font=\fontsize{12}{12}\selectfont]

\begin{axis}[
  xlabel= Global rounds,
  ylabel= Loss,
  legend pos=north east,
  xticklabels from table={ablation_study_team_Mnist_Full_team_full_client_cnn_loss.dat}{GR},xtick=data
]
\addplot[blue,thick,mark=diamond*] table [y=gteam_2,x=GR]{ablation_study_team_Mnist_Full_team_full_client_cnn_loss.dat};

\addplot[brown,dashed,thick,mark=square*]  table [y=gteam_3,x=GR]{ablation_study_team_Mnist_Full_team_full_client_cnn_loss.dat};

\addplot[black,dashed,thick,mark=star]  table [y=gteam_4,x=GR]{ablation_study_team_Mnist_Full_team_full_client_cnn_loss.dat};

\addplot[cyan,thick,mark=diamond*] table [y=gteam_5,x=GR]{ablation_study_team_Mnist_Full_team_full_client_cnn_loss.dat};

\end{axis}

\end{tikzpicture}
}

    \caption{Global Loss \\ (Validation)}
    \label{Fig: mnist_ftfc_cnn_gl}
\end{subfigure}
\begin{subfigure}[!t]{0.24\linewidth}
            \resizebox{\linewidth}{!}{
                \begin{tikzpicture}
                \tikzstyle{every node}=[font=\fontsize{12}{12}\selectfont]

\begin{axis}[
  xlabel= Global rounds,
  ylabel= Personalized Accuracy,
  legend pos=south east,
  xticklabels from table={ablation_study_team_Mnist_Full_team_full_client_cnn_acc.dat}{GR},xtick=data
]
\addplot[blue,thick,mark=diamond*] table [y=pteam_2,x=GR]{ablation_study_team_Mnist_Full_team_full_client_cnn_acc.dat};

\addplot[brown,dashed,thick,mark=square*]  table [y=pteam_3,x=GR]{ablation_study_team_Mnist_Full_team_full_client_cnn_acc.dat};

\addplot[black,dashed,thick,mark=star]  table [y=pteam_4,x=GR]{ablation_study_team_Mnist_Full_team_full_client_cnn_acc.dat};

\addplot[cyan,thick,mark=diamond*] table [y=pteam_5,x=GR]{ablation_study_team_Mnist_Full_team_full_client_cnn_acc.dat};

\end{axis}

\end{tikzpicture}

            }
        \caption{Personalized Accuracy \\ (Validation)}
        \label{Fig: mnist_ftfc_cnn_pa}
    \end{subfigure} 
    \begin{subfigure}[!t]{0.24\linewidth}
        \resizebox{\linewidth}{!}{
        \begin{tikzpicture}
        \tikzstyle{every node}=[font=\fontsize{12}{12}\selectfont]
\begin{axis}[
  xlabel= Global rounds,
  ylabel= Global Accuracy,
  legend pos=south east,
  xticklabels from table={ablation_study_team_Mnist_Full_team_full_client_cnn_acc.dat}{GR},xtick=data
]
\addplot[blue,thick,mark=diamond*] table [y=gteam_2,x=GR]{ablation_study_team_Mnist_Full_team_full_client_cnn_acc.dat};

\addplot[brown,dashed,thick,mark=square*]  table [y=gteam_3,x=GR]{ablation_study_team_Mnist_Full_team_full_client_cnn_acc.dat};

\addplot[black,dashed,thick,mark=star]  table [y=gteam_4,x=GR]{ablation_study_team_Mnist_Full_team_full_client_cnn_acc.dat};

\addplot[cyan,thick,mark=diamond*] table [y=gteam_5,x=GR]{ablation_study_team_Mnist_Full_team_full_client_cnn_acc.dat};

\end{axis}

\end{tikzpicture}
}

    \caption{Global Accuracy \\ (Validation)}
    \label{Fig: mnist_ftfc_cnn_ga}
\end{subfigure}

\caption{Full participation teams and devices on MNIST datasets in non-convex settings (CNN)}
\label{Fig: mnist_ftfc_cnn}
\end{figure*}



\begin{figure*}[!ht]
        \centering
        \begin{subfigure}[!t]{0.24\linewidth}
            \resizebox{\linewidth}{!}{
                \begin{tikzpicture}
                \tikzstyle{every node}=[font=\fontsize{12}{12}\selectfont]
\begin{axis}[
  xlabel= Global rounds,
  ylabel= Loss,
  legend style ={nodes={scale=1.4, transform shape}},
  legend pos=north east,
  xticklabels from table={ablation_study_team_Mnist_Full_team_full_client_mclr_loss.dat}{GR},xtick=data
]
\addplot[blue,thick,mark=diamond*] table [y=pteam_2,x=GR]{ablation_study_team_Mnist_Full_team_full_client_mclr_loss.dat};
\addlegendentry{$team = 2$}

\addplot[brown,dashed,thick,mark=square*]  table [y=pteam_3,x=GR]{ablation_study_team_Mnist_Full_team_full_client_mclr_loss.dat};
\addlegendentry{$ team = 3$}

\addplot[black,dashed,thick,mark=star]  table [y=pteam_4,x=GR]{ablation_study_team_Mnist_Full_team_full_client_mclr_loss.dat};
\addlegendentry{$ team = 4$}

\addplot[cyan,thick,mark=diamond*] table [y=pteam_5,x=GR]{ablation_study_team_Mnist_Full_team_full_client_mclr_loss.dat};
\addlegendentry{$ team = 5$}

\end{axis}
\end{tikzpicture}

            }
        \caption{Personalized Loss \\ (Validation)}
        \label{Fig: mnist_ftfc_mclr_pl}
    \end{subfigure} 
    \begin{subfigure}[!t]{0.24\linewidth}
        \resizebox{\linewidth}{!}{
        \begin{tikzpicture}
        \tikzstyle{every node}=[font=\fontsize{12}{12}\selectfont]

\begin{axis}[
  xlabel= Global rounds,
  ylabel= Loss,
  legend pos=north east,
  xticklabels from table={ablation_study_team_Mnist_Full_team_full_client_mclr_loss.dat}{GR},xtick=data
]
\addplot[blue,thick,mark=diamond*] table [y=gteam_2,x=GR]{ablation_study_team_Mnist_Full_team_full_client_mclr_loss.dat};

\addplot[brown,dashed,thick,mark=square*]  table [y=gteam_3,x=GR]{ablation_study_team_Mnist_Full_team_full_client_mclr_loss.dat};

\addplot[black,dashed,thick,mark=star]  table [y=gteam_4,x=GR]{ablation_study_team_Mnist_Full_team_full_client_mclr_loss.dat};

\addplot[cyan,thick,mark=diamond*] table [y=gteam_5,x=GR]{ablation_study_team_Mnist_Full_team_full_client_mclr_loss.dat};

\end{axis}

\end{tikzpicture}
}

    \caption{Global Loss \\ (Validation)}
    \label{Fig: mnist_ftfc_mclr_gl}
\end{subfigure}
\begin{subfigure}[!t]{0.24\linewidth}
            \resizebox{\linewidth}{!}{
                \begin{tikzpicture}
                \tikzstyle{every node}=[font=\fontsize{12}{12}\selectfont]

\begin{axis}[
  xlabel= Global rounds,
  ylabel= Personalized Accuracy,
  legend pos=south east,
  xticklabels from table={ablation_study_team_Mnist_Full_team_full_client_mclr_acc.dat}{GR},xtick=data
]
\addplot[blue,thick,mark=diamond*] table [y=pteam_2,x=GR]{ablation_study_team_Mnist_Full_team_full_client_mclr_acc.dat};

\addplot[brown,dashed,thick,mark=square*]  table [y=pteam_3,x=GR]{ablation_study_team_Mnist_Full_team_full_client_mclr_acc.dat};

\addplot[black,dashed,thick,mark=star]  table [y=pteam_4,x=GR]{ablation_study_team_Mnist_Full_team_full_client_mclr_acc.dat};

\addplot[cyan,thick,mark=diamond*] table [y=pteam_5,x=GR]{ablation_study_team_Mnist_Full_team_full_client_mclr_acc.dat};

\end{axis}

\end{tikzpicture}

            }
        \caption{Personalized Accuracy \\ (Validation)}
        \label{Fig: mnist_ftfc_mclr_pa}
    \end{subfigure} 
    \begin{subfigure}[!t]{0.24\linewidth}
        \resizebox{\linewidth}{!}{
        \begin{tikzpicture}
        \tikzstyle{every node}=[font=\fontsize{12}{12}\selectfont]
\begin{axis}[
  xlabel= Global rounds,
  ylabel= Global Accuracy,
  legend pos=south east,
  xticklabels from table={ablation_study_team_Mnist_Full_team_full_client_mclr_acc.dat}{GR},xtick=data
]
\addplot[blue,thick,mark=diamond*] table [y=gteam_2,x=GR]{ablation_study_team_Mnist_Full_team_full_client_mclr_acc.dat};

\addplot[brown,dashed,thick,mark=square*]  table [y=gteam_3,x=GR]{ablation_study_team_Mnist_Full_team_full_client_mclr_acc.dat};

\addplot[black,dashed,thick,mark=star]  table [y=gteam_4,x=GR]{ablation_study_team_Mnist_Full_team_full_client_mclr_acc.dat};

\addplot[cyan,thick,mark=diamond*] table [y=gteam_5,x=GR]{ablation_study_team_Mnist_Full_team_full_client_mclr_acc.dat};

\end{axis}

\end{tikzpicture}
}

    \caption{Global Accuracy \\ (Validation)}
    \label{Fig: mnist_ftfc_mclr_ga}
\end{subfigure}

\caption{Full participation teams and devices on MNIST datasets in convex settings (MCLR)}
\label{Fig: mnist_ftfc_mclr}
\end{figure*}
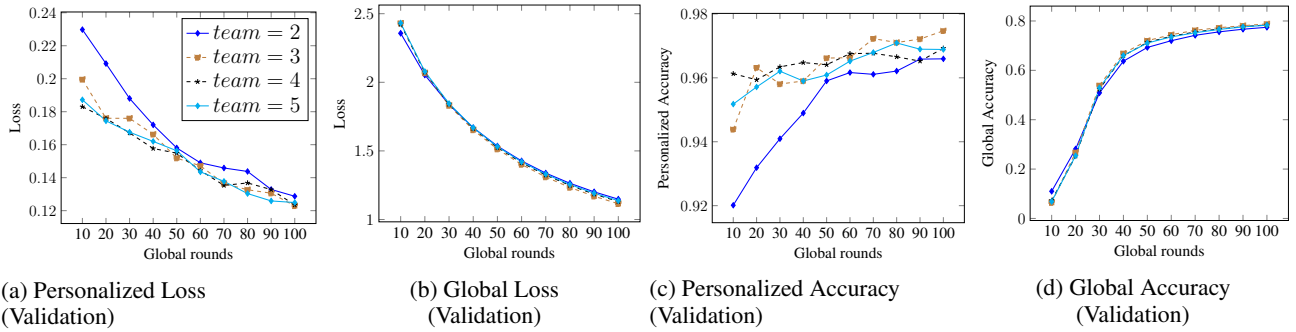




\begin{figure*}[!ht]
        \centering
        \begin{subfigure}[!t]{0.24\linewidth}
            \resizebox{\linewidth}{!}{
                \begin{tikzpicture}
                \tikzstyle{every node}=[font=\fontsize{12}{12}\selectfont]
\begin{axis}[
  xlabel= Global rounds,
  ylabel= Loss,
  legend pos=north east,
  xticklabels from table={ablation_study_team_FMnist_Full_team_full_client_mclr_loss.dat}{GR},xtick=data
]
\addplot[blue,thick,mark=diamond*] table [y=pteam_2,x=GR]{ablation_study_team_FMnist_Full_team_full_client_mclr_loss.dat};

\addplot[brown,dashed,thick,mark=square*]  table [y=pteam_3,x=GR]{ablation_study_team_FMnist_Full_team_full_client_cnn_loss.dat};

\addplot[black,dashed,thick,mark=star]  table [y=pteam_4,x=GR]{ablation_study_team_FMnist_Full_team_full_client_cnn_loss.dat};

\addplot[cyan,thick,mark=diamond*] table [y=pteam_5,x=GR]{ablation_study_team_FMnist_Full_team_full_client_cnn_loss.dat};

\end{axis}
\end{tikzpicture}

            }
        \caption{Personalized Loss \\ (Validation)}
        \label{Fig: fmnist_ftfc_cnn_pl}
    \end{subfigure} 
    \begin{subfigure}[!t]{0.24\linewidth}
        \resizebox{\linewidth}{!}{
        \begin{tikzpicture}
        \tikzstyle{every node}=[font=\fontsize{12}{12}\selectfont]

\begin{axis}[
  xlabel= Global rounds,
  ylabel= Loss,
  legend pos=north east,
  legend style ={nodes={scale=1.4, transform shape}},
  xticklabels from table={ablation_study_team_FMnist_Full_team_full_client_cnn_loss.dat}{GR},xtick=data
]
\addplot[blue,thick,mark=diamond*] table [y=gteam_2,x=GR]{ablation_study_team_FMnist_Full_team_full_client_cnn_loss.dat};
\addlegendentry{$ team = 2$}

\addplot[brown,dashed,thick,mark=square*]  table [y=gteam_3,x=GR]{ablation_study_team_FMnist_Full_team_full_client_cnn_loss.dat};
\addlegendentry{$ team = 3$}

\addplot[black,dashed,thick,mark=star]  table [y=gteam_4,x=GR]{ablation_study_team_FMnist_Full_team_full_client_cnn_loss.dat};
\addlegendentry{$ team = 4$}

\addplot[cyan,thick,mark=diamond*] table [y=gteam_5,x=GR]{ablation_study_team_FMnist_Full_team_full_client_cnn_loss.dat};
\addlegendentry{$ team = 5$}

\end{axis}

\end{tikzpicture}
}

    \caption{Global Loss \\ (Validation)}
    \label{Fig: fmnist_ftfc_cnn_gl}
\end{subfigure}
\begin{subfigure}[!t]{0.24\linewidth}
            \resizebox{\linewidth}{!}{
                \begin{tikzpicture}
                \tikzstyle{every node}=[font=\fontsize{12}{12}\selectfont]

\begin{axis}[
  xlabel= Global rounds,
  ylabel= Personalized Accuracy,
  legend pos=south east,
  xticklabels from table={ablation_study_team_FMnist_Full_team_full_client_cnn_acc.dat}{GR},xtick=data
]
\addplot[blue,thick,mark=diamond*] table [y=pteam_2,x=GR]{ablation_study_team_FMnist_Full_team_full_client_cnn_acc.dat};

\addplot[brown,dashed,thick,mark=square*]  table [y=pteam_3,x=GR]{ablation_study_team_FMnist_Full_team_full_client_cnn_acc.dat};

\addplot[black,dashed,thick,mark=star]  table [y=pteam_4,x=GR]{ablation_study_team_FMnist_Full_team_full_client_cnn_acc.dat};

\addplot[cyan,thick,mark=diamond*] table [y=pteam_5,x=GR]{ablation_study_team_FMnist_Full_team_full_client_cnn_acc.dat};

\end{axis}

\end{tikzpicture}

            }
        \caption{Personalized Accuracy \\ (Validation)}
        \label{Fig: fmnist_ftfc_cnn_pa}
    \end{subfigure} 
    \begin{subfigure}[!t]{0.24\linewidth}
        \resizebox{\linewidth}{!}{
        \begin{tikzpicture}
        \tikzstyle{every node}=[font=\fontsize{12}{12}\selectfont]
\begin{axis}[
  xlabel= Global rounds,
  ylabel= Global Accuracy,
  legend pos=south east,
  xticklabels from table={ablation_study_team_FMnist_Full_team_full_client_cnn_acc.dat}{GR},xtick=data
]
\addplot[blue,thick,mark=diamond*] table [y=gteam_2,x=GR]{ablation_study_team_FMnist_Full_team_full_client_cnn_acc.dat};

\addplot[brown,dashed,thick,mark=square*]  table [y=gteam_3,x=GR]{ablation_study_team_FMnist_Full_team_full_client_cnn_acc.dat};

\addplot[black,dashed,thick,mark=star]  table [y=gteam_4,x=GR]{ablation_study_team_FMnist_Full_team_full_client_cnn_acc.dat};

\addplot[cyan,thick,mark=diamond*] table [y=gteam_5,x=GR]{ablation_study_team_FMnist_Full_team_full_client_cnn_acc.dat};

\end{axis}

\end{tikzpicture}
}

    \caption{Global Accuracy \\ (Validation)}
    \label{Fig: fmnist_ftfc_cnn_ga}
\end{subfigure}

\caption{Full participation teams and devices on FMNIST datasets in non-convex settings (CNN)}
\label{Fig: fmnist_ftfc_cnn}
\end{figure*}
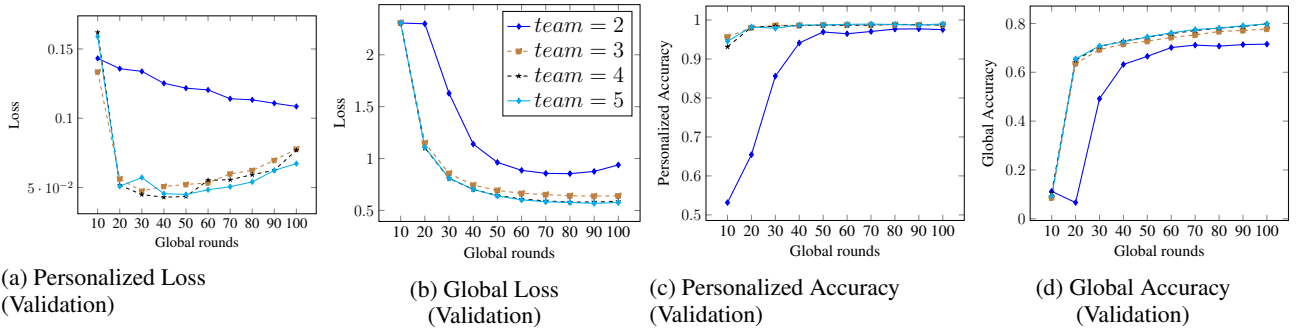



\begin{figure*}[!ht]
        \centering
        \begin{subfigure}[!t]{0.24\linewidth}
            \resizebox{\linewidth}{!}{
                \begin{tikzpicture}
                \tikzstyle{every node}=[font=\fontsize{12}{12}\selectfont]
\begin{axis}[
  xlabel= Global rounds,
  ylabel= Loss,
  legend pos=north east,
  legend style ={nodes={scale=1.4, transform shape}},
  xticklabels from table={ablation_study_team_FMnist_Full_team_full_client_mclr_loss.dat}{GR},xtick=data
]
\addplot[blue,thick,mark=diamond*] table [y=pteam_2,x=GR]{ablation_study_team_FMnist_Full_team_full_client_mclr_loss.dat};
\addlegendentry{$team = 2$}

\addplot[brown,dashed,thick,mark=square*]  table [y=pteam_3,x=GR]{ablation_study_team_FMnist_Full_team_full_client_mclr_loss.dat};
\addlegendentry{$ team = 3$}

\addplot[black,dashed,thick,mark=star]  table [y=pteam_4,x=GR]{ablation_study_team_FMnist_Full_team_full_client_mclr_loss.dat};
\addlegendentry{$ team = 4$}

\addplot[cyan,thick,mark=diamond*] table [y=pteam_5,x=GR]{ablation_study_team_FMnist_Full_team_full_client_mclr_loss.dat};
\addlegendentry{$ team = 5$}

\end{axis}
\end{tikzpicture}

            }
        \caption{Personalized Loss \\ (Validation)}
        \label{Fig: fmnist_ftfc_mclr_pl}
    \end{subfigure} 
    \begin{subfigure}[!t]{0.24\linewidth}
        \resizebox{\linewidth}{!}{
        \begin{tikzpicture}
        \tikzstyle{every node}=[font=\fontsize{12}{12}\selectfont]

\begin{axis}[
  xlabel= Global rounds,
  ylabel= Loss,
  legend pos=north east,
  xticklabels from table={ablation_study_team_FMnist_Full_team_full_client_mclr_loss.dat}{GR},xtick=data
]
\addplot[blue,thick,mark=diamond*] table [y=gteam_2,x=GR]{ablation_study_team_FMnist_Full_team_full_client_mclr_loss.dat};

\addplot[brown,dashed,thick,mark=square*]  table [y=gteam_3,x=GR]{ablation_study_team_FMnist_Full_team_full_client_mclr_loss.dat};

\addplot[black,dashed,thick,mark=star]  table [y=gteam_4,x=GR]{ablation_study_team_FMnist_Full_team_full_client_mclr_loss.dat};

\addplot[cyan,thick,mark=diamond*] table [y=gteam_5,x=GR]{ablation_study_team_FMnist_Full_team_full_client_mclr_loss.dat};

\end{axis}

\end{tikzpicture}
}

    \caption{Global Loss \\ (Validation)}
    \label{Fig: fmnist_ftfc_mclr_gl}
\end{subfigure}
\begin{subfigure}[!t]{0.24\linewidth}
            \resizebox{\linewidth}{!}{
                \begin{tikzpicture}
                \tikzstyle{every node}=[font=\fontsize{12}{12}\selectfont]

\begin{axis}[
  xlabel= Global rounds,
  ylabel= Personalized Accuracy,
  legend pos=south east,
  xticklabels from table={ablation_study_team_FMnist_Full_team_full_client_mclr_acc.dat}{GR},xtick=data
]
\addplot[blue,thick,mark=diamond*] table [y=pteam_2,x=GR]{ablation_study_team_FMnist_Full_team_full_client_mclr_acc.dat};

\addplot[brown,dashed,thick,mark=square*]  table [y=pteam_3,x=GR]{ablation_study_team_FMnist_Full_team_full_client_mclr_acc.dat};

\addplot[black,dashed,thick,mark=star]  table [y=pteam_4,x=GR]{ablation_study_team_FMnist_Full_team_full_client_mclr_acc.dat};

\addplot[cyan,thick,mark=diamond*] table [y=pteam_5,x=GR]{ablation_study_team_FMnist_Full_team_full_client_mclr_acc.dat};

\end{axis}

\end{tikzpicture}

            }
        \caption{Personalized Accuracy \\ (Validation)}
        \label{Fig: fmnist_ftfc_mclr_pa}
    \end{subfigure} 
    \begin{subfigure}[!t]{0.24\linewidth}
        \resizebox{\linewidth}{!}{
        \begin{tikzpicture}
        \tikzstyle{every node}=[font=\fontsize{12}{12}\selectfont]
\begin{axis}[
  xlabel= Global rounds,
  ylabel= Global Accuracy,
  legend pos=south east,
  xticklabels from table={ablation_study_team_FMnist_Full_team_full_client_mclr_acc.dat}{GR},xtick=data
]
\addplot[blue,thick,mark=diamond*] table [y=gteam_2,x=GR]{ablation_study_team_FMnist_Full_team_full_client_mclr_acc.dat};

\addplot[brown,dashed,thick,mark=square*]  table [y=gteam_3,x=GR]{ablation_study_team_FMnist_Full_team_full_client_mclr_acc.dat};

\addplot[black,dashed,thick,mark=star]  table [y=gteam_4,x=GR]{ablation_study_team_FMnist_Full_team_full_client_mclr_acc.dat};

\addplot[cyan,thick,mark=diamond*] table [y=gteam_5,x=GR]{ablation_study_team_FMnist_Full_team_full_client_mclr_acc.dat};

\end{axis}

\end{tikzpicture}
}

    \caption{Global Accuracy \\ (Validation)}
    \label{Fig: fmnist_ftfc_mclr_ga}
\end{subfigure}

\caption{Full participation teams and devices on FMNIST datasets in convex settings (MCLR)}
\label{Fig: fmnist_ftfc_mclr}
\end{figure*}
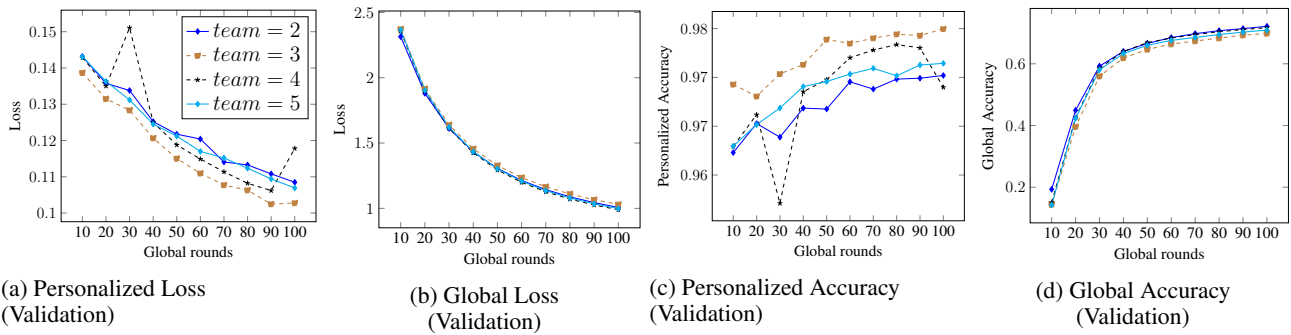




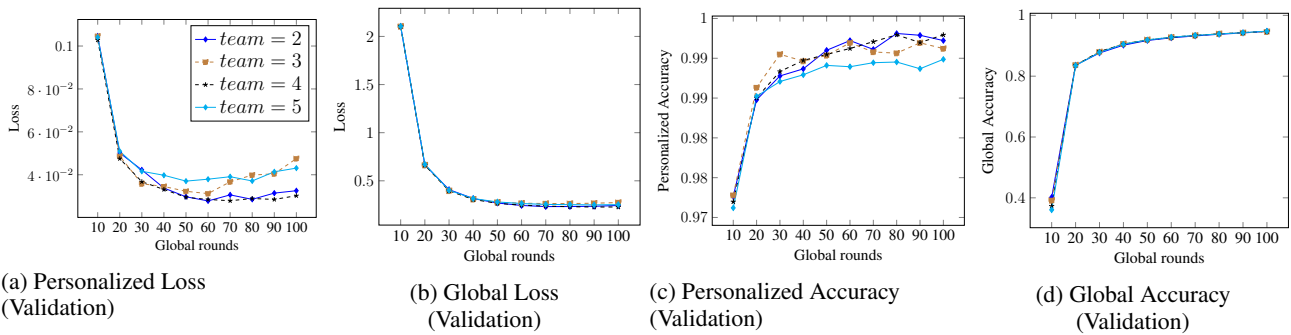
\begin{figure*}[!ht]
        \centering
        \begin{subfigure}[!t]{0.24\linewidth}
            \resizebox{\linewidth}{!}{
                \begin{tikzpicture}
                \tikzstyle{every node}=[font=\fontsize{12}{12}\selectfont]
\begin{axis}[
  xlabel= Global rounds,
  ylabel= Loss,
  legend pos=north east,
  legend style ={nodes={scale=1.4, transform shape}},
  xticklabels from table={ablation_study_team_Emnist_Full_team_full_client_cnn_loss.dat}{GR},xtick=data
]
\addplot[blue,thick,mark=diamond*] table [y=pteam_2,x=GR]{ablation_study_team_Emnist_Full_team_full_client_cnn_loss.dat};
\addlegendentry{$team = 2$}

\addplot[brown,dashed,thick,mark=square*]  table [y=pteam_3,x=GR]{ablation_study_team_Emnist_Full_team_full_client_cnn_loss.dat};
\addlegendentry{$ team = 3$}

\addplot[black,dashed,thick,mark=star]  table [y=pteam_4,x=GR]{ablation_study_team_Emnist_Full_team_full_client_cnn_loss.dat};
\addlegendentry{$ team = 4$}

\addplot[cyan,thick,mark=diamond*] table [y=pteam_5,x=GR]{ablation_study_team_Emnist_Full_team_full_client_cnn_loss.dat};
\addlegendentry{$ team = 5$}

\end{axis}
\end{tikzpicture}

            }
        \caption{Personalized Loss \\ (Validation)}
        \label{Fig: Emnist_ftfc_cnn_pl}
    \end{subfigure} 
    \begin{subfigure}[!t]{0.24\linewidth}
        \resizebox{\linewidth}{!}{
        \begin{tikzpicture}
        \tikzstyle{every node}=[font=\fontsize{12}{12}\selectfont]

\begin{axis}[
  xlabel= Global rounds,
  ylabel= Loss,
  legend pos=north east,
  xticklabels from table={ablation_study_team_Emnist_Full_team_full_client_cnn_loss.dat}{GR},xtick=data
]
\addplot[blue,thick,mark=diamond*] table [y=gteam_2,x=GR]{ablation_study_team_Emnist_Full_team_full_client_cnn_loss.dat};

\addplot[brown,dashed,thick,mark=square*]  table [y=gteam_3,x=GR]{ablation_study_team_Emnist_Full_team_full_client_cnn_loss.dat};

\addplot[black,dashed,thick,mark=star]  table [y=gteam_4,x=GR]{ablation_study_team_Emnist_Full_team_full_client_cnn_loss.dat};

\addplot[cyan,thick,mark=diamond*] table [y=gteam_5,x=GR]{ablation_study_team_Emnist_Full_team_full_client_cnn_loss.dat};

\end{axis}

\end{tikzpicture}
}

    \caption{Global Loss \\ (Validation)}
    \label{Fig: Emnist_ftfc_cnn_gl}
\end{subfigure}
\begin{subfigure}[!t]{0.24\linewidth}
            \resizebox{\linewidth}{!}{
                \begin{tikzpicture}
                \tikzstyle{every node}=[font=\fontsize{12}{12}\selectfont]

\begin{axis}[
  xlabel= Global rounds,
  ylabel= Personalized Accuracy,
  legend pos=south east,
  xticklabels from table={ablation_study_team_Emnist_Full_team_full_client_cnn_acc.dat}{GR},xtick=data
]
\addplot[blue,thick,mark=diamond*] table [y=pteam_2,x=GR]{ablation_study_team_Emnist_Full_team_full_client_cnn_acc.dat};

\addplot[brown,dashed,thick,mark=square*]  table [y=pteam_3,x=GR]{ablation_study_team_Emnist_Full_team_full_client_cnn_acc.dat};

\addplot[black,dashed,thick,mark=star]  table [y=pteam_4,x=GR]{ablation_study_team_Emnist_Full_team_full_client_cnn_acc.dat};

\addplot[cyan,thick,mark=diamond*] table [y=pteam_5,x=GR]{ablation_study_team_Emnist_Full_team_full_client_cnn_acc.dat};

\end{axis}

\end{tikzpicture}

            }
        \caption{Personalized Accuracy \\ (Validation)}
        \label{Fig: Emnist_ftfc_cnn_pa}
    \end{subfigure} 
    \begin{subfigure}[!t]{0.24\linewidth}
        \resizebox{\linewidth}{!}{
        \begin{tikzpicture}
        \tikzstyle{every node}=[font=\fontsize{12}{12}\selectfont]
\begin{axis}[
  xlabel= Global rounds,
  ylabel= Global Accuracy,
  legend pos=south east,
  xticklabels from table={ablation_study_team_Emnist_Full_team_full_client_cnn_acc.dat}{GR},xtick=data
]
\addplot[blue,thick,mark=diamond*] table [y=gteam_2,x=GR]{ablation_study_team_Emnist_Full_team_full_client_cnn_acc.dat};

\addplot[brown,dashed,thick,mark=square*]  table [y=gteam_3,x=GR]{ablation_study_team_Emnist_Full_team_full_client_cnn_acc.dat};

\addplot[black,dashed,thick,mark=star]  table [y=gteam_4,x=GR]{ablation_study_team_Emnist_Full_team_full_client_cnn_acc.dat};

\addplot[cyan,thick,mark=diamond*] table [y=gteam_5,x=GR]{ablation_study_team_Emnist_Full_team_full_client_cnn_acc.dat};

\end{axis}

\end{tikzpicture}
}

    \caption{Global Accuracy \\ (Validation)}
    \label{Fig: emnist_ftfc_cnn_ga}
\end{subfigure}

\caption{Full participation teams and devices on EMNIST datasets in non-convex settings (CNN)}
\label{Fig: emnist_ftfc_cnn}
\end{figure*}



\begin{figure*}[!ht]
        \centering
        \begin{subfigure}[!t]{0.24\linewidth}
            \resizebox{\linewidth}{!}{
                \begin{tikzpicture}
                \tikzstyle{every node}=[font=\fontsize{12}{12}\selectfont]
\begin{axis}[
  xlabel= Global rounds,
  ylabel= Loss,
  legend pos=north east,
  legend style ={nodes={scale=1.4, transform shape}},
  xticklabels from table={ablation_study_team_Emnist_Full_team_full_client_mclr_loss.dat}{GR},xtick=data
]

\addplot[brown,dashed,thick,mark=square*]  table [y=pteam_3,x=GR]{ablation_study_team_Emnist_Full_team_full_client_mclr_loss.dat};
\addlegendentry{$ team = 3$}

\addplot[black,dashed,thick,mark=star]  table [y=pteam_4,x=GR]{ablation_study_team_Emnist_Full_team_full_client_mclr_loss.dat};
\addlegendentry{$ team = 4$}

\addplot[cyan,thick,mark=diamond*] table [y=pteam_5,x=GR]{ablation_study_team_Emnist_Full_team_full_client_mclr_loss.dat};
\addlegendentry{$ team = 5$}

\addplot[blue,thick,mark=diamond*] table [y=pteam_10,x=GR]{ablation_study_team_Emnist_Full_team_full_client_mclr_loss.dat};
\addlegendentry{$team = 10$}

\end{axis}
\end{tikzpicture}

            }
        \caption{Personalized Loss \\ (Validation)}
        \label{Fig: emnist_ftfc_mclr_pl}
    \end{subfigure} 
    \begin{subfigure}[!t]{0.24\linewidth}
        \resizebox{\linewidth}{!}{
        \begin{tikzpicture}
        \tikzstyle{every node}=[font=\fontsize{12}{12}\selectfont]

\begin{axis}[
  xlabel= Global rounds,
  ylabel= Loss,
  legend pos=north east,
  xticklabels from table={ablation_study_team_Emnist_Full_team_full_client_mclr_loss.dat}{GR},xtick=data
]

\addplot[brown,dashed,thick,mark=square*]  table [y=gteam_3,x=GR]{ablation_study_team_Emnist_Full_team_full_client_mclr_loss.dat};

\addplot[black,dashed,thick,mark=star]  table [y=gteam_4,x=GR]{ablation_study_team_Emnist_Full_team_full_client_mclr_loss.dat};

\addplot[cyan,thick,mark=diamond*] table [y=gteam_5,x=GR]{ablation_study_team_Emnist_Full_team_full_client_mclr_loss.dat};

\addplot[blue,thick,mark=diamond*] table [y=gteam_10,x=GR]{ablation_study_team_Emnist_Full_team_full_client_mclr_loss.dat};

\end{axis}

\end{tikzpicture}
}

    \caption{Global Loss \\ (Validation)}
    
    \label{Fig: emnist_ftfc_mclr_gl}
\end{subfigure}
\begin{subfigure}[!t]{0.24\linewidth}
            \resizebox{\linewidth}{!}{
                \begin{tikzpicture}
                \tikzstyle{every node}=[font=\fontsize{12}{12}\selectfont]

\begin{axis}[
  xlabel= Global rounds,
  ylabel= Personalized Accuracy,
  legend pos=south east,
  xticklabels from table={ablation_study_team_Emnist_Full_team_full_client_mclr_acc.dat}{GR},xtick=data
]

\addplot[brown,dashed,thick,mark=square*]  table [y=pteam_3,x=GR]{ablation_study_team_Emnist_Full_team_full_client_mclr_acc.dat};

\addplot[black,dashed,thick,mark=star]  table [y=pteam_4,x=GR]{ablation_study_team_Emnist_Full_team_full_client_mclr_acc.dat};

\addplot[cyan,thick,mark=diamond*] table [y=pteam_5,x=GR]{ablation_study_team_Emnist_Full_team_full_client_mclr_acc.dat};

\addplot[blue,thick,mark=diamond*] table [y=pteam_10,x=GR]{ablation_study_team_Emnist_Full_team_full_client_mclr_acc.dat};

\end{axis}

\end{tikzpicture}

            }
        \caption{Personalized Accuracy \\ (Validation)}
        \label{Fig: Emnist_ftfc_mclr_pa}
    \end{subfigure} 
    \begin{subfigure}[!t]{0.24\linewidth}
        \resizebox{\linewidth}{!}{
        \begin{tikzpicture}
        \tikzstyle{every node}=[font=\fontsize{12}{12}\selectfont]
\begin{axis}[
  xlabel= Global rounds,
  ylabel= Global Accuracy,
  legend pos=south east,
  xticklabels from table={ablation_study_team_Emnist_Full_team_full_client_mclr_acc.dat}{GR},xtick=data
]

\addplot[brown,dashed,thick,mark=square*]  table [y=gteam_3,x=GR]{ablation_study_team_Emnist_Full_team_full_client_mclr_acc.dat};

\addplot[black,dashed,thick,mark=star]  table [y=gteam_4,x=GR]{ablation_study_team_Emnist_Full_team_full_client_mclr_acc.dat};

\addplot[cyan,thick,mark=diamond*] table [y=gteam_5,x=GR]{ablation_study_team_Emnist_Full_team_full_client_mclr_acc.dat};

\addplot[blue,thick,mark=diamond*] table [y=gteam_10,x=GR]{ablation_study_team_Emnist_Full_team_full_client_mclr_acc.dat};

\end{axis}

\end{tikzpicture}
}

    \caption{Global Accuracy \\ (Validation)}
    \label{Fig: emnist_ftfc_mclr_ga}
\end{subfigure}

\caption{Full participation teams and devices on EMNIST datasets in convex settings (MCLR)}
\label{Fig: emnist_ftfc_mclr}
\end{figure*}
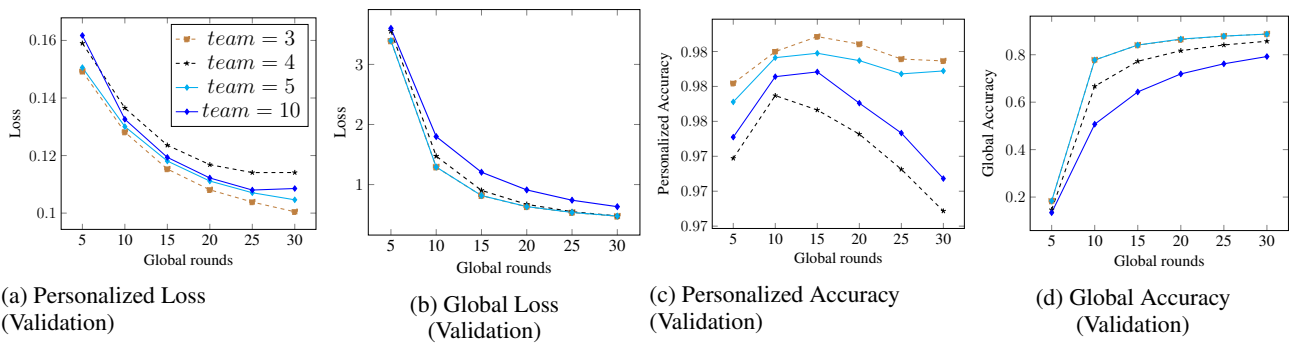




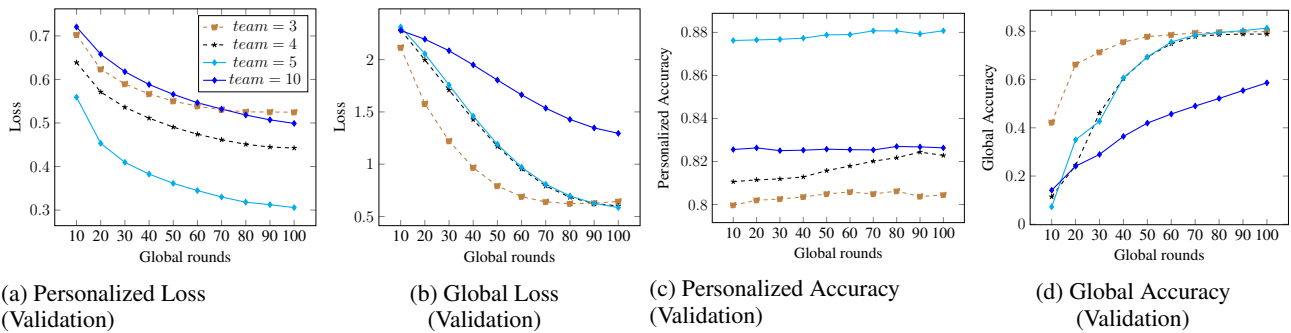
\begin{figure*}[!ht]
        \centering
        \begin{subfigure}[!t]{0.24\linewidth}
            \resizebox{\linewidth}{!}{
                \begin{tikzpicture}
                \tikzstyle{every node}=[font=\fontsize{12}{12}\selectfont]
\begin{axis}[
  xlabel= Global rounds,
  ylabel= Loss,
  legend pos=north east,
  xticklabels from table={ablation_study_team_Synthetic_Full_team_full_client_mclr_loss.dat}{GR},xtick=data
]

\addplot[brown,dashed,thick,mark=square*]  table [y=pteam_3,x=GR]{ablation_study_team_Synthetic_Full_team_full_client_dnn_loss.dat};
\addlegendentry{$ team = 3$}

\addplot[black,dashed,thick,mark=star]  table [y=pteam_4,x=GR]{ablation_study_team_Synthetic_Full_team_full_client_dnn_loss.dat};
\addlegendentry{$ team = 4$}

\addplot[cyan,thick,mark=diamond*] table [y=pteam_5,x=GR]{ablation_study_team_Synthetic_Full_team_full_client_dnn_loss.dat};
\addlegendentry{$ team = 5$}

\addplot[blue,thick,mark=diamond*] table [y=pteam_10,x=GR]{ablation_study_team_Synthetic_Full_team_full_client_dnn_loss.dat};
\addlegendentry{$team = 10$}

\end{axis}
\end{tikzpicture}

            }
        \caption{Personalized Loss \\ (Validation)}
        \label{Fig: synthetic_ftfc_dnn_pl}
    \end{subfigure} 
    \begin{subfigure}[!t]{0.24\linewidth}
        \resizebox{\linewidth}{!}{
        \begin{tikzpicture}
        \tikzstyle{every node}=[font=\fontsize{12}{12}\selectfont]

\begin{axis}[
  xlabel= Global rounds,
  ylabel= Loss,
  legend pos=north east,
  xticklabels from table={ablation_study_team_Synthetic_Full_team_full_client_dnn_loss.dat}{GR},xtick=data
]

\addplot[brown,dashed,thick,mark=square*]  table [y=gteam_3,x=GR]{ablation_study_team_Synthetic_Full_team_full_client_dnn_loss.dat};

\addplot[black,dashed,thick,mark=star]  table [y=gteam_4,x=GR]{ablation_study_team_Synthetic_Full_team_full_client_dnn_loss.dat};

\addplot[cyan,thick,mark=diamond*] table [y=gteam_5,x=GR]{ablation_study_team_Synthetic_Full_team_full_client_dnn_loss.dat};

\addplot[blue,thick,mark=diamond*] table [y=gteam_10,x=GR]{ablation_study_team_Synthetic_Full_team_full_client_dnn_loss.dat};

\end{axis}

\end{tikzpicture}
}

    \caption{Global Loss \\ (Validation)}
    \label{Fig: synthetic_ftfc_dnn_gl}
\end{subfigure}
\begin{subfigure}[!t]{0.24\linewidth}
            \resizebox{\linewidth}{!}{
                \begin{tikzpicture}
                \tikzstyle{every node}=[font=\fontsize{12}{12}\selectfont]

\begin{axis}[
  xlabel= Global rounds,
  ylabel= Personalized Accuracy,
  legend pos=south east,
  xticklabels from table={ablation_study_team_Synthetic_Full_team_full_client_dnn_acc.dat}{GR},xtick=data
]

\addplot[brown,dashed,thick,mark=square*]  table [y=pteam_3,x=GR]{ablation_study_team_Synthetic_Full_team_full_client_dnn_acc.dat};

\addplot[black,dashed,thick,mark=star]  table [y=pteam_4,x=GR]{ablation_study_team_Synthetic_Full_team_full_client_dnn_acc.dat};

\addplot[cyan,thick,mark=diamond*] table [y=pteam_5,x=GR]{ablation_study_team_Synthetic_Full_team_full_client_dnn_acc.dat};

\addplot[blue,thick,mark=diamond*] table [y=pteam_10,x=GR]{ablation_study_team_Synthetic_Full_team_full_client_dnn_acc.dat};

\end{axis}

\end{tikzpicture}

            }
        \caption{Personalized Accuracy \\ (Validation)}
        \label{Fig: synthetic_ftfc_dnn_pa}
    \end{subfigure} 
    \begin{subfigure}[!t]{0.24\linewidth}
        \resizebox{\linewidth}{!}{
        \begin{tikzpicture}
        \tikzstyle{every node}=[font=\fontsize{12}{12}\selectfont]
\begin{axis}[
  xlabel= Global rounds,
  ylabel= Global Accuracy,
  legend pos=south east,
  xticklabels from table={ablation_study_team_Synthetic_Full_team_full_client_dnn_acc.dat}{GR},xtick=data
]

\addplot[brown,dashed,thick,mark=square*]  table [y=gteam_3,x=GR]{ablation_study_team_Synthetic_Full_team_full_client_dnn_acc.dat};

\addplot[black,dashed,thick,mark=star]  table [y=gteam_4,x=GR]{ablation_study_team_Synthetic_Full_team_full_client_dnn_acc.dat};

\addplot[cyan,thick,mark=diamond*] table [y=gteam_5,x=GR]{ablation_study_team_Synthetic_Full_team_full_client_dnn_acc.dat};

\addplot[blue,thick,mark=diamond*] table [y=gteam_10,x=GR]{ablation_study_team_Synthetic_Full_team_full_client_dnn_acc.dat};

\end{axis}

\end{tikzpicture}
}

    \caption{Global Accuracy \\ (Validation)}
    \label{Fig: synthetic_ftfc_dnn_ga}
\end{subfigure}

\caption{Full participation teams and devices on Synthetic datasets in non-convex settings (DNN)}
\label{Fig: synthetic_ftfc_cnn}
\end{figure*}



\begin{figure*}[!ht]
        \centering
        \begin{subfigure}[!t]{0.24\linewidth}
            \resizebox{\linewidth}{!}{
                \begin{tikzpicture}
                \tikzstyle{every node}=[font=\fontsize{12}{12}\selectfont]
\begin{axis}[
  xlabel= Global rounds,
  ylabel= Loss,
  legend pos=north east,
  xticklabels from table={ablation_study_team_Synthetic_Full_team_full_client_mclr_loss.dat}{GR},xtick=data
]
\addplot[blue,thick,mark=diamond*] table [y=pteam_2,x=GR]{ablation_study_team_Synthetic_Full_team_full_client_mclr_loss.dat};

\addplot[brown,dashed,thick,mark=square*]  table [y=pteam_3,x=GR]{ablation_study_team_Synthetic_Full_team_full_client_mclr_loss.dat};

\addplot[black,dashed,thick,mark=star]  table [y=pteam_4,x=GR]{ablation_study_team_Synthetic_Full_team_full_client_mclr_loss.dat};

\addplot[cyan,thick,mark=diamond*] table [y=pteam_5,x=GR]{ablation_study_team_Synthetic_Full_team_full_client_mclr_loss.dat};

\end{axis}
\end{tikzpicture}

            }
        \caption{Personalized Loss \\ (Validation)}
        \label{Fig: synthetic_ftfc_mclr_pl}
    \end{subfigure} 
    \begin{subfigure}[!t]{0.24\linewidth}
        \resizebox{\linewidth}{!}{
        \begin{tikzpicture}
        \tikzstyle{every node}=[font=\fontsize{12}{12}\selectfont]

\begin{axis}[
  xlabel= Global rounds,
  ylabel= Loss,
  legend pos=north east,
  legend style ={nodes={scale=1.4, transform shape}},
  xticklabels from table={ablation_study_team_Synthetic_Full_team_full_client_mclr_loss.dat}{GR},xtick=data
]
\addplot[blue,thick,mark=diamond*] table [y=gteam_2,x=GR]{ablation_study_team_Synthetic_Full_team_full_client_mclr_loss.dat};
\addlegendentry{$ team = 2$}

\addplot[brown,dashed,thick,mark=square*]  table [y=gteam_3,x=GR]{ablation_study_team_Synthetic_Full_team_full_client_mclr_loss.dat};
\addlegendentry{$ team = 3$}

\addplot[black,dashed,thick,mark=star]  table [y=gteam_4,x=GR]{ablation_study_team_Synthetic_Full_team_full_client_mclr_loss.dat};
\addlegendentry{$ team = 4$}

\addplot[cyan,thick,mark=diamond*] table [y=gteam_5,x=GR]{ablation_study_team_Synthetic_Full_team_full_client_mclr_loss.dat};
\addlegendentry{$ team = 5$}

\end{axis}

\end{tikzpicture}
}

    \caption{Global Loss \\ (Validation)}
    \label{Fig: synthetic_ftfc_mclr_gl}
\end{subfigure}
\begin{subfigure}[!t]{0.24\linewidth}
            \resizebox{\linewidth}{!}{
                \begin{tikzpicture}
                \tikzstyle{every node}=[font=\fontsize{12}{12}\selectfont]

\begin{axis}[
  xlabel= Global rounds,
  ylabel= Personalized Accuracy,
  legend pos=south east,
  xticklabels from table={ablation_study_team_Synthetic_Full_team_full_client_mclr_acc.dat}{GR},xtick=data
]
\addplot[blue,thick,mark=diamond*] table [y=pteam_2,x=GR]{ablation_study_team_Synthetic_Full_team_full_client_mclr_acc.dat};

\addplot[brown,dashed,thick,mark=square*]  table [y=pteam_3,x=GR]{ablation_study_team_Synthetic_Full_team_full_client_mclr_acc.dat};

\addplot[black,dashed,thick,mark=star]  table [y=pteam_4,x=GR]{ablation_study_team_Synthetic_Full_team_full_client_mclr_acc.dat};

\addplot[cyan,thick,mark=diamond*] table [y=pteam_5,x=GR]{ablation_study_team_Synthetic_Full_team_full_client_mclr_acc.dat};

\end{axis}

\end{tikzpicture}

            }
        \caption{Personalized Accuracy \\ (Validation)}
        \label{Fig: synthetic_ftfc_mclr_pa}
    \end{subfigure} 
    \begin{subfigure}[!t]{0.24\linewidth}
        \resizebox{\linewidth}{!}{
        \begin{tikzpicture}
        \tikzstyle{every node}=[font=\fontsize{12}{12}\selectfont]
\begin{axis}[
  xlabel= Global rounds,
  ylabel= Global Accuracy,
  legend pos=south east,
  xticklabels from table={ablation_study_team_Synthetic_Full_team_full_client_mclr_acc.dat}{GR},xtick=data
]
\addplot[blue,thick,mark=diamond*] table [y=gteam_2,x=GR]{ablation_study_team_Synthetic_Full_team_full_client_mclr_acc.dat};

\addplot[brown,dashed,thick,mark=square*]  table [y=gteam_3,x=GR]{ablation_study_team_Synthetic_Full_team_full_client_mclr_acc.dat};

\addplot[black,dashed,thick,mark=star]  table [y=gteam_4,x=GR]{ablation_study_team_Synthetic_Full_team_full_client_mclr_acc.dat};

\addplot[cyan,thick,mark=diamond*] table [y=gteam_5,x=GR]{ablation_study_team_Synthetic_Full_team_full_client_mclr_acc.dat};

\end{axis}

\end{tikzpicture}
}

    \caption{Global Accuracy \\ (Validation)}
    \label{Fig: synthetic_ftfc_mclr_ga}
\end{subfigure}

\caption{Full participation teams and devices on Synthetic datasets in convex settings (MCLR)}
\label{Fig: synthetic_ftfc_mclr}
\end{figure*}
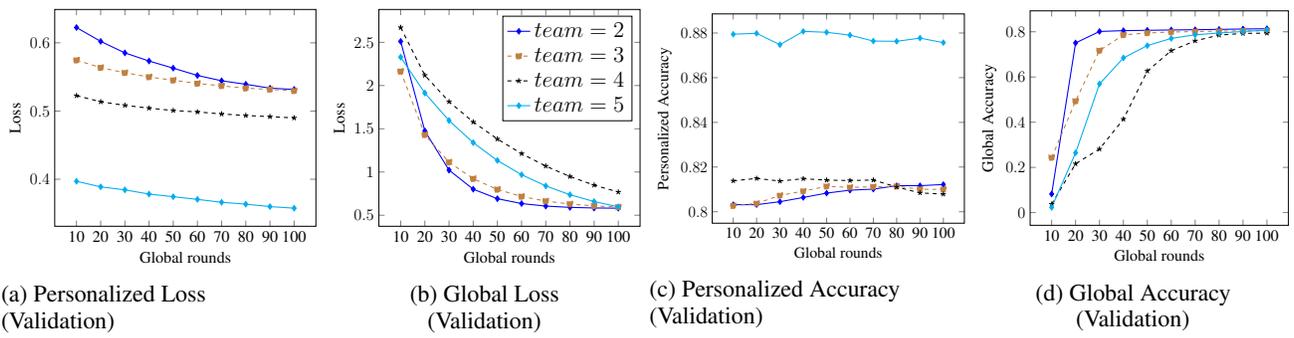


\clearpage
\subsubsection{Full participation of Teams and partial participation of Devices} 

Here, we conducted experiments in convex settings using the MCLR approach; we observed that \ourmodel{} with a low percentage of device participation converges slower compared to a high percentage of device participation (see \Cref{Fig: FMnist_ftpc_team_5} to \Cref{Fig: Emnist_ftpc_team_5}). Increasing the number of participating devices resulted in faster convergence of the global model. Therefore, to achieve a better global model with faster convergence, it is recommended to increase the number of device participants per team iteration.


\begin{figure*}[!ht]
        \centering
        \begin{subfigure}[!t]{0.24\linewidth}
            \resizebox{\linewidth}{!}{
                \begin{tikzpicture}
\begin{axis}[
  xlabel= Global rounds,
  ylabel= Loss,
  legend pos=north east,
  legend style ={nodes={scale=1.4, transform shape}},
  xticklabels from table={ablation_study_team_FMnist_Full_team_part_client_team_5_loss.dat}{GR},xtick=data
]
\addplot[blue,thick,mark=diamond*] table [y=p_10,x=GR]{ablation_study_team_FMnist_Full_team_part_client_team_5_loss.dat};
\addlegendentry{$ 10\%$}

\addplot[brown,dashed,thick,mark=star]  table [y=p_20,x=GR]{ablation_study_team_FMnist_Full_team_part_client_team_5_loss.dat};
\addlegendentry{$ 20\%$}

\addplot[black,dashed,thick,mark=triangle*]  table [y=p_25,x=GR]{ablation_study_team_FMnist_Full_team_part_client_team_5_loss.dat};
\addlegendentry{$ 25\%$}

\addplot[cyan,thick,mark=square*] table [y=p_50,x=GR]{ablation_study_team_FMnist_Full_team_part_client_team_5_loss.dat};
\addlegendentry{$ 50\%$}

\end{axis}
\end{tikzpicture}

            }
        \caption{Personalized Loss \\ (Validation)}
        \label{Fig: FMnist_ftpc_pl}
    \end{subfigure} 
    \begin{subfigure}[!t]{0.24\linewidth}
        \resizebox{\linewidth}{!}{
        \begin{tikzpicture}

\begin{axis}[
  xlabel= Global rounds,
  ylabel= Loss,
  legend pos=north east,
  xticklabels from table={ablation_study_team_FMnist_Full_team_part_client_team_5_loss.dat}{GR},xtick=data
]
\addplot[blue,thick,mark=diamond*] table [y=g_10,x=GR]{ablation_study_team_FMnist_Full_team_part_client_team_5_loss.dat};

\addplot[brown,dashed,thick,mark=star]  table [y=g_20,x=GR]{ablation_study_team_FMnist_Full_team_part_client_team_5_loss.dat};

\addplot[black,dashed,thick,mark=triangle*]  table [y=g_25,x=GR]{ablation_study_team_FMnist_Full_team_part_client_team_5_loss.dat};

\addplot[cyan,thick,mark=square*] table [y=g_50,x=GR]{ablation_study_team_FMnist_Full_team_part_client_team_5_loss.dat};

\end{axis}

\end{tikzpicture}
}

    \caption{Global Loss \\ (Validation)}
    \label{Fig: FMnist_ftpc_gl}
\end{subfigure}
\begin{subfigure}[!t]{0.24\linewidth}
            \resizebox{\linewidth}{!}{
                \begin{tikzpicture}

\begin{axis}[
  xlabel= Global rounds,
  ylabel= Personalized Accuracy,
  legend pos=south east,
  xticklabels from table={ablation_study_team_FMnist_Full_team_part_client_team_5_acc.dat}{GR},xtick=data
]
\addplot[blue,thick,mark=diamond*] table [y=p_10,x=GR]{ablation_study_team_FMnist_Full_team_part_client_team_5_acc.dat};

\addplot[brown,dashed,thick,mark=star]  table [y=p_20,x=GR]{ablation_study_team_FMnist_Full_team_part_client_team_5_acc.dat};

\addplot[black,dashed,thick,mark=triangle*]  table [y=p_25,x=GR]{ablation_study_team_FMnist_Full_team_part_client_team_5_acc.dat};

\addplot[cyan,thick,mark=square*] table [y=p_50,x=GR]{ablation_study_team_FMnist_Full_team_part_client_team_5_acc.dat};

\end{axis}

\end{tikzpicture}

            }
        \caption{Personalized Accuracy \\ (Validation)}
        \label{Fig: FMnist_ftpc_pa}
    \end{subfigure} 
    \begin{subfigure}[!t]{0.24\linewidth}
        \resizebox{\linewidth}{!}{
        \begin{tikzpicture}
\begin{axis}[
  xlabel= Global rounds,
  ylabel= Global Accuracy,
  legend pos=south east,
  xticklabels from table={ablation_study_team_FMnist_Full_team_part_client_team_5_acc.dat}{GR},xtick=data
]
\addplot[blue,thick,mark=diamond*] table [y=g_10,x=GR]{ablation_study_team_FMnist_Full_team_part_client_team_5_acc.dat};

\addplot[brown,dashed,thick,mark=star]  table [y=g_20,x=GR]{ablation_study_team_FMnist_Full_team_part_client_team_5_acc.dat};

\addplot[black,dashed,thick,mark=triangle*]  table [y=g_25,x=GR]{ablation_study_team_FMnist_Full_team_part_client_team_5_acc.dat};

\addplot[cyan,thick,mark=square*] table [y=g_50,x=GR]{ablation_study_team_FMnist_Full_team_part_client_team_5_acc.dat};

\end{axis}

\end{tikzpicture}
}

    \caption{Global Accuracy \\ (Validation)}
    \label{Fig: FMnist_ftpc_ga}
\end{subfigure}

\caption{Full team participation (5 teams) but partial devices participation on FMNIST datasets in convex settings (MCLR)}
\label{Fig: FMnist_ftpc_team_5}
\end{figure*}
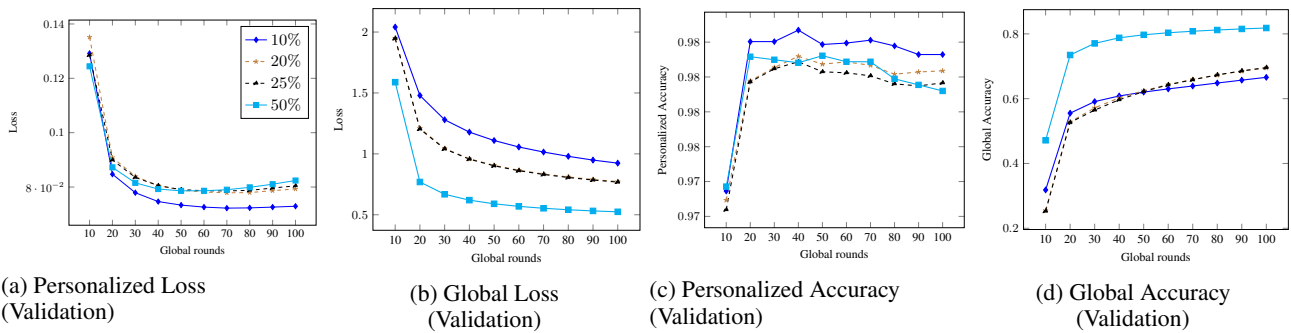



\begin{figure*}[!ht]
        \centering
        \begin{subfigure}[!t]{0.24\linewidth}
            \resizebox{\linewidth}{!}{
                \begin{tikzpicture}
\begin{axis}[
  xlabel= Global rounds,
  ylabel= Loss,
  legend pos=north east,
  legend style ={nodes={scale=1.4, transform shape}},
  xticklabels from table={ablation_study_team_FMnist_Full_team_part_client_team_5_loss.dat}{GR},xtick=data
]
\addplot[blue,thick,mark=diamond*] table [y=p_8,x=GR]{ablation_study_team_FMnist_Full_team_part_client_team_4_loss.dat};
\addlegendentry{$ 8\%$}

\addplot[brown,dashed,thick,mark=star]  table [y=p_12,x=GR]{ablation_study_team_FMnist_Full_team_part_client_team_4_loss.dat};
\addlegendentry{$ 12\%$}

\addplot[black,dashed,thick,mark=triangle*]  table [y=p_20,x=GR]{ablation_study_team_FMnist_Full_team_part_client_team_4_loss.dat};
\addlegendentry{$ 20\%$}

\addplot[cyan,thick,mark=square*] table [y=p_48,x=GR]{ablation_study_team_FMnist_Full_team_part_client_team_4_loss.dat};
\addlegendentry{$ 48\%$}

\end{axis}
\end{tikzpicture}

            }
        \caption{Personalized Loss \\ (Validation)}
        \label{Fig: FMnist_ftpc_pl_4}
    \end{subfigure} 
    \begin{subfigure}[!t]{0.24\linewidth}
        \resizebox{\linewidth}{!}{
        \begin{tikzpicture}

\begin{axis}[
  xlabel= Global rounds,
  ylabel= Loss,
  legend pos=north east,
  xticklabels from table={ablation_study_team_FMnist_Full_team_part_client_team_4_loss.dat}{GR},xtick=data
]
\addplot[blue,thick,mark=diamond*] table [y=g_8,x=GR]{ablation_study_team_FMnist_Full_team_part_client_team_4_loss.dat};

\addplot[brown,dashed,thick,mark=star]  table [y=g_12,x=GR]{ablation_study_team_FMnist_Full_team_part_client_team_4_loss.dat};

\addplot[black,dashed,thick,mark=triangle*]  table [y=g_20,x=GR]{ablation_study_team_FMnist_Full_team_part_client_team_4_loss.dat};

\addplot[cyan,thick,mark=square*] table [y=g_48,x=GR]{ablation_study_team_FMnist_Full_team_part_client_team_4_loss.dat};

\end{axis}

\end{tikzpicture}
}

    \caption{Global Loss \\ (Validation)}
    \label{Fig: FMnist_ftpc_gl_4}
\end{subfigure}
\begin{subfigure}[!t]{0.24\linewidth}
            \resizebox{\linewidth}{!}{
                \begin{tikzpicture}

\begin{axis}[
  xlabel= Global rounds,
  ylabel= Personalized Accuracy,
  legend pos=south east,
  xticklabels from table={ablation_study_team_FMnist_Full_team_part_client_team_4_acc.dat}{GR},xtick=data
]
\addplot[blue,thick,mark=diamond*] table [y=p_8,x=GR]{ablation_study_team_FMnist_Full_team_part_client_team_4_acc.dat};

\addplot[brown,dashed,thick,mark=star]  table [y=p_12,x=GR]{ablation_study_team_FMnist_Full_team_part_client_team_4_acc.dat};

\addplot[black,dashed,thick,mark=triangle*]  table [y=p_20,x=GR]{ablation_study_team_FMnist_Full_team_part_client_team_4_acc.dat};

\addplot[cyan,thick,mark=square*] table [y=p_48,x=GR]{ablation_study_team_FMnist_Full_team_part_client_team_4_acc.dat};

\end{axis}

\end{tikzpicture}

            }
        \caption{Personalized Accuracy \\ (Validation)}
        \label{Fig: FMnist_ftpc_pa_4}
    \end{subfigure} 
    \begin{subfigure}[!t]{0.24\linewidth}
        \resizebox{\linewidth}{!}{
        \begin{tikzpicture}
\begin{axis}[
  xlabel= Global rounds,
  ylabel= Global Accuracy,
  legend pos=south east,
  xticklabels from table={ablation_study_team_FMnist_Full_team_part_client_team_4_acc.dat}{GR},xtick=data
]
\addplot[blue,thick,mark=diamond*] table [y=g_8,x=GR]{ablation_study_team_FMnist_Full_team_part_client_team_4_acc.dat};

\addplot[brown,dashed,thick,mark=star]  table [y=g_12,x=GR]{ablation_study_team_FMnist_Full_team_part_client_team_4_acc.dat};

\addplot[black,dashed,thick,mark=triangle*]  table [y=g_20,x=GR]{ablation_study_team_FMnist_Full_team_part_client_team_4_acc.dat};

\addplot[cyan,thick,mark=square*] table [y=g_48,x=GR]{ablation_study_team_FMnist_Full_team_part_client_team_4_acc.dat};

\end{axis}

\end{tikzpicture}
}

    \caption{Global Accuracy \\ (Validation)}
    \label{Fig: FMnist_ftpc_ga_4}
\end{subfigure}

\caption{Full team participation (4 teams) but partial devices participation on FMNIST datasets in convex settings (MCLR)}
\label{Fig: FMnist_ftpc_team_4}
\end{figure*}
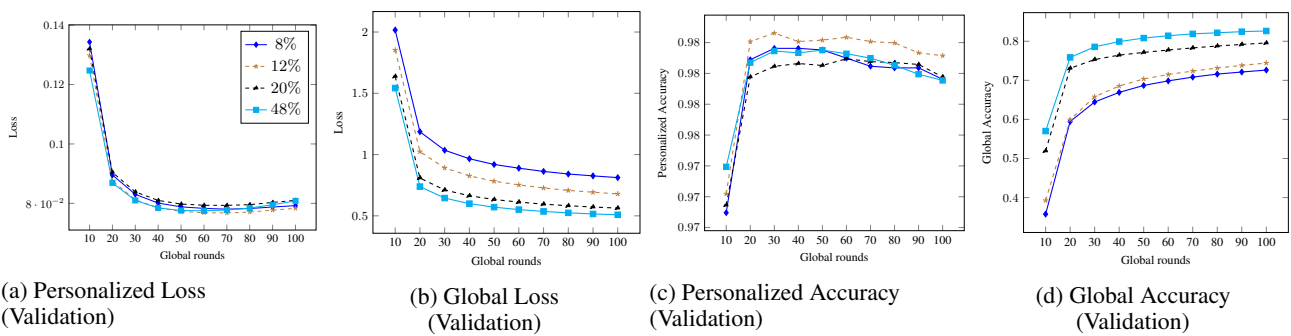



\begin{figure*}[!ht]
        \centering
        \begin{subfigure}[!t]{0.24\linewidth}
            \resizebox{\linewidth}{!}{
                \begin{tikzpicture}
\begin{axis}[
  xlabel= Global rounds,
  ylabel= Loss,
  legend pos=north east,
  legend style ={nodes={scale=1.4, transform shape}},
  xticklabels from table={ablation_study_team_FMnist_Full_team_part_client_team_2_loss.dat}{GR},xtick=data
]
\addplot[blue,thick,mark=diamond*] table [y=p_4,x=GR]{ablation_study_team_FMnist_Full_team_part_client_team_2_loss.dat};
\addlegendentry{$ 4\%$}

\addplot[brown,dashed,thick,mark=star]  table [y=p_12,x=GR]{ablation_study_team_FMnist_Full_team_part_client_team_2_loss.dat};
\addlegendentry{$ 12\%$}

\addplot[black,dashed,thick,mark=triangle*]  table [y=p_20,x=GR]{ablation_study_team_FMnist_Full_team_part_client_team_2_loss.dat};
\addlegendentry{$ 20\%$}

\addplot[cyan,thick,mark=square*] table [y=p_50,x=GR]{ablation_study_team_FMnist_Full_team_part_client_team_2_loss.dat};
\addlegendentry{$ 50\%$}

\end{axis}
\end{tikzpicture}

            }
        \caption{Personalized Loss \\ (Validation)}
        \label{Fig: FMnist_ftpc_pl_2}
    \end{subfigure} 
    \begin{subfigure}[!t]{0.24\linewidth}
        \resizebox{\linewidth}{!}{
        \begin{tikzpicture}

\begin{axis}[
  xlabel= Global rounds,
  ylabel= Loss,
  legend pos=north east,
  xticklabels from table={ablation_study_team_FMnist_Full_team_part_client_team_2_loss.dat}{GR},xtick=data
]
\addplot[blue,thick,mark=diamond*] table [y=g_4,x=GR]{ablation_study_team_FMnist_Full_team_part_client_team_2_loss.dat};

\addplot[brown,dashed,thick,mark=star]  table [y=g_12,x=GR]{ablation_study_team_FMnist_Full_team_part_client_team_2_loss.dat};

\addplot[black,dashed,thick,mark=triangle*]  table [y=g_20,x=GR]{ablation_study_team_FMnist_Full_team_part_client_team_2_loss.dat};

\addplot[cyan,thick,mark=square*] table [y=g_50,x=GR]{ablation_study_team_FMnist_Full_team_part_client_team_2_loss.dat};

\end{axis}

\end{tikzpicture}
}

    \caption{Global Loss \\ (Validation)}
    \label{Fig: FMnist_ftpc_gl_2}
\end{subfigure}
\begin{subfigure}[!t]{0.24\linewidth}
            \resizebox{\linewidth}{!}{
                \begin{tikzpicture}

\begin{axis}[
  xlabel= Global rounds,
  ylabel= Personalized Accuracy,
  legend pos=south east,
  xticklabels from table={ablation_study_team_FMnist_Full_team_part_client_team_2_acc.dat}{GR},xtick=data
]
\addplot[blue,thick,mark=diamond*] table [y=p_4,x=GR]{ablation_study_team_FMnist_Full_team_part_client_team_2_acc.dat};

\addplot[brown,dashed,thick,mark=star]  table [y=p_12,x=GR]{ablation_study_team_FMnist_Full_team_part_client_team_2_acc.dat};

\addplot[black,dashed,thick,mark=triangle*]  table [y=p_20,x=GR]{ablation_study_team_FMnist_Full_team_part_client_team_2_acc.dat};

\addplot[cyan,thick,mark=square*] table [y=p_50,x=GR]{ablation_study_team_FMnist_Full_team_part_client_team_2_acc.dat};

\end{axis}

\end{tikzpicture}

            }
        \caption{Personalized Accuracy \\ (Validation)}
        \label{Fig: FMnist_ftpc_pa_2}
    \end{subfigure} 
    \begin{subfigure}[!t]{0.24\linewidth}
        \resizebox{\linewidth}{!}{
        \begin{tikzpicture}
\begin{axis}[
  xlabel= Global rounds,
  ylabel= Global Accuracy,
  legend pos=south east,
  xticklabels from table={ablation_study_team_FMnist_Full_team_part_client_team_2_acc.dat}{GR},xtick=data
]
\addplot[blue,thick,mark=diamond*] table [y=g_4,x=GR]{ablation_study_team_FMnist_Full_team_part_client_team_2_acc.dat};

\addplot[brown,dashed,thick,mark=star]  table [y=g_12,x=GR]{ablation_study_team_FMnist_Full_team_part_client_team_2_acc.dat};

\addplot[black,dashed,thick,mark=triangle*]  table [y=g_20,x=GR]{ablation_study_team_FMnist_Full_team_part_client_team_2_acc.dat};

\addplot[cyan,thick,mark=square*] table [y=g_50,x=GR]{ablation_study_team_FMnist_Full_team_part_client_team_2_acc.dat};

\end{axis}

\end{tikzpicture}
}

    \caption{Global Accuracy \\ (Validation)}
    \label{Fig: FMnist_ftpc_ga_2}
\end{subfigure}

\caption{Full team participation (2 teams) but partial devices participation on FMNIST datasets in convex settings (MCLR)}
\label{Fig: FMnist_ftpc_team_2}
\end{figure*}
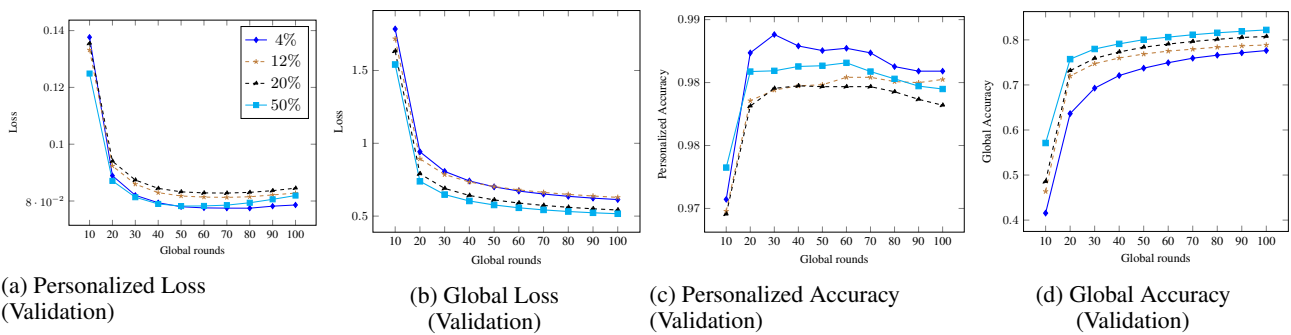



\begin{figure*}[!ht]
        \centering
        \begin{subfigure}[!t]{0.24\linewidth}
            \resizebox{\linewidth}{!}{
                \begin{tikzpicture}
\begin{axis}[
  xlabel= Global rounds,
  ylabel= Loss,
  legend pos=north east,
  legend style ={nodes={scale=1.4, transform shape}},
  xticklabels from table={ablation_study_team_Mnist_Full_team_part_client_team_2_loss.dat}{GR},xtick=data
]
\addplot[blue,thick,mark=diamond*] table [y=p_4,x=GR]{ablation_study_team_Mnist_Full_team_part_client_team_2_loss.dat};
\addlegendentry{$ 4\%$}

\addplot[brown,dashed,thick,mark=star]  table [y=p_12,x=GR]{ablation_study_team_Mnist_Full_team_part_client_team_2_loss.dat};
\addlegendentry{$ 12\%$}

\addplot[black,dashed,thick,mark=triangle*]  table [y=p_20,x=GR]{ablation_study_team_Mnist_Full_team_part_client_team_2_loss.dat};
\addlegendentry{$ 20\%$}

\addplot[cyan,thick,mark=square*] table [y=p_50,x=GR]{ablation_study_team_Mnist_Full_team_part_client_team_2_loss.dat};
\addlegendentry{$ 50\%$}

\end{axis}
\end{tikzpicture}

            }
        \caption{Personalized Loss \\ (Validation)}
        \label{Fig: Mnist_ftpc_pl_2}
    \end{subfigure} 
    \begin{subfigure}[!t]{0.24\linewidth}
        \resizebox{\linewidth}{!}{
        \begin{tikzpicture}

\begin{axis}[
  xlabel= Global rounds,
  ylabel= Loss,
  legend pos=north east,
  xticklabels from table={ablation_study_team_FMnist_Full_team_part_client_team_2_loss.dat}{GR},xtick=data
]
\addplot[blue,thick,mark=diamond*] table [y=g_4,x=GR]{ablation_study_team_Mnist_Full_team_part_client_team_2_loss.dat};

\addplot[brown,dashed,thick,mark=star]  table [y=g_12,x=GR]{ablation_study_team_Mnist_Full_team_part_client_team_2_loss.dat};

\addplot[black,dashed,thick,mark=triangle*]  table [y=g_20,x=GR]{ablation_study_team_Mnist_Full_team_part_client_team_2_loss.dat};

\addplot[cyan,thick,mark=square*] table [y=g_50,x=GR]{ablation_study_team_Mnist_Full_team_part_client_team_2_loss.dat};

\end{axis}

\end{tikzpicture}
}

    \caption{Global Loss \\ (Validation)}
    \label{Fig: Mnist_ftpc_gl_2}
\end{subfigure}
\begin{subfigure}[!t]{0.24\linewidth}
            \resizebox{\linewidth}{!}{
                \begin{tikzpicture}

\begin{axis}[
  xlabel= Global rounds,
  ylabel= Personalized Accuracy,
  legend pos=south east,
  xticklabels from table={ablation_study_team_Mnist_Full_team_part_client_team_2_acc.dat}{GR},xtick=data
]
\addplot[blue,thick,mark=diamond*] table [y=p_4,x=GR]{ablation_study_team_Mnist_Full_team_part_client_team_2_acc.dat};

\addplot[brown,dashed,thick,mark=star]  table [y=p_12,x=GR]{ablation_study_team_Mnist_Full_team_part_client_team_2_acc.dat};

\addplot[black,dashed,thick,mark=triangle*]  table [y=p_20,x=GR]{ablation_study_team_Mnist_Full_team_part_client_team_2_acc.dat};

\addplot[cyan,thick,mark=square*] table [y=p_50,x=GR]{ablation_study_team_Mnist_Full_team_part_client_team_2_acc.dat};

\end{axis}

\end{tikzpicture}

            }
        \caption{Personalized Accuracy \\ (Validation)}
        \label{Fig: Mnist_ftpc_pa_2}
    \end{subfigure} 
    \begin{subfigure}[!t]{0.24\linewidth}
        \resizebox{\linewidth}{!}{
        \begin{tikzpicture}
\begin{axis}[
  xlabel= Global rounds,
  ylabel= Global Accuracy,
  legend pos=south east,
  xticklabels from table={ablation_study_team_Mnist_Full_team_part_client_team_2_acc.dat}{GR},xtick=data
]
\addplot[blue,thick,mark=diamond*] table [y=g_4,x=GR]{ablation_study_team_Mnist_Full_team_part_client_team_2_acc.dat};

\addplot[brown,dashed,thick,mark=star]  table [y=g_12,x=GR]{ablation_study_team_Mnist_Full_team_part_client_team_2_acc.dat};

\addplot[black,dashed,thick,mark=triangle*]  table [y=g_20,x=GR]{ablation_study_team_Mnist_Full_team_part_client_team_2_acc.dat};

\addplot[cyan,thick,mark=square*] table [y=g_50,x=GR]{ablation_study_team_Mnist_Full_team_part_client_team_2_acc.dat};

\end{axis}

\end{tikzpicture}
}

    \caption{Global Accuracy \\ (Validation)}
    \label{Fig: Mnist_ftpc_ga_2}
\end{subfigure}

\caption{Full team participation (2 teams) but partial devices participation on MNIST datasets in convex settings (MCLR)}
\label{Fig: Mnist_ftpc_team_2}
\end{figure*}
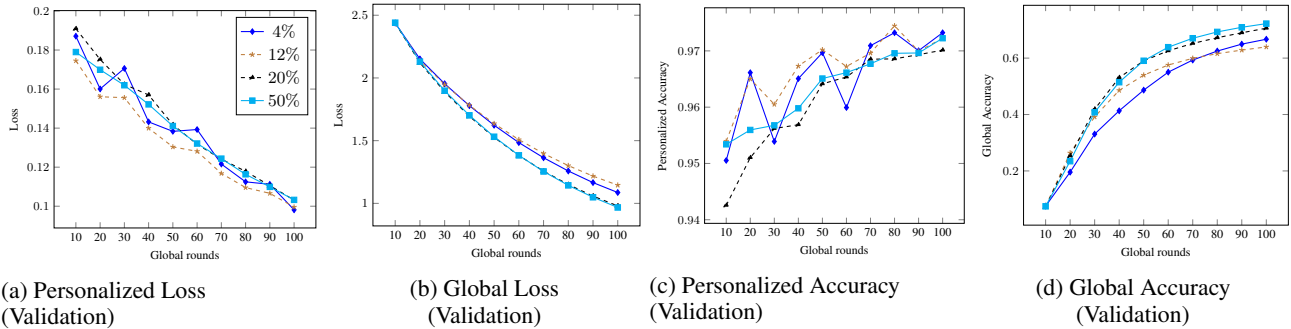



\begin{figure*}[!ht]
        \centering
        \begin{subfigure}[!t]{0.24\linewidth}
            \resizebox{\linewidth}{!}{
                \begin{tikzpicture}
\begin{axis}[
  xlabel= Global rounds,
  ylabel= Loss,
  legend pos=north east,
  legend style ={nodes={scale=1.4, transform shape}},
  xticklabels from table={ablation_study_team_Mnist_Full_team_part_client_team_4_loss.dat}{GR},xtick=data
]
\addplot[blue,thick,mark=diamond*] table [y=p_8,x=GR]{ablation_study_team_Mnist_Full_team_part_client_team_4_loss.dat};
\addlegendentry{$ 8\%$}

\addplot[brown,dashed,thick,mark=star]  table [y=p_12,x=GR]{ablation_study_team_Mnist_Full_team_part_client_team_4_loss.dat};
\addlegendentry{$ 12\%$}

\addplot[black,dashed,thick,mark=triangle*]  table [y=p_20,x=GR]{ablation_study_team_Mnist_Full_team_part_client_team_4_loss.dat};
\addlegendentry{$ 20\%$}

\addplot[cyan,thick,mark=square*] table [y=p_48,x=GR]{ablation_study_team_Mnist_Full_team_part_client_team_4_loss.dat};
\addlegendentry{$ 48\%$}

\end{axis}
\end{tikzpicture}

            }
        \caption{Personalized Loss \\ (Validation)}
        \label{Fig: Mnist_ftpc_pl_4}
    \end{subfigure} 
    \begin{subfigure}[!t]{0.24\linewidth}
        \resizebox{\linewidth}{!}{
        \begin{tikzpicture}

\begin{axis}[
  xlabel= Global rounds,
  ylabel= Loss,
  legend pos=north east,
  xticklabels from table={ablation_study_team_Mnist_Full_team_part_client_team_4_loss.dat}{GR},xtick=data
]
\addplot[blue,thick,mark=diamond*] table [y=g_8,x=GR]{ablation_study_team_Mnist_Full_team_part_client_team_4_loss.dat};

\addplot[brown,dashed,thick,mark=star]  table [y=g_12,x=GR]{ablation_study_team_Mnist_Full_team_part_client_team_4_loss.dat};

\addplot[black,dashed,thick,mark=triangle*]  table [y=g_20,x=GR]{ablation_study_team_Mnist_Full_team_part_client_team_4_loss.dat};

\addplot[cyan,thick,mark=square*] table [y=g_48,x=GR]{ablation_study_team_Mnist_Full_team_part_client_team_4_loss.dat};

\end{axis}

\end{tikzpicture}
}

    \caption{Global Loss \\ (Validation)}
    \label{Fig: Mnist_ftpc_gl_4}
\end{subfigure}
\begin{subfigure}[!t]{0.24\linewidth}
            \resizebox{\linewidth}{!}{
                \begin{tikzpicture}

\begin{axis}[
  xlabel= Global rounds,
  ylabel= Personalized Accuracy,
  legend pos=south east,
  xticklabels from table={ablation_study_team_Mnist_Full_team_part_client_team_4_acc.dat}{GR},xtick=data
]
\addplot[blue,thick,mark=diamond*] table [y=p_8,x=GR]{ablation_study_team_Mnist_Full_team_part_client_team_4_acc.dat};

\addplot[brown,dashed,thick,mark=star]  table [y=p_12,x=GR]{ablation_study_team_Mnist_Full_team_part_client_team_4_acc.dat};

\addplot[black,dashed,thick,mark=triangle*]  table [y=p_20,x=GR]{ablation_study_team_Mnist_Full_team_part_client_team_4_acc.dat};

\addplot[cyan,thick,mark=square*] table [y=p_48,x=GR]{ablation_study_team_Mnist_Full_team_part_client_team_4_acc.dat};

\end{axis}

\end{tikzpicture}

            }
        \caption{Personalized Accuracy \\ (Validation)}
        \label{Fig: Mnist_ftpc_pa_4}
    \end{subfigure} 
    \begin{subfigure}[!t]{0.24\linewidth}
        \resizebox{\linewidth}{!}{
        \begin{tikzpicture}
\begin{axis}[
  xlabel= Global rounds,
  ylabel= Global Accuracy,
  legend pos=south east,
  xticklabels from table={ablation_study_team_Mnist_Full_team_part_client_team_4_acc.dat}{GR},xtick=data
]
\addplot[blue,thick,mark=diamond*] table [y=g_8,x=GR]{ablation_study_team_Mnist_Full_team_part_client_team_4_acc.dat};

\addplot[brown,dashed,thick,mark=star]  table [y=g_12,x=GR]{ablation_study_team_Mnist_Full_team_part_client_team_4_acc.dat};

\addplot[black,dashed,thick,mark=triangle*]  table [y=g_20,x=GR]{ablation_study_team_Mnist_Full_team_part_client_team_4_acc.dat};

\addplot[cyan,thick,mark=square*] table [y=g_48,x=GR]{ablation_study_team_Mnist_Full_team_part_client_team_4_acc.dat};

\end{axis}

\end{tikzpicture}
}

    \caption{Global Accuracy \\ (Validation)}
    \label{Fig: Mnist_ftpc_ga_4}
\end{subfigure}

\caption{Full team participation (4 teams) but partial devices participation on MNIST datasets in convex settings (MCLR)}
\label{Fig: Mnist_ftpc_team_4}
\end{figure*}
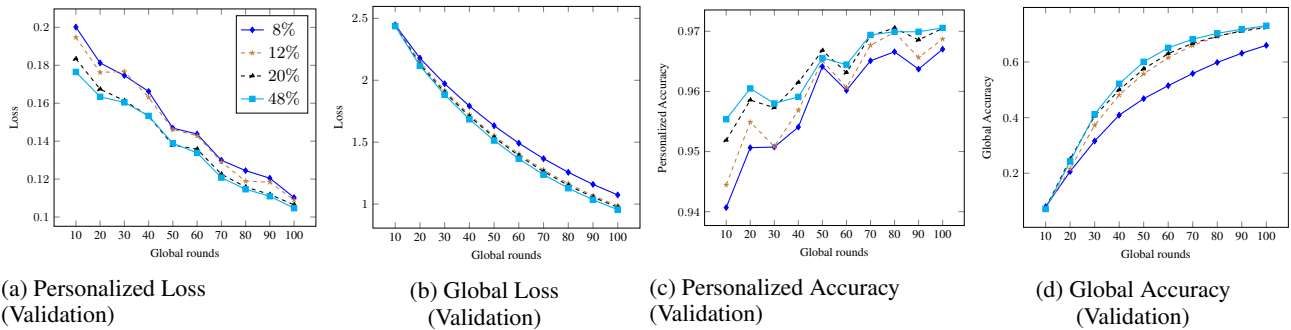



\begin{figure*}[!ht]
        \centering
        \begin{subfigure}[!t]{0.24\linewidth}
            \resizebox{\linewidth}{!}{
                \begin{tikzpicture}
\begin{axis}[
  xlabel= Global rounds,
  ylabel= Loss,
  legend pos=north east,
  legend style ={nodes={scale=1.4, transform shape}},
  xticklabels from table={ablation_study_team_Mnist_Full_team_part_client_team_5_loss.dat}{GR},xtick=data
]
\addplot[blue,thick,mark=diamond*] table [y=p_10,x=GR]{ablation_study_team_Mnist_Full_team_part_client_team_5_loss.dat};
\addlegendentry{$ 10\%$}

\addplot[brown,dashed,thick,mark=star]  table [y=p_20,x=GR]{ablation_study_team_Mnist_Full_team_part_client_team_5_loss.dat};
\addlegendentry{$ 20\%$}

\addplot[black,dashed,thick,mark=triangle*]  table [y=p_30,x=GR]{ablation_study_team_Mnist_Full_team_part_client_team_5_loss.dat};
\addlegendentry{$ 30\%$}

\addplot[cyan,thick,mark=square*] table [y=p_50,x=GR]{ablation_study_team_Mnist_Full_team_part_client_team_5_loss.dat};
\addlegendentry{$ 50\%$}

\end{axis}
\end{tikzpicture}

            }
        \caption{Personalized Loss \\ (Validation)}
        \label{Fig: Mnist_ftpc_pl_5}
    \end{subfigure} 
    \begin{subfigure}[!t]{0.24\linewidth}
        \resizebox{\linewidth}{!}{
        \begin{tikzpicture}

\begin{axis}[
  xlabel= Global rounds,
  ylabel= Loss,
  legend pos=north east,
  xticklabels from table={ablation_study_team_Mnist_Full_team_part_client_team_5_loss.dat}{GR},xtick=data
]
\addplot[blue,thick,mark=diamond*] table [y=g_10,x=GR]{ablation_study_team_Mnist_Full_team_part_client_team_5_loss.dat};

\addplot[brown,dashed,thick,mark=star]  table [y=g_20,x=GR]{ablation_study_team_Mnist_Full_team_part_client_team_5_loss.dat};

\addplot[black,dashed,thick,mark=triangle*]  table [y=g_30,x=GR]{ablation_study_team_Mnist_Full_team_part_client_team_5_loss.dat};

\addplot[cyan,thick,mark=square*] table [y=g_50,x=GR]{ablation_study_team_Mnist_Full_team_part_client_team_5_loss.dat};

\end{axis}

\end{tikzpicture}
}

    \caption{Global Loss \\ (Validation)}
    \label{Fig: Mnist_ftpc_gl_5}
\end{subfigure}
\begin{subfigure}[!t]{0.24\linewidth}
            \resizebox{\linewidth}{!}{
                \begin{tikzpicture}

\begin{axis}[
  xlabel= Global rounds,
  ylabel= Personalized Accuracy,
  legend pos=south east,
  xticklabels from table={ablation_study_team_Mnist_Full_team_part_client_team_5_acc.dat}{GR},xtick=data
]
\addplot[blue,thick,mark=diamond*] table [y=p_10,x=GR]{ablation_study_team_Mnist_Full_team_part_client_team_5_acc.dat};

\addplot[brown,dashed,thick,mark=star]  table [y=p_20,x=GR]{ablation_study_team_Mnist_Full_team_part_client_team_5_acc.dat};

\addplot[black,dashed,thick,mark=triangle*]  table [y=p_30,x=GR]{ablation_study_team_Mnist_Full_team_part_client_team_5_acc.dat};

\addplot[cyan,thick,mark=square*] table [y=p_50,x=GR]{ablation_study_team_Mnist_Full_team_part_client_team_5_acc.dat};

\end{axis}

\end{tikzpicture}

            }
        \caption{Personalized Accuracy \\ (Validation)}
        \label{Fig: Mnist_ftpc_pa_5}
    \end{subfigure} 
    \begin{subfigure}[!t]{0.24\linewidth}
        \resizebox{\linewidth}{!}{
        \begin{tikzpicture}
\begin{axis}[
  xlabel= Global rounds,
  ylabel= Global Accuracy,
  legend pos=south east,
  xticklabels from table={ablation_study_team_Mnist_Full_team_part_client_team_5_acc.dat}{GR},xtick=data
]
\addplot[blue,thick,mark=diamond*] table [y=g_10,x=GR]{ablation_study_team_Mnist_Full_team_part_client_team_5_acc.dat};

\addplot[brown,dashed,thick,mark=star]  table [y=g_20,x=GR]{ablation_study_team_Mnist_Full_team_part_client_team_5_acc.dat};

\addplot[black,dashed,thick,mark=triangle*]  table [y=g_30,x=GR]{ablation_study_team_Mnist_Full_team_part_client_team_5_acc.dat};

\addplot[cyan,thick,mark=square*] table [y=g_50,x=GR]{ablation_study_team_Mnist_Full_team_part_client_team_5_acc.dat};

\end{axis}

\end{tikzpicture}
}

    \caption{Global Accuracy \\ (Validation)}
    \label{Fig: Mnist_ftpc_ga_5}
\end{subfigure}

\caption{Full team participation (5 teams) but partial devices participation on MNIST datasets in convex settings (MCLR)}
\label{Fig: Mnist_ftpc_team_5}
\end{figure*}
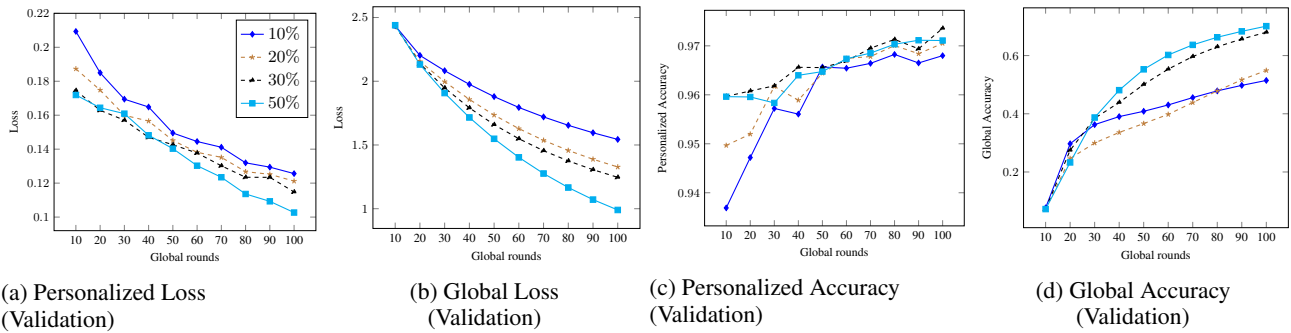



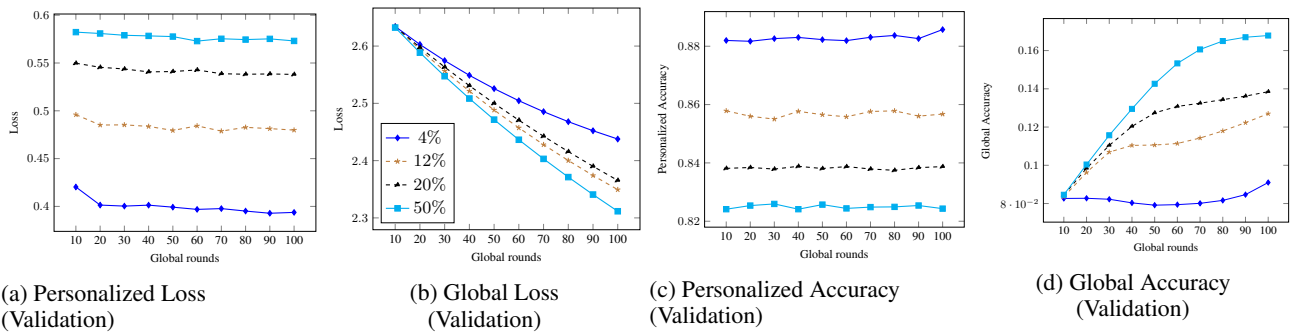
\begin{figure*}[!ht]
        \centering
        \begin{subfigure}[!t]{0.24\linewidth}
            \resizebox{\linewidth}{!}{
                \begin{tikzpicture}
\begin{axis}[
  xlabel= Global rounds,
  ylabel= Loss,
  legend pos=north east,
  xticklabels from table={ablation_study_team_Synthetic_Full_team_part_client_team_2_loss.dat}{GR},xtick=data
]
\addplot[blue,thick,mark=diamond*] table [y=p_part_4,x=GR]{ablation_study_team_Synthetic_Full_team_part_client_team_2_loss.dat};

\addplot[brown,dashed,thick,mark=star]  table [y=p_part_12,x=GR]{ablation_study_team_Synthetic_Full_team_part_client_team_2_loss.dat};

\addplot[black,dashed,thick,mark=triangle*]  table [y=p_part_20,x=GR]{ablation_study_team_Synthetic_Full_team_part_client_team_2_loss.dat};

\addplot[cyan,thick,mark=square*] table [y=p_part_50,x=GR]{ablation_study_team_Synthetic_Full_team_part_client_team_2_loss.dat};

\end{axis}
\end{tikzpicture}

            }
        \caption{Personalized Loss \\ (Validation)}
        \label{Fig: Synthetic_ftpc_pl_2}
    \end{subfigure} 
    \begin{subfigure}[!t]{0.24\linewidth}
        \resizebox{\linewidth}{!}{
        \begin{tikzpicture}

\begin{axis}[
  xlabel= Global rounds,
  ylabel= Loss,
  legend pos=south west,
  legend style ={nodes={scale=1.4, transform shape}},
  xticklabels from table={ablation_study_team_Synthetic_Full_team_part_client_team_2_loss.dat}{GR},xtick=data
]
\addplot[blue,thick,mark=diamond*] table [y=g_part_4,x=GR]{ablation_study_team_Synthetic_Full_team_part_client_team_2_loss.dat};
\addlegendentry{$ 4\%$}

\addplot[brown,dashed,thick,mark=star]  table [y=g_part_12,x=GR]{ablation_study_team_Synthetic_Full_team_part_client_team_2_loss.dat};
\addlegendentry{$ 12\%$}

\addplot[black,dashed,thick,mark=triangle*]  table [y=g_part_20,x=GR]{ablation_study_team_Synthetic_Full_team_part_client_team_2_loss.dat};
\addlegendentry{$ 20\%$}

\addplot[cyan,thick,mark=square*] table [y=g_part_50,x=GR]{ablation_study_team_Synthetic_Full_team_part_client_team_2_loss.dat};
\addlegendentry{$ 50\%$}

\end{axis}

\end{tikzpicture}
}

    \caption{Global Loss \\ (Validation)}
    \label{Fig: Synthetic_ftpc_gl_2}
\end{subfigure}
\begin{subfigure}[!t]{0.24\linewidth}
            \resizebox{\linewidth}{!}{
                \begin{tikzpicture}

\begin{axis}[
  xlabel= Global rounds,
  ylabel= Personalized Accuracy,
  legend pos=south west,
  legend style ={nodes={scale=1.4, transform shape}},
  xticklabels from table={ablation_study_team_Synthetic_Full_team_part_client_team_2_acc.dat}{GR},xtick=data
]
\addplot[blue,thick,mark=diamond*] table [y=p_part_4,x=GR]{ablation_study_team_Synthetic_Full_team_part_client_team_2_acc.dat};

\addplot[brown,dashed,thick,mark=star]  table [y=p_part_12,x=GR]{ablation_study_team_Synthetic_Full_team_part_client_team_2_acc.dat};

\addplot[black,dashed,thick,mark=triangle*]  table [y=p_part_20,x=GR]{ablation_study_team_Synthetic_Full_team_part_client_team_2_acc.dat};

\addplot[cyan,thick,mark=square*] table [y=p_part_50,x=GR]{ablation_study_team_Synthetic_Full_team_part_client_team_2_acc.dat};

\end{axis}

\end{tikzpicture}

            }
        \caption{Personalized Accuracy \\ (Validation)}
        \label{Fig: Synthetic_ftpc_pa_2}
    \end{subfigure} 
    \begin{subfigure}[!t]{0.24\linewidth}
        \resizebox{\linewidth}{!}{
        \begin{tikzpicture}
\begin{axis}[
  xlabel= Global rounds,
  ylabel= Global Accuracy,
  legend pos=south east,
  xticklabels from table={ablation_study_team_Synthetic_Full_team_part_client_team_2_acc.dat}{GR},xtick=data
]
\addplot[blue,thick,mark=diamond*] table [y=g_part_4,x=GR]{ablation_study_team_Synthetic_Full_team_part_client_team_2_acc.dat};

\addplot[brown,dashed,thick,mark=star]  table [y=g_part_12,x=GR]{ablation_study_team_Synthetic_Full_team_part_client_team_2_acc.dat};

\addplot[black,dashed,thick,mark=triangle*]  table [y=g_part_20,x=GR]{ablation_study_team_Synthetic_Full_team_part_client_team_2_acc.dat};

\addplot[cyan,thick,mark=square*] table [y=g_part_50,x=GR]{ablation_study_team_Synthetic_Full_team_part_client_team_2_acc.dat};

\end{axis}

\end{tikzpicture}
}

    \caption{Global Accuracy \\ (Validation)}
    \label{Fig: Synthetic_ftpc_ga_2}
\end{subfigure}

\caption{Full team participation (2 teams) but partial devices participation on Synthetic datasets in convex settings (MCLR)}
\label{Fig: Synthetic_ftpc_team_2}
\end{figure*}



\begin{figure*}[!ht]
        \centering
        \begin{subfigure}[!t]{0.24\linewidth}
            \resizebox{\linewidth}{!}{
                \begin{tikzpicture}
\begin{axis}[
  xlabel= Global rounds,
  ylabel= Loss,
  legend pos=north east,
  xticklabels from table={ablation_study_team_Synthetic_Full_team_part_client_team_4_loss.dat}{GR},xtick=data
]
\addplot[blue,thick,mark=diamond*] table [y=p_part_8,x=GR]{ablation_study_team_Synthetic_Full_team_part_client_team_4_loss.dat};

\addplot[brown,dashed,thick,mark=star]  table [y=p_part_12,x=GR]{ablation_study_team_Synthetic_Full_team_part_client_team_4_loss.dat};

\addplot[black,dashed,thick,mark=triangle*]  table [y=p_part_20,x=GR]{ablation_study_team_Synthetic_Full_team_part_client_team_4_loss.dat};

\addplot[cyan,thick,mark=square*] table [y=p_part_48,x=GR]{ablation_study_team_Synthetic_Full_team_part_client_team_4_loss.dat};

\end{axis}
\end{tikzpicture}

            }
        \caption{Personalized Loss \\ (Validation)}
        \label{Fig: Synthetic_ftpc_pl_4}
    \end{subfigure} 
    \begin{subfigure}[!t]{0.24\linewidth}
        \resizebox{\linewidth}{!}{
        \begin{tikzpicture}

\begin{axis}[
  xlabel= Global rounds,
  ylabel= Loss,
  legend pos=south west,
  legend style ={nodes={scale=1.4, transform shape}},
  xticklabels from table={ablation_study_team_Synthetic_Full_team_part_client_team_4_loss.dat}{GR},xtick=data
]
\addplot[blue,thick,mark=diamond*] table [y=g_part_8,x=GR]{ablation_study_team_Synthetic_Full_team_part_client_team_4_loss.dat};
\addlegendentry{$ 8\%$}

\addplot[brown,dashed,thick,mark=star]  table [y=g_part_12,x=GR]{ablation_study_team_Synthetic_Full_team_part_client_team_4_loss.dat};
\addlegendentry{$ 12\%$}

\addplot[black,dashed,thick,mark=triangle*]  table [y=g_part_20,x=GR]{ablation_study_team_Synthetic_Full_team_part_client_team_4_loss.dat};
\addlegendentry{$ 20\%$}

\addplot[cyan,thick,mark=square*] table [y=g_part_48,x=GR]{ablation_study_team_Synthetic_Full_team_part_client_team_4_loss.dat};
\addlegendentry{$ 48\%$}

\end{axis}

\end{tikzpicture}
}

    \caption{Global Loss \\ (Validation)}
    \label{Fig: Synthetic_ftpc_gl_4}
\end{subfigure}
\begin{subfigure}[!t]{0.24\linewidth}
            \resizebox{\linewidth}{!}{
                \begin{tikzpicture}

\begin{axis}[
  xlabel= Global rounds,
  ylabel= Personalized Accuracy,
  legend pos=south east,
  xticklabels from table={ablation_study_team_Synthetic_Full_team_part_client_team_4_acc.dat}{GR},xtick=data
]
\addplot[blue,thick,mark=diamond*] table [y=p_part_8,x=GR]{ablation_study_team_Synthetic_Full_team_part_client_team_4_acc.dat};

\addplot[brown,dashed,thick,mark=star]  table [y=p_part_12,x=GR]{ablation_study_team_Synthetic_Full_team_part_client_team_4_acc.dat};

\addplot[black,dashed,thick,mark=triangle*]  table [y=p_part_20,x=GR]{ablation_study_team_Synthetic_Full_team_part_client_team_4_acc.dat};

\addplot[cyan,thick,mark=square*] table [y=p_part_48,x=GR]{ablation_study_team_Synthetic_Full_team_part_client_team_4_acc.dat};

\end{axis}

\end{tikzpicture}

            }
        \caption{Personalized Accuracy \\ (Validation)}
        \label{Fig: Synthetic_ftpc_pa_4}
    \end{subfigure} 
    \begin{subfigure}[!t]{0.24\linewidth}
        \resizebox{\linewidth}{!}{
        \begin{tikzpicture}
\begin{axis}[
  xlabel= Global rounds,
  ylabel= Global Accuracy,
  legend pos=south east,
  xticklabels from table={ablation_study_team_Synthetic_Full_team_part_client_team_4_acc.dat}{GR},xtick=data
]
\addplot[blue,thick,mark=diamond*] table [y=g_part_8,x=GR]{ablation_study_team_Synthetic_Full_team_part_client_team_4_acc.dat};

\addplot[brown,dashed,thick,mark=star]  table [y=g_part_12,x=GR]{ablation_study_team_Synthetic_Full_team_part_client_team_4_acc.dat};

\addplot[black,dashed,thick,mark=triangle*]  table [y=g_part_20,x=GR]{ablation_study_team_Synthetic_Full_team_part_client_team_4_acc.dat};

\addplot[cyan,thick,mark=square*] table [y=g_part_48,x=GR]{ablation_study_team_Synthetic_Full_team_part_client_team_4_acc.dat};

\end{axis}

\end{tikzpicture}
}

    \caption{Global Accuracy \\ (Validation)}
    \label{Fig: Synthetic_ftpc_ga_4}
\end{subfigure}

\caption{Full team participation (4 teams) but partial devices participation on Synthetic datasets in convex settings (MCLR)}
\label{Fig: Synthetic_ftpc_team_4}
\end{figure*}
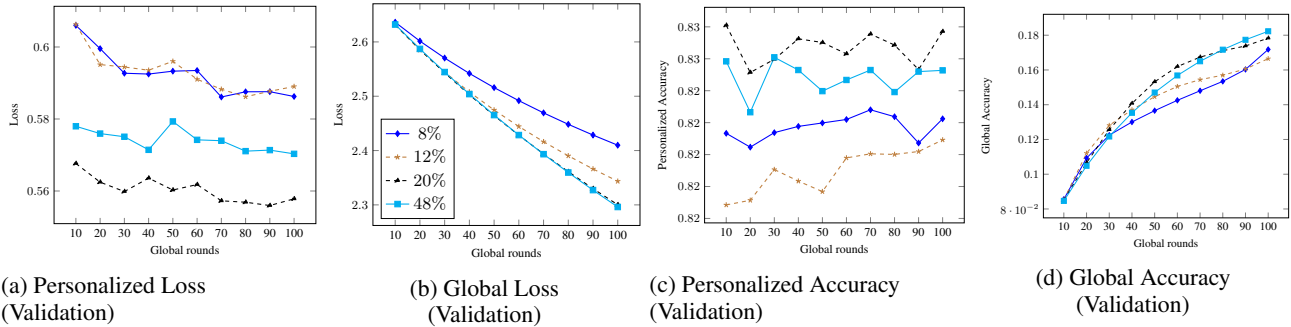



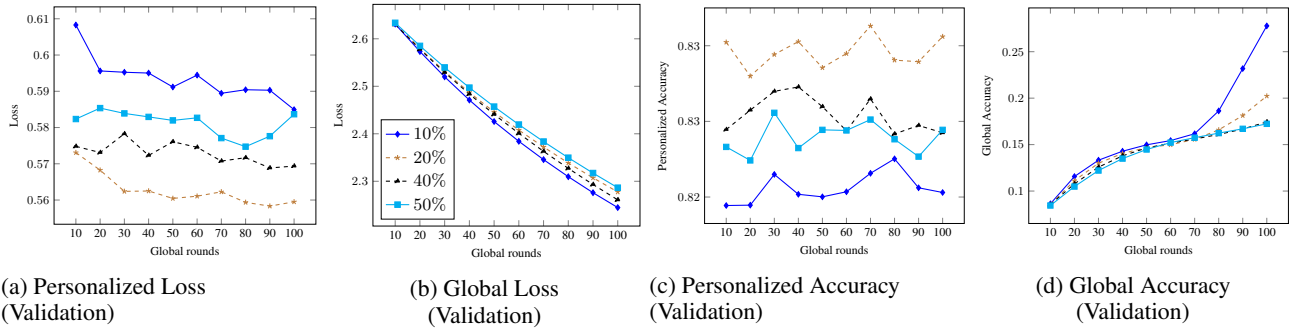
\begin{figure*}[!ht]
        \centering
        \begin{subfigure}[!t]{0.24\linewidth}
            \resizebox{\linewidth}{!}{
                \begin{tikzpicture}
\begin{axis}[
  xlabel= Global rounds,
  ylabel= Loss,
  legend pos=north east,
  xticklabels from table={ablation_study_team_Synthetic_Full_team_part_client_team_5_loss.dat}{GR},xtick=data
]
\addplot[blue,thick,mark=diamond*] table [y=p_part_10,x=GR]{ablation_study_team_Synthetic_Full_team_part_client_team_5_loss.dat};

\addplot[brown,dashed,thick,mark=star]  table [y=p_part_20,x=GR]{ablation_study_team_Synthetic_Full_team_part_client_team_5_loss.dat};

\addplot[black,dashed,thick,mark=triangle*]  table [y=p_part_40,x=GR]{ablation_study_team_Synthetic_Full_team_part_client_team_5_loss.dat};

\addplot[cyan,thick,mark=square*] table [y=p_part_50,x=GR]{ablation_study_team_Synthetic_Full_team_part_client_team_5_loss.dat};

\end{axis}
\end{tikzpicture}

            }
        \caption{Personalized Loss \\ (Validation)}
        \label{Fig: Synthetic_ftpc_pl_5}
    \end{subfigure} 
    \begin{subfigure}[!t]{0.24\linewidth}
        \resizebox{\linewidth}{!}{
        \begin{tikzpicture}

\begin{axis}[
  xlabel= Global rounds,
  ylabel= Loss,
  legend pos=south west,
  legend style ={nodes={scale=1.4, transform shape}},
  xticklabels from table={ablation_study_team_Synthetic_Full_team_part_client_team_5_loss.dat}{GR},xtick=data
]
\addplot[blue,thick,mark=diamond*] table [y=g_part_10,x=GR]{ablation_study_team_Synthetic_Full_team_part_client_team_5_loss.dat};
\addlegendentry{$ 10\%$}

\addplot[brown,dashed,thick,mark=star]  table [y=g_part_20,x=GR]{ablation_study_team_Synthetic_Full_team_part_client_team_5_loss.dat};
\addlegendentry{$ 20\%$}

\addplot[black,dashed,thick,mark=triangle*]  table [y=g_part_40,x=GR]{ablation_study_team_Synthetic_Full_team_part_client_team_5_loss.dat};
\addlegendentry{$ 40\%$}

\addplot[cyan,thick,mark=square*] table [y=g_part_50,x=GR]{ablation_study_team_Synthetic_Full_team_part_client_team_5_loss.dat};
\addlegendentry{$ 50\%$}

\end{axis}

\end{tikzpicture}
}

    \caption{Global Loss \\ (Validation)}
    \label{Fig: Synthetic_ftpc_gl_5}
\end{subfigure}
\begin{subfigure}[!t]{0.24\linewidth}
            \resizebox{\linewidth}{!}{
                \begin{tikzpicture}

\begin{axis}[
  xlabel= Global rounds,
  ylabel= Personalized Accuracy,
  legend pos=south east,
  xticklabels from table={ablation_study_team_Synthetic_Full_team_part_client_team_5_acc.dat}{GR},xtick=data
]
\addplot[blue,thick,mark=diamond*] table [y=p_part_10,x=GR]{ablation_study_team_Synthetic_Full_team_part_client_team_5_acc.dat};

\addplot[brown,dashed,thick,mark=star]  table [y=p_part_20,x=GR]{ablation_study_team_Synthetic_Full_team_part_client_team_5_acc.dat};

\addplot[black,dashed,thick,mark=triangle*]  table [y=p_part_40,x=GR]{ablation_study_team_Synthetic_Full_team_part_client_team_5_acc.dat};

\addplot[cyan,thick,mark=square*] table [y=p_part_50,x=GR]{ablation_study_team_Synthetic_Full_team_part_client_team_5_acc.dat};

\end{axis}

\end{tikzpicture}

            }
        \caption{Personalized Accuracy \\ (Validation)}
        \label{Fig: Synthetic_ftpc_pa_5}
    \end{subfigure} 
    \begin{subfigure}[!t]{0.24\linewidth}
        \resizebox{\linewidth}{!}{
        \begin{tikzpicture}
\begin{axis}[
  xlabel= Global rounds,
  ylabel= Global Accuracy,
  legend pos=south east,
  xticklabels from table={ablation_study_team_Synthetic_Full_team_part_client_team_5_acc.dat}{GR},xtick=data
]
\addplot[blue,thick,mark=diamond*] table [y=g_part_10,x=GR]{ablation_study_team_Synthetic_Full_team_part_client_team_5_acc.dat};

\addplot[brown,dashed,thick,mark=star]  table [y=g_part_20,x=GR]{ablation_study_team_Synthetic_Full_team_part_client_team_5_acc.dat};

\addplot[black,dashed,thick,mark=triangle*]  table [y=g_part_40,x=GR]{ablation_study_team_Synthetic_Full_team_part_client_team_5_acc.dat};

\addplot[cyan,thick,mark=square*] table [y=g_part_50,x=GR]{ablation_study_team_Synthetic_Full_team_part_client_team_5_acc.dat};

\end{axis}

\end{tikzpicture}
}

    \caption{Global Accuracy \\ (Validation)}
    \label{Fig: Synthetic_ftpc_ga_5}
\end{subfigure}

\caption{Full team participation (5 teams) but partial devices participation on Synthetic datasets in convex settings (MCLR)}
\label{Fig: Synthetic_ftpc_team_5}
\end{figure*}



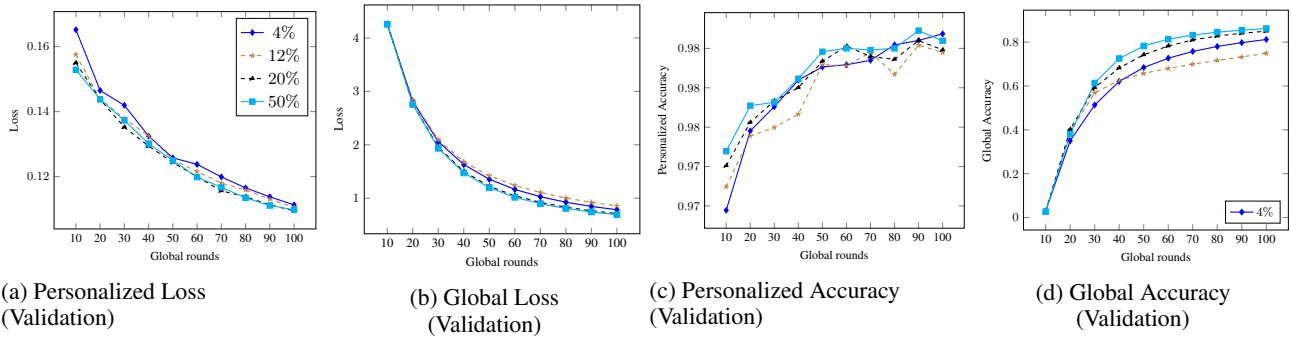
\begin{figure*}[!ht]
        \centering
        \begin{subfigure}[!t]{0.24\linewidth}
            \resizebox{\linewidth}{!}{
                \begin{tikzpicture}
\begin{axis}[
  xlabel= Global rounds,
  ylabel= Loss,
  legend pos=north east,
  legend style ={nodes={scale=1.4, transform shape}},
  xticklabels from table={ablation_study_team_Emnist_Full_team_part_client_team_2_loss.dat}{GR},xtick=data
]
\addplot[blue,thick,mark=diamond*] table [y=p_part_4,x=GR]{ablation_study_team_Emnist_Full_team_part_client_team_2_loss.dat};
\addlegendentry{$ 4\%$}

\addplot[brown,dashed,thick,mark=star]  table [y=p_part_12,x=GR]{ablation_study_team_Emnist_Full_team_part_client_team_2_loss.dat};
\addlegendentry{$ 12\%$}

\addplot[black,dashed,thick,mark=triangle*]  table [y=p_part_20,x=GR]{ablation_study_team_Emnist_Full_team_part_client_team_2_loss.dat};
\addlegendentry{$ 20\%$}

\addplot[cyan,thick,mark=square*] table [y=p_part_50,x=GR]{ablation_study_team_Emnist_Full_team_part_client_team_2_loss.dat};
\addlegendentry{$ 50\%$}

\end{axis}
\end{tikzpicture}

            }
        \caption{Personalized Loss \\ (Validation)}
        \label{Fig: Emnist_ftpc_pl_2}
    \end{subfigure} 
    \begin{subfigure}[!t]{0.24\linewidth}
        \resizebox{\linewidth}{!}{
        \begin{tikzpicture}

\begin{axis}[
  xlabel= Global rounds,
  ylabel= Loss,
  legend pos=north east,
  xticklabels from table={ablation_study_team_Emnist_Full_team_part_client_team_2_loss.dat}{GR},xtick=data
]
\addplot[blue,thick,mark=diamond*] table [y=g_part_4,x=GR]{ablation_study_team_Emnist_Full_team_part_client_team_2_loss.dat};

\addplot[brown,dashed,thick,mark=star]  table [y=g_part_12,x=GR]{ablation_study_team_Emnist_Full_team_part_client_team_2_loss.dat};

\addplot[black,dashed,thick,mark=triangle*]  table [y=g_part_20,x=GR]{ablation_study_team_Emnist_Full_team_part_client_team_2_loss.dat};

\addplot[cyan,thick,mark=square*] table [y=g_part_50,x=GR]{ablation_study_team_Emnist_Full_team_part_client_team_2_loss.dat};

\end{axis}

\end{tikzpicture}
}

    \caption{Global Loss \\ (Validation)}
    \label{Fig: Emnist_ftpc_gl_2}
\end{subfigure}
\begin{subfigure}[!t]{0.24\linewidth}
            \resizebox{\linewidth}{!}{
                \begin{tikzpicture}

\begin{axis}[
  xlabel= Global rounds,
  ylabel= Personalized Accuracy,
  legend pos=south east,
  xticklabels from table={ablation_study_team_Emnist_Full_team_part_client_team_2_acc.dat}{GR},xtick=data
]
\addplot[blue,thick,mark=diamond*] table [y=p_part_4,x=GR]{ablation_study_team_Emnist_Full_team_part_client_team_2_acc.dat};

\addplot[brown,dashed,thick,mark=star]  table [y=p_part_12,x=GR]{ablation_study_team_Emnist_Full_team_part_client_team_2_acc.dat};

\addplot[black,dashed,thick,mark=triangle*]  table [y=p_part_20,x=GR]{ablation_study_team_Emnist_Full_team_part_client_team_2_acc.dat};

\addplot[cyan,thick,mark=square*] table [y=p_part_50,x=GR]{ablation_study_team_Emnist_Full_team_part_client_team_2_acc.dat};

\end{axis}

\end{tikzpicture}

            }
        \caption{Personalized Accuracy \\ (Validation)}
        \label{Fig: Emnist_ftpc_pa_2}
    \end{subfigure} 
    \begin{subfigure}[!t]{0.24\linewidth}
        \resizebox{\linewidth}{!}{
        \begin{tikzpicture}
\begin{axis}[
  xlabel= Global rounds,
  ylabel= Global Accuracy,
  legend pos=south east,
  xticklabels from table={ablation_study_team_Emnist_Full_team_part_client_team_2_acc.dat}{GR},xtick=data
]
\addplot[blue,thick,mark=diamond*] table [y=g_part_4,x=GR]{ablation_study_team_Emnist_Full_team_part_client_team_2_acc.dat};
\addlegendentry{$ 4\%$}

\addplot[brown,dashed,thick,mark=star]  table [y=g_part_12,x=GR]{ablation_study_team_Emnist_Full_team_part_client_team_2_acc.dat};

\addplot[black,dashed,thick,mark=triangle*]  table [y=g_part_20,x=GR]{ablation_study_team_Emnist_Full_team_part_client_team_2_acc.dat};

\addplot[cyan,thick,mark=square*] table [y=g_part_50,x=GR]{ablation_study_team_Emnist_Full_team_part_client_team_2_acc.dat};

\end{axis}

\end{tikzpicture}
}

    \caption{Global Accuracy \\ (Validation)}
    \label{Fig: Emnist_ftpc_ga_2}
\end{subfigure}

\caption{Full team participation (2 teams) but partial devices participation on EMNIST datasets in convex settings (MCLR)}
\label{Fig: Emnist_ftpc_team_2}
\end{figure*}



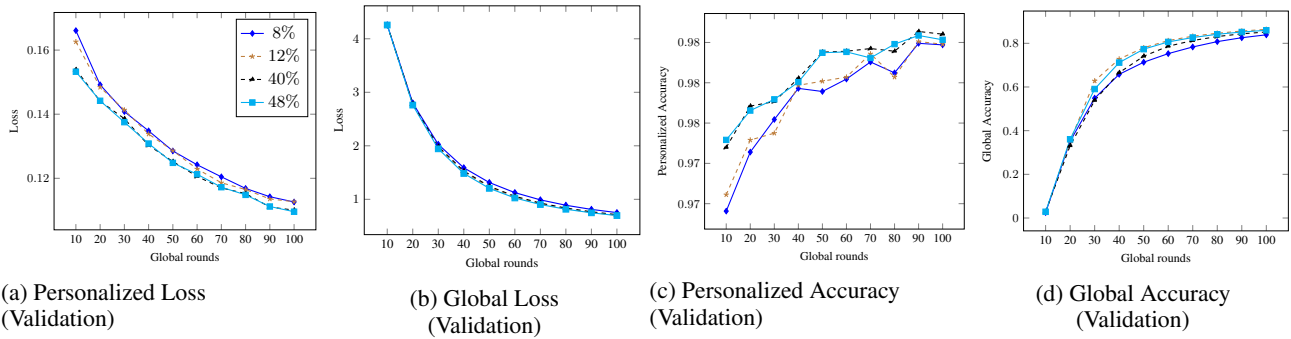
\begin{figure*}[!ht]
        \centering
        \begin{subfigure}[!t]{0.24\linewidth}
            \resizebox{\linewidth}{!}{
                \begin{tikzpicture}
\begin{axis}[
  xlabel= Global rounds,
  ylabel= Loss,
  legend pos=north east,
  legend style ={nodes={scale=1.4, transform shape}},
  xticklabels from table={ablation_study_team_Emnist_Full_team_part_client_team_4_loss.dat}{GR},xtick=data
]
\addplot[blue,thick,mark=diamond*] table [y=p_part_8,x=GR]{ablation_study_team_Emnist_Full_team_part_client_team_4_loss.dat};
\addlegendentry{$ 8\%$}

\addplot[brown,dashed,thick,mark=star]  table [y=p_part_12,x=GR]{ablation_study_team_Emnist_Full_team_part_client_team_4_loss.dat};
\addlegendentry{$ 12\%$}

\addplot[black,dashed,thick,mark=triangle*]  table [y=p_part_40,x=GR]{ablation_study_team_Emnist_Full_team_part_client_team_4_loss.dat};
\addlegendentry{$ 40\%$}

\addplot[cyan,thick,mark=square*] table [y=p_part_48,x=GR]{ablation_study_team_Emnist_Full_team_part_client_team_4_loss.dat};
\addlegendentry{$ 48\%$}

\end{axis}
\end{tikzpicture}

            }
        \caption{Personalized Loss \\ (Validation)}
        \label{Fig: Emnist_ftpc_pl_4}
    \end{subfigure} 
    \begin{subfigure}[!t]{0.24\linewidth}
        \resizebox{\linewidth}{!}{
        \begin{tikzpicture}

\begin{axis}[
  xlabel= Global rounds,
  ylabel= Loss,
  legend pos=north east,
  xticklabels from table={ablation_study_team_Emnist_Full_team_part_client_team_4_loss.dat}{GR},xtick=data
]
\addplot[blue,thick,mark=diamond*] table [y=g_part_8,x=GR]{ablation_study_team_Emnist_Full_team_part_client_team_4_loss.dat};

\addplot[brown,dashed,thick,mark=star]  table [y=g_part_12,x=GR]{ablation_study_team_Emnist_Full_team_part_client_team_4_loss.dat};

\addplot[black,dashed,thick,mark=triangle*]  table [y=g_part_40,x=GR]{ablation_study_team_Emnist_Full_team_part_client_team_4_loss.dat};

\addplot[cyan,thick,mark=square*] table [y=g_part_48,x=GR]{ablation_study_team_Emnist_Full_team_part_client_team_4_loss.dat};

\end{axis}

\end{tikzpicture}
}

    \caption{Global Loss \\ (Validation)}
    \label{Fig: Emnist_ftpc_gl_4}
\end{subfigure}
\begin{subfigure}[!t]{0.24\linewidth}
            \resizebox{\linewidth}{!}{
                \begin{tikzpicture}

\begin{axis}[
  xlabel= Global rounds,
  ylabel= Personalized Accuracy,
  legend pos=south east,
  xticklabels from table={ablation_study_team_Emnist_Full_team_part_client_team_4_acc.dat}{GR},xtick=data
]
\addplot[blue,thick,mark=diamond*] table [y=p_part_8,x=GR]{ablation_study_team_Emnist_Full_team_part_client_team_4_acc.dat};

\addplot[brown,dashed,thick,mark=star]  table [y=p_part_12,x=GR]{ablation_study_team_Emnist_Full_team_part_client_team_4_acc.dat};

\addplot[black,dashed,thick,mark=triangle*]  table [y=p_part_40,x=GR]{ablation_study_team_Emnist_Full_team_part_client_team_4_acc.dat};

\addplot[cyan,thick,mark=square*] table [y=p_part_48,x=GR]{ablation_study_team_Emnist_Full_team_part_client_team_4_acc.dat};

\end{axis}

\end{tikzpicture}

            }
        \caption{Personalized Accuracy \\ (Validation)}
        \label{Fig: Emnist_ftpc_pa_4}
    \end{subfigure} 
    \begin{subfigure}[!t]{0.24\linewidth}
        \resizebox{\linewidth}{!}{
        \begin{tikzpicture}
\begin{axis}[
  xlabel= Global rounds,
  ylabel= Global Accuracy,
  legend pos=south east,
  xticklabels from table={ablation_study_team_Emnist_Full_team_part_client_team_4_acc.dat}{GR},xtick=data
]
\addplot[blue,thick,mark=diamond*] table [y=g_part_8,x=GR]{ablation_study_team_Emnist_Full_team_part_client_team_4_acc.dat};

\addplot[brown,dashed,thick,mark=star]  table [y=g_part_12,x=GR]{ablation_study_team_Emnist_Full_team_part_client_team_4_acc.dat};

\addplot[black,dashed,thick,mark=triangle*]  table [y=g_part_40,x=GR]{ablation_study_team_Emnist_Full_team_part_client_team_4_acc.dat};

\addplot[cyan,thick,mark=square*] table [y=g_part_48,x=GR]{ablation_study_team_Emnist_Full_team_part_client_team_4_acc.dat};

\end{axis}

\end{tikzpicture}
}

    \caption{Global Accuracy \\ (Validation)}
    \label{Fig: Emnist_ftpc_ga_4}
\end{subfigure}

\caption{Full team participation (4 teams) but partial devices participation on EMNIST datasets in convex settings (MCLR)}
\label{Fig: Emnist_ftpc_team_4}
\end{figure*}



\begin{figure*}[!ht]
        \centering
        \begin{subfigure}[!t]{0.24\linewidth}
            \resizebox{\linewidth}{!}{
                \begin{tikzpicture}
\begin{axis}[
  xlabel= Global rounds,
  ylabel= Loss,
  legend pos=north east,
  legend style ={nodes={scale=1.4, transform shape}},
  xticklabels from table={ablation_study_team_Emnist_Full_team_part_client_team_5_loss.dat}{GR},xtick=data
]
\addplot[blue,thick,mark=diamond*] table [y=p_part_10,x=GR]{ablation_study_team_Emnist_Full_team_part_client_team_5_loss.dat};
\addlegendentry{$ 10\%$}

\addplot[brown,dashed,thick,mark=star]  table [y=p_part_20,x=GR]{ablation_study_team_Emnist_Full_team_part_client_team_5_loss.dat};
\addlegendentry{$ 20\%$}

\addplot[black,dashed,thick,mark=triangle*]  table [y=p_part_40,x=GR]{ablation_study_team_Emnist_Full_team_part_client_team_5_loss.dat};
\addlegendentry{$ 40\%$}

\addplot[cyan,thick,mark=square*] table [y=p_part_50,x=GR]{ablation_study_team_Emnist_Full_team_part_client_team_5_loss.dat};
\addlegendentry{$ 50\%$}

\end{axis}
\end{tikzpicture}

            }
        \caption{Personalized Loss \\ (Validation)}
        \label{Fig: Emnist_ftpc_pl_5}
    \end{subfigure} 
    \begin{subfigure}[!t]{0.24\linewidth}
        \resizebox{\linewidth}{!}{
        \begin{tikzpicture}

\begin{axis}[
  xlabel= Global rounds,
  ylabel= Loss,
  legend pos=north east,
  xticklabels from table={ablation_study_team_Emnist_Full_team_part_client_team_5_loss.dat}{GR},xtick=data
]
\addplot[blue,thick,mark=diamond*] table [y=g_part_10,x=GR]{ablation_study_team_Emnist_Full_team_part_client_team_5_loss.dat};

\addplot[brown,dashed,thick,mark=star]  table [y=g_part_20,x=GR]{ablation_study_team_Emnist_Full_team_part_client_team_5_loss.dat};

\addplot[black,dashed,thick,mark=triangle*]  table [y=g_part_40,x=GR]{ablation_study_team_Emnist_Full_team_part_client_team_5_loss.dat};

\addplot[cyan,thick,mark=square*] table [y=g_part_50,x=GR]{ablation_study_team_Emnist_Full_team_part_client_team_5_loss.dat};

\end{axis}

\end{tikzpicture}
}

    \caption{Global Loss \\ (Validation)}
    \label{Fig: Emnist_ftpc_gl_5}
\end{subfigure}
\begin{subfigure}[!t]{0.24\linewidth}
            \resizebox{\linewidth}{!}{
                \begin{tikzpicture}

\begin{axis}[
  xlabel= Global rounds,
  ylabel= Personalized Accuracy,
  legend pos=south east,
  xticklabels from table={ablation_study_team_Emnist_Full_team_part_client_team_5_acc.dat}{GR},xtick=data
]
\addplot[blue,thick,mark=diamond*] table [y=p_part_10,x=GR]{ablation_study_team_Emnist_Full_team_part_client_team_5_acc.dat};

\addplot[brown,dashed,thick,mark=star]  table [y=p_part_20,x=GR]{ablation_study_team_Emnist_Full_team_part_client_team_5_acc.dat};

\addplot[black,dashed,thick,mark=triangle*]  table [y=p_part_40,x=GR]{ablation_study_team_Emnist_Full_team_part_client_team_5_acc.dat};

\addplot[cyan,thick,mark=square*] table [y=p_part_50,x=GR]{ablation_study_team_Emnist_Full_team_part_client_team_5_acc.dat};

\end{axis}

\end{tikzpicture}

            }
        \caption{Personalized Accuracy \\ (Validation)}
        \label{Fig: Emnist_ftpc_pa_5}
    \end{subfigure} 
    \begin{subfigure}[!t]{0.24\linewidth}
        \resizebox{\linewidth}{!}{
        \begin{tikzpicture}
\begin{axis}[
  xlabel= Global rounds,
  ylabel= Global Accuracy,
  legend pos=south east,
  xticklabels from table={ablation_study_team_Emnist_Full_team_part_client_team_5_acc.dat}{GR},xtick=data
]
\addplot[blue,thick,mark=diamond*] table [y=g_part_10,x=GR]{ablation_study_team_Emnist_Full_team_part_client_team_5_acc.dat};

\addplot[brown,dashed,thick,mark=star]  table [y=g_part_20,x=GR]{ablation_study_team_Emnist_Full_team_part_client_team_5_acc.dat};

\addplot[black,dashed,thick,mark=triangle*]  table [y=g_part_40,x=GR]{ablation_study_team_Emnist_Full_team_part_client_team_5_acc.dat};

\addplot[cyan,thick,mark=square*] table [y=g_part_50,x=GR]{ablation_study_team_Emnist_Full_team_part_client_team_5_acc.dat};

\end{axis}

\end{tikzpicture}
}

    \caption{Global Accuracy \\ (Validation)}
    \label{Fig: Emnist_ftpc_ga_5}
\end{subfigure}

\caption{Full team participation (5 teams) but partial devices participation on EMNIST datasets in convex settings (MCLR)}
\label{Fig: Emnist_ftpc_team_5}
\end{figure*}
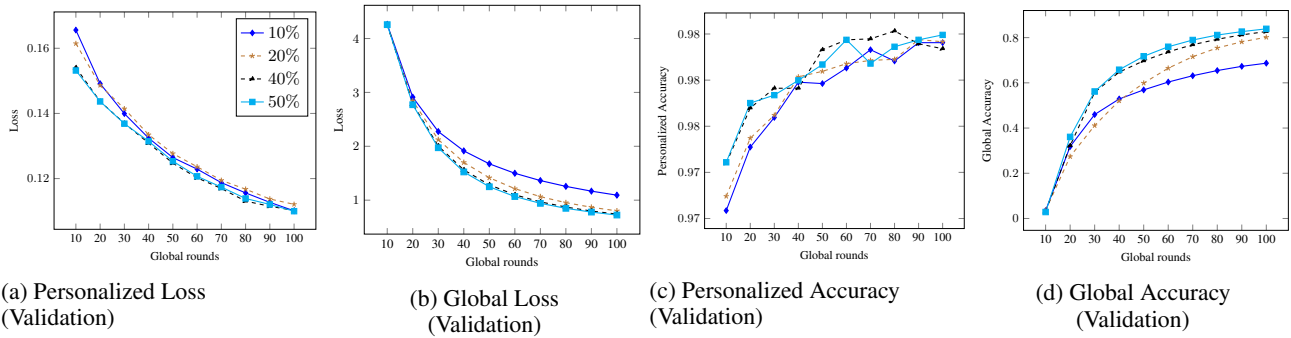


\clearpage
\subsubsection{Partial participation of Teams and full participation of Devices}

We conducted a series of experiments in both convex  (see \Cref{Fig: emnist_ptfc_mclr}, \Cref{Fig: mnist_ptfc_mclr}, and \Cref{Fig: fmnist_ptfc_mclr}) and non-convex settings (see \Cref{Fig: emnist_ptfc_cnn}, \Cref{Fig: mnist_ptfc_cnn}, and \Cref{Fig: fmnist_ptfc_cnn}) using EMNIST, MNIST, and FMNIST datasets. From there, we observed if the participation of teams is limited, then the personalized and global model both converges slowly. The same behaviour has been observed in all datasets for both convex and non-convex experiments (see \Cref{Fig: emnist_ptfc_mclr} to \Cref{Fig: fmnist_ptfc_cnn}).

\begin{figure*}[!ht]
        \centering
        \begin{subfigure}[!t]{0.24\linewidth}
            \resizebox{\linewidth}{!}{
                \begin{tikzpicture}
\begin{axis}[
  xlabel= Global rounds,
  ylabel= Loss,
  legend pos=north west,
  legend style ={nodes={scale=1.4, transform shape}},
  xticklabels from table={ablation_study_team_Emnist_Part_team_full_client_cnn_loss.dat}{GR},xtick=data
]
\addplot[blue,thick,mark=diamond*] table [y=pround2,x=GR]{ablation_study_team_Emnist_Part_team_full_client_mclr_loss.dat};
\addlegendentry{$ 2\%$}

\addplot[brown,dashed,thick,mark=square*]  table [y=pround4,x=GR]{ablation_study_team_Emnist_Part_team_full_client_mclr_loss.dat};
\addlegendentry{$ 4\%$}

\addplot[black,dashed,thick,mark=triangle*]  table [y=pround10,x=GR]{ablation_study_team_Emnist_Part_team_full_client_mclr_loss.dat};
\addlegendentry{$ 10\%$}

\end{axis}
\end{tikzpicture}

            }
        \caption{Personalized Loss \\ (Validation)}
        \label{Fig: emnist_ptfc_mclr_pl}
    \end{subfigure} 
    \begin{subfigure}[!t]{0.24\linewidth}
        \resizebox{\linewidth}{!}{
        \begin{tikzpicture}

\begin{axis}[
  xlabel= Global rounds,
  ylabel= Loss,
  legend pos=north east,
  xticklabels from table={ablation_study_team_Emnist_Part_team_full_client_mclr_loss.dat}{GR},xtick=data
]
\addplot[blue,thick,mark=diamond*] table [y=ground2,x=GR]{ablation_study_team_Emnist_Part_team_full_client_mclr_loss.dat};

\addplot[brown,dashed,thick,mark=square*]  table [y=ground4,x=GR]{ablation_study_team_Emnist_Part_team_full_client_mclr_loss.dat};

\addplot[black,dashed,thick,mark=triangle*]  table [y=ground10,x=GR]{ablation_study_team_Emnist_Part_team_full_client_mclr_loss.dat};

\end{axis}

\end{tikzpicture}
}

    \caption{Global Loss \\ (Validation)}
    \label{Fig: emnist_ptfc_mclr_gl}
\end{subfigure}
\begin{subfigure}[!t]{0.24\linewidth}
            \resizebox{\linewidth}{!}{
                \begin{tikzpicture}

\begin{axis}[
  xlabel= Global rounds,
  ylabel= Personalized Accuracy,
  legend pos=south east,
  xticklabels from table={ablation_study_team_Emnist_Part_team_full_client_mclr_acc.dat}{GR},xtick=data
]
\addplot[blue,thick,mark=diamond*] table [y=pround2,x=GR]{ablation_study_team_Emnist_Part_team_full_client_mclr_acc.dat};

\addplot[brown,dashed,thick,mark=square*]  table [y=pround4,x=GR]{ablation_study_team_Emnist_Part_team_full_client_mclr_acc.dat};

\addplot[black,dashed,thick,mark=triangle*]  table [y=pround10,x=GR]{ablation_study_team_Emnist_Part_team_full_client_mclr_acc.dat};

\end{axis}

\end{tikzpicture}

            }
        \caption{Personalized Accuracy \\ (Validation)}
        \label{Fig: emnist_ptfc_mclr_pa}
    \end{subfigure} 
    \begin{subfigure}[!t]{0.24\linewidth}
        \resizebox{\linewidth}{!}{
        \begin{tikzpicture}
\begin{axis}[
  xlabel= Global rounds,
  ylabel= Global Accuracy,
  legend pos=south east,
  xticklabels from table={ablation_study_team_Emnist_Part_team_full_client_mclr_acc.dat}{GR},xtick=data
]
\addplot[blue,thick,mark=diamond*] table [y=ground2,x=GR]{ablation_study_team_Emnist_Part_team_full_client_mclr_acc.dat};

\addplot[brown,dashed,thick,mark=square*]  table [y=ground4,x=GR]{ablation_study_team_Emnist_Part_team_full_client_mclr_acc.dat};

\addplot[black,dashed,thick,mark=triangle*]  table [y=ground10,x=GR]{ablation_study_team_Emnist_Part_team_full_client_mclr_acc.dat};

\end{axis}

\end{tikzpicture}
}

    \caption{Global Accuracy \\ (Validation)}
    \label{Fig: emnist_ptfc_mclr_ga}
\end{subfigure}

\caption{Partial participation of team ($2\%, 4\%, \ and\ 10\% $) and full devices participation of devices on EMNIST  datasets in convex settings (MCLR)}
\label{Fig: emnist_ptfc_mclr}
\end{figure*}
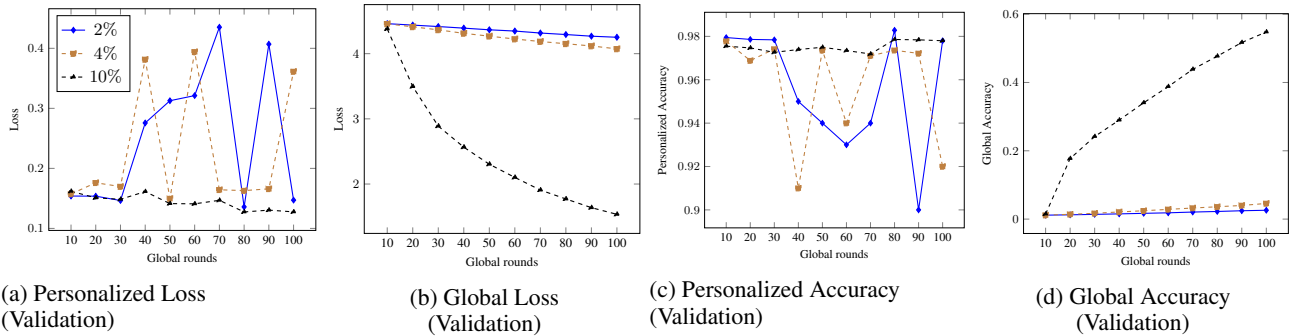


\begin{figure*}[!ht]
        \centering
        \begin{subfigure}[!t]{0.24\linewidth}
            \resizebox{\linewidth}{!}{
                \begin{tikzpicture}
\begin{axis}[
  xlabel= Global rounds,
  ylabel= Loss,
  legend pos=north east,
  legend style ={nodes={scale=1.4, transform shape}},
  xticklabels from table={ablation_study_team_Emnist_Part_team_full_client_cnn_loss.dat}{GR},xtick=data
]
\addplot[blue,thick,mark=diamond*] table [y=pteams_2,x=GR]{ablation_study_team_Emnist_Part_team_full_client_cnn_loss.dat};
\addlegendentry{$ 20\%$}

\addplot[brown,dashed,thick,mark=square*]  table [y=pteams_3,x=GR]{ablation_study_team_Emnist_Part_team_full_client_cnn_loss.dat};
\addlegendentry{$ 30\%$}

\addplot[black,dashed,thick,mark=triangle*]  table [y=pteams_4,x=GR]{ablation_study_team_Emnist_Part_team_full_client_cnn_loss.dat};
\addlegendentry{$ 40\%$}

\addplot[cyan,dashed,thick,mark=star]  table [y=pteams_5,x=GR]{ablation_study_team_Emnist_Part_team_full_client_cnn_loss.dat};
\addlegendentry{$ 50\%$}

\end{axis}
\end{tikzpicture}

            }
        \caption{Personalized Loss \\ (Validation)}
        \label{Fig: emnist_ptfc_cnn_pl}
    \end{subfigure} 
    \begin{subfigure}[!t]{0.24\linewidth}
        \resizebox{\linewidth}{!}{
        \begin{tikzpicture}

\begin{axis}[
  xlabel= Global rounds,
  ylabel= Loss,
  legend pos=north east,
  xticklabels from table={ablation_study_team_Emnist_Part_team_full_client_cnn_loss.dat}{GR},xtick=data
]
\addplot[blue,thick,mark=diamond*] table [y=gteams_2,x=GR]{ablation_study_team_Emnist_Part_team_full_client_cnn_loss.dat};

\addplot[brown,dashed,thick,mark=square*]  table [y=gteams_3,x=GR]{ablation_study_team_Emnist_Part_team_full_client_cnn_loss.dat};

\addplot[black,dashed,thick,mark=triangle*]  table [y=gteams_4,x=GR]{ablation_study_team_Emnist_Part_team_full_client_cnn_loss.dat};

\addplot[cyan,dashed,thick,mark=star]  table [y=gteams_5,x=GR]{ablation_study_team_Emnist_Part_team_full_client_cnn_loss.dat};

\end{axis}

\end{tikzpicture}
}

    \caption{Global Loss \\ (Validation)}
    \label{Fig: emnist_ptfc_cnn_gl}
\end{subfigure}
\begin{subfigure}[!t]{0.24\linewidth}
            \resizebox{\linewidth}{!}{
                \begin{tikzpicture}

\begin{axis}[
  xlabel= Global rounds,
  ylabel= Personalized Accuracy,
  legend pos=south east,
  xticklabels from table={ablation_study_team_Emnist_Part_team_full_client_cnn_acc.dat}{GR},xtick=data
]
\addplot[blue,thick,mark=diamond*] table [y=pteams_2,x=GR]{ablation_study_team_Emnist_Part_team_full_client_cnn_acc.dat};

\addplot[brown,dashed,thick,mark=square*]  table [y=pteams_3,x=GR]{ablation_study_team_Emnist_Part_team_full_client_cnn_acc.dat};

\addplot[black,dashed,thick,mark=triangle*]  table [y=pteams_4,x=GR]{ablation_study_team_Emnist_Part_team_full_client_cnn_acc.dat};

\addplot[cyan,dashed,thick,mark=star]  table [y=pteams_5,x=GR]{ablation_study_team_Emnist_Part_team_full_client_cnn_acc.dat};

\end{axis}

\end{tikzpicture}

            }
        \caption{Personalized Accuracy \\ (Validation)}
        \label{Fig: emnist_ptfc_cnn_pa}
    \end{subfigure} 
    \begin{subfigure}[!t]{0.24\linewidth}
        \resizebox{\linewidth}{!}{
        \begin{tikzpicture}
\begin{axis}[
  xlabel= Global rounds,
  ylabel= Global Accuracy,
  legend pos=south east,
  xticklabels from table={ablation_study_team_Emnist_Part_team_full_client_cnn_acc.dat}{GR},xtick=data
]
\addplot[blue,thick,mark=diamond*] table [y=gteams_2,x=GR]{ablation_study_team_Emnist_Part_team_full_client_cnn_acc.dat};

\addplot[brown,dashed,thick,mark=square*]  table [y=gteams_3,x=GR]{ablation_study_team_Emnist_Part_team_full_client_cnn_acc.dat};

\addplot[black,dashed,thick,mark=triangle*]  table [y=gteams_4,x=GR]{ablation_study_team_Emnist_Part_team_full_client_cnn_acc.dat};

\addplot[cyan,dashed,thick,mark=star]  table [y=gteams_5,x=GR]{ablation_study_team_Emnist_Part_team_full_client_cnn_acc.dat};

\end{axis}

\end{tikzpicture}
}

    \caption{Global Accuracy (Validation)}
    \label{Fig: emnist_ptfc_cnn_ga}
\end{subfigure}

\caption{Partial participation of team ($20\%, 30\%, 40\%, \ and\ 50\% $) and full devices participation of devices on EMNIST  datasets in convex settings (CNN)}
\label{Fig: emnist_ptfc_cnn}
\end{figure*}
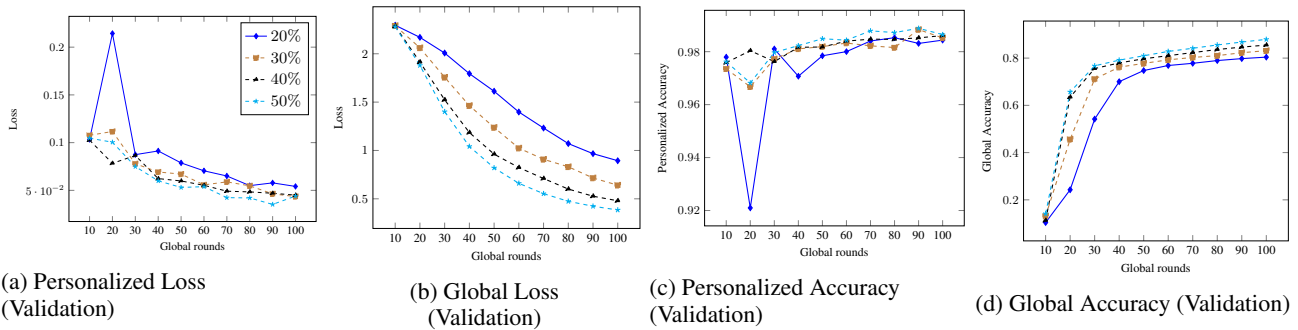


\begin{figure*}[!ht]
        \centering
        \begin{subfigure}[!t]{0.24\linewidth}
            \resizebox{\linewidth}{!}{
                \begin{tikzpicture}
\begin{axis}[
  xlabel= Global rounds,
  ylabel= Loss,
  legend pos=south west,
  legend style ={nodes={scale=1.4, transform shape}},
  xticklabels from table={ablation_study_team_Mnist_Part_team_full_client_cnn_loss.dat}{GR},xtick=data
]
\addplot[blue,thick,mark=diamond*] table [y=pteams_2,x=GR]{ablation_study_team_Mnist_Part_team_full_client_mclr_loss.dat};
\addlegendentry{$ 20\%$}

\addplot[brown,dashed,mark=square*] table [y=pteams_3,x=GR]{ablation_study_team_Mnist_Part_team_full_client_mclr_loss.dat};
\addlegendentry{$ 30\%$}

\addplot[black,dashed,thick,mark=triangle*]  table [y=pteams_4,x=GR]{ablation_study_team_Mnist_Part_team_full_client_mclr_loss.dat};
\addlegendentry{$ 40\%$}

\addplot[cyan,dashed,thick,mark=star]  table [y=pteams_5,x=GR]{ablation_study_team_Mnist_Part_team_full_client_mclr_loss.dat};
\addlegendentry{$ 50\%$}

\end{axis}
\end{tikzpicture}

            }
        \caption{Personalized Loss \\ (Validation)}
        \label{Fig: mnist_ptfc_mclr_pl}
    \end{subfigure} 
    \begin{subfigure}[!t]{0.24\linewidth}
        \resizebox{\linewidth}{!}{
        \begin{tikzpicture}

\begin{axis}[
  xlabel= Global rounds,
  ylabel= Loss,
  legend pos=north east,
  xticklabels from table={ablation_study_team_Mnist_Part_team_full_client_mclr_loss.dat}{GR},xtick=data
]
\addplot[blue,thick,mark=diamond*] table [y=gteams_2,x=GR]{ablation_study_team_Mnist_Part_team_full_client_mclr_loss.dat};

\addplot[brown,dashed,thick,mark=square*]  table [y=gteams_3,x=GR]{ablation_study_team_Mnist_Part_team_full_client_mclr_loss.dat};

\addplot[black,dashed,thick,mark=triangle*]  table [y=gteams_4,x=GR]{ablation_study_team_Mnist_Part_team_full_client_mclr_loss.dat};

\addplot[cyan,dashed,thick,mark=star]  table [y=gteams_5,x=GR]{ablation_study_team_Mnist_Part_team_full_client_mclr_loss.dat};

\end{axis}

\end{tikzpicture}
}

    \caption{Global Loss (Validation)}
    \label{Fig: mnist_ptfc_mclr_gl}
\end{subfigure}
\begin{subfigure}[!t]{0.24\linewidth}
            \resizebox{\linewidth}{!}{
                \begin{tikzpicture}

\begin{axis}[
  xlabel= Global rounds,
  ylabel= Personalized Accuracy,
  legend pos=south east,
  xticklabels from table={ablation_study_team_Mnist_Part_team_full_client_mclr_acc.dat}{GR},xtick=data
]
\addplot[blue,thick,mark=diamond*] table [y=pteams_2,x=GR]{ablation_study_team_Mnist_Part_team_full_client_mclr_acc.dat};

\addplot[brown,dashed,thick,mark=square*]  table [y=pteams_3,x=GR]{ablation_study_team_Mnist_Part_team_full_client_mclr_acc.dat};

\addplot[black,dashed,thick,mark=triangle*]  table [y=pteams_4,x=GR]{ablation_study_team_Mnist_Part_team_full_client_mclr_acc.dat};

\addplot[cyan,dashed,thick,mark=star]  table [y=pteams_5,x=GR]{ablation_study_team_Mnist_Part_team_full_client_mclr_acc.dat};

\end{axis}

\end{tikzpicture}

            }
        \caption{Personalized Accuracy \\ (Validation)}
        \label{Fig: mnist_ptfc_mclr_pa}
    \end{subfigure} 
    \begin{subfigure}[!t]{0.24\linewidth}
        \resizebox{\linewidth}{!}{
        \begin{tikzpicture}
\begin{axis}[
  xlabel= Global rounds,
  ylabel= Global Accuracy,
  legend pos=south east,
  xticklabels from table={ablation_study_team_Mnist_Part_team_full_client_mclr_acc.dat}{GR},xtick=data
]
\addplot[blue,thick,mark=diamond*] table [y=gteams_2,x=GR]{ablation_study_team_Mnist_Part_team_full_client_mclr_acc.dat};

\addplot[brown,dashed,thick,mark=square*]  table [y=gteams_3,x=GR]{ablation_study_team_Mnist_Part_team_full_client_mclr_acc.dat};

\addplot[black,dashed,thick,mark=triangle*]  table [y=gteams_4,x=GR]{ablation_study_team_Mnist_Part_team_full_client_mclr_acc.dat};

\addplot[cyan,dashed,thick,mark=star]  table [y=gteams_5,x=GR]{ablation_study_team_Mnist_Part_team_full_client_mclr_acc.dat};
\end{axis}

\end{tikzpicture}
}

    \caption{Global Accuracy \\ (Validation)}
    \label{Fig: mnist_ptfc_mclr_ga}
\end{subfigure}

\caption{Partial participation of team ($20\%, 30\%, 40\%, \ and\ 50\% $) and full devices participation of devices on MNISTdatasets in convex settings (MCLR)}
\label{Fig: mnist_ptfc_mclr}
\end{figure*}


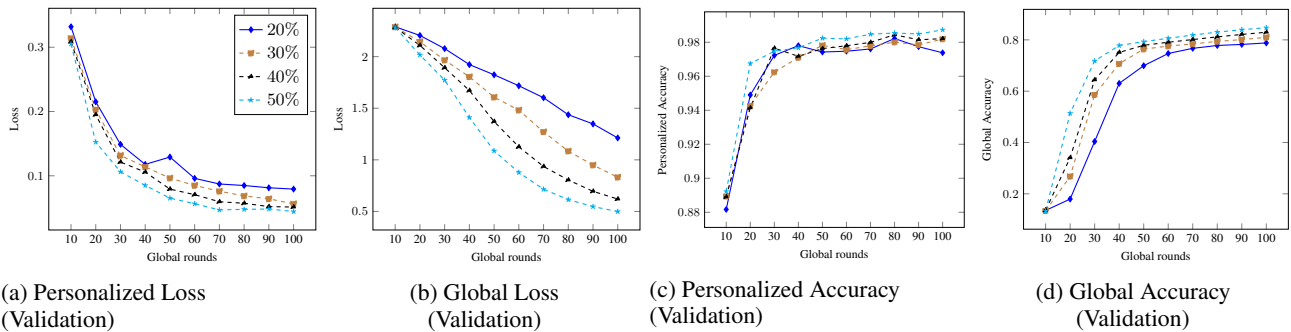
\begin{figure*}[!ht]
        \centering
        \begin{subfigure}[!t]{0.24\linewidth}
            \resizebox{\linewidth}{!}{
                \begin{tikzpicture}
\begin{axis}[
  xlabel= Global rounds,
  ylabel= Loss,
  legend pos=north east,
  legend style ={nodes={scale=1.4, transform shape}},
  xticklabels from table={ablation_study_team_Mnist_Part_team_full_client_cnn_loss.dat}{GR},xtick=data
]
\addplot[blue,thick,mark=diamond*] table [y=pteams_2,x=GR]{ablation_study_team_Mnist_Part_team_full_client_cnn_loss.dat};
\addlegendentry{$ 20\%$}

\addplot[brown,dashed,thick,mark=square*]  table [y=pteams_3,x=GR]{ablation_study_team_Mnist_Part_team_full_client_cnn_loss.dat};
\addlegendentry{$ 30\%$}

\addplot[black,dashed,thick,mark=triangle*]  table [y=pteams_4,x=GR]{ablation_study_team_Mnist_Part_team_full_client_cnn_loss.dat};
\addlegendentry{$ 40\%$}

\addplot[cyan,dashed,thick,mark=star]  table [y=pteams_5,x=GR]{ablation_study_team_Mnist_Part_team_full_client_cnn_loss.dat};
\addlegendentry{$ 50\%$}

\end{axis}
\end{tikzpicture}

            }
        \caption{Personalized Loss \\ (Validation)}
        \label{Fig: mnist_ptfc_cnn_pl}
    \end{subfigure} 
    \begin{subfigure}[!t]{0.24\linewidth}
        \resizebox{\linewidth}{!}{
        \begin{tikzpicture}

\begin{axis}[
  xlabel= Global rounds,
  ylabel= Loss,
  legend pos=north east,
  xticklabels from table={ablation_study_team_Mnist_Part_team_full_client_cnn_loss.dat}{GR},xtick=data
]
\addplot[blue,thick,mark=diamond*] table [y=gteams_2,x=GR]{ablation_study_team_Mnist_Part_team_full_client_cnn_loss.dat};

\addplot[brown,dashed,thick,mark=square*]  table [y=gteams_3,x=GR]{ablation_study_team_Mnist_Part_team_full_client_cnn_loss.dat};

\addplot[black,dashed,thick,mark=triangle*]  table [y=gteams_4,x=GR]{ablation_study_team_Mnist_Part_team_full_client_cnn_loss.dat};

\addplot[cyan,dashed,thick,mark=star]  table [y=gteams_5,x=GR]{ablation_study_team_Mnist_Part_team_full_client_cnn_loss.dat};

\end{axis}

\end{tikzpicture}
}

    \caption{Global Loss \\ (Validation)}
    \label{Fig: mnist_ptfc_cnn_gl}
\end{subfigure}
\begin{subfigure}[!t]{0.24\linewidth}
            \resizebox{\linewidth}{!}{
                \begin{tikzpicture}

\begin{axis}[
  xlabel= Global rounds,
  ylabel= Personalized Accuracy,
  legend pos=south east,
  xticklabels from table={ablation_study_team_Mnist_Part_team_full_client_cnn_acc.dat}{GR},xtick=data
]
\addplot[blue,thick,mark=diamond*] table [y=pteams_2,x=GR]{ablation_study_team_Mnist_Part_team_full_client_cnn_acc.dat};

\addplot[brown,dashed,thick,mark=square*]  table [y=pteams_3,x=GR]{ablation_study_team_Mnist_Part_team_full_client_cnn_acc.dat};

\addplot[black,dashed,thick,mark=triangle*]  table [y=pteams_4,x=GR]{ablation_study_team_Mnist_Part_team_full_client_cnn_acc.dat};

\addplot[cyan,dashed,thick,mark=star]  table [y=pteams_5,x=GR]{ablation_study_team_Mnist_Part_team_full_client_cnn_acc.dat};

\end{axis}

\end{tikzpicture}

            }
        \caption{Personalized Accuracy \\ (Validation)}
        \label{Fig: mnist_ptfc_cnn_pa}
    \end{subfigure} 
    \begin{subfigure}[!t]{0.24\linewidth}
        \resizebox{\linewidth}{!}{
        \begin{tikzpicture}
\begin{axis}[
  xlabel= Global rounds,
  ylabel= Global Accuracy,
  legend pos=south east,
  xticklabels from table={ablation_study_team_Mnist_Part_team_full_client_cnn_acc.dat}{GR},xtick=data
]
\addplot[blue,thick,mark=diamond*] table [y=gteams_2,x=GR]{ablation_study_team_Mnist_Part_team_full_client_cnn_acc.dat};

\addplot[brown,dashed,thick,mark=square*]  table [y=gteams_3,x=GR]{ablation_study_team_Mnist_Part_team_full_client_cnn_acc.dat};

\addplot[black,dashed,thick,mark=triangle*]  table [y=gteams_4,x=GR]{ablation_study_team_Mnist_Part_team_full_client_cnn_acc.dat};

\addplot[cyan,dashed,thick,mark=star]  table [y=gteams_5,x=GR]{ablation_study_team_Mnist_Part_team_full_client_cnn_acc.dat};

\end{axis}

\end{tikzpicture}
}

    \caption{Global Accuracy \\ (Validation)}
    \label{Fig: mnist_ptfc_cnn_ga}
\end{subfigure}

\caption{Partial participation of team ($20\%, 30\%, 40\%, \ and\ 50\% $) and full devices participation of devices on MNIST datasets in convex settings (CNN)}
\label{Fig: mnist_ptfc_cnn}
\end{figure*}


\begin{figure*}[!ht]
        \centering
        \begin{subfigure}[!t]{0.24\linewidth}
            \resizebox{\linewidth}{!}{
                \begin{tikzpicture}
\begin{axis}[
  xlabel= Global rounds,
  ylabel= Loss,
  legend pos=north east,
  legend style ={nodes={scale=1.4, transform shape}},
  xticklabels from table={ablation_study_team_FMnist_Part_team_full_client_cnn_loss.dat}{GR},xtick=data
]
\addplot[blue,thick,mark=diamond*] table [y=pteams_2,x=GR]{ablation_study_team_FMnist_Part_team_full_client_mclr_loss.dat};
\addlegendentry{$ 20\%$}

\addplot[brown,thick,mark=square*] table [y=pteams_3,x=GR]{ablation_study_team_FMnist_Part_team_full_client_mclr_loss.dat};
\addlegendentry{$ 30\%$}

\addplot[black,dashed,thick,mark=triangle*]  table [y=pteams_4,x=GR]{ablation_study_team_FMnist_Part_team_full_client_mclr_loss.dat};
\addlegendentry{$ 40\%$}

\addplot[green,dashed,thick,mark=star]  table [y=pteams_5,x=GR]{ablation_study_team_FMnist_Part_team_full_client_mclr_loss.dat};
\addlegendentry{$ 50\%$}

\end{axis}
\end{tikzpicture}

            }
        \caption{Personalized Loss \\ (Validation)}
        \label{Fig: fmnist_ptfc_mclr_pl}
    \end{subfigure} 
    \begin{subfigure}[!t]{0.24\linewidth}
        \resizebox{\linewidth}{!}{
        \begin{tikzpicture}

\begin{axis}[
  xlabel= Global rounds,
  ylabel= Loss,
  legend pos=north east,
  xticklabels from table={ablation_study_team_FMnist_Part_team_full_client_mclr_loss.dat}{GR},xtick=data
]
\addplot[blue,thick,mark=diamond*] table [y=gteams_2,x=GR]{ablation_study_team_FMnist_Part_team_full_client_mclr_loss.dat};

\addplot[brown,dashed,thick,mark=square*]  table [y=gteams_3,x=GR]{ablation_study_team_FMnist_Part_team_full_client_mclr_loss.dat};

\addplot[black,dashed,thick,mark=triangle*]  table [y=gteams_4,x=GR]{ablation_study_team_FMnist_Part_team_full_client_mclr_loss.dat};

\addplot[green,dashed,thick,mark=star]  table [y=gteams_5,x=GR]{ablation_study_team_FMnist_Part_team_full_client_mclr_loss.dat};

\end{axis}

\end{tikzpicture}
}

    \caption{Global Loss \\ (Validation)}
    \label{Fig: fmnist_ptfc_mclr_gl}
\end{subfigure}
\begin{subfigure}[!t]{0.24\linewidth}
            \resizebox{\linewidth}{!}{
                \begin{tikzpicture}

\begin{axis}[
  xlabel= Global rounds,
  ylabel= Personalized Accuracy,
  legend pos=south east,
  xticklabels from table={ablation_study_team_FMnist_Part_team_full_client_mclr_acc.dat}{GR},xtick=data
]
\addplot[blue,thick,mark=diamond*] table [y=pteams_2,x=GR]{ablation_study_team_FMnist_Part_team_full_client_mclr_acc.dat};

\addplot[brown,dashed,thick,mark=square*]  table [y=pteams_3,x=GR]{ablation_study_team_FMnist_Part_team_full_client_mclr_acc.dat};

\addplot[black,dashed,thick,mark=triangle*]  table [y=pteams_4,x=GR]{ablation_study_team_FMnist_Part_team_full_client_mclr_acc.dat};

\addplot[green,dashed,thick,mark=star]  table [y=pteams_5,x=GR]{ablation_study_team_FMnist_Part_team_full_client_mclr_acc.dat};

\end{axis}

\end{tikzpicture}

            }
        \caption{Personalized Accuracy \\ (Validation)}
        \label{Fig: fmnist_ptfc_mclr_pa}
    \end{subfigure} 
    \begin{subfigure}[!t]{0.24\linewidth}
        \resizebox{\linewidth}{!}{
        \begin{tikzpicture}
\begin{axis}[
  xlabel= Global rounds,
  ylabel= Global Accuracy,
  legend pos=south east,
  xticklabels from table={ablation_study_team_FMnist_Part_team_full_client_mclr_acc.dat}{GR},xtick=data
]
\addplot[blue,thick,mark=diamond*] table [y=gteams_2,x=GR]{ablation_study_team_FMnist_Part_team_full_client_mclr_acc.dat};

\addplot[brown,dashed,thick,mark=square*]  table [y=gteams_3,x=GR]{ablation_study_team_FMnist_Part_team_full_client_mclr_acc.dat};

\addplot[black,dashed,thick,mark=triangle*]  table [y=gteams_4,x=GR]{ablation_study_team_FMnist_Part_team_full_client_mclr_acc.dat};

\addplot[green,dashed,thick,mark=star]  table [y=gteams_5,x=GR]{ablation_study_team_FMnist_Part_team_full_client_mclr_acc.dat};
\end{axis}

\end{tikzpicture}
}

    \caption{Global Accuracy \\ (Validation)}
    \label{Fig: fmnist_ptfc_mclr_ga}
\end{subfigure}

\caption{Partial participation of team ($20\%, 30\%, 40\%, \ and\ 50\% $) and full devices participation of devices on FMNIST datasets in convex settings (MCLR)}
\label{Fig: fmnist_ptfc_mclr}
\end{figure*}
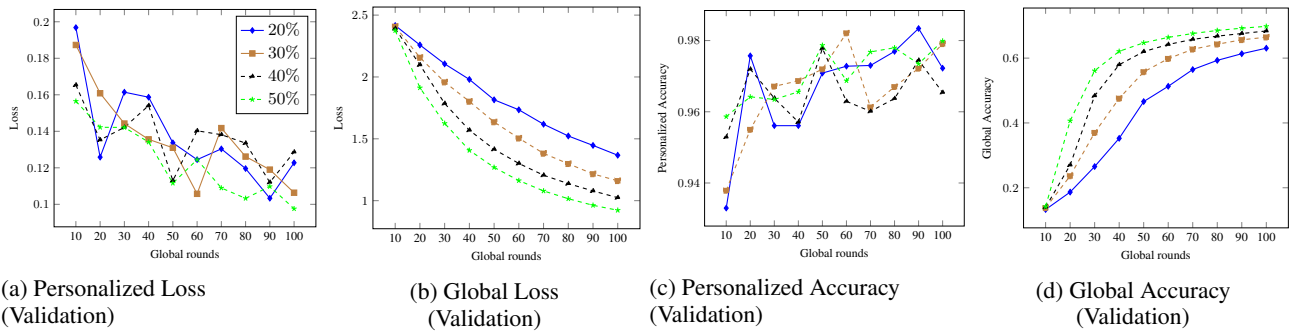


\begin{figure*}[!ht]
        \centering
        \begin{subfigure}[!t]{0.24\linewidth}
            \resizebox{\linewidth}{!}{
                \begin{tikzpicture}
\begin{axis}[
  xlabel= Global rounds,
  ylabel= Loss,
  legend pos=south west,
  legend style ={nodes={scale=1.4, transform shape}},
  xticklabels from table={ablation_study_team_FMnist_Part_team_full_client_cnn_loss.dat}{GR},xtick=data
]
\addplot[blue,thick,mark=diamond*] table [y=pteams_2,x=GR]{ablation_study_team_FMnist_Part_team_full_client_cnn_loss.dat};
\addlegendentry{$ 20\%$}

\addplot[brown,dashed,thick,mark=square*]  table [y=pteams_3,x=GR]{ablation_study_team_FMnist_Part_team_full_client_cnn_loss.dat};
\addlegendentry{$ 30\%$}

\addplot[black,dashed,thick,mark=triangle*]  table [y=pteams_4,x=GR]{ablation_study_team_FMnist_Part_team_full_client_cnn_loss.dat};
\addlegendentry{$ 40\%$}

\addplot[cyan,dashed,thick,mark=star]  table [y=pteams_5,x=GR]{ablation_study_team_FMnist_Part_team_full_client_cnn_loss.dat};
\addlegendentry{$ 50\%$}

\end{axis}
\end{tikzpicture}

            }
        \caption{Personalized Loss \\ (Validation)}
        \label{Fig: fmnist_ptfc_cnn_pl}
    \end{subfigure} 
    \begin{subfigure}[!t]{0.24\linewidth}
        \resizebox{\linewidth}{!}{
        \begin{tikzpicture}

\begin{axis}[
  xlabel= Global rounds,
  ylabel= Loss,
  legend pos=south west,
  xticklabels from table={ablation_study_team_FMnist_Part_team_full_client_cnn_loss.dat}{GR},xtick=data
]
\addplot[blue,thick,mark=diamond*] table [y=gteams_2,x=GR]{ablation_study_team_FMnist_Part_team_full_client_cnn_loss.dat};

\addplot[brown,dashed,thick,mark=square*]  table [y=gteams_3,x=GR]{ablation_study_team_FMnist_Part_team_full_client_cnn_loss.dat};

\addplot[black,dashed,thick,mark=triangle*]  table [y=gteams_4,x=GR]{ablation_study_team_FMnist_Part_team_full_client_cnn_loss.dat};

\addplot[cyan,dashed,thick,mark=star]  table [y=gteams_5,x=GR]{ablation_study_team_FMnist_Part_team_full_client_cnn_loss.dat};

\end{axis}

\end{tikzpicture}
}

    \caption{Global Loss \\ (Validation)}
    \label{Fig: fmnist_ptfc_cnn_gl}
\end{subfigure}
\begin{subfigure}[!t]{0.24\linewidth}
            \resizebox{\linewidth}{!}{
                \begin{tikzpicture}

\begin{axis}[
  xlabel= Global rounds,
  ylabel= Personalized Accuracy,
  legend pos=south east,
  xticklabels from table={ablation_study_team_FMnist_Part_team_full_client_cnn_acc.dat}{GR},xtick=data
]
\addplot[blue,thick,mark=diamond*] table [y=pteams_2,x=GR]{ablation_study_team_FMnist_Part_team_full_client_cnn_acc.dat};

\addplot[brown,dashed,thick,mark=square*]  table [y=pteams_3,x=GR]{ablation_study_team_FMnist_Part_team_full_client_cnn_acc.dat};

\addplot[black,dashed,thick,mark=triangle*]  table [y=pteams_4,x=GR]{ablation_study_team_FMnist_Part_team_full_client_cnn_acc.dat};

\addplot[cyan,dashed,thick,mark=star]  table [y=pteams_5,x=GR]{ablation_study_team_FMnist_Part_team_full_client_cnn_acc.dat};

\end{axis}

\end{tikzpicture}

            }
        \caption{Personalized Accuracy \\ (Validation)}
        \label{Fig: fmnist_ptfc_cnn_pa}
    \end{subfigure} 
    \begin{subfigure}[!t]{0.24\linewidth}
        \resizebox{\linewidth}{!}{
        \begin{tikzpicture}
\begin{axis}[
  xlabel= Global rounds,
  ylabel= Global Accuracy,
  legend pos=south east,
  xticklabels from table={ablation_study_team_FMnist_Part_team_full_client_cnn_acc.dat}{GR},xtick=data
]
\addplot[blue,thick,mark=diamond*] table [y=gteams_2,x=GR]{ablation_study_team_FMnist_Part_team_full_client_cnn_acc.dat};

\addplot[brown,dashed,thick,mark=square*]  table [y=gteams_3,x=GR]{ablation_study_team_FMnist_Part_team_full_client_cnn_acc.dat};

\addplot[black,dashed,thick,mark=triangle*]  table [y=gteams_4,x=GR]{ablation_study_team_FMnist_Part_team_full_client_cnn_acc.dat};

\addplot[cyan,dashed,thick,mark=star]  table [y=gteams_5,x=GR]{ablation_study_team_FMnist_Part_team_full_client_cnn_acc.dat};

\end{axis}

\end{tikzpicture}
}

    \caption{Global Accuracy \\ (Validation)}
    \label{Fig: fmnist_ptfc_cnn_ga}
\end{subfigure}

\caption{Partial participation of teams ($20\%, 30\%, 40\%, \ and\ 50\% $) and full devices participation of devices on FMNIST datasets in convex settings (CNN)}
\label{Fig: fmnist_ptfc_cnn}
\end{figure*}
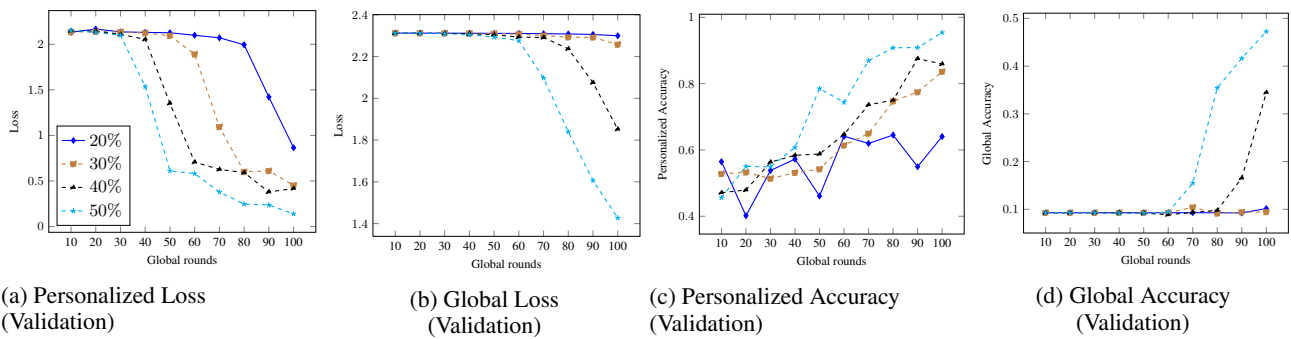

\clearpage
\subsubsection{Partial participation of Teams and Devices}
We conducted experiments for convex scenarios on EMNIST, MNIST, and FMNIST datasets (see \Cref{Fig: emnist_ptpc_team_2_mclr}, \Cref{Fig: mnist_ptpc_team_2_mclr}, and \Cref{Fig: fmnist_ptpc_team_20_mclr}). From the experiments, it was observed that when 20\% of teams participated in each global round, increasing the number of devices participating in each team iteration  resulted in improved convergence of both personalized and global models.

\begin{figure*}[!ht]
        \centering
        \begin{subfigure}[!t]{0.24\linewidth}
            \resizebox{\linewidth}{!}{
                \begin{tikzpicture}
\begin{axis}[
  xlabel= Global rounds,
  ylabel= Loss,
  legend pos=north east,
  legend style ={nodes={scale=1.4, transform shape}},
  xticklabels from table={ablation_study_team_Emnist_Part_team_part_client_team_2_mclr_loss.dat}{GR},xtick=data
]
\addplot[blue,thick,mark=diamond*] table [y=pusers_2,x=GR]{ablation_study_team_Emnist_Part_team_part_client_team_2_mclr_loss.dat};

\addplot[brown,thick,mark=diamond*] table [y=pusers_4,x=GR]{ablation_study_team_Emnist_Part_team_part_client_team_2_mclr_loss.dat};

\addplot[black,dashed,thick,mark=oval*]  table [y=pusers_10,x=GR]{ablation_study_team_Emnist_Part_team_part_client_team_2_mclr_loss.dat};

\addplot[cyan,dashed,thick,mark=triangle*]  table [y=pusers_20,x=GR]{ablation_study_team_Emnist_Part_team_part_client_team_2_mclr_loss.dat};

\end{axis}
\end{tikzpicture}

            }
        \caption{Personalized Loss \\ (Validation)}
        \label{Fig: emnist_ptpc_team_2_mclr_pl}
    \end{subfigure} 
    \begin{subfigure}[!t]{0.24\linewidth}
        \resizebox{\linewidth}{!}{
        \begin{tikzpicture}

\begin{axis}[
  xlabel= Global rounds,
  ylabel= Loss,
  legend pos=north east,
  legend style ={nodes={scale=1.4, transform shape}},
  xticklabels from table={ablation_study_team_Emnist_Part_team_part_client_team_2_mclr_loss.dat}{GR},xtick=data
]
\addplot[blue,thick,mark=diamond*] table [y=gusers_2,x=GR]{ablation_study_team_Emnist_Part_team_part_client_team_2_mclr_loss.dat};
\addlegendentry{$ 2\%$}

\addplot[brown,dashed,thick,mark=oval*]  table [y=gusers_4,x=GR]{ablation_study_team_Emnist_Part_team_part_client_team_2_mclr_loss.dat};
\addlegendentry{$ 4\%$}

\addplot[black,dashed,thick,mark=triangle*]  table [y=gusers_10,x=GR]{ablation_study_team_Emnist_Part_team_part_client_team_2_mclr_loss.dat};
\addlegendentry{$ 10\%$}

\addplot[cyan,dashed,thick,mark=triangle*]  table [y=gusers_20,x=GR]{ablation_study_team_Emnist_Part_team_part_client_team_2_mclr_loss.dat};
\addlegendentry{$ 20\%$}

\end{axis}

\end{tikzpicture}
}

    \caption{Global Loss (Validation)}
    \label{Fig: emnist_ptpc_team_2_mclr_gl}
\end{subfigure}
\begin{subfigure}[!t]{0.24\linewidth}
            \resizebox{\linewidth}{!}{
                \begin{tikzpicture}

\begin{axis}[
  xlabel= Global rounds,
  ylabel= Personalized Accuracy,
  legend pos=south east,
  xticklabels from table={ablation_study_team_Emnist_Part_team_part_client_team_2_mclr_acc.dat}{GR},xtick=data
]
\addplot[blue,thick,mark=diamond*] table [y=pusers_2,x=GR]{ablation_study_team_Emnist_Part_team_part_client_team_2_mclr_acc.dat};

\addplot[brown,dashed,thick,mark=oval*]  table [y=pusers_4,x=GR]{ablation_study_team_Emnist_Part_team_part_client_team_2_mclr_acc.dat};

\addplot[black,dashed,thick,mark=triangle*]  table [y=pusers_10,x=GR]{ablation_study_team_Emnist_Part_team_part_client_team_2_mclr_acc.dat};

\addplot[cyan,dashed,thick,mark=triangle*]  table [y=pusers_20,x=GR]{ablation_study_team_Emnist_Part_team_part_client_team_2_mclr_acc.dat};

\end{axis}

\end{tikzpicture}

            }
        \caption{Personalized Accuracy \\ (Validation)}
        \label{Fig: emnist_ptpc_team_2_mclr_pa}
    \end{subfigure} 
    \begin{subfigure}[!t]{0.24\linewidth}
        \resizebox{\linewidth}{!}{
        \begin{tikzpicture}
\begin{axis}[
  xlabel= Global rounds,
  ylabel= Global Accuracy,
  legend pos=south east,
  xticklabels from table={ablation_study_team_Emnist_Part_team_part_client_team_2_mclr_acc.dat}{GR},xtick=data
]
\addplot[blue,thick,mark=diamond*] table [y=gusers_2,x=GR]{ablation_study_team_Emnist_Part_team_part_client_team_2_mclr_acc.dat};

\addplot[brown,dashed,thick,mark=oval*]  table [y=gusers_4,x=GR]{ablation_study_team_Emnist_Part_team_part_client_team_2_mclr_acc.dat};

\addplot[black,dashed,thick,mark=triangle*]  table [y=gusers_10,x=GR]{ablation_study_team_Emnist_Part_team_part_client_team_2_mclr_acc.dat};

\addplot[cyan,dashed,thick,mark=triangle*]  table [y=gusers_20,x=GR]{ablation_study_team_Emnist_Part_team_part_client_team_2_mclr_acc.dat};
\end{axis}

\end{tikzpicture}
}

    \caption{Global Accuracy (Validation)}
    \label{Fig: emnist_ptpc_team_2_mclr_ga}
\end{subfigure}

\caption{Partial participation of team ($20\%$) and devices ( $2\%, 4\%, 10\%,\ and\ 30\%$)  on EMNIST datasets in convex settings (MCLR)}
\label{Fig: emnist_ptpc_team_2_mclr}
\end{figure*}
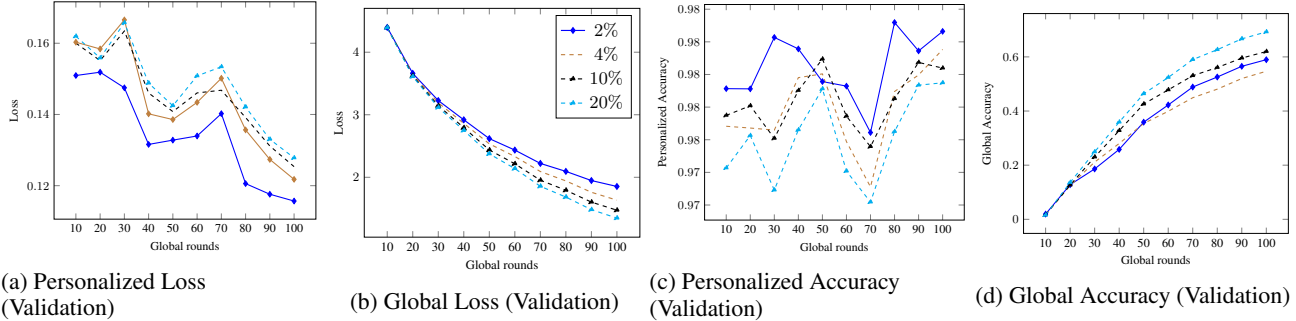


\begin{figure*}[!ht]
        \centering
        \begin{subfigure}[!t]{0.24\linewidth}
            \resizebox{\linewidth}{!}{
                \begin{tikzpicture}
\begin{axis}[
  xlabel= Global rounds,
  ylabel= Loss,
  legend pos=north east,
  legend style ={nodes={scale=1.4, transform shape}},
  xticklabels from table={ablation_study_team_Mnist_Part_team_part_client_team_2_mclr_loss.dat}{GR},xtick=data
]
\addplot[blue,thick,mark=diamond*] table [y=pusers_2,x=GR]{ablation_study_team_Mnist_Part_team_part_client_team_2_mclr_loss.dat};
\addlegendentry{$ 2\%$}

\addplot[brown,thick,mark=square*] table [y=pusers_4,x=GR]{ablation_study_team_Mnist_Part_team_part_client_team_2_mclr_loss.dat};
\addlegendentry{$ 4\%$}

\addplot[black,dashed,thick,mark=triangle*]  table [y=pusers_10,x=GR]{ablation_study_team_Mnist_Part_team_part_client_team_2_mclr_loss.dat};
\addlegendentry{$ 10\%$}

\addplot[cyan,dashed,thick,mark=star]  table [y=pusers_20,x=GR]{ablation_study_team_Mnist_Part_team_part_client_team_2_mclr_loss.dat};
\addlegendentry{$ 20\%$}

\end{axis}
\end{tikzpicture}

            }
        \caption{Personalized Loss \\ (Validation)}
        \label{Fig: mnist_ptpc_team_2_mclr_pl}
    \end{subfigure} 
    \begin{subfigure}[!t]{0.24\linewidth}
        \resizebox{\linewidth}{!}{
        \begin{tikzpicture}

\begin{axis}[
  xlabel= Global rounds,
  ylabel= Loss,
  legend pos=north east,
  legend style ={nodes={scale=1.4, transform shape}},
  xticklabels from table={ablation_study_team_Mnist_Part_team_part_client_team_2_mclr_loss.dat}{GR},xtick=data
]
\addplot[blue,thick,mark=diamond*] table [y=gusers_2,x=GR]{ablation_study_team_Mnist_Part_team_part_client_team_2_mclr_loss.dat};

\addplot[brown,dashed,thick,mark=square*]  table [y=gusers_4,x=GR]{ablation_study_team_Mnist_Part_team_part_client_team_2_mclr_loss.dat};

\addplot[black,dashed,thick,mark=triangle*]  table [y=gusers_10,x=GR]{ablation_study_team_Mnist_Part_team_part_client_team_2_mclr_loss.dat};

\addplot[cyan,dashed,thick,mark=star]  table [y=gusers_20,x=GR]{ablation_study_team_Mnist_Part_team_part_client_team_2_mclr_loss.dat};

\end{axis}

\end{tikzpicture}
}

    \caption{Global Loss (Validation)}
    \label{Fig: mnist_ptpc_team_2_mclr_gl}
\end{subfigure}
\begin{subfigure}[!t]{0.24\linewidth}
            \resizebox{\linewidth}{!}{
                \begin{tikzpicture}

\begin{axis}[
  xlabel= Global rounds,
  ylabel= Personalized Accuracy,
  legend pos=south east,
  xticklabels from table={ablation_study_team_Mnist_Part_team_part_client_team_2_mclr_acc.dat}{GR},xtick=data
]
\addplot[blue,thick,mark=diamond*] table [y=pusers_2,x=GR]{ablation_study_team_Mnist_Part_team_part_client_team_2_mclr_acc.dat};

\addplot[brown,dashed,thick,mark=square*]  table [y=pusers_4,x=GR]{ablation_study_team_Mnist_Part_team_part_client_team_2_mclr_acc.dat};

\addplot[black,dashed,thick,mark=triangle*]  table [y=pusers_10,x=GR]{ablation_study_team_Mnist_Part_team_part_client_team_2_mclr_acc.dat};

\addplot[cyan,dashed,thick,mark=star]  table [y=pusers_20,x=GR]{ablation_study_team_Mnist_Part_team_part_client_team_2_mclr_acc.dat};

\end{axis}

\end{tikzpicture}

            }
        \caption{Personalized Accuracy \\ (Validation)}
        \label{Fig: mnist_ptpc_team_2_mclr_pa}
    \end{subfigure} 
    \begin{subfigure}[!t]{0.24\linewidth}
        \resizebox{\linewidth}{!}{
        \begin{tikzpicture}
\begin{axis}[
  xlabel= Global rounds,
  ylabel= Global Accuracy,
  legend pos=south east,
  xticklabels from table={ablation_study_team_Mnist_Part_team_part_client_team_2_mclr_acc.dat}{GR},xtick=data
]
\addplot[blue,thick,mark=diamond*] table [y=gusers_2,x=GR]{ablation_study_team_Mnist_Part_team_part_client_team_2_mclr_acc.dat};

\addplot[brown,dashed,thick,mark=square*]  table [y=gusers_4,x=GR]{ablation_study_team_Mnist_Part_team_part_client_team_2_mclr_acc.dat};

\addplot[black,dashed,thick,mark=triangle*]  table [y=gusers_10,x=GR]{ablation_study_team_Mnist_Part_team_part_client_team_2_mclr_acc.dat};

\addplot[cyan,dashed,thick,mark=star]  table [y=gusers_20,x=GR]{ablation_study_team_Mnist_Part_team_part_client_team_2_mclr_acc.dat};
\end{axis}

\end{tikzpicture}
}

    \caption{Global Accuracy (Validation)}
    \label{Fig: mnist_ptpc_team_2_mclr_ga}
\end{subfigure}

\caption{Partial participation of team ($20\%$) and devices ( $2\%, 4\%, 10\%,\ and\ 30\%$)  on MNIST datasets in convex settings (MCLR)}
\label{Fig: mnist_ptpc_team_2_mclr}
\end{figure*}
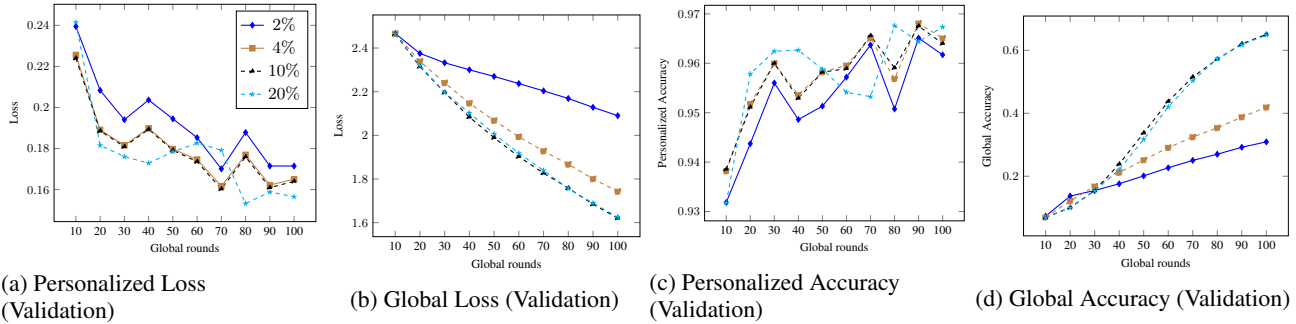


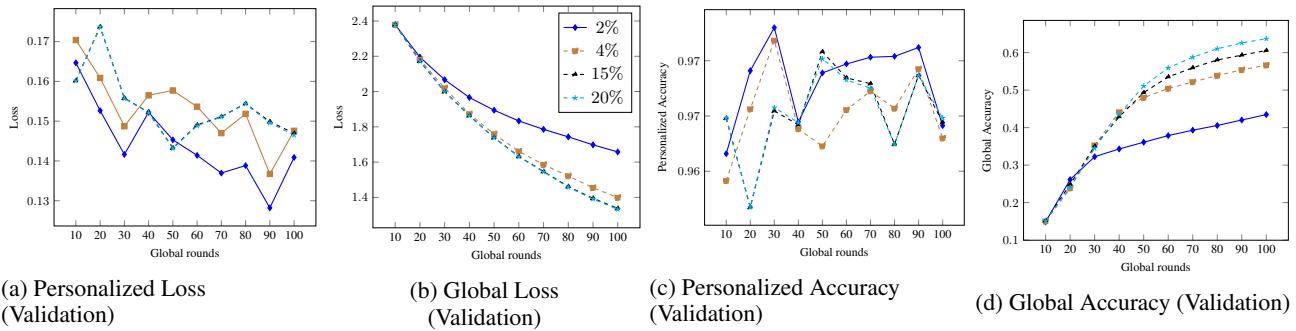
\begin{figure*}[!ht]
        \centering
        \begin{subfigure}[!t]{0.24\linewidth}
            \resizebox{\linewidth}{!}{
                \begin{tikzpicture}
\begin{axis}[
  xlabel= Global rounds,
  ylabel= Loss,
  legend pos=north east,
  legend style ={nodes={scale=1.4, transform shape}},
  xticklabels from table={ablation_study_team_FMnist_Part_team_part_client_team_20_mclr_loss.dat}{GR},xtick=data
]
\addplot[blue,thick,mark=diamond*] table [y=pusers_2,x=GR]{ablation_study_team_FMnist_Part_team_part_client_team_20_mclr_loss.dat};

\addplot[brown,thick,mark=square*] table [y=pusers_4,x=GR]{ablation_study_team_FMnist_Part_team_part_client_team_20_mclr_loss.dat};

\addplot[black,dashed,thick,mark=triangle*]  table [y=pusers_15,x=GR]{ablation_study_team_FMnist_Part_team_part_client_team_20_mclr_loss.dat};

\addplot[cyan,dashed,thick,mark=star]  table [y=pusers_20,x=GR]{ablation_study_team_FMnist_Part_team_part_client_team_20_mclr_loss.dat};

\end{axis}
\end{tikzpicture}

            }
        \caption{Personalized Loss \\ (Validation)}
        \label{Fig: mnist_ptpc_team_20_mclr_pl}
    \end{subfigure} 
    \begin{subfigure}[!t]{0.24\linewidth}
        \resizebox{\linewidth}{!}{
        \begin{tikzpicture}

\begin{axis}[
  xlabel= Global rounds,
  ylabel= Loss,
  legend pos=north east,
  legend style ={nodes={scale=1.4, transform shape}},
  xticklabels from table={ablation_study_team_FMnist_Part_team_part_client_team_20_mclr_loss.dat}{GR},xtick=data
]
\addplot[blue,thick,mark=diamond*] table [y=gusers_2,x=GR]{ablation_study_team_FMnist_Part_team_part_client_team_20_mclr_loss.dat};
\addlegendentry{$ 2\%$}

\addplot[brown,dashed,thick,mark=square*]  table [y=gusers_4,x=GR]{ablation_study_team_FMnist_Part_team_part_client_team_20_mclr_loss.dat};
\addlegendentry{$ 4\%$}

\addplot[black,dashed,thick,mark=triangle*]  table [y=gusers_15,x=GR]{ablation_study_team_FMnist_Part_team_part_client_team_20_mclr_loss.dat};
\addlegendentry{$ 15\%$}

\addplot[cyan,dashed,thick,mark=star]  table [y=gusers_20,x=GR]{ablation_study_team_FMnist_Part_team_part_client_team_20_mclr_loss.dat};
\addlegendentry{$ 20\%$}

\end{axis}

\end{tikzpicture}
}

    \caption{Global Loss \\ (Validation)}
    \label{Fig: mnist_ptpc_team_20_mclr_gl}
\end{subfigure}
\begin{subfigure}[!t]{0.24\linewidth}
            \resizebox{\linewidth}{!}{
                \begin{tikzpicture}

\begin{axis}[
  xlabel= Global rounds,
  ylabel= Personalized Accuracy,
  legend pos=south east,
  xticklabels from table={ablation_study_team_FMnist_Part_team_part_client_team_20_mclr_acc.dat}{GR},xtick=data
]
\addplot[blue,thick,mark=diamond*] table [y=pusers_2,x=GR]{ablation_study_team_FMnist_Part_team_part_client_team_20_mclr_acc.dat};

\addplot[brown,dashed,thick,mark=square*]  table [y=pusers_4,x=GR]{ablation_study_team_FMnist_Part_team_part_client_team_20_mclr_acc.dat};

\addplot[black,dashed,thick,mark=triangle*]  table [y=pusers_15,x=GR]{ablation_study_team_FMnist_Part_team_part_client_team_20_mclr_acc.dat};

\addplot[cyan,dashed,thick,mark=star]  table [y=pusers_20,x=GR]{ablation_study_team_FMnist_Part_team_part_client_team_20_mclr_acc.dat};

\end{axis}

\end{tikzpicture}

            }
        \caption{Personalized Accuracy \\ (Validation)}
        \label{Fig: mnist_ptpc_team_20_mclr_pa}
    \end{subfigure} 
    \begin{subfigure}[!t]{0.24\linewidth}
        \resizebox{\linewidth}{!}{
        \begin{tikzpicture}
\begin{axis}[
  xlabel= Global rounds,
  ylabel= Global Accuracy,
  legend pos=south east,
  xticklabels from table={ablation_study_team_FMnist_Part_team_part_client_team_20_mclr_acc.dat}{GR},xtick=data
]
\addplot[blue,thick,mark=diamond*] table [y=gusers_2,x=GR]{ablation_study_team_FMnist_Part_team_part_client_team_20_mclr_acc.dat};

\addplot[brown,dashed,thick,mark=square*]  table [y=gusers_4,x=GR]{ablation_study_team_FMnist_Part_team_part_client_team_20_mclr_acc.dat};

\addplot[black,dashed,thick,mark=triangle*]  table [y=gusers_15,x=GR]{ablation_study_team_FMnist_Part_team_part_client_team_20_mclr_acc.dat};

\addplot[cyan,dashed,thick,mark=star]  table [y=gusers_20,x=GR]{ablation_study_team_FMnist_Part_team_part_client_team_20_mclr_acc.dat};
\end{axis}

\end{tikzpicture}
}

    \caption{Global Accuracy (Validation)}
    \label{Fig: mnist_ptpc_team_20_mclr_ga}
\end{subfigure}

\caption{Partial participation of team ($20\%$) and devices ( $2\%, 4\%, 15\%,\ and\ 30\%$)  on FMNIST datasets in convex settings (MCLR)}
\label{Fig: fmnist_ptpc_team_20_mclr}
\end{figure*}


\paragraph{\textit{Discussions:}} Based on our empirical observations, we conclude that the convergence and performance of \ourmodel{} are optimal when both teams and devices are fully present throughout the entire global iterations. However, \ourmodel{} also demonstrates good performance even with variations in team and device participation. When teams fully participate but there is partial device participation, \ourmodel{} achieves fast convergence. On the other hand, when the number of participating teams is limited, the convergence of the global model is slower, requiring a higher number of global rounds to converge. It is important to note that the convergence of \ourmodel{} is slowest when both teams and devices have very low participation ($2\%$) in each global rounds. 

\clearpage

\subsection{Effect of team iterations on convergence of \ourmodel{}} \label{appx: effect_team_iters}
In this study, experiments were conducted to investigate the impact of team iterations on the convergence of \ourmodel{}. The objective was to gain insights into how varying the number of team iterations influences the convergence behaviour of the model. Here we studied (1) the effect of the number of team iterations on the convergence of \ourmodel{} when teams have full involvement but devices have fractional involvement in the entire learning process. (2) Effect of team iterations on the convergence of \ourmodel{} while teams and devices both have fractional involvement. i.e., all devices are not participating in each global rounds only a fraction of devices are participating. 

\subsubsection{Effect of number of team iterations on the full participation of Teams and partial participation of Devices.}

In our experiments (see \Cref{Fig:effect_team_mnist_cnn}, \Cref{Fig:effect_team_mnist_mclr} \Cref{Fig:effect_team_fmnist}, \Cref{Fig:team_iters_synthetic_cnn}, and \Cref{Fig:team_iters_synthetic_mclr}), we observed for both convex and non-convex scenarios, the convergence of \ourmodel{}'s personalized and global model improves if we increase the team iterations. This improvement occurs because when we increase the team iterations, more devices within the teams can actively contribute to the process.

\subsubsection{Effect of number of team iterations on the partial participation of Teams and Devices}
In our experiments, we specifically investigated the impact of low team and device participation on the convergence of the global model. The objective was to determine whether increasing the number of team iterations could expedite the convergence process. Our findings, as depicted in \Cref{Fig: te_emnist_ptpc_team_2_mclr} and \Cref{Fig: te_mnist_ptpc_team_2_mclr}, indicate that when team participation is set at $20\%$, increasing the team iterations lead to improved convergence of the global model. However, when team participation is extremely low ($2\%$), as shown in \Cref{Fig: te_emnist_very_low_ptpc_team_2_mclr} and \Cref{Fig: te_mnist_very_low_ptpc_team_2_mclr}, simply increasing team iterations is insufficient. In such cases, a higher number of global iterations is necessary to achieve convergence for the global model.

\paragraph{\textit{Discussions:}}Based on our findings, we can infer that team iterations play a crucial role in improving the performance of both the global and personalized models. However, when there is limited participation from teams and devices in each global round, relying solely on team iterations is inadequate. In such cases, it becomes necessary to increase the number of global rounds to enable more teams to participate and, consequently, enhance the performance of \ourmodel{}.

\begin{figure*}[!ht]
        \centering
        \begin{subfigure}[!t]{0.24\linewidth}
            \resizebox{\linewidth}{!}{
                \begin{tikzpicture}
                 \tikzstyle{every node}=[font=\fontsize{12}{12}\selectfont]
\begin{axis}[
  xlabel= Global rounds,
  ylabel= Accuracy,
  legend pos=south east,
  legend style ={nodes={scale=1.2, transform shape}},
  xticklabels from table={final-pm-gm-tr-cnn-mnist-acc.dat}{GR},xtick=data
]
\addplot[blue,dashed,thick,mark=diamond*] table [y=pt40,x=GR]{final-pm-gm-tr-cnn-mnist-acc.dat};
\addlegendentry{$K = 40$ }

\addplot[brown,dashed,thick,mark=square*]  table [y=pt20,x=GR]{final-pm-gm-tr-cnn-mnist-acc.dat};
\addlegendentry{$K = 20$ }

\addplot[black,dashed,thick,mark=triangle*]  table [y=pt10,x=GR]{final-pm-gm-tr-cnn-mnist-acc.dat};
\addlegendentry{$K = 10$ }

\end{axis}
\end{tikzpicture}

            }
        \caption{Personalized Accuracy (Validation)}
        \label{Fig:team_non_conv_mnist_acc}
    \end{subfigure} 
    \begin{subfigure}[!t]{0.24\linewidth}
        \resizebox{\linewidth}{!}{
        \begin{tikzpicture}
         \tikzstyle{every node}=[font=\fontsize{12}{12}\selectfont]
\begin{axis}[
  xlabel= Global rounds,
  ylabel= Accuracy,
  legend pos=south east,
  xticklabels from table={final-pm-gm-tr-cnn-mnist-acc.dat}{GR},xtick=data
]
\addplot[blue,dashed,thick,mark=diamond*] table [y=gt40,x=GR]{final-pm-gm-tr-cnn-mnist-acc.dat};

\addplot[brown,dashed,thick,mark=square*]  table [y=gt20,x=GR]{final-pm-gm-tr-cnn-mnist-acc.dat};

\addplot[black,dashed,thick,mark=triangle*]  table [y=gt10,x=GR]{final-pm-gm-tr-cnn-mnist-acc.dat};

\end{axis}
\end{tikzpicture}
}

    \caption{Global Accuracy (Validation)}
    \label{Fig:team_non_conv_mnist_global_acc}
\end{subfigure}
\begin{subfigure}[!t]{0.24\linewidth}
            \resizebox{\linewidth}{!}{
                \begin{tikzpicture}
                 \tikzstyle{every node}=[font=\fontsize{12}{12}\selectfont]

\begin{axis}[
  xlabel= Global rounds,
  ylabel= Loss,
  legend pos= north east,
  xticklabels from table={final-pm-gm-tr-cnn-mnist.dat}{GR},xtick=data
]
\addplot[blue,dashed, thick,mark=diamond*] table [y=pt40,x=GR]{final-pm-gm-tr-cnn-mnist.dat};

\addplot[brown,dashed,thick,mark=square*]  table [y=pt20,x=GR]{final-pm-gm-tr-cnn-mnist.dat};

\addplot[black,dashed,thick,mark=triangle*]  table [y=pt10,x=GR]{final-pm-gm-tr-cnn-mnist.dat};

\end{axis}

\end{tikzpicture}

            }
        \caption{Personalized Loss (Validation)}
        \label{Fig:team_non_convex_mnist_loss}
    \end{subfigure} 
    \begin{subfigure}[!t]{0.24\linewidth}
        \resizebox{\linewidth}{!}{
        \begin{tikzpicture}
         \tikzstyle{every node}=[font=\fontsize{12}{12}\selectfont]
        \begin{axis}[
  xlabel= Global rounds,
  ylabel= Loss,
  legend pos= north east,
  xticklabels from table={final-pm-gm-tr-cnn-mnist.dat}{GR},xtick=data
]

\addplot[blue,dashed,thick,mark=diamond*] table [y=gt40,x=GR]{final-pm-gm-tr-cnn-mnist.dat};

\addplot[brown,dashed,thick,mark=square*]  table [y=gt20,x=GR]{final-pm-gm-tr-cnn-mnist.dat};

\addplot[black,dashed,thick,mark=triangle*]  table [y=gt10,x=GR]{final-pm-gm-tr-cnn-mnist.dat};

\end{axis}
\end{tikzpicture}
}

    \caption{Global Loss (validation)}
    \label{Fig:team_non_convex_mnist_global_loss}
\end{subfigure}

\caption{Effect of Team iterations on the convergence of \ourmodel{} in non-convex settings (CNN) on MNIST while teams and devices have full participation}
\label{Fig:effect_team_mnist_cnn}
\end{figure*}
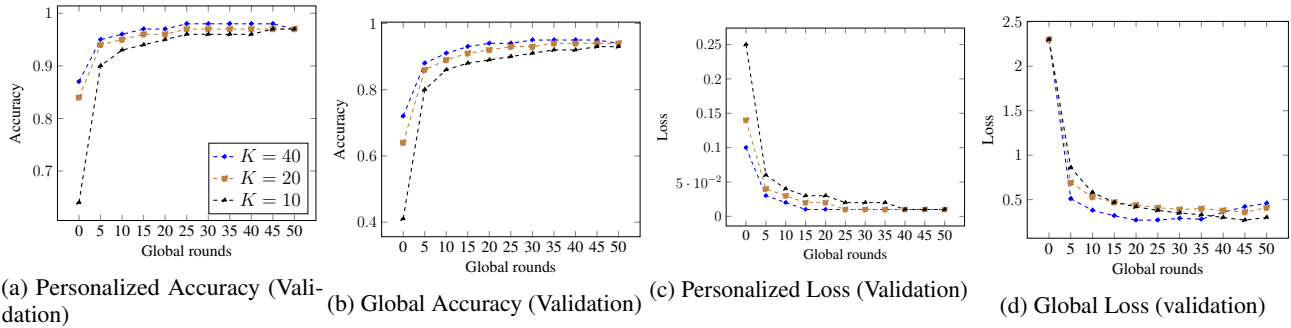


\begin{figure*}[!ht]
        \centering
        \begin{subfigure}[!t]{0.24\linewidth}
            \resizebox{\linewidth}{!}{
                \begin{tikzpicture}
                 \tikzstyle{every node}=[font=\fontsize{12}{12}\selectfont]
\begin{axis}[
  xlabel= Global rounds,
  ylabel= Accuracy,
  legend pos=south east,
  legend style ={nodes={scale=1.2, transform shape}},
  xticklabels from table={final-pm-gm-tr-mclr-mnist-acc.dat}{GR},xtick=data
]
\addplot[blue,dashed,thick,mark=diamond*] table [y=pt40,x=GR]{final-pm-gm-tr-mclr-mnist-acc.dat};
\addlegendentry{$K = 40$}

\addplot[brown,dashed,thick,mark=square*]  table [y=pt20,x=GR]{final-pm-gm-tr-mclr-mnist-acc.dat};
\addlegendentry{$K = 20$}

\addplot[black,dashed,thick,mark=triangle*]  table [y=pt10,x=GR]{final-pm-gm-tr-mclr-mnist-acc.dat};
\addlegendentry{$K = 10$}

\end{axis}
\end{tikzpicture}

            }
        \caption{Personalized Accuracy}
        \label{Fig:team_conv_mnist_acc}
    \end{subfigure} 
    \begin{subfigure}[!t]{0.24\linewidth}
        \resizebox{\linewidth}{!}{
        \begin{tikzpicture}
         \tikzstyle{every node}=[font=\fontsize{12}{12}\selectfont]
\begin{axis}[
  xlabel= Global rounds,
  ylabel= Accuracy,
  legend pos=south east,
  xticklabels from table={final-pm-gm-tr-mclr-mnist-acc.dat}{GR},xtick=data
]

\addplot[blue,dashed, thick,mark=diamond*] table [y=gt40,x=GR]{final-pm-gm-tr-mclr-mnist-acc.dat};

\addplot[brown,dashed,thick,mark=square*]  table [y=gt20,x=GR]{final-pm-gm-tr-mclr-mnist-acc.dat};

\addplot[black,dashed,thick,mark=triangle*]  table [y=gt10,x=GR]{final-pm-gm-tr-mclr-mnist-acc.dat};

\end{axis}
\end{tikzpicture}
}

    \caption{Global Accuracy (Validation)}
    \label{Fig:team_convex_mnist_global_acc}
\end{subfigure}
\begin{subfigure}[!t]{0.24\linewidth}
            \resizebox{\linewidth}{!}{
                \begin{tikzpicture}
                 \tikzstyle{every node}=[font=\fontsize{12}{12}\selectfont]

\begin{axis}[
  xlabel= Global rounds,
  ylabel= Loss,
  legend pos= north east,
  xticklabels from table={final-pm-gm-tr-mclr-mnist.dat}{GR},xtick=data
]
\addplot[blue,dashed,thick,mark=diamond*] table [y=pt40,x=GR]{final-pm-gm-tr-mclr-mnist.dat};

\addplot[brown,dashed,thick,mark=square*]  table [y=pt20,x=GR]{final-pm-gm-tr-mclr-mnist.dat};

\addplot[black,dashed,thick,mark=triangle*]  table [y=pt10,x=GR]{final-pm-gm-tr-mclr-mnist.dat};

\end{axis}

\end{tikzpicture}

            }
        \caption{Personalized Loss (Validation)}
        \label{Fig:team_convex_mnist_loss}
    \end{subfigure} 
    \begin{subfigure}[!t]{0.24\linewidth}
        \resizebox{\linewidth}{!}{
        \begin{tikzpicture}
         \tikzstyle{every node}=[font=\fontsize{12}{12}\selectfont]
        \begin{axis}[
  xlabel= Global rounds,
  ylabel= Loss,
  legend pos= north east,
  xticklabels from table={final-pm-gm-tr-mclr-mnist.dat}{GR},xtick=data
]

\addplot[blue,dashed, thick,mark=diamond*] table [y=gt40,x=GR]{final-pm-gm-tr-mclr-mnist.dat};

\addplot[brown,dashed,thick,mark=square*]  table [y=gt20,x=GR]{final-pm-gm-tr-mclr-mnist.dat};

\addplot[black,dashed,thick,mark=triangle*]  table [y=gt10,x=GR]{final-pm-gm-tr-mclr-mnist.dat};

\end{axis}
\end{tikzpicture}
}

    \caption{Global Loss (Validation}
    \label{Fig:team_convex_mnist_global_loss}
\end{subfigure}

\caption{Effect of Team iterations on the convergence of \ourmodel{} in strongly convex (MCLR) settings on MNIST while teams and devices have full participation }
\label{Fig:effect_team_mnist_mclr}
\end{figure*}
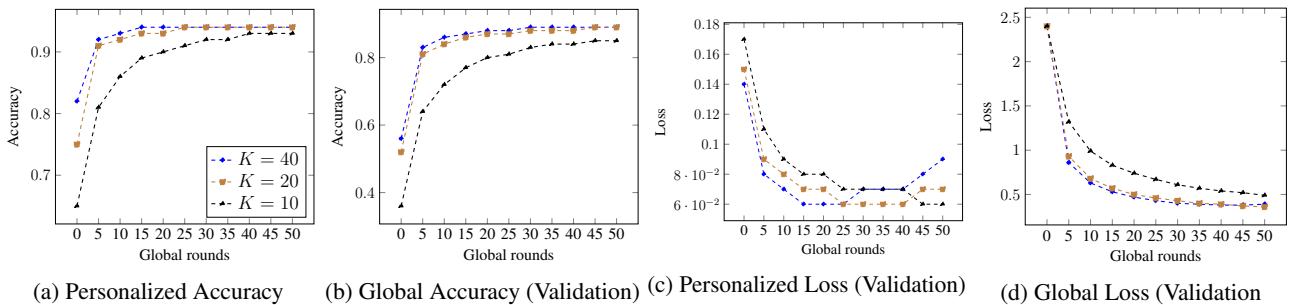


\begin{figure*}[!ht]
        \centering
        \begin{subfigure}[!t]{0.24\linewidth}
            \resizebox{\linewidth}{!}{
                \begin{tikzpicture}
                \tikzstyle{every node}=[font=\fontsize{12}{12}\selectfont]
\begin{axis}[
  xlabel= Global rounds,
  ylabel= Accuracy,
  legend pos=south east,
  legend style ={nodes={scale=1.0, transform shape}},
  xticklabels from table={final-tr-fmnist-cnn-acc.dat}{GR},xtick=data
]
\addplot[blue,dashed,thick,mark=diamond*] table [y=pt40,x=GR]{final-tr-fmnist-cnn-acc.dat};
\addlegendentry{$K = 40$}

\addplot[brown,dashed,thick,mark=square*]  table [y=pt20,x=GR]{final-tr-fmnist-cnn-acc.dat};
\addlegendentry{$K = 20$}

\addplot[black,dashed,thick,mark=triangle*]  table [y=pt10,x=GR]{final-tr-fmnist-cnn-acc.dat};
\addlegendentry{$K = 10$}

\end{axis}
\end{tikzpicture}

            }
        \caption{Personalized Accuracy (Validation)}
        \label{Fig:tr_fmnist_cnn_acc}
    \end{subfigure} 
    \begin{subfigure}[!t]{0.24\linewidth}
        \resizebox{\linewidth}{!}{
        \begin{tikzpicture}
        \tikzstyle{every node}=[font=\fontsize{12}{12}\selectfont]
\begin{axis}[
  xlabel= Global rounds,
  ylabel= Accuracy,
  xticklabels from table={final-tr-fmnist-cnn-acc.dat}{GR},xtick=data
]

\addplot[blue,dashed,thick,mark=diamond*] table [y=gt40,x=GR]{final-tr-fmnist-cnn-acc.dat};

\addplot[brown,dashed,thick,mark=square*]  table [y=gt20,x=GR]{final-tr-fmnist-cnn-acc.dat};

\addplot[black,dashed,thick,mark=triangle*]  table [y=gt10,x=GR]{final-tr-fmnist-cnn-acc.dat};

\end{axis}
\end{tikzpicture}
}

    \caption{Global Accuracy (Validation)}
    \label{Fig:tr_fmnist_cnn_global_acc}
\end{subfigure}
\begin{subfigure}[!t]{0.24\linewidth}
            \resizebox{\linewidth}{!}{
                \begin{tikzpicture}
                \tikzstyle{every node}=[font=\fontsize{12}{12}\selectfont]

\begin{axis}[
  xlabel= Global rounds,
  ylabel= Loss,
  legend pos= north east,
  xticklabels from table={final-tr-fmnist-cnn-loss.dat}{GR},xtick=data
]
\addplot[blue,dashed,thick,mark=diamond*] table [y=pt40,x=GR]{final-tr-fmnist-cnn-loss.dat};

\addplot[brown,dashed,thick,mark=square*]  table [y=pt20,x=GR]{final-tr-fmnist-cnn-loss.dat};

\addplot[black,dashed,thick,mark=triangle*]  table [y=pt10,x=GR]{final-tr-fmnist-cnn-loss.dat};

\end{axis}

\end{tikzpicture}

            }
        \caption{Personalized Loss (Validation)}
        \label{Fig:tr_fmnist_cnn_loss}
    \end{subfigure} 
    \begin{subfigure}[!t]{0.24\linewidth}
        \resizebox{\linewidth}{!}{
        \begin{tikzpicture}
        \tikzstyle{every node}=[font=\fontsize{12}{12}\selectfont]
        \begin{axis}[
  xlabel= Global rounds,
  ylabel= Loss,
  xticklabels from table={final-tr-fmnist-cnn-loss.dat}{GR},xtick=data
]

\addplot[blue,dashed,thick,mark=diamond*] table [y=gt40,x=GR]{final-tr-fmnist-cnn-loss.dat};

\addplot[brown,dashed,thick,mark=square*]  table [y=gt20,x=GR]{final-tr-fmnist-cnn-loss.dat};

\addplot[black,dashed,thick,mark=triangle*]  table [y=gt10,x=GR]{final-tr-fmnist-cnn-loss.dat};

\end{axis}
\end{tikzpicture}
}

    \caption{Global Loss (Validation)}
    \label{Fig:tr_fmnist_cnn_global_loss}
\end{subfigure}

\caption{Effect of Team iterations on the convergence of \ourmodel{} in non-convex settings (CNN) on FMNIST while teams and devices have full participation}
\label{Fig:effect_team_fmnist}
\end{figure*}


\begin{figure*}[!ht]
        \centering
        \begin{subfigure}[!t]{0.24\linewidth}
            \resizebox{\linewidth}{!}{
                \begin{tikzpicture}
                \tikzstyle{every node}=[font=\fontsize{12}{12}\selectfont]
\begin{axis}[
  xlabel= Global rounds,
  ylabel= Accuracy,
  legend pos=south east,
  xticklabels from table={final-tr-synthetic-cnn-acc.dat}{GR},xtick=data
]
\addplot[blue,dashed,thick,mark=diamond*] table [y=pt40,x=GR]{final-tr-synthetic-cnn-acc.dat};

\addplot[brown,dashed,thick,mark=square*]  table [y=pt20,x=GR]{final-tr-synthetic-cnn-acc.dat};

\addplot[black,dashed,thick,mark=triangle*]  table [y=pt10,x=GR]{final-tr-synthetic-cnn-acc.dat};

\end{axis}
\end{tikzpicture}

            }
        \caption{Personalized Accuracy (Validation)}
        \label{Fig:tr_synthetic_cnn_acc}
    \end{subfigure} 
    \begin{subfigure}[!t]{0.24\linewidth}
        \resizebox{\linewidth}{!}{
        \begin{tikzpicture}
        \tikzstyle{every node}=[font=\fontsize{12}{12}\selectfont]
\begin{axis}[
  xlabel= Global rounds,
  ylabel= Accuracy,
  xticklabels from table={final-tr-fmnist-cnn-acc.dat}{GR},xtick=data
]
\addplot[blue,dashed, thick,mark=diamond*] table [y=gt40,x=GR]{final-tr-synthetic-cnn-acc.dat};

\addplot[brown,dashed,thick,mark=square*]  table [y=gt20,x=GR]{final-tr-synthetic-cnn-acc.dat};

\addplot[black,dashed,thick,mark=triangle*]  table [y=gt10,x=GR]{final-tr-synthetic-cnn-acc.dat};

\end{axis}
\end{tikzpicture}
}

    \caption{Global Accuracy (Validation)}
    \label{Fig:tr_synthetic_cnn_global_acc}
\end{subfigure}
\begin{subfigure}[!t]{0.24\linewidth}
            \resizebox{\linewidth}{!}{
                \begin{tikzpicture}
                \tikzstyle{every node}=[font=\fontsize{12}{12}\selectfont]

\begin{axis}[
  xlabel= Global rounds,
  ylabel= Loss,
  legend pos= north east,
  xticklabels from table={final-tr-synthetic-cnn-loss.dat}{GR},xtick=data
]
\addplot[blue,dashed,,thick,mark=diamond*] table [y=pt40,x=GR]{final-tr-synthetic-cnn-loss.dat};

\addplot[brown,dashed,thick,mark=square*]  table [y=pt20,x=GR]{final-tr-synthetic-cnn-loss.dat};

\addplot[black,dashed,thick,mark=triangle*]  table [y=pt10,x=GR]{final-tr-synthetic-cnn-loss.dat};

\end{axis}

\end{tikzpicture}

            }
        \caption{Personalized Loss (Validation)}
        \label{Fig:tr_synthetic_cnn_loss}
    \end{subfigure} 
    \begin{subfigure}[!t]{0.24\linewidth}
        \resizebox{\linewidth}{!}{
        \begin{tikzpicture}
        \tikzstyle{every node}=[font=\fontsize{12}{12}\selectfont]
        \begin{axis}[
  xlabel= Global rounds,
  ylabel= Loss,
  legend pos= north east,
  xticklabels from table={final-tr-synthetic-cnn-loss.dat}{GR},xtick=data
]
\addplot[blue,dashed,thick,mark=diamond*] table [y=gt40,x=GR]{final-tr-synthetic-cnn-loss.dat};
\addlegendentry{$K = 40$}

\addplot[brown,dashed,thick,mark=square*]  table [y=gt20,x=GR]{final-tr-synthetic-cnn-loss.dat};
\addlegendentry{$K = 20$}

\addplot[black,dashed,thick,mark=triangle*]  table [y=gt10,x=GR]{final-tr-synthetic-cnn-loss.dat};
\addlegendentry{$K = 10$}

\end{axis}
\end{tikzpicture}
}

    \caption{Global Loss (Validation)}
    \label{Fig:tr_synthetic_cnn_global_loss}
\end{subfigure}

\caption{Effect of Team iterations on the convergence of \ourmodel{} in non-convex settings (DNN) on Synthetic dataset while teams and devices fully participate. }
\label{Fig:team_iters_synthetic_cnn}
\end{figure*}


\begin{figure*}[!ht]
        \centering
        \begin{subfigure}[!t]{0.24\linewidth}
            \resizebox{\linewidth}{!}{
                \begin{tikzpicture}
                \tikzstyle{every node}=[font=\fontsize{12}{12}\selectfont]
\begin{axis}[
  xlabel= Global rounds,
  ylabel= Accuracy,
  legend pos=south east,
  xticklabels from table={final-tr-synthetic-mclr-acc.dat}{GR},xtick=data
]
\addplot[blue,dashed,thick,mark=diamond*] table [y=pt40,x=GR]{final-tr-synthetic-mclr-acc.dat};

\addplot[brown,dashed,thick,mark=square*]  table [y=pt20,x=GR]{final-tr-synthetic-mclr-acc.dat};

\addplot[black,dashed,thick,mark=triangle*]  table [y=pt10,x=GR]{final-tr-synthetic-mclr-acc.dat};

\end{axis}
\end{tikzpicture}

            }
        \caption{Personalized Accuracy (Validation)}
        \label{Fig:tr_synthetic_mclr_acc}
    \end{subfigure} 
    \begin{subfigure}[!t]{0.24\linewidth}
        \resizebox{\linewidth}{!}{
        \begin{tikzpicture}
        \tikzstyle{every node}=[font=\fontsize{12}{12}\selectfont]
\begin{axis}[
  xlabel= Global rounds,
  ylabel= Accuracy,
  xticklabels from table={final-tr-fmnist-mclr-acc.dat}{GR},xtick=data
]
\addplot[blue,dashed, thick,mark=diamond*] table [y=gt40,x=GR]{final-tr-synthetic-mclr-acc.dat};

\addplot[brown,dashed,thick,mark=square*]  table [y=gt20,x=GR]{final-tr-synthetic-mclr-acc.dat};

\addplot[black,dashed,thick,mark=triangle*]  table [y=gt10,x=GR]{final-tr-synthetic-mclr-acc.dat};

\end{axis}
\end{tikzpicture}
}

    \caption{Global Accuracy (Validation)}
    \label{Fig:tr_synthetic_mclr_global_acc}
\end{subfigure}
\begin{subfigure}[!t]{0.24\linewidth}
            \resizebox{\linewidth}{!}{
                \begin{tikzpicture}
                \tikzstyle{every node}=[font=\fontsize{12}{12}\selectfont]

\begin{axis}[
  xlabel= Global rounds,
  ylabel= Loss,
  legend pos= north east,
  xticklabels from table={final-tr-synthetic-mclr-loss.dat}{GR},xtick=data
]
\addplot[blue,dashed,,thick,mark=diamond*] table [y=pt40,x=GR]{final-tr-synthetic-mclr-loss.dat};

\addplot[brown,dashed,thick,mark=square*]  table [y=pt20,x=GR]{final-tr-synthetic-mclr-loss.dat};

\addplot[black,dashed,thick,mark=triangle*]  table [y=pt10,x=GR]{final-tr-synthetic-mclr-loss.dat};

\end{axis}

\end{tikzpicture}

            }
        \caption{Personalized Loss (Validation)}
        \label{Fig:tr_synthetic_mclr_loss}
    \end{subfigure} 
    \begin{subfigure}[!t]{0.24\linewidth}
        \resizebox{\linewidth}{!}{
        \begin{tikzpicture}
        \tikzstyle{every node}=[font=\fontsize{12}{12}\selectfont]
        \begin{axis}[
  xlabel= Global rounds,
  ylabel= Loss,
  legend pos= north east,
  xticklabels from table={final-tr-synthetic-mclr-loss.dat}{GR},xtick=data
]
\addplot[blue,dashed,thick,mark=diamond*] table [y=gt40,x=GR]{final-tr-synthetic-mclr-loss.dat};
\addlegendentry{$K = 40$ }

\addplot[brown,dashed,thick,mark=square*]  table [y=gt20,x=GR]{final-tr-synthetic-mclr-loss.dat};
\addlegendentry{$K = 20$ }

\addplot[black,dashed,thick,mark=triangle*]  table [y=gt10,x=GR]{final-tr-synthetic-mclr-loss.dat};
\addlegendentry{$K = 10$ }

\end{axis}
\end{tikzpicture}
}

    \caption{Global Loss (Validation)}
    \label{Fig:tr_synthetic_mclr_global_loss}
\end{subfigure}

\caption{Effect of Team iterations on the convergence of \ourmodel{} in strongly convex settings (MCLR) on Synthetic dataset while teams and devices fully participate. }
\label{Fig:team_iters_synthetic_mclr}
\end{figure*}


\begin{figure*}[!ht]
        \centering
        \begin{subfigure}[!t]{0.24\linewidth}
            \resizebox{\linewidth}{!}{
                \begin{tikzpicture}
                \tikzstyle{every node}=[font=\fontsize{12}{12}\selectfont]
\begin{axis}[
  xlabel= Global rounds,
  ylabel= Loss,
  legend pos=north east,
  legend style ={nodes={scale=1.2, transform shape}},
  xticklabels from table={ablation_study_team_Emnist_TE_part_team_part_client_TE_team_2_client_2_loss.dat}{GR},xtick=data
]
\addplot[blue,thick,mark=diamond*] table [y=pTE_30,x=GR]{ablation_study_team_Emnist_TE_part_team_part_client_TE_team_2_client_2_loss.dat};
\addlegendentry{$K = 30$}

\addplot[brown,thick,mark=square*] table [y=pTE_50,x=GR]{ablation_study_team_Emnist_TE_part_team_part_client_TE_team_2_client_2_loss.dat};
\addlegendentry{$K = 50$}

\addplot[black,dashed,thick,mark=triangle*]  table [y=pTE_100,x=GR]{ablation_study_team_Emnist_TE_part_team_part_client_TE_team_2_client_2_loss.dat};
\addlegendentry{$K = 100$}

\addplot[green,dashed,thick,mark=star]  table [y=pTE_150,x=GR]{ablation_study_team_Emnist_TE_part_team_part_client_TE_team_2_client_2_loss.dat};
\addlegendentry{$K = 150$}

\end{axis}
\end{tikzpicture}

            }
        \caption{Personalized Loss \\ (Validation)}
        \label{Fig: te_emnist_ptpc_team_2_mclr_pl}
\end{subfigure} 
\begin{subfigure}[!t]{0.24\linewidth}
        \resizebox{\linewidth}{!}{
        \begin{tikzpicture}
        \tikzstyle{every node}=[font=\fontsize{12}{12}\selectfont]

\begin{axis}[
  xlabel= Global rounds,
  ylabel= Loss,
  xticklabels from table={ablation_study_team_Emnist_TE_part_team_part_client_TE_team_2_client_2_loss.dat}{GR},xtick=data
]
\addplot[blue,thick,mark=diamond*] table [y=gTE_30,x=GR]{ablation_study_team_Emnist_TE_part_team_part_client_TE_team_2_client_2_loss.dat};

\addplot[brown,dashed,thick,mark=square*]  table [y=gTE_50,x=GR]{ablation_study_team_Emnist_TE_part_team_part_client_TE_team_2_client_2_loss.dat};

\addplot[black,dashed,thick,mark=triangle*]  table [y=gTE_100,x=GR]{ablation_study_team_Emnist_TE_part_team_part_client_TE_team_2_client_2_loss.dat};

\addplot[green,dashed,thick,mark=star]  table [y=gTE_150,x=GR]{ablation_study_team_Emnist_TE_part_team_part_client_TE_team_2_client_2_loss.dat};

\end{axis}

\end{tikzpicture}
}

    \caption{Global Loss \\ (Validation)}
    \label{Fig: te_emnist_ptpc_team_2_mclr_gl}
\end{subfigure}
\begin{subfigure}[!t]{0.24\linewidth}
            \resizebox{\linewidth}{!}{
                \begin{tikzpicture}
                \tikzstyle{every node}=[font=\fontsize{12}{12}\selectfont]

\begin{axis}[
  xlabel= Global rounds,
  ylabel= Personalized Accuracy,
  xticklabels from table={ablation_study_team_Emnist_TE_part_team_part_client_TE_team_2_client_2_acc.dat}{GR},xtick=data
]
\addplot[blue,thick,mark=diamond*] table [y=pTE_30,x=GR]{ablation_study_team_Emnist_TE_part_team_part_client_TE_team_2_client_2_acc.dat};

\addplot[brown,dashed,thick,mark=square*]  table [y=pTE_50,x=GR]{ablation_study_team_Emnist_TE_part_team_part_client_TE_team_2_client_2_acc.dat};

\addplot[black,dashed,thick,mark=triangle*]  table [y=pTE_100,x=GR]{ablation_study_team_Emnist_TE_part_team_part_client_TE_team_2_client_2_acc.dat};

\addplot[green,dashed,thick,mark=star]  table [y=pTE_150,x=GR]{ablation_study_team_Emnist_TE_part_team_part_client_TE_team_2_client_2_acc.dat};

\end{axis}

\end{tikzpicture}

            }
        \caption{Personalized Accuracy \\ (Validation)}
        \label{Fig: te_emnist_ptpc_team_2_mclr_pa}
    \end{subfigure} 
    \begin{subfigure}[!t]{0.24\linewidth}
        \resizebox{\linewidth}{!}{
        \begin{tikzpicture}
        \tikzstyle{every node}=[font=\fontsize{12}{12}\selectfont]
\begin{axis}[
  xlabel= Global rounds,
  ylabel= Global Accuracy,
  xticklabels from table={ablation_study_team_Emnist_TE_part_team_part_client_TE_team_2_client_2_acc.dat}{GR},xtick=data
]
\addplot[blue,thick,mark=diamond*] table [y=gTE_30,x=GR]{ablation_study_team_Emnist_TE_part_team_part_client_TE_team_2_client_2_acc.dat};

\addplot[brown,dashed,thick,mark=square*]  table [y=gTE_50,x=GR]{ablation_study_team_Emnist_TE_part_team_part_client_TE_team_2_client_2_acc.dat};

\addplot[black,dashed,thick,mark=triangle*]  table [y=gTE_100,x=GR]{ablation_study_team_Emnist_TE_part_team_part_client_TE_team_2_client_2_acc.dat};

\addplot[green,dashed,thick,mark=star]  table [y=gTE_150,x=GR]{ablation_study_team_Emnist_TE_part_team_part_client_TE_team_2_client_2_acc.dat};
\end{axis}

\end{tikzpicture}
}

    \caption{Global Accuracy \\ (Validation)}
    \label{Fig: te_emnist_ptpc_team_2_mclr_ga}
\end{subfigure}

\caption{Effect of Team iterations $\{TE = 30, 50, 100,\ and\ 150\}$ when team ($20\%$) and devices ($2\%$) are partially participated in the \ourmodel{} using EMNIST datasets in convex settings (MCLR)}
\label{Fig: te_emnist_ptpc_team_2_mclr}
\end{figure*}


\begin{figure*}[!ht]
        \centering
        \begin{subfigure}[!t]{0.24\linewidth}
            \resizebox{\linewidth}{!}{
                \begin{tikzpicture}
                \tikzstyle{every node}=[font=\fontsize{12}{12}\selectfont]
\begin{axis}[
  xlabel= Global rounds,
  ylabel= Loss,
  legend pos=south west,
  xticklabels from table={ablation_study_team_Emnist_Very_low_part_team_part_client_TE_team_2_client_2_loss.dat}{GR},xtick=data
]
\addplot[blue,thick,mark=diamond*] table [y=pTE_30,x=GR]{ablation_study_team_Emnist_Very_low_part_team_part_client_TE_team_2_client_2_loss.dat};

\addplot[brown,dashed,thick,mark=square*]  table [y=pTE_100,x=GR]{ablation_study_team_Emnist_Very_low_part_team_part_client_TE_team_2_client_2_loss.dat};

\addplot[black,dashed,thick,mark=triangle*]  table [y=pTE_150,x=GR]{ablation_study_team_Emnist_Very_low_part_team_part_client_TE_team_2_client_2_loss.dat};

\addplot[green,thick,mark=star] table [y=pTE_200,x=GR]{ablation_study_team_Emnist_Very_low_part_team_part_client_TE_team_2_client_2_loss.dat};

\end{axis}
\end{tikzpicture}

            }
        \caption{Personalized Loss \\ (Validation)}
        \label{Fig: te_emnist_very_low_ptpc_team_2_mclr_pl}
    \end{subfigure} 
    \begin{subfigure}[!t]{0.24\linewidth}
        \resizebox{\linewidth}{!}{
        \begin{tikzpicture}
        \tikzstyle{every node}=[font=\fontsize{12}{12}\selectfont]

\begin{axis}[
  xlabel= Global rounds,
  ylabel= Loss,
  legend style ={nodes={scale=1.2, transform shape}},
  legend pos=south west,
  xticklabels from table={ablation_study_team_Emnist_Very_low_part_team_part_client_TE_team_2_client_2_loss.dat}{GR},xtick=data
]
\addplot[blue,thick,mark=diamond*] table [y=gTE_30,x=GR]{ablation_study_team_Emnist_Very_low_part_team_part_client_TE_team_2_client_2_loss.dat};
\addlegendentry{$K = 30$}

\addplot[brown,dashed,thick,mark=square*]  table [y=gTE_100,x=GR]{ablation_study_team_Emnist_Very_low_part_team_part_client_TE_team_2_client_2_loss.dat};
\addlegendentry{$K = 100$}

\addplot[black,dashed,thick,mark=triangle*]  table [y=gTE_150,x=GR]{ablation_study_team_Emnist_Very_low_part_team_part_client_TE_team_2_client_2_loss.dat};
\addlegendentry{$K = 150$}

\addplot[green,dashed,thick,mark=star]  table [y=gTE_200,x=GR]{ablation_study_team_Emnist_Very_low_part_team_part_client_TE_team_2_client_2_loss.dat};
\addlegendentry{$K = 200$}

\end{axis}

\end{tikzpicture}
}

    \caption{Global Loss (Validation)}
    \label{Fig: te_emnist_very_low_ptpc_team_2_mclr_gl}
\end{subfigure}
\begin{subfigure}[!t]{0.24\linewidth}
            \resizebox{\linewidth}{!}{
                \begin{tikzpicture}
                \tikzstyle{every node}=[font=\fontsize{12}{12}\selectfont]

\begin{axis}[
  xlabel= Global rounds,
  ylabel= Personalized Accuracy,
  xticklabels from table={ablation_study_team_Emnist_Very_low_part_team_part_client_TE_team_2_client_2_acc.dat}{GR},xtick=data
]
\addplot[blue,thick,mark=diamond*] table [y=pTE_30,x=GR]{ablation_study_team_Emnist_Very_low_part_team_part_client_TE_team_2_client_2_acc.dat};

\addplot[brown,dashed,thick,mark=square*]  table [y=pTE_100,x=GR]{ablation_study_team_Emnist_Very_low_part_team_part_client_TE_team_2_client_2_acc.dat};

\addplot[black,dashed,thick,mark=triangle*]  table [y=pTE_150,x=GR]{ablation_study_team_Emnist_Very_low_part_team_part_client_TE_team_2_client_2_acc.dat};

\addplot[green,dashed,thick,mark=star]  table [y=pTE_200,x=GR]{ablation_study_team_Emnist_Very_low_part_team_part_client_TE_team_2_client_2_acc.dat};

\end{axis}

\end{tikzpicture}

            }
        \caption{Personalized Accuracy \\ (Validation)}
        \label{Fig: te_emnist_very_low_ptpc_team_2_mclr_pa}
    \end{subfigure} 
    \begin{subfigure}[!t]{0.24\linewidth}
        \resizebox{\linewidth}{!}{
        \begin{tikzpicture}
        \tikzstyle{every node}=[font=\fontsize{12}{12}\selectfont]
\begin{axis}[
  xlabel= Global rounds,
  ylabel= Global Accuracy,
  xticklabels from table={ablation_study_team_Emnist_Very_low_part_team_part_client_TE_team_2_client_2_acc.dat}{GR},xtick=data
]
\addplot[blue,thick,mark=diamond*] table [y=gTE_30,x=GR]{ablation_study_team_Emnist_Very_low_part_team_part_client_TE_team_2_client_2_acc.dat};

\addplot[brown,dashed,thick,mark=square*]  table [y=gTE_100,x=GR]{ablation_study_team_Emnist_Very_low_part_team_part_client_TE_team_2_client_2_acc.dat};

\addplot[black,dashed,thick,mark=triangle*]  table [y=gTE_150,x=GR]{ablation_study_team_Emnist_Very_low_part_team_part_client_TE_team_2_client_2_acc.dat};

\addplot[green,dashed,thick,mark=star]  table [y=gTE_200,x=GR]{ablation_study_team_Emnist_Very_low_part_team_part_client_TE_team_2_client_2_acc.dat};
\end{axis}

\end{tikzpicture}
}

    \caption{Global Accuracy \\ (Validation)}
    \label{Fig: te_emnist_very_low_ptpc_team_2_mclr_ga}
\end{subfigure}

\caption{Effect of Team iterations $\{TE = 30, 100,  150\ and\ 200\}$ when team ($2\%$) and devices ($2\%$) are partially participated in the \ourmodel{} using EMNIST datasets in convex settings (MCLR)}
\label{Fig: te_emnist_very_low_ptpc_team_2_mclr}
\end{figure*}


\begin{figure*}[!ht]
    \centering
        \begin{subfigure}[!t]{0.24\linewidth}
            \resizebox{\linewidth}{!}{
                \begin{tikzpicture}
                \tikzstyle{every node}=[font=\fontsize{12}{12}\selectfont]
\begin{axis}[
  xlabel= Global rounds,
  ylabel= Loss,
  legend pos=north east,
  legend style ={nodes={scale=1.3, transform shape}},
  xticklabels from table={ablation_study_team_Mnist_TE_part_team_part_client_TE_team_2_client_2_loss.dat}{GR},xtick=data
]
\addplot[blue,thick,mark=diamond*] table [y=pTE_30,x=GR]{ablation_study_team_Mnist_TE_part_team_part_client_TE_team_2_client_2_loss.dat};
\addlegendentry{$K = 30$}

\addplot[black,dashed,thick,mark=square*]  table [y=pTE_100,x=GR]{ablation_study_team_Mnist_TE_part_team_part_client_TE_team_2_client_2_loss.dat};
\addlegendentry{$K = 100$}

\addplot[green,dashed,thick,mark=triangle*]  table [y=pTE_150,x=GR]{ablation_study_team_Mnist_TE_part_team_part_client_TE_team_2_client_2_loss.dat};
\addlegendentry{$K = 150$}

\addplot[brown,thick,mark=star] table [y=pTE_200,x=GR]{ablation_study_team_Mnist_TE_part_team_part_client_TE_team_2_client_2_loss.dat};
\addlegendentry{$K = 200$}

\end{axis}
\end{tikzpicture}

            }
        \caption{Personalized Loss \\ (Validation)}
        \label{Fig: te_mnist_ptpc_team_2_mclr_pl}
    \end{subfigure} 
    \begin{subfigure}[!t]{0.24\linewidth}
        \resizebox{\linewidth}{!}{
        \begin{tikzpicture}
        \tikzstyle{every node}=[font=\fontsize{12}{12}\selectfont]

\begin{axis}[
  xlabel= Global rounds,
  ylabel= Loss,
  xticklabels from table={ablation_study_team_Mnist_TE_part_team_part_client_TE_team_2_client_2_loss.dat}{GR},xtick=data
]
\addplot[blue,thick,mark=diamond*] table [y=gTE_30,x=GR]{ablation_study_team_Mnist_TE_part_team_part_client_TE_team_2_client_2_loss.dat};

\addplot[black,dashed,thick,mark=square*]  table [y=gTE_100,x=GR]{ablation_study_team_Mnist_TE_part_team_part_client_TE_team_2_client_2_loss.dat};

\addplot[green,dashed,thick,mark=triangle*]  table [y=gTE_150,x=GR]{ablation_study_team_Mnist_TE_part_team_part_client_TE_team_2_client_2_loss.dat};

\addplot[brown,dashed,thick,mark=star]  table [y=gTE_200,x=GR]{ablation_study_team_Mnist_TE_part_team_part_client_TE_team_2_client_2_loss.dat};

\end{axis}

\end{tikzpicture}
}

    \caption{Global Loss \\ (Validation)}
    \label{Fig: te_mnist_ptpc_team_2_mclr_gl}
\end{subfigure}
\begin{subfigure}[!t]{0.24\linewidth}
            \resizebox{\linewidth}{!}{
                \begin{tikzpicture}
                \tikzstyle{every node}=[font=\fontsize{12}{12}\selectfont]

\begin{axis}[
  xlabel= Global rounds,
  ylabel= Personalized Accuracy,
  xticklabels from table={ablation_study_team_Mnist_TE_part_team_part_client_TE_team_2_client_2_acc.dat}{GR},xtick=data
]
\addplot[blue,thick,mark=diamond*] table [y=pTE_30,x=GR]{ablation_study_team_Mnist_TE_part_team_part_client_TE_team_2_client_2_acc.dat};

\addplot[black,dashed,thick,mark=square*]  table [y=pTE_100,x=GR]{ablation_study_team_Mnist_TE_part_team_part_client_TE_team_2_client_2_acc.dat};

\addplot[green,dashed,thick,mark=triangle*]  table [y=pTE_150,x=GR]{ablation_study_team_Mnist_TE_part_team_part_client_TE_team_2_client_2_acc.dat};

\addplot[brown,dashed,thick,mark=star]  table [y=pTE_200,x=GR]{ablation_study_team_Mnist_TE_part_team_part_client_TE_team_2_client_2_acc.dat};

\end{axis}

\end{tikzpicture}

            }
        \caption{Personalized Accuracy \\ (Validation)}
        \label{Fig: te_mnist_ptpc_team_2_mclr_pa}
    \end{subfigure} 
    \begin{subfigure}[!t]{0.24\linewidth}
        \resizebox{\linewidth}{!}{
        \begin{tikzpicture}
        \tikzstyle{every node}=[font=\fontsize{12}{12}\selectfont]
\begin{axis}[
  xlabel= Global rounds,
  ylabel= Global Accuracy,
  xticklabels from table={ablation_study_team_Mnist_TE_part_team_part_client_TE_team_2_client_2_acc.dat}{GR},xtick=data
]
\addplot[blue,thick,mark=diamond*] table [y=gTE_30,x=GR]{ablation_study_team_Mnist_TE_part_team_part_client_TE_team_2_client_2_acc.dat};

\addplot[black,dashed,thick,mark=square*]  table [y=gTE_100,x=GR]{ablation_study_team_Mnist_TE_part_team_part_client_TE_team_2_client_2_acc.dat};

\addplot[green,dashed,thick,mark=triangle*]  table [y=gTE_150,x=GR]{ablation_study_team_Mnist_TE_part_team_part_client_TE_team_2_client_2_acc.dat};

\addplot[brown,dashed,thick,mark=star]  table [y=gTE_200,x=GR]{ablation_study_team_Mnist_TE_part_team_part_client_TE_team_2_client_2_acc.dat};

\end{axis}

\end{tikzpicture}
}

    \caption{Global Accuracy \\ (Validation)}
    \label{Fig: te_mnist_ptpc_team_2_mclr_ga}
\end{subfigure}

\caption{Effect of Team iterations $\{TE = 30, 100, 150 \ and\ 200\}$ when team ($20\%$) and devices ($2\%$) are partially participated in the \ourmodel{} using MNIST datasets in convex settings (MCLR)}
\label{Fig: te_mnist_ptpc_team_2_mclr}
\end{figure*}



\begin{figure*}[!ht]
        \centering
        \begin{subfigure}[!t]{0.24\linewidth}
            \resizebox{\linewidth}{!}{
                \begin{tikzpicture}
                \tikzstyle{every node}=[font=\fontsize{12}{12}\selectfont]
\begin{axis}[
  xlabel= Global rounds,
  ylabel= Loss,
  legend pos=south west,
  xticklabels from table={ablation_study_team_Mnist_Very_low_part_team_part_client_TE_team_2_client_2_loss.dat}{GR},xtick=data
]
\addplot[blue,thick,mark=diamond*] table [y=pTE_30,x=GR]{ablation_study_team_Mnist_Very_low_part_team_part_client_TE_team_2_client_2_loss.dat};

\addplot[brown,dashed,thick,mark=square*]  table [y=pTE_100,x=GR]{ablation_study_team_Mnist_Very_low_part_team_part_client_TE_team_2_client_2_loss.dat};

\addplot[black,dashed,thick,mark=triangle*]  table [y=pTE_150,x=GR]{ablation_study_team_Mnist_Very_low_part_team_part_client_TE_team_2_client_2_loss.dat};

\addplot[green,thick,mark=star] table [y=pTE_200,x=GR]{ablation_study_team_Mnist_Very_low_part_team_part_client_TE_team_2_client_2_loss.dat};

\end{axis}
\end{tikzpicture}

            }
        \caption{Personalized Loss \\ (Validation)}
        \label{Fig: te_mnist_very_low_ptpc_team_2_mclr_pl}
    \end{subfigure} 
    \begin{subfigure}[!t]{0.24\linewidth}
        \resizebox{\linewidth}{!}{
        \begin{tikzpicture}
        \tikzstyle{every node}=[font=\fontsize{12}{12}\selectfont]

\begin{axis}[
  xlabel= Global rounds,
  ylabel= Loss,
  legend style ={nodes={scale=1.2, transform shape}},
  legend pos=south west,
  xticklabels from table={ablation_study_team_Mnist_Very_low_part_team_part_client_TE_team_2_client_2_loss.dat}{GR},xtick=data
]
\addplot[blue,thick,mark=diamond*] table [y=gTE_30,x=GR]{ablation_study_team_Mnist_Very_low_part_team_part_client_TE_team_2_client_2_loss.dat};
\addlegendentry{$K = 30$}

\addplot[brown,dashed,thick,mark=square*]  table [y=gTE_100,x=GR]{ablation_study_team_Mnist_Very_low_part_team_part_client_TE_team_2_client_2_loss.dat};
\addlegendentry{$K = 100$}

\addplot[black,dashed,thick,mark=triangle*]  table [y=gTE_150,x=GR]{ablation_study_team_Mnist_Very_low_part_team_part_client_TE_team_2_client_2_loss.dat};
\addlegendentry{$K = 150$}

\addplot[green,dashed,thick,mark=star]  table [y=gTE_200,x=GR]{ablation_study_team_Mnist_Very_low_part_team_part_client_TE_team_2_client_2_loss.dat};
\addlegendentry{$K = 200$}

\end{axis}

\end{tikzpicture}
}

    \caption{Global Loss \\ (Validation)}
    \label{Fig: te_mnist_very_low_ptpc_team_2_mclr_gl}
\end{subfigure}
\begin{subfigure}[!t]{0.24\linewidth}
            \resizebox{\linewidth}{!}{
                \begin{tikzpicture}
                \tikzstyle{every node}=[font=\fontsize{12}{12}\selectfont]

\begin{axis}[
  xlabel= Global rounds,
  ylabel= Personalized Accuracy,
  xticklabels from table={ablation_study_team_Mnist_Very_low_part_team_part_client_TE_team_2_client_2_acc.dat}{GR},xtick=data
]
\addplot[blue,thick,mark=diamond*] table [y=pTE_30,x=GR]{ablation_study_team_Mnist_Very_low_part_team_part_client_TE_team_2_client_2_acc.dat};

\addplot[brown,dashed,thick,mark=square*]  table [y=pTE_100,x=GR]{ablation_study_team_Mnist_Very_low_part_team_part_client_TE_team_2_client_2_acc.dat};

\addplot[black,dashed,thick,mark=triangle*]  table [y=pTE_150,x=GR]{ablation_study_team_Mnist_Very_low_part_team_part_client_TE_team_2_client_2_acc.dat};

\addplot[green,dashed,thick,mark=star]  table [y=pTE_200,x=GR]{ablation_study_team_Mnist_Very_low_part_team_part_client_TE_team_2_client_2_acc.dat};

\end{axis}

\end{tikzpicture}

            }
        \caption{Personalized Accuracy \\ (Validation)}
        \label{Fig: te_mnist_very_low_ptpc_team_2_mclr_pa}
    \end{subfigure} 
    \begin{subfigure}[!t]{0.24\linewidth}
        \resizebox{\linewidth}{!}{
        \begin{tikzpicture}
        \tikzstyle{every node}=[font=\fontsize{12}{12}\selectfont]
\begin{axis}[
  xlabel= Global rounds,
  ylabel= Global Accuracy,
  xticklabels from table={ablation_study_team_Mnist_Very_low_part_team_part_client_TE_team_2_client_2_acc.dat}{GR},xtick=data
]
\addplot[blue,thick,mark=diamond*] table [y=gTE_30,x=GR]{ablation_study_team_Mnist_Very_low_part_team_part_client_TE_team_2_client_2_acc.dat};

\addplot[brown,dashed,thick,mark=square*]  table [y=gTE_100,x=GR]{ablation_study_team_Mnist_Very_low_part_team_part_client_TE_team_2_client_2_acc.dat};

\addplot[black,dashed,thick,mark=triangle*]  table [y=gTE_150,x=GR]{ablation_study_team_Mnist_Very_low_part_team_part_client_TE_team_2_client_2_acc.dat};

\addplot[green,dashed,thick,mark=star]  table [y=gTE_200,x=GR]{ablation_study_team_Mnist_Very_low_part_team_part_client_TE_team_2_client_2_acc.dat};
\end{axis}

\end{tikzpicture}
}

    \caption{Global Accuracy \\ (Validation)}
    \label{Fig: te_mnist_very_low_ptpc_team_2_mclr_ga}
\end{subfigure}

\caption{Effect of Team iterations $\{TE = 30, 100,  150\ and\ 200\}$ when team ($2\%$) and devices ($2\%$) are partially participated in the \ourmodel{} using EMNIST datasets in convex settings (MCLR)}
\label{Fig: te_mnist_very_low_ptpc_team_2_mclr}
\end{figure*}


\clearpage
\subsection{Convergence analysis } \label{appx: exp_conv_analysis}
Based on the convergence results presented in  \Cref{Fig:convergence_sota_mnist}, \Cref{Fig:convergence_sota_emnist}, and \Cref{Fig:convergence_sota_synthetic}, we observed that \ourmodel{}(PM) achieved faster convergence compared to AL2GD. Additionally, the convergence of \ourmodel{}(GM) and h-SGD was found to be similar for both strongly convex and non-convex scenarios.


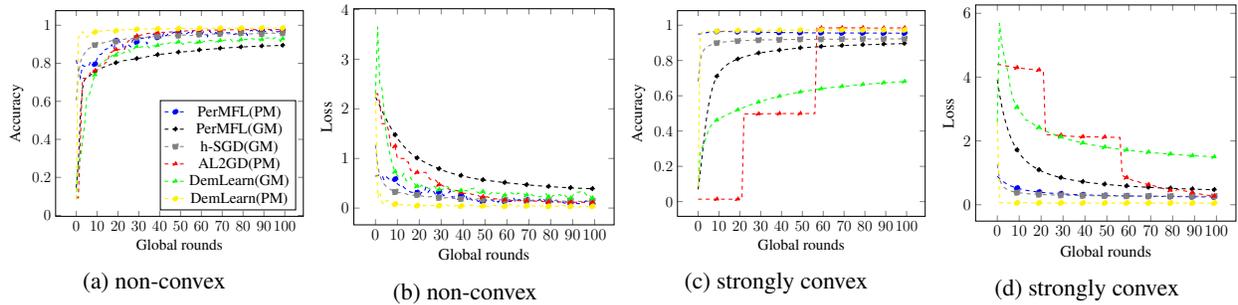
\begin{figure*}[!ht]
\centering
\begin{subfigure}[!t]{0.23\linewidth}
\resizebox{\linewidth}{!}{
\begin{tikzpicture}
\tikzstyle{every node}=[font=\fontsize{12}{12}\selectfont]
\begin{axis}[
xlabel= Global rounds,
ylabel= Accuracy,
legend style ={nodes={scale=1, transform shape}},
legend pos=south east,
xtick={0,10,20,30,40,50,60,70,80,90,100}
]
\addplot[blue,dashed,thick,mark=otimes*, mark repeat=10, mark phase=10] table [y=PerMFL(PM) ,x=GR]{final-perf-emnist-cnn-acc.dat};
\addlegendentry{\ourmodel{}(PM)}

\addplot[black,dashed,thick,mark=diamond*, mark repeat=10, mark phase=10]  table [y=PerMFL(GM) ,x=GR]{final-perf-emnist-cnn-acc.dat};
\addlegendentry{\ourmodel{}(GM)}

\addplot[gray,dashed,thick,mark=square*, mark repeat=10, mark phase=10]  table [y=h-sgd,x=GR]{final-perf-emnist-cnn-acc.dat};
\addlegendentry{h-SGD(GM)}

\addplot[red,dashed,thick,mark=triangle*, mark repeat=10, mark phase=10]  table [y=AL2GD,x=GR]{final-perf-emnist-cnn-acc.dat};
\addlegendentry{AL2GD(PM)}

\addplot[green,thick,dashed,mark=triangle*, mark repeat=10, mark phase=10]  table [y=DemLearn(C-GEN),x=GR]{final-perf-emnist-cnn-acc.dat};
\addlegendentry{DemLearn(GM)}

\addplot[yellow,thick,dashed,mark=*, mark repeat=10, mark phase=10]  table [y=DemLearn(C-SPE),x=GR]{final-perf-emnist-cnn-acc.dat};
\addlegendentry{DemLearn(PM)}

\end{axis}
\end{tikzpicture}

}
\caption{non-convex}
\label{Fig:emnist_non_conv_acc_sota}
\end{subfigure} 
\begin{subfigure}[!t]{0.23\linewidth}
\resizebox{\linewidth}{!}{
\begin{tikzpicture}
\tikzstyle{every node}=[font=\fontsize{12}{12}\selectfont]
\begin{axis}[
xlabel= Global rounds,
ylabel= Loss,
xtick={0,10,20,30,40,50,60,70,80,90,100},
legend style ={nodes={scale=1.50, transform shape}},
legend pos= north east,
]
\addplot[blue,dashed,thick,mark=otimes*, mark repeat=10, mark phase=10] table [y=PerMFL(PM) ,x=GR]{final-perf-emnist-cnn-loss.dat};

\addplot[black,dashed,thick,mark=diamond*, mark repeat=10, mark phase=10]  table [y=PerMFL(GM) ,x=GR]{final-perf-emnist-cnn-loss.dat};

\addplot[gray,dashed,thick,mark=square*, mark repeat=10, mark phase=10]  table [y=h-sgd,x=GR]{final-perf-emnist-cnn-loss.dat};

\addplot[red,dashed,thick,mark=triangle*, mark repeat=10, mark phase=10]  table [y=AL2GD,x=GR]{final-perf-emnist-cnn-loss.dat};

\addplot[green,thick,dashed,mark=triangle*, mark repeat=10, mark phase=10]  table [y=DemLearn(C-GEN),x=GR]{final-perf-emnist-cnn-loss.dat};

\addplot[yellow,thick,dashed,mark=*, mark repeat=10, mark phase=10]  table [y=DemLearn(C-SPE),x=GR]{final-perf-emnist-cnn-loss.dat};

\end{axis}
\end{tikzpicture}
}

\caption{non-convex}
\label{Fig:emnist_non_conv_loss_sota}
\end{subfigure}
\begin{subfigure}[!t]{0.23\linewidth}
\resizebox{\linewidth}{!}{
\begin{tikzpicture}
\tikzstyle{every node}=[font=\fontsize{12}{12}\selectfont]

\begin{axis}[
xlabel= Global rounds,
ylabel= Accuracy,
xtick={0,10,20,30,40,50,60,70,80,90,100},
legend style ={nodes={scale=1.3, transform shape}},
legend pos= south east
]
\addplot[blue,thick,dashed,mark=otimes*, mark repeat=10, mark phase=10] table [y=PerMFL(PM) ,x=GR]{final-perf-emnist-mclr-acc.dat};

\addplot[black,dashed,thick,mark=diamond*, mark repeat=10, mark phase=10]  table [y=PerMFL(GM) ,x=GR]{final-perf-emnist-mclr-acc.dat};

\addplot[gray,dashed,thick,mark=square*, mark repeat=10, mark phase=10]  table [y=h-sgd,x=GR]{final-perf-emnist-mclr-acc.dat};

\addplot[red,dashed,thick,mark=triangle*, mark repeat=10, mark phase=10]  table [y=AL2GD,x=GR]{final-perf-emnist-mclr-acc.dat};

\addplot[green,thick,dashed,mark=triangle*, mark repeat=10, mark phase=10]  table [y=DemLearn(C-GEN),x=GR]{final-perf-emnist-mclr-acc.dat};

\addplot[yellow,thick,dashed,mark=*, mark repeat=10, mark phase=10]  table [y=DemLearn(C-SPE),x=GR]{final-perf-emnist-mclr-acc.dat};

\end{axis}

\end{tikzpicture}

}
\caption{strongly convex}
\label{Fig:emnist_conv_acc_sota}
\end{subfigure} 
\begin{subfigure}[!t]{0.23\linewidth}
\resizebox{\linewidth}{!}{
\begin{tikzpicture}
\tikzstyle{every node}=[font=\fontsize{12}{12}\selectfont]
\begin{axis}[
xlabel= Global rounds,
ylabel= Loss,
xtick={0,10,20,30,40,50,60,70,80,90,100},
legend style ={nodes={scale=1.45, transform shape}},
legend pos= north east,
]

\addplot[blue,thick,dashed,mark=otimes*, mark repeat=10, mark phase=10] table [y=PerMFL(PM) ,x=GR]{final-perf-emnist-mclr-loss.dat};

\addplot[black,dashed,thick,mark=diamond*, mark repeat=10, mark phase=10]  table [y=PerMFL(GM) ,x=GR]{final-perf-emnist-mclr-loss.dat};

\addplot[gray,dashed,thick,mark=square*, mark repeat=10, mark phase=10]  table [y=h-sgd,x=GR]{final-perf-emnist-mclr-loss.dat};

\addplot[red,dashed,thick,mark=triangle*, mark repeat=10, mark phase=10]  table [y=AL2GD,x=GR]{final-perf-emnist-mclr-loss.dat};

\addplot[green,thick,dashed,mark=triangle*, mark repeat=10, mark phase=10]  table [y=DemLearn(C-GEN),x=GR]{final-perf-emnist-mclr-loss.dat};

\addplot[yellow,thick,dashed,mark=*, mark repeat=10, mark phase=10]  table [y=DemLearn(C-SPE),x=GR]{final-perf-emnist-mclr-loss.dat};

\end{axis}
\end{tikzpicture}
}

\caption{strongly convex}
\label{Fig:emnist_conv_loss_sota}
\end{subfigure}

\caption{ Convergence comparison of \ourmodel{} with multi-tier SOTA in strongly convex and non-convex settings on EMNIST}
\label{Fig:convergence_sota_emnist}
\end{figure*}




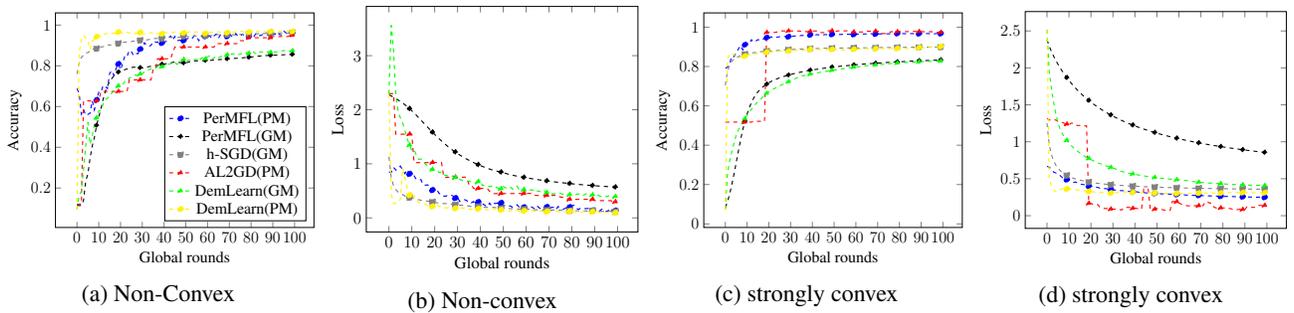
\begin{figure*}[!ht]
\centering
\begin{subfigure}[!t]{0.24\linewidth}
\resizebox{\linewidth}{!}{
\begin{tikzpicture}
\tikzstyle{every node}=[font=\fontsize{12}{12}\selectfont]

\begin{axis}[
xlabel= Global rounds,
ylabel= Accuracy,
legend style ={nodes={scale=1, transform shape}},
legend pos=south east,
xtick={0,10,20,30,40,50,60,70,80,90,100}
]
\addplot[blue,dashed,thick,mark=otimes*, mark repeat=10, mark phase=10] table [y=PerMFL(PM) ,x=GR]{final-perf-mnist-cnn-acc.dat};
\addlegendentry{\ourmodel{}(PM)}

\addplot[black,dashed,thick,mark=diamond*, mark repeat=10, mark phase=10]  table [y=PerMFL(GM) ,x=GR]{final-perf-mnist-cnn-acc.dat};
\addlegendentry{\ourmodel{}(GM)}

\addplot[gray,dashed,thick,mark=square*, mark repeat=10, mark phase=10]  table [y=h-sgd,x=GR]{final-perf-mnist-cnn-acc.dat};
\addlegendentry{h-SGD(GM)}

\addplot[red,dashed,thick,mark=triangle*, mark repeat=10, mark phase=10]  table [y=AL2GD,x=GR]{final-perf-mnist-cnn-acc.dat};
\addlegendentry{AL2GD(PM)}

\addplot[green,thick,dashed,mark=triangle*, mark repeat=10, mark phase=10]  table [y=DemLearn(C-GEN),x=GR]{final-perf-mnist-cnn-acc.dat};
\addlegendentry{DemLearn(GM)}

\addplot[yellow,thick,dashed,mark=*, mark repeat=10, mark phase=10]  table [y=DemLearn(C-SPE),x=GR]{final-perf-mnist-cnn-acc.dat};
\addlegendentry{DemLearn(PM)}

\end{axis}
\end{tikzpicture}

}
\caption{Non-Convex}
\label{Fig:mnist_non_conv_acc_sota}
\end{subfigure} 
\begin{subfigure}[!t]{0.24\linewidth}
\resizebox{\linewidth}{!}{
\begin{tikzpicture}
\tikzstyle{every node}=[font=\fontsize{12}{12}\selectfont]
\begin{axis}[
xlabel= Global rounds,
ylabel= Loss,
legend style ={nodes={scale=1.45, transform shape}},
legend pos= north east,
xtick={0,10,20,30,40,50,60,70,80,90,100}
]
\addplot[blue,dashed,thick,mark=otimes*, mark repeat=10, mark phase=10] table [y=PerMFL(PM) ,x=GR]{final-perf-mnist-cnn-loss.dat};

\addplot[black,dashed,thick,mark=diamond*, mark repeat=10, mark phase=10]  table [y=PerMFL(GM) ,x=GR]{final-perf-mnist-cnn-loss.dat};

\addplot[gray,dashed,thick,mark=square*, mark repeat=10, mark phase=10]  table [y=h-sgd,x=GR]{final-perf-mnist-cnn-loss.dat};

\addplot[red,dashed,thick,mark=triangle*, mark repeat=10, mark phase=10]  table [y=AL2GD,x=GR]{final-perf-mnist-cnn-loss.dat};

\addplot[green,thick,dashed,mark=triangle*, mark repeat=10, mark phase=10]  table [y=DemLearn(C-GEN),x=GR]{final-perf-mnist-cnn-loss.dat};

\addplot[yellow,thick,dashed,mark=*, mark repeat=10, mark phase=10]  table [y=DemLearn(C-SPE),x=GR]{final-perf-mnist-cnn-loss.dat};

\end{axis}
\end{tikzpicture}
}

\caption{Non-convex}
\label{Fig:mnist_non_conv_loss_sota}
\end{subfigure}
\begin{subfigure}[!t]{0.24\linewidth}
\resizebox{\linewidth}{!}{
\begin{tikzpicture}
\tikzstyle{every node}=[font=\fontsize{12}{12}\selectfont]

\begin{axis}[
xlabel= Global rounds,
ylabel= Accuracy,
legend style ={nodes={scale=1.45, transform shape}},
legend pos= south east,
xtick={0,10,20,30,40,50,60,70,80,90,100},
legend style ={nodes={scale=1.3, transform shape}},
legend pos= south east
]
\addplot[blue,thick,dashed,mark=otimes*, mark repeat=10, mark phase=10] table [y=PerMFL(PM) ,x=GR]{final-perf-mnist-mclr-acc.dat};

\addplot[black,dashed,thick,mark=diamond*, mark repeat=10, mark phase=10]  table [y=PerMFL(GM) ,x=GR]{final-perf-mnist-mclr-acc.dat};

\addplot[gray,dashed,thick,mark=square*, mark repeat=10, mark phase=10]  table [y=h-sgd,x=GR]{final-perf-mnist-mclr-acc.dat};

\addplot[red,dashed,thick,mark=triangle*, mark repeat=10, mark phase=10]  table [y=AL2GD,x=GR]{final-perf-mnist-mclr-acc.dat};

\addplot[green,thick,dashed,mark=triangle*, mark repeat=10, mark phase=10]  table [y=DemLearn(C-GEN),x=GR]{final-perf-mnist-mclr-acc.dat};

\addplot[yellow,thick,dashed,mark=*, mark repeat=10, mark phase=10]  table [y=DemLearn(C-SPE),x=GR]{final-perf-mnist-mclr-acc.dat};

\end{axis}

\end{tikzpicture}

}
\caption{strongly convex}
\label{Fig:mnist_conv_acc_sota}
\end{subfigure} 
\begin{subfigure}[!t]{0.24\linewidth}
\resizebox{\linewidth}{!}{
\begin{tikzpicture}
\tikzstyle{every node}=[font=\fontsize{12}{12}\selectfont]
\begin{axis}[
xlabel= Global rounds,
ylabel= Loss,
legend style ={nodes={scale=1.45, transform shape}},
legend pos= north east,
xtick={0,10,20,30,40,50,60,70,80,90,100}
]
\addplot[blue,dashed,thick,mark=otimes*, mark repeat=10, mark phase=10] table [y=PerMFL(PM) ,x=GR]{final-perf-mnist-mclr-loss.dat};

\addplot[black,dashed,thick,mark=diamond*, mark repeat=10, mark phase=10]  table [y=PerMFL(GM) ,x=GR]{final-perf-mnist-mclr-loss.dat};

\addplot[gray,dashed,thick,mark=square*, mark repeat=10, mark phase=10]  table [y=h-sgd,x=GR]{final-perf-mnist-mclr-loss.dat};

\addplot[red,dashed,thick,mark=triangle*, mark repeat=10, mark phase=10]  table [y=AL2GD,x=GR]{final-perf-mnist-mclr-loss.dat};

\addplot[green,thick,dashed,mark=triangle*, mark repeat=10, mark phase=10]  table [y=DemLearn(C-GEN),x=GR]{final-perf-mnist-mclr-loss.dat};

\addplot[yellow,thick,dashed,mark=*, mark repeat=10, mark phase=10]  table [y=DemLearn(C-SPE),x=GR]{final-perf-mnist-mclr-loss.dat};

\end{axis}
\end{tikzpicture}
}

\caption{strongly convex}
\label{Fig:mnist_conv_loss_sota}
\end{subfigure}

\caption{ Convergence comparison of \ourmodel{} with multi-tier SOTA in strongly convex and non-convex settings on MNIST }
\label{Fig:convergence_sota_mnist}
\end{figure*}


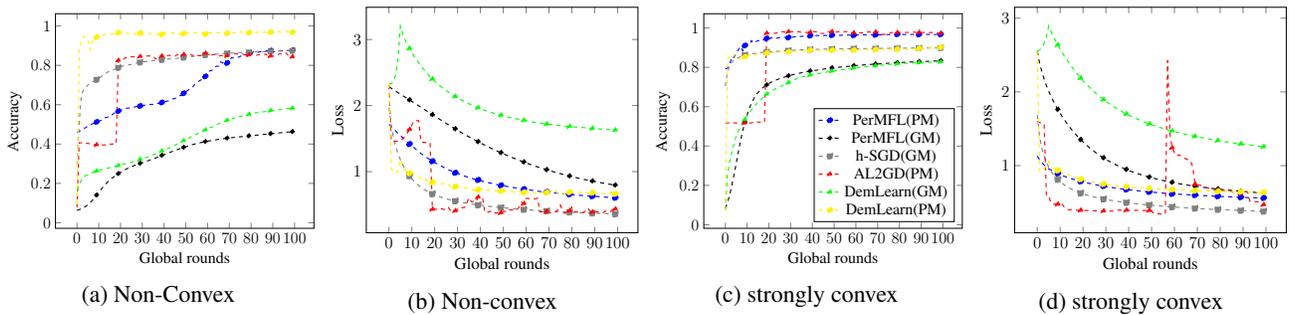
\begin{figure*}[!ht]
\centering
\begin{subfigure}[!t]{0.24\linewidth}
\resizebox{\linewidth}{!}{
\begin{tikzpicture}
\tikzstyle{every node}=[font=\fontsize{12}{12}\selectfont]

\begin{axis}[
xlabel= Global rounds,
ylabel= Accuracy,
legend style ={nodes={scale=1, transform shape}},
legend pos=south east,
xtick={0,10,20,30,40,50,60,70,80,90,100}
]
\addplot[blue,dashed,thick,mark=otimes*, mark repeat=10, mark phase=10] table [y=PerMFL(PM) ,x=GR]{final-perf-synthetic-dnn-acc.dat};

\addplot[black,dashed,thick,mark=diamond*, mark repeat=10, mark phase=10]  table [y=PerMFL(GM) ,x=GR]{final-perf-synthetic-dnn-acc.dat};

\addplot[gray,dashed,thick,mark=square*, mark repeat=10, mark phase=10]  table [y=h-sgd,x=GR]{final-perf-synthetic-dnn-acc.dat};

\addplot[red,dashed,thick,mark=triangle*, mark repeat=10, mark phase=10]  table [y=AL2GD,x=GR]{final-perf-synthetic-dnn-acc.dat};

\addplot[green,thick,dashed,mark=triangle*, mark repeat=10, mark phase=10]  table [y=DemLearn(C-GEN),x=GR]{final-perf-synthetic-dnn-acc.dat};

\addplot[yellow,thick,dashed,mark=*, mark repeat=10, mark phase=10]  table [y=DemLearn(C-SPE),x=GR]{final-perf-mnist-cnn-acc.dat};

\end{axis}
\end{tikzpicture}

}
\caption{Non-Convex}
\label{Fig:synthetic_non_conv_acc_sota}
\end{subfigure} 
\begin{subfigure}[!t]{0.24\linewidth}
\resizebox{\linewidth}{!}{
\begin{tikzpicture}
\tikzstyle{every node}=[font=\fontsize{12}{12}\selectfont]
\begin{axis}[
xlabel= Global rounds,
ylabel= Loss,
legend style ={nodes={scale=1.45, transform shape}},
legend pos= north east,
xtick={0,10,20,30,40,50,60,70,80,90,100}
]
\addplot[blue,dashed,thick,mark=otimes*, mark repeat=10, mark phase=10] table [y=PerMFL(PM) ,x=GR]{final-perf-synthetic-dnn-loss.dat};

\addplot[black,dashed,thick,mark=diamond*, mark repeat=10, mark phase=10]  table [y=PerMFL(GM) ,x=GR]{final-perf-synthetic-dnn-loss.dat};

\addplot[gray,dashed,thick,mark=square*, mark repeat=10, mark phase=10]  table [y=h-sgd,x=GR]{final-perf-synthetic-dnn-loss.dat};

\addplot[red,dashed,thick,mark=triangle*, mark repeat=10, mark phase=10]  table [y=AL2GD,x=GR]{final-perf-synthetic-dnn-loss.dat};

\addplot[green,thick,dashed,mark=triangle*, mark repeat=10, mark phase=10]  table [y=DemLearn(C-GEN),x=GR]{final-perf-synthetic-dnn-loss.dat};

\addplot[yellow,thick,dashed,mark=*, mark repeat=10, mark phase=10]  table [y=DemLearn(C-SPE),x=GR]{final-perf-synthetic-dnn-loss.dat};

\end{axis}
\end{tikzpicture}
}

\caption{Non-convex}
\label{Fig:synthetic_non_conv_loss_sota}
\end{subfigure}
\begin{subfigure}[!t]{0.24\linewidth}
\resizebox{\linewidth}{!}{
\begin{tikzpicture}
\tikzstyle{every node}=[font=\fontsize{12}{12}\selectfont]

\begin{axis}[
xlabel= Global rounds,
ylabel= Accuracy,
legend style ={nodes={scale=1.1, transform shape}},
legend pos= south east,
xtick={0,10,20,30,40,50,60,70,80,90,100},
legend style ={nodes={scale=0.9, transform shape}},
legend pos= south east
]
\addplot[blue,thick,dashed,mark=otimes*, mark repeat=10, mark phase=10] table [y=PerMFL(PM) ,x=GR]{final-perf-mnist-mclr-acc.dat};
\addlegendentry{\ourmodel{}(PM)}

\addplot[black,dashed,thick,mark=diamond*, mark repeat=10, mark phase=10]  table [y=PerMFL(GM) ,x=GR]{final-perf-mnist-mclr-acc.dat};
\addlegendentry{\ourmodel{}(GM)}

\addplot[gray,dashed,thick,mark=square*, mark repeat=10, mark phase=10]  table [y=h-sgd,x=GR]{final-perf-mnist-mclr-acc.dat};
\addlegendentry{h-SGD(GM)}

\addplot[red,dashed,thick,mark=triangle*, mark repeat=10, mark phase=10]  table [y=AL2GD,x=GR]{final-perf-mnist-mclr-acc.dat};
\addlegendentry{AL2GD(PM)}

\addplot[green,thick,dashed,mark=triangle*, mark repeat=10, mark phase=10]  table [y=DemLearn(C-GEN),x=GR]{final-perf-mnist-mclr-acc.dat};
\addlegendentry{DemLearn(GM)}

\addplot[yellow,thick,dashed,mark=*, mark repeat=10, mark phase=10]  table [y=DemLearn(C-SPE),x=GR]{final-perf-mnist-mclr-acc.dat};
\addlegendentry{DemLearn(PM)}

\end{axis}

\end{tikzpicture}

}
\caption{strongly convex}
\label{Fig:synthetic_conv_acc_sota}
\end{subfigure} 
\begin{subfigure}[!t]{0.24\linewidth}
\resizebox{\linewidth}{!}{
\begin{tikzpicture}
\tikzstyle{every node}=[font=\fontsize{12}{12}\selectfont]
\begin{axis}[
xlabel= Global rounds,
ylabel= Loss,
legend style ={nodes={scale=1.45, transform shape}},
legend pos= north east,
xtick={0,10,20,30,40,50,60,70,80,90,100}
]
\addplot[blue,dashed,thick,mark=otimes*, mark repeat=10, mark phase=10] table [y=PerMFL(PM) ,x=GR]{final-perf-synthetic-mclr-loss.dat};

\addplot[black,dashed,thick,mark=diamond*, mark repeat=10, mark phase=10]  table [y=PerMFL(GM) ,x=GR]{final-perf-synthetic-mclr-loss.dat};

\addplot[gray,dashed,thick,mark=square*, mark repeat=10, mark phase=10]  table [y=h-sgd,x=GR]{final-perf-synthetic-mclr-loss.dat};

\addplot[red,dashed,thick,mark=triangle*, mark repeat=10, mark phase=10]  table [y=AL2GD,x=GR]{final-perf-synthetic-mclr-loss.dat};

\addplot[green,thick,dashed,mark=triangle*, mark repeat=10, mark phase=10]  table [y=DemLearn(C-GEN),x=GR]{final-perf-synthetic-mclr-loss.dat};

\addplot[yellow,thick,dashed,mark=*, mark repeat=10, mark phase=10]  table [y=DemLearn(C-SPE),x=GR]{final-perf-synthetic-mclr-loss.dat};

\end{axis}
\end{tikzpicture}
}

\caption{strongly convex}
\label{Fig:synthetic_conv_loss_sota}
\end{subfigure}

\caption{ Convergence comparison of \ourmodel{} with multi-tier SOTA in strongly convex and non-convex settings on Synthetic datasets }
\label{Fig:convergence_sota_synthetic}
\end{figure*}

\clearpage

\subsection{Performance analysis on FEMNIST and CIFAR100 datasets } \label{appx: perf_femnist_cifar100}
We have experimented the performance of \ourmodel{} with the state-of-the-art multi-tire federated learning algorithms on Federated EMNIST (FEMNIST) and CIFAR100 datasets (see \Cref{tab: perf_femnist_cifar100}) for the strongly convex setup. From there, we observed \ourmodel{}(PM) produce better results than AL2GD. h-sgd produces a better global model than \ourmodel{}. For the CIFAR100 dataset, AL2GD overcame the performance of \ourmodel{}(PM). The performance of \ourmodel{}(GM) and h-SGD both are equivalent. 
\begin{table*}[!ht]
\centering
\caption{Performance comparison of \ourmodel{} with SOTA. (Validation accuracy(mean/std)) } \label{tab: perf_femnist_cifar100}
\begin{adjustbox}{max width=\textwidth}
\setstretch{1.0}
\begin{tabular}{| c | c | c |}
    \hline
    Algorithm & \multicolumn{2}{| c |}{MCLR (Strongly convex)}  \\
    \cline{2-3} 
    & FEMNIST & CIFAR100 \\
    \hline
    h-SGD (GM) & 0.6405($\pm 0.005$) & 0.1232($0.001$) \\
    AL2GD (PM) & 0.4467($\pm 0.01$) & 0.65.87($\pm 0.07$)  \\
    \ourmodel{} (GM) & 0.5757 ($\pm 0.0$) & 0.1368 ($\pm 0.0$)  \\
    \ourmodel{} (PM) & 0.8129 ($\pm 0.0$) &  0.6695 ($\pm 0.001$) \\
    \hline
   
\end{tabular}
\end{adjustbox}
\end{table*}

\end{document}